%% file: paper.tex
\documentclass[]{bytedance_seed}

\usepackage[toc,page,header]{appendix}

\usepackage{titlesec}
\usepackage{enumitem}
\usepackage{graphicx}
\usepackage{makecell}
\usepackage{array}
\usepackage{booktabs}
\usepackage{wrapfig}
\usepackage{indentfirst}
\usepackage{amsmath}
\usepackage{amssymb}
\usepackage{bm}
\usepackage{xspace}
\usepackage{hyperref}
\usepackage{xcolor}
\usepackage{tcolorbox}
\usepackage{pifont}
\usepackage{amsthm}
\usepackage{mathtools}
\usepackage{multirow}
\usepackage{diagbox}
\usepackage{marvosym}
\usepackage{tabularx}

\newcolumntype{C}{>{\centering\arraybackslash}X}

\newcommand{\xmark}{\textcolor{red}{\ding{55}}}
\newcommand{\cmark}{\textcolor{green!60!black}{\ding{51}}}

\usepackage{algorithm}
\usepackage{algpseudocode}
\usepackage{caption}

\newenvironment{sidealgorithm}[2]{%
  \refstepcounter{algorithm}%
  \noindent\hrule height 0.8pt\relax
  \vspace{2pt}
  \noindent\textbf{Algorithm \thealgorithm}~#2\label{#1}\par
  \vspace{2pt}
  \noindent\hrule
  \vspace{3pt}
  \begin{algorithmic}[1]
}{%
  \end{algorithmic}
  \vfill
  \noindent\hrule
}

\usepackage{xparse}

\newcounter{implication}
\renewcommand{\theimplication}{\arabic{implication}}

\NewDocumentCommand{\implicationbox}{m m}{%
    \refstepcounter{implication}%
    \label{#1}%
    \obsbox{%
        \textbf{Implication~\theimplication. }%
        \textit{#2}%
    }%
}

\numberwithin{equation}{section}
\theoremstyle{plain}
\newtheorem{theorem}{Theorem}[section]

\newtheorem{proposition}[theorem]{Proposition}
\newtheorem{corollary}[theorem]{Corollary}
\theoremstyle{definition}

\newtheorem{assumption}[theorem]{Assumption}
\theoremstyle{remark}
\newtheorem{remark}[theorem]{Remark}

\newcommand{\Lcal}{\mathcal{L}}
\newcommand{\Ncal}{\mathcal{N}}
\newcommand{\E}{\mathbb{E}}
\newcommand{\KL}{\mathrm{KL}}

\newcommand{\vv}{\bm{v}}
\newcommand{\veps}{\bm{\epsilon}}
\newcommand{\dd}{\mathrm{d}}
\newcommand{\pdata}{p_{\mathrm{data}}}

\newcommand{\method}{$\mathcal{C}ola$ \textsc{DLM}\xspace}
\newcommand{\obsbox}[1]{%
    \begin{tcolorbox}[colframe=black!70, colback=lightgray!15, boxrule=1pt, arc=2mm]
        \small#1
    \end{tcolorbox}
}

\usepackage[most]{tcolorbox}

\newtcblisting{promptbox}[1][]{
  breakable,
  enhanced,
  colback=gray!3,
  colframe=black!35,
  boxrule=0.5pt,
  arc=2pt,
  left=6pt,
  right=6pt,
  top=6pt,
  bottom=6pt,
  listing only,
  listing engine=listings,
  listing options={
    basicstyle=\ttfamily\small,
    columns=fullflexible,
    keepspaces=true,
    breaklines=true
  },
  title=#1,
  fonttitle=\bfseries,
}

\title{\raisebox{-0.3ex}{\includegraphics[height=2.3ex]{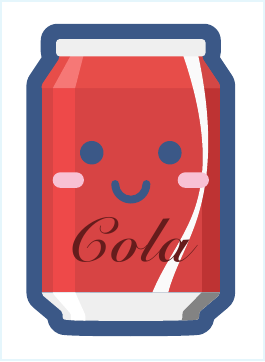}} \textcolor{red}{Co}ntinuous \textcolor{red}{La}tent Diffusion Language Model}

\author[1,2]{Hongcan Guo}
\author[3,\dagger]{Qinyu Zhao}
\author[1,4]{Yian Zhao}
\author[1,5]{Shen Nie}
\author[1]{Rui Zhu}
\author[1]{Qiushan Guo}
\author[1]{Feng Wang}
\author[1]{Tao Yang}
\author[2]{Hengshuang Zhao}
\author[1]{Guoqiang Wei}
\author[1, \textsuperscript{\Letter}]{Yan Zeng}

\affiliation[1]{ByteDance Seed}
\affiliation[2]{The University of Hong Kong} 
\affiliation[3]{The Australian National University} 
\affiliation[4]{Peking University}
\affiliation[5]{Renmin University of China}

\contribution[\dagger]{Work done during an internship at Bytedance Seed}
\contribution[\textsuperscript{\Letter}]{Corresponding author}

\abstract{
Large language models have achieved remarkable success under the autoregressive paradigm, yet high-quality text generation need not be tied to a fixed left-to-right order. Existing alternatives still struggle to jointly achieve generation efficiency, scalable representation learning, and effective global semantic modeling. We propose \method, a hierarchical latent diffusion language model that frames text generation through hierarchical information decomposition. \method first learns a stable text-to-latent mapping with a Text VAE, then models a global semantic prior in continuous latent space with a block-causal DiT, and finally generates text through conditional decoding. From a unified Markov-path perspective, its diffusion process performs latent prior transport rather than token-level observation recovery, thereby separating global semantic organization from local textual realization. This design yields a more flexible non-autoregressive inductive bias, supports semantic compression and prior fitting in continuous space, and naturally extends to other continuous modalities. Through experiments spanning 4 research questions, 8 benchmarks, strictly matched $\sim$2B-parameter autoregressive and LLaDA baselines, and scaling curves up to about 2000 EFLOPs, we identify an effective overall configuration of \method and verify its strong scaling behavior for text generation. Taken together, the results establish hierarchical continuous latent prior modeling as a principled alternative to strictly token-level language modeling, where generation quality and scaling behavior may better reflect model capability than likelihood, while also suggesting a concrete path toward unified modeling across discrete text and continuous modalities.
}

\date{\today}
\correspondence{Yan Zeng at \email{zengyan.yanne@bytedance.com}}

\checkdata[Project Page]{\url{https://hongcanguo.github.io/Cola-DLM/}}

\begin{document}
\maketitle

\input{sections/introduction}
\input{sections/relatedwork}
\input{sections/approach}
\input{sections/experiments}

\input{sections/discussion}
\input{sections/conclusion}

\input{sections/afterword}

\clearpage

\bibliographystyle{plainnat}
\bibliography{main}

\clearpage

\beginappendix

\input{sections/appendix}

\end{document}

%% file: sections/introduction.tex
\section{Introduction}

Large language models have achieved remarkable success under the autoregressive paradigm \citep{brown2020language,guo2025deepseek,liu2024deepseekv2,hurst2024gpt,team2026kimi}. By factorizing the discrete text distribution through the chain rule \citep{malach2023auto,barrault2024large,katharopoulos2020transformers,hutchins2022block,du2022glm,yang2019xlnet}, autoregressive language models have driven major advances in large-scale pretraining, open-ended generation, and downstream transfer, and have become the dominant approach to modern language modeling \citep{pan2025survey,minaee2024large,chang2024survey,wan2023efficient,zhao2023survey}. However, this paradigm tightly couples generation to a fixed left-to-right order, making inference inherently sequential and restricting the model’s inductive bias to a single token ordering \citep{bachmann2024pitfalls,wu2025constrained,li2025beyond,monea2023pass,aishwarya2024tandem,fu2024break,you2024linear}. Recent progress in both discrete and continuous diffusion-based text modeling suggests that high-quality language generation need not rely on such a fixed order; instead, language models can also be defined through more general state evolution and denoising paths \citep{ye2024beyond,yu2025dimple,ou2024your,strudel2022self,chen2025dlm,mahabadi2024tess}.

Despite extensive exploration along autoregressive, discrete diffusion, and continuous diffusion directions \citep{song2020score,von2025generalized,li2025unifying,xu2024energy,goel2026skip,jarrett2025language,jeoung2025examining,lahoti2023improving}, existing methods still struggle to simultaneously achieve generation efficiency, scalable representation, and global semantic modeling. Autoregressive models directly parameterize token-level conditional probabilities, yielding a clear training objective, but their fixed generation order incurs inherent sequential inference cost and introduces a strong hand-crafted inductive bias, which limits performance on more general generation tasks \citep{lin2021limitations,berglund2023reversal,deschenaux2024promises,dalal2019autoregressive,zhu2024towards}. Discrete diffusion language models remove explicit left-to-right factorization \citep{hoogeboom2021autoregressive,gat2024discrete,hoogeboom2021argmax,zhao2025d1}, yet they still typically perform observation recovery in discrete token space, leading to costly multi-step sampling, while intermediate discrete states are not well suited to stably represent global semantic structure \citep{zheng2023reparameterized,zhou2024diffusion,venkatraman2024amortizing,srikanth2025discrete,takida2022sq,meng2022concrete,jang2016categorical}. Continuous diffusion methods further introduce continuous representation spaces \citep{sahoo2025diffusion,gong2024scaling,tae2025tess}, but most existing approaches still use the diffusion path to recover token-aligned representations rather than to explicitly model a latent prior \citep{gulrajani2023likelihood,dieleman2022continuous}. As a result, current methods have not yet provided a unified framework that systematically combines non-autoregressive generation, continuous representation, and probabilistic text modeling.

To address this gap, we propose \method, a hierarchical latent-space diffusion language model. \method first learns a stable mapping between text and continuous latent variables through a Text VAE \citep{zheng2025diffusion,xu2020variational,semeniuta2017hybrid,li2020optimus,bowman2016generating,kingma2013auto}, then models the latent prior in continuous latent space with a block-causal DiT \citep{peebles2023scalable,chandrasegaran2025exploring,park2024switch,mo2023dit,liu2025bcat,cannizzaro2023towards,bandyopadhyay2025block,zhang2020invariant}, and finally generates text through a conditional decoder. The key idea of \method is to use diffusion not for token-level observation recovery, but for latent prior transport. From a unified Markov-path perspective, this design explicitly decomposes text generation into two levels: global semantic organization in continuous latent space and local textual realization through conditional decoding. This decomposition weakens the inductive bias imposed by fixed token order, allows the geometry of continuous space to directly support semantic compression and prior fitting, and enables a more flexible non-autoregressive generation process. Moreover, block-causal prior modeling preserves cross-block causal structure while allowing more efficient parallel computation within each block. Grounded in hierarchical latent-space modeling, \method is also highly modular and naturally extensible to alternative latent modeling components and other continuous modalities \citep{zheng2025diffusion,deng2026generative}.

Motivated by these observations, we systematically study diffusion language modeling in continuous latent space from both theoretical and empirical perspectives. Our contributions are as follows.

\begin{itemize}
     \item We propose \method, a hierarchical latent-space language model that explicitly decomposes text generation into global semantic modeling and local textual realization within a unified probabilistic framework, while using diffusion-based prior modeling in continuous latent space to connect the two, thereby establishing a new paradigm for language generation from the perspective of hierarchical information decomposition.
    \item We analyze the differences between \method and existing language modeling paradigms from a unified Markov-path perspective, clarifying its advantages in global semantic modeling, non-autoregressive inductive bias, and theoretical interpretability, which are further validated in the subsequent experiments.
    \item Through extensive experiments spanning 4 research questions, 8 benchmarks, strictly matched $\sim$2B-parameter autoregressive and LLaDA baselines, and scaling curves up to about 2000 EFLOPs, we systematically validate the central claims of \method, identify an effective overall configuration, and verify its strong potential and favorable scaling behavior for text generation.
    \item We further analyze several issues beyond the core framework, including the mismatch between likelihood estimation and generation quality, first-block conditioning, and latent compression. We also provide preliminary evidence that \method offers a natural bridge from discrete text to continuous modalities such as vision, pointing to a broader unified generative paradigm.
\end{itemize}

%% file: sections/relatedwork.tex
\section{Related Work}

\subsection{Autoregressive Language Models}
Autoregressive language models \citep{radford2019language,touvron2023llama,yang2025qwen3,liu2024deepseek} factorize the discrete text distribution by the chain rule and are trained with token-level maximum likelihood, making them the most widely adopted paradigm for text modeling. Their limitations are that generation is constrained by a fixed left-to-right order, inference is inherently sequential, and they are less suitable for non-monotonic generation tasks such as infilling, local editing, and global reorganization. In contrast, \method first models a global semantic prior in a continuous latent space and then performs conditional decoding, thereby alleviating token-level ordering bias and improving generation efficiency with a block-causal DiT.

\subsection{Discrete Diffusion Language Models}
Discrete diffusion language models mainly fall into two categories. The first category is based on discrete transition kernels \citep{austin2021structured,campbell2022continuous,sun2022score}, which define forward perturbation and reverse recovery in discrete token space and achieve non-autoregressive generation through multi-step denoising; however, sampling is usually slow and these methods cannot easily exploit the smooth semantic structure of continuous spaces. The second category is based on masking or absorbing states \citep{sahoo2024simple,nie2025large,zhu2025llada,you2025llada,zhu2025lladamoe,shi2024simplified,zheng2025continuously,zheng2024masked,sahoo2025diffusion,nie2024scaling}, which construct training objectives by progressively mapping tokens to masks or absorbing states and then recovering the original text; however, information loss in intermediate states limits global semantic planning and fine-grained control. In contrast, \method moves the diffusion process to a continuous latent space, where compressible latent variables carry global semantics, thus combining the manipulability of continuous spaces with hierarchical semantic modeling.

\subsection{Continuous Diffusion Language Models}
Continuous diffusion language models can be broadly divided into three categories. The first category consists of high-dimensional vocabulary-aligned continuous methods \citep{han2023ssd,richemond2022categorical,mahabadi2024tess,jo2025continuous}, which perform continuous diffusion or flow modeling directly on one-hot vectors, logit simplexes, or probability simplexes to preserve alignment with discrete vocabularies; however, their representation dimension scales with vocabulary size, which limits scalability. The second category consists of token-embedding-based continuous methods \citep{li2022diffusion,strudel2022self,gao2024empowering,gong2022diffuseq,gulrajani2023likelihood,chen2023cheaper,lin2023text,dieleman2022continuous}, which first map text into continuous embedding spaces and then apply diffusion or flow modeling to improve generation flexibility; however, their generation process remains essentially the recovery of noisy target representations, lacking an explicit hierarchical latent-variable interpretation and a unified marginal-likelihood view of text distributions. The third category consists of latent-space continuous methods \citep{meshchaninov2025cosmos,kang2025ladir,lovelace2023latent,zhang2023planner}, which compress text into latent spaces with autoencoders or VAEs and then perform diffusion modeling. These methods typically rely on latent-space design and autoregressive decoders, and usually treat the latent space as a fixed representation rather than modeling it under a hierarchical latent-variable framework. In contrast, \method explicitly separates global semantics from local realization through hierarchical latent-variable modeling, and learns a semantic prior in a dynamic continuous latent space, thereby better balancing modeling flexibility, inference efficiency, and theoretical interpretability.

%% file: sections/approach.tex
\section{Continuous Latent Diffusion Language Model}

This section first presents \method as a hierarchical latent-variable language model with a rigorous probabilistic definition. We also outline the overall workflow of \method. We then place \method in a unified theoretical framework together with AR models, discrete denoising language models, and continuous token-space methods. Detailed derivations and proofs are deferred to Appendices~\ref{app:prob_model}, \ref{app:inference_and_likelihood}, \ref{app:comparison} and \ref{app:advantages}.

\subsection{Theoretical Foundations of Cola DLM}

In this subsection, we present \method as a hierarchical latent-variable language model with a rigorous probabilistic definition. We then introduce its unconditional and conditional probability estimators. Detailed derivations and proofs are provided in Appendices~\ref{app:prob_model} and~\ref{app:inference_and_likelihood}.

\subsubsection{Theoretical Formulation of Cola DLM}

\paragraph{\textbf{Hierarchical latent-variable modeling.}}
Let \(x \in \mathcal X\) denote a discrete text sequence, and let \(z_0 \in \mathbb R^d\) denote its continuous latent variable. The generative model of \method consists of a conditional decoder \(p_\theta(x \mid z_0)\) and a latent prior \(p_\psi(z_0)\):
\begin{equation}
p(x,z_0)=p_\theta(x\mid z_0)\,p_\psi(z_0),
\qquad
p(x)=\int p_\theta(x\mid z_0)\,p_\psi(z_0)\,dz_0.
\label{eq:main_ldlm_joint}
\end{equation}
Here, \(q_\phi(z_0\mid x)\) is used only for variational inference during training, and is not part of the generative model itself.

We model \(p_\psi(z_0)\) with a continuous-flow prior. Let the base distribution be \(p_1(z_1)=\mathcal N(0,I)\), and let \(v_\psi(z_t,t)\) be the vector field. Then
\begin{equation}
z_1\sim p_1,\qquad
\frac{dz_t}{dt}=v_\psi(z_t,t),\qquad
z_0=\Phi^\psi_{0\leftarrow 1}(z_1),
\label{eq:main_cnf_prior}
\end{equation}
which induces \(p_\psi=(\Phi^\psi_{0\leftarrow 1})_\sharp p_1\). In the sequence implementation, the latent is further decomposed into blocks, \(z_0=(z_0^{(1)},\ldots,z_0^{(B)})\), with
\begin{equation}
p_\psi(z_0)=p_\psi(z_0^{(1)})\prod_{b=2}^{B}p_\psi(z_0^{(b)}\mid z_0^{(<b)}).
\label{eq:main_block_prior}
\end{equation}
This factorization directly corresponds to the block-causal prior learning and block-wise inference used later.

\paragraph{\textbf{ELBO and prior learning.}}
By Jensen's inequality, the training lower bound of \method is
\begin{equation}
\log p(x)\ge
\mathbb E_{q_\phi(z_0\mid x)}
\!\left[
\log p_\theta(x\mid z_0)+\log p_\psi(z_0)-\log q_\phi(z_0\mid x)
\right]
=: \mathcal L_{\mathrm{ELBO}}(x).
\label{eq:main_elbo}
\end{equation}
Training therefore maximizes \(\mathcal L_{\mathrm{ELBO}}(x)\), or equivalently minimizes \(-\mathcal L_{\mathrm{ELBO}}(x)\).

Let the aggregated posterior be \(\bar q_\phi(z_0)=\int q_\phi(z_0\mid x)\,\pdata(x)\,dx\). The expected ELBO can then be written as
\begin{equation}
\mathbb E_{\pdata(x)}[\mathcal L_{\mathrm{ELBO}}(x)]
=
\mathbb E_{q(x,z_0)}[\log p_\theta(x\mid z_0)]
-
I_q(X;Z_0)
-
\KL(\bar q_\phi(z_0)\,\|\,p_\psi(z_0)),
\label{eq:main_avg_elbo}
\end{equation}
where \(q(x,z_0)=\pdata(x)q_\phi(z_0\mid x)\). This decomposition shows that \method separates text modeling into conditional reconstruction, information compression, and prior matching.

When the encoder and decoder are fixed, prior learning reduces to
\begin{equation}
\max_\psi\ \mathbb E_{z_0\sim \bar q_\phi}[\log p_\psi(z_0)]
\quad\Longleftrightarrow\quad
\min_\psi\ \KL(\bar q_\phi(z_0)\,\|\,p_\psi(z_0)).
\label{eq:main_prior_subproblem}
\end{equation}
In practice, we do not optimize the density directly. Instead, we learn the corresponding vector field with Flow Matching. For block \(b\), the conditional FM objective is
\begin{equation}
\mathcal L_{\mathrm{FM}}
=
\sum_{b=1}^{B}
\mathbb E_{t,z_0,z_1}
\left[
\left\|
v_\psi\!\bigl(z_t^{(b)},t;z_0^{(<b)}\bigr)-u_t^{(b)}(z_0,z_1)
\right\|_2^2
\right].
\label{eq:main_fm}
\end{equation}
Flow Matching is therefore a solver for the prior in \method, rather than the definition of the model itself.

\obsbox{
\textbf{Summary.}
The generative distribution of \method is defined by the hierarchical factorization in Eq.~\eqref{eq:main_ldlm_joint}: the latent prior \(p_\psi(z_0)\) generates global continuous semantics, and the decoder \(p_\theta(x\mid z_0)\) realizes discrete text. The encoder \(q_\phi(z_0\mid x)\) is not part of the generative model, but an inference model that lifts the observed data distribution \(\pdata(x)\) into a model-dependent latent joint distribution,
\[
q_\phi(x,z_0)=\pdata(x)q_\phi(z_0\mid x),
\qquad
\bar q_\phi(z_0)=\int q_\phi(z_0\mid x)\pdata(x)\,dx.
\]
Thus, while \(\pdata(x)\) is fixed, the induced latent distribution \(\bar q_\phi\) is not. When the representation is fixed, prior learning fits \(p_\psi\) to \(\bar q_\phi\). Under joint training, however, \(\bar q_\phi\) and \(p_\psi\) co-evolve: the encoder reshapes the latent data distribution, while the learned prior regularizes and organizes the latent space. Flow Matching is therefore only an implementation choice for learning this prior transport; the underlying model remains a hierarchical latent-variable language model.
}

\begin{figure}[t]
\centering

\begin{minipage}[t][0.185\textheight][t]{0.485\linewidth}
\vspace{0pt}
\footnotesize
\begin{sidealgorithm}{alg:uncond_loglik}{Unconditional log-likelihood estimation}
\Require $x$, $K$
\For{$k = 1,\dots,K$}
    \State sample $z_0^{(k)} \sim q_\phi(z_0 \mid x)$
    \State $w^{(k)} \gets \log p_\theta(x \mid z_0^{(k)}) + \log p_\psi(z_0^{(k)})$
    \Statex \hspace{\algorithmicindent}$\qquad\quad - \log q_\phi(z_0^{(k)} \mid x)$
\EndFor
\State output $\log \widehat p_{\mathrm{ELBO}}(x)=\frac{1}{K}\sum_{k=1}^K w^{(k)}$
\Statex \hspace{\algorithmicindent}or $\log \widehat p_{\mathrm{IWAE}}(x)=\log \frac{1}{K}\sum_{k=1}^K \exp(w^{(k)})$
\Statex \hspace{\algorithmicindent}according to Eq.~\eqref{eq:main_p_estimators}
\end{sidealgorithm}
\end{minipage}
\hfill
\begin{minipage}[t][0.185\textheight][t]{0.485\linewidth}
\vspace{0pt}
\footnotesize
\begin{sidealgorithm}{alg:cond_loglik}{Conditional log-likelihood estimation}
\Require $x_{\mathrm{pre}}$, $x_{\mathrm{res}}$
\State estimate $\log \widehat p(x_{\mathrm{pre}},x_{\mathrm{res}})$
\Statex \hspace{\algorithmicindent}using Algorithm~\ref{alg:uncond_loglik}
\State estimate $\log \widehat p(x_{\mathrm{pre}})$
\Statex \hspace{\algorithmicindent}using Algorithm~\ref{alg:uncond_loglik}
\State \Return $\log \widehat p(x_{\mathrm{res}} \mid x_{\mathrm{pre}})$
\Statex \hspace{\algorithmicindent}$= \log \widehat p(x_{\mathrm{pre}},x_{\mathrm{res}}) - \log \widehat p(x_{\mathrm{pre}})$
\Statex \hspace{\algorithmicindent}according to Eq.~\eqref{eq:main_cond_estimator}
\end{sidealgorithm}
\end{minipage}

\end{figure}
\subsubsection{Probability Estimation for Cola DLM}

\paragraph{\textbf{Unconditional probability estimation.}}
At evaluation time, we approximate \(\log p(x)\). For samples \(z_0^{(k)}\sim q_\phi(z_0\mid x)\), define the importance weight
\begin{equation}
\log w^{(k)}
=
\log p_\theta(x\mid z_0^{(k)})
+\log p_\psi(z_0^{(k)})
-\log q_\phi(z_0^{(k)}\mid x).
\label{eq:main_importance_weight}
\end{equation}

The prior term \(\log p_\psi(z_0^{(k)})\) is evaluated by the CNF change-of-variables formula. Concretely, we solve the augmented ODE
\begin{equation}
\frac{\dd}{\dd t}
\begin{bmatrix}
z_t \\
\ell_t
\end{bmatrix}
=
\begin{bmatrix}
\vv_{\psi}(z_t,t) \\
\nabla\!\cdot \vv_{\psi}(z_t,t)
\end{bmatrix},
\qquad
\begin{bmatrix}
z_0 \\
\ell_0
\end{bmatrix}
=
\begin{bmatrix}
z_0^{(k)} \\
0
\end{bmatrix},
\label{eq:main_augmented_ode}
\end{equation}
from \(t=0\) to \(t=1\), yielding \((z_1^{(k)},\ell_1^{(k)})\). We then obtain
\begin{equation}
\log p_\psi(z_0^{(k)})
=
\log p_1(z_1^{(k)})+\ell_1^{(k)},
\label{eq:main_prior_eval}
\end{equation}
where \(p_1\) is the terminal base distribution. In high dimensions, the divergence term is estimated with Hutchinson's trace estimator:
\begin{equation}
\nabla\!\cdot \vv_{\psi}(z_t,t)
=
\mathrm{Tr}\!\left(\frac{\partial \vv_{\psi}(z_t,t)}{\partial z_t}\right)
\approx
\veps^{\top}
\frac{\partial \vv_{\psi}(z_t,t)}{\partial z_t}
\veps,
\qquad
\veps\sim\mathcal N(0,I),
\label{eq:main_hutchinson}
\end{equation}
where the same \(\veps\) is fixed within one ODE solve.

This gives two standard estimators, namely the ELBO-style and IWAE-style estimators:
\begin{equation}
\log \widehat p_{\mathrm{ELBO},K}(x)
=
\frac{1}{K}\sum_{k=1}^{K}\log w^{(k)},
\qquad
\log \widehat p_{\mathrm{IWAE},K}(x)
=
\log\!\left(\frac{1}{K}\sum_{k=1}^{K}e^{\log w^{(k)}}\right).
\label{eq:main_p_estimators}
\end{equation}
The IWAE-style estimator is typically tighter.

\paragraph{\textbf{Conditional probability estimation.}}
For a prefix--response decomposition \(x=(x_{\mathrm{pre}},x_{\mathrm{res}})\), the exact identity is
\begin{equation}
\log p(x_{\mathrm{res}}\mid x_{\mathrm{pre}})
=
\log p(x_{\mathrm{pre}},x_{\mathrm{res}})
-
\log p(x_{\mathrm{pre}}).
\label{eq:main_cond_identity}
\end{equation}
We therefore obtain a plug-in estimator by scoring the joint sequence and the prefix with the same unconditional estimator:
\begin{equation}
\log \widehat p(x_{\mathrm{res}}\mid x_{\mathrm{pre}})
=
\log \widehat p(x_{\mathrm{pre}},x_{\mathrm{res}})
-
\log \widehat p(x_{\mathrm{pre}}).
\label{eq:main_cond_estimator}
\end{equation}

\subsection{Workflow of Cola DLM}
In this section, we describe the overall workflow of \method in detail. As illustrated in Figure~\ref{fig:cola_main_fig}, we explain the method from three perspectives: the pretraining of the Text VAE, the pretraining of prior learning with the Text DiT, and the inference process of \method.
\subsubsection{Text VAE Pretraining}

In the first stage, we learn a stable latent--text correspondence. The encoder maps text into the latent space, and the decoder reconstructs the original text conditioned on the latent:
\begin{equation}
z_0\sim q_\phi(z_0\mid x),
\qquad
\hat x\sim p_\theta(x\mid z_0).
\label{eq:main_stage1_io}
\end{equation}
The goal of this stage is not to learn the final prior, but to establish a stable division of labor between information stored in the latent and information recovered by the decoder.

The corresponding objective is
\begin{equation}
\mathcal L_{\mathrm{VAE}}
=
-\mathbb E_{q_\phi(z_0\mid x)}\log p_\theta(x\mid z_0)
+\beta\,\KL(q_\phi(z_0\mid x)\,\|\,p_{\mathrm{base}}(z_0))
+\lambda_{\mathrm{mask}}\mathcal L_{\mathrm{mask}}.
\label{eq:main_stage1_loss}
\end{equation}
Here, \(\mathcal L_{\mathrm{mask}}\) is the BERT-style masking loss shown in the figure. It prevents the VAE encoder from collapsing semantically while the decoder merely memorizes surface text. In our experiments, the VAE does not compress the sequence length. To prevent information leakage and facilitate subsequent streaming generation, both our VAE encoder and decoder are strictly causal.

\begin{figure}[t]
    \centering
    \includegraphics[width=1\textwidth]{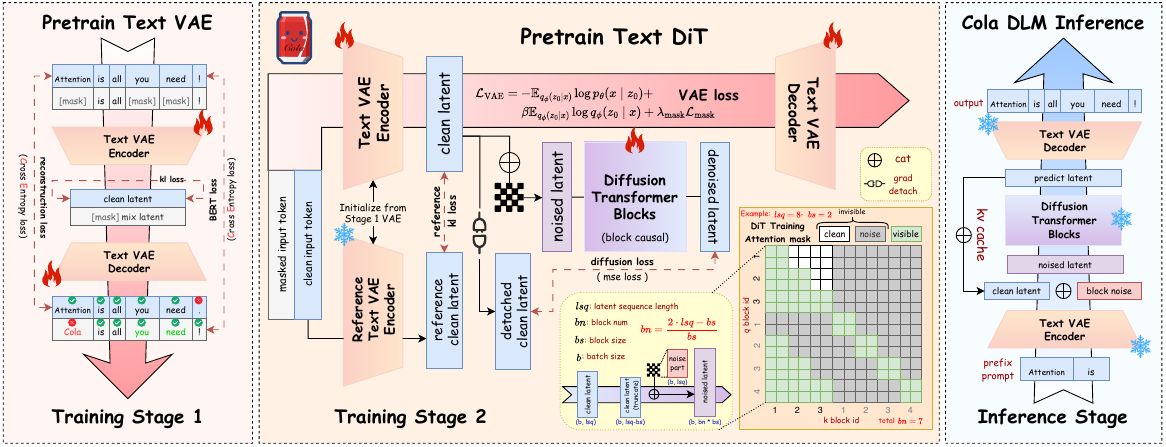}
    \caption{\textbf{The Overall Workflow of Cola DLM.} Detailed illustration of the training and inference pipeline of \method. \emph{Training Stage 1} shows Text VAE pretraining with reconstruction, BERT, and KL losses. \emph{Training Stage 2} shows joint pretraining of the Text VAE and Text DiT with gradient control for stable optimization, where a specialized block-causal mechanism is adopted in the DiT. \emph{Inference Stage} illustrates the decoding process with KV cache.}
    \label{fig:cola_main_fig}
\end{figure}

\subsubsection{Prior Learning with Block-Causal DiT}

In the second stage, we learn a conditional prior on the stabilized latent space. For block \(b\), the visible set consists of the historical clean latent blocks and the current noisy block:
\begin{equation}
\mathcal V_b=\left\{\operatorname{sg}(z_0^{(<b)}),\,z_t^{(b)}\right\},
\label{eq:main_visible_set}
\end{equation}
where \(\operatorname{sg}(\cdot)\) denotes stop-gradient. This visibility constraint enforces bidirectional attention within each block and causal dependence across blocks, consistent with Eq.~\eqref{eq:main_block_prior}.

Under this design, prior learning uses a joint objective that combines conditional Flow Matching with a reference-encoder regularizer:
\begin{equation}
\begin{aligned}
\mathcal L_{\mathrm{stage2}}
=&\;
\lambda_{\mathrm{VAE}}
\Bigl(
-\mathbb E_{q_\phi(z_0\mid x)}\log p_\theta(x\mid z_0)
+\beta\mathbb E_{q_\phi(z_0\mid x)}\log q_\phi(z_0\mid x)
+\lambda_{\mathrm{mask}}\mathcal L_{\mathrm{mask}}
\Bigr) \\
&\;
+\lambda_{\mathrm{fm}}\mathcal L_{\mathrm{FM}}
+\lambda_{\mathrm{ref}}
\mathbb E_{\pdata(x)}
\KL\!\Bigl(q_\phi(z_0\mid x)\,\|\,q_{\phi_{\mathrm{ref}}}(z_0\mid x)\Bigr).
\end{aligned}
\label{eq:main_stage2_loss}
\end{equation}
The first group preserves the autoencoding structure with regularized latent learning, the second term learns the block-level conditional prior, and the third term suppresses latent drift during joint training.

\subsubsection{Inference: Prefix Encoding, Block-wise Generation, and Conditional Decoding}

At inference time, the model first encodes the prefix into clean latent conditions:
\begin{equation}
z^{\mathrm{pre}}\sim q_\phi(z^{\mathrm{pre}}\mid x^{\mathrm{pre}}).
\label{eq:main_prefix_encode}
\end{equation}
It then generates the response latent block by block. Each block is obtained by transporting a noise seed under the historical condition:
\begin{equation}
\hat z_0^{(b)}
=
\Phi^\psi_{0\leftarrow 1}
\!\Bigl(\epsilon^{(b)};\,z^{\mathrm{pre}},\hat z_0^{(<b)}\Bigr),
\qquad
\epsilon^{(b)}\sim\mathcal N(0,I).
\label{eq:main_block_generation}
\end{equation}
Finally, the decoder outputs the text response conditioned on the prefix and the generated latent blocks:
\begin{equation}
\hat x^{\mathrm{res}}
\sim
p_\theta\!\bigl(x^{\mathrm{res}}\mid z^{\mathrm{pre}},\hat z_0^{(1:B)}\bigr).
\label{eq:main_decode}
\end{equation}

\obsbox{
\textbf{Summary.}
The workflow of \method implements the above hierarchical probabilistic model through two training stages and one inference stage, rather than a mechanical cascade of VAE, DiT, and decoder. In Stage 1, the base prior \(p_{\mathrm{base}}\) regularizes the latent--text interface and stabilizes the autoencoding representation, but it is not the final generative prior. In Stage 2, the block-causal DiT learns the final latent prior \(p_\psi(z_0)\) while the VAE remains trainable under reconstruction, masking, and reference regularization. This makes prior learning a controlled co-adaptation between the latent representation and the learned flow prior. At inference time, the model encodes the prefix, generates future latent blocks autoregressively in latent space, and realizes the response through the conditional decoder.
}

\subsection{A Unified View of Cola DLM and Existing Methods}
In this section, we compare \method with AR, LLaDA, and Plaid under a unified Markov-path perspective, and theoretically characterize the specific advantages of \method. More detailed analysis and proofs are provided in Appendices~\ref{app:comparison} and~\ref{app:advantages}.
\subsubsection{Text Modeling under a Unified Stochastic-Path View}

For a unified comparison, let \(\tau=(S_t)_{t\in\mathcal T}\) be a stochastic process on state space \(\mathcal S\), with initial distribution \(\mu_\Theta\), transition kernel \(K_t^\Theta\), and emission mechanism \(e_\Theta(x\mid \tau)\). A process-based generative model can be written as
\begin{equation}
p_\Theta(x)
=
\int e_\Theta(x\mid \tau)\,P_\Theta(d\tau),
\qquad
P_\Theta(d\tau)
=
\mu_\Theta(ds_0)\prod_{t>0}K_t^\Theta(ds_t\mid s_{<t}).
\label{eq:cmp_unified_path}
\end{equation}
This common outer form does not determine the nature of the model. The essential distinction lies in the state space of the path and its semantic role: a path over text or near-lossless text-aligned representations is an observation path, whereas a path used only to generate a latent prior is a prior path.

For AR, the path is the prefix expansion itself, yielding an exact chain factorization but binding generation to a left-to-right filtration:
\begin{equation}
p_{\mathrm{AR}}(x)=\prod_{i=1}^{L}p_\eta(x_i\mid x_{<i}).
\label{eq:main_ar}
\end{equation}

For LLaDA, the path is a discrete corruption--recovery trajectory, whose objective is observation reconstruction in a discrete state space:
\begin{equation}
q(s_{1:T}\mid x)=q_1(s_1\mid x)\prod_{t=2}^{T}q_t(s_t\mid s_{t-1}),
\qquad
p_\theta(s_{0:T})=p(s_T)\prod_{t=1}^{T}p_\theta(s_{t-1}\mid s_t).
\label{eq:main_llada}
\end{equation}
Thus, LLaDA weakens the handcrafted left-to-right bias, but still modifies the observation-recovery process rather than introducing an explicit hierarchical latent variable.

Plaid further moves this recovery process to a continuous token-aligned representation \(h_0=E(x)\):
\begin{equation}
q(h_{1:T}\mid h_0)=q_1(h_1\mid h_0)\prod_{t=2}^{T}q_t(h_t\mid h_{t-1}),
\qquad
p_\theta(h_{0:T})=p(h_T)\prod_{t=1}^{T}p_\theta(h_{t-1}\mid h_t).
\label{eq:main_plaid}
\end{equation}
Its core target is therefore continuous observation recovery, rather than a decomposition into a prior and a conditional decoder.

In \method, by contrast, the stochastic path only transports the latent prior:
\begin{equation}
z_1\sim p_1,\qquad
z_0=\Phi^\psi_{0\leftarrow 1}(z_1),\qquad
x\sim p_\theta(x\mid z_0),
\label{eq:main_ldlm_process}
\end{equation}
with the marginal still given by Eq.~\eqref{eq:main_ldlm_joint}. Hence, diffusion is used to learn a flexible continuous prior, not to impose a left-to-right inductive bias on text.

The reason for using a continuous path is not that continuous modeling is inherently superior, but that it naturally captures the geometry of the latent distribution. In \method, continuity appears in \(p_\psi(z_0)\), rather than in an observation-recovery trajectory:
\begin{equation}
\frac{dz_t}{dt}=v_\psi(z_t,t),
\qquad
p_\psi=(\Phi^\psi_{0\leftarrow 1})_\sharp p_1.
\label{eq:main_why_continuous}
\end{equation}
Thus, the distinction between \method and LLaDA lies in both state space and modeling target.

\begin{table*}[t]
\centering
\footnotesize
\setlength{\tabcolsep}{4pt}
\renewcommand{\arraystretch}{1.18}

\begin{tabularx}{\textwidth}{@{}
>{\hsize=0.90\hsize\centering\arraybackslash}X
>{\hsize=1.10\hsize\centering\arraybackslash}X
>{\hsize=1.15\hsize\centering\arraybackslash}X
>{\hsize=1.00\hsize\centering\arraybackslash}X
>{\hsize=1.20\hsize\centering\arraybackslash}X
>{\hsize=0.40\hsize\centering\arraybackslash}X
@{}}
\toprule
\bfseries Method
& \bfseries State Space
& \bfseries Path Role
& \bfseries \makecell[c]{Generative\\Factorization}
& \bfseries \makecell[c]{Where Continuity\\Appears}
& \bfseries \makecell[c]{Explicit\\Latent} \\
\specialrule{0.8pt}{0pt}{0pt}

\bfseries AR 
& Prefix Tokens 
& Direct Generation Path 
& \mbox{$\prod_i p(x_i \mid x_{<i})$}
& None 
& {\large\xmark} \\
\midrule

\bfseries LLaDA 
& Discrete Masked Sequences 
& Discrete Observation-Recovery Path 
& \mbox{$p(s_T)\prod_t p_\theta(s_{t-1}\mid s_t)$}
& \makecell[c]{Discrete\\Token Space}
& {\large\xmark} \\
\midrule

\bfseries Plaid 
& Continuous Token-Aligned Representations 
& Continuous Observation-Recovery Path 
& \mbox{$p(h_T)\prod_t p_\theta(h_{t-1}\mid h_t)$}
& \makecell[c]{Continuous\\Token Space}
& {\large\xmark} \\
\midrule

\bfseries Cola DLM 
& Compressed Latent Sequences 
& Prior-Transport Path 
& \mbox{$\int p_\theta(x\mid z_0)\,p_\psi(z_0)\,dz_0$}
& Latent Space 
& {\large\cmark} \\
\bottomrule
\end{tabularx}
\caption{\textbf{Unified Perspective.} Key differences among text models under a unified Markov-path view.}
\label{tab:main_unified_compare}
\end{table*}

Finally, the reason for using a latent variable is to explicitly separate semantic structure from token realization. The information decomposition of the expected ELBO,
\begin{equation}
\mathbb E_{\pdata(x)}[\mathcal L_{\mathrm{ELBO}}(x)]
=
\mathbb E_{q(x,z_0)}[\log p_\theta(x\mid z_0)]
-
I_q(X;Z_0)
-
\KL(\bar q_\phi(z_0)\,\|\,p_\psi(z_0)),
\label{eq:main_info_split_again}
\end{equation}
shows that \(z_0\) is not merely a continuous surrogate for discrete text, but an explicit marginalized intermediate variable: global semantics are compressed into \(z_0\), while local token realization is delegated to the decoder.

\subsubsection{Theoretical Advantages of Cola DLM}
\label{sec:theory_advantage}

\paragraph{\textbf{A unified criterion.}}
Let the lower bound of the approximation error for a model family \(\mathcal M\) be
\begin{equation}
\mathcal E(\mathcal M):=\inf_{p\in\mathcal M}\KL(\pdata(x)\,\|\,p(x)).
\label{eq:main_model_error}
\end{equation}
For AR, the population risk is determined solely by \(\mathcal E(\mathcal M_{\mathrm{AR}})\). In contrast, \method also incurs a variational inference gap:
\begin{equation}
G_{\mathrm{infer}}^{\mathrm{ColaDLM}}
:=
\mathbb E_{\pdata(x)}
\KL\!\bigl(q_\phi(z_0\mid x)\,\|\,p_{\theta,\psi}(z_0\mid x)\bigr).
\label{eq:main_infer_gap}
\end{equation}
Its total statistical burden is therefore
\begin{equation}
R_{\mathrm{ColaDLM}}
=
\mathcal E(\mathcal M_{\mathrm{ColaDLM}})
+
\inf_{\phi,\theta,\psi}G_{\mathrm{infer}}^{\mathrm{ColaDLM}}.
\label{eq:main_total_burden}
\end{equation}

\begin{proposition}[Unified criterion]
At the population level, \method outperforms a comparison model if and only if its total statistical burden is smaller. Taking AR as an example,
\begin{equation}
\mathcal{C}olaDLM \succ \mathrm{AR}
\quad\Longleftrightarrow\quad
R_{\mathrm{ColaDLM}}<\mathcal{E}(\mathcal{M}_{\mathrm{AR}}).
\label{eq:main_unified_criterion}
\end{equation}
\end{proposition}

\paragraph{\textbf{Rate-distortion and structured generation.}}
Whether a latent bottleneck is beneficial depends on whether the data admits a low-rate but informative global representation. Define the representation rate-distortion function as
\begin{equation}
D(R)
:=
\inf_{\substack{q(z_0\mid x):\\ I_q(X;Z_0)\le R}}
\ \inf_{p_\theta(x\mid z_0)}
\mathbb E_{q(x,z_0)}[-\log p_\theta(x\mid z_0)].
\label{eq:main_rate_distortion}
\end{equation}
If \(D(R)\) is already small at a low rate \(R\), then the data admits a low-dimensional semantic variable, and a latent bottleneck is more likely to reduce the overall mismatch. Conversely, if high-quality reconstruction requires a high information rate, aggressive compression only makes conditional reconstruction harder.

This can be characterized further through a structured-generation assumption. Suppose there exists a global variable \(G\) such that
\begin{equation}
\pdata(x)=\int p^\star(x\mid g)\,p^\star(g)\,dg,
\qquad
H(X\mid G)\ll H(X),
\qquad
\dim(G)\ll \dim(E(X)),
\label{eq:main_structured_generation}
\end{equation}
then the factorization of \method is closer to the true generative mechanism: the prior models the distribution of \(G\), while the decoder handles conditional realization. In this case, the latent bottleneck helps rather than hurts.

\paragraph{\textbf{Three governing curves and the applicability boundary.}}
The applicability of \method is ultimately determined by three curves: the representation rate-distortion curve \(D(R)\), the prior approximation curve, and the inference-gap curve \(G_{\mathrm{infer}}^{\mathrm{ColaDLM}}\). More compactly,
\begin{equation}
\begin{aligned}
    \text{\method is advantageous} \Longleftrightarrow \quad
    & \Bigl[D(R)\ \text{is already small at low }R\Bigr] \\
    \wedge\ & \Bigl[\mathcal E(\mathcal M_{\mathrm{ColaDLM}})\ \text{decreases}\Bigr] \\
    \wedge\ & \Bigl[G_{\mathrm{infer}}^{\mathrm{ColaDLM}}\ \text{is controllable}\Bigr].
\end{aligned}
\label{eq:main_three_curves}
\end{equation}
The benefit of \method is therefore not guaranteed by diffusion or continuity alone. It depends on whether the data exhibits a structure with low-dimensional global semantics and high-dimensional local token realization.

\obsbox{
\textbf{Summary.}
The central advantage of \method is not denoising itself, but the latent decomposition that separates text modeling into a global prior and a conditional realization process.}

%% file: sections/experiments.tex
\section{Experiments}

In this section, we conduct experiments to address the following research questions:

\begin{itemize}
    \item \textbf{RQ1: }Does a global semantic structure exist within the latent space?
    \item \textbf{RQ2: }What type of latent space is optimal for text generation?
    \item \textbf{RQ3: }Which diffusion process is most effective for text generation?
    \item \textbf{RQ4: }Why use a \emph{continuous} \emph{latent} \emph{diffusion} model for language modeling?
\end{itemize}

\subsection{Experimental Setup}

\textbf{Datasets. }
For training, we use external open-source pretraining data. For evaluation, the internal component analysis of \method (Sections~\ref{exp:rq1}, \ref{exp:rq2} and \ref{exp:rq3}) is conducted on randomly sampled subsets from the test sets of LAMBADA \citep{paperno2016lambada}, MMLU \citep{hendrycks2020measuring}, and SIQA \citep{sap2019social}. LAMBADA is a continuation benchmark, whereas the remaining two are multiple-choice benchmarks. For external comparisons (Section~\ref{exp:rq4}), we additionally evaluate on the test sets of SQuAD \citep{rajpurkar2016squad}, Story Cloze \citep{mostafazadeh2016corpus}, OBQA \citep{mihaylov2018can}, RACE \citep{lai2017race}, and HellaSwag \citep{zellers2019hellaswag}. Additional dataset details are deferred to Appendix~\ref{app:dataset}.

\textbf{Baselines. } 
In the internal component comparison experiments (Sections~\ref{exp:rq1}, \ref{exp:rq2} and \ref{exp:rq3}), we specify the different configurations of \method. In Section~\ref{exp:rq4}, for the scaling comparison, we independently train the autoregressive and LLaDA baselines under strictly matched settings. Specifically, the autoregressive and discrete diffusion models are randomly initialized using the official modeling implementations of LLaMA \citep{touvron2023llama} and LLaDA \citep{nie2025large}, respectively. Details are provided in Appendix~\ref{app:baseline}.

\textbf{Metrics. }
As discussed in Section~\ref{dis:ppl}, the estimated perplexity exhibits a substantial mismatch with the actual generation quality of \method. Moreover, prior work \citep{wang2023knn, hashimoto2019unifying, holtzman2019curious, meister2021language} has noted that perplexity is not strictly correlated with generation performance. To enable the most objective and fair comparison, we therefore evaluate all models under a unified few-shot setting across both multiple-choice and generative tasks. For multiple-choice benchmarks and continuation tasks such as Lambada and SQuAD, accuracy is computed by strict string matching between the model output and the ground-truth answer under predefined rules. Additional evaluation details are provided in Appendix~\ref{app:metrics}.

\textbf{Setup. }
\method uses the OLMo 2 \citep{olmo20242} tokenizer and is trained with AdamW. The learning rate starts at \(1\times10^{-6}\), is linearly warmed up to \(1.5\times10^{-4}\) over the first 5{,}000 steps, and is then decayed with a cosine schedule to \(1\times10^{-5}\) by 1{,}000{,}000 steps. All evaluations are conducted using the checkpoint at the corresponding FLOPs budget, without EMA weights. The same tokenizer, optimization and evaluation settings are used for \method and all external baselines. In \method, the VAE contains 500M parameters and the DiT contains 1.8B parameters. For the autoregressive and discrete diffusion baselines, the embedding layer has approximately 400M parameters and the non-embedding backbone has 1.8B parameters. We thus keep the total model size of the two model families at a comparable scale of roughly 2B parameters. All methods are trained with the same random seed so that the training data are matched across runs, with the maximum sequence length set to 512. Additional details are provided in Appendix~\ref{app:setup} and~\ref{app:fairness_embed_vae_latent_space}.

\subsection{Evidence of Global Semantic Structures in Cola DLM (RQ1)}
\label{exp:rq1}
In this section, we first present an implication for the existence of global semantic structures, and then provide strong empirical evidence for their existence by quantitatively examining the performance of latent spaces with different dimensions under different timestep shifts. The full theoretical derivation, proof, and technical details are provided in Appendix~\ref{sec:global_semantic_structure}.

\implicationbox{imp:rq1_global_semantic_structure}{
If the latent representation is purely local and fully separable, then the optimal timeshift does not exhibit a stable drift as the latent dimension changes. Therefore, if the optimal timeshift is observed to shift systematically with the latent dimension in experiments, this indicates the existence of cross-dimensional shared structures in the latent space; if this phenomenon is mainly reflected in semantic metrics, it further supports that these shared structures are related to high-level semantics.
}

Based on Implication~\ref{imp:rq1_global_semantic_structure}, the focus of this section is not merely on the specific optimal loc values under different latent dimensions, but on whether the peak position of the optimal timeshift exhibits a stable and regular shift as the latent dimension varies. Figure~\ref{fig:rq1_global_semantic_structure} presents the corresponding experimental results.

\textbf{Obs. \ding{182} The optimal timeshift exhibits a systematic drift with the latent dimension. }As shown in the left panel of Figure~\ref{fig:rq1_global_semantic_structure}, the best loc for Task Avg shifts from approximately $1.0$ at $d=16$, to approximately $1.7$ at $d=64$, and further to approximately $2.3$ at $d=128$. This trend is clear and approximately monotonic. This phenomenon directly contradicts the separable null hypothesis. A more plausible explanation is that changing the latent dimension alters the effective noise calibration position of some cross-dimensional shared structure.

\textbf{Obs. \ding{183} This trend is consistent across multiple semantic metrics. }The right panel of Figure~\ref{fig:rq1_global_semantic_structure} further shows that, although the best loc values for LAMBADA, MMLU, SIQA, and Task Avg are not exactly the same, they all overall favor larger loc regions as the latent dimension increases. This indicates that the peak drift is not an accidental fluctuation of any single task, but rather a stable phenomenon jointly supported by multiple semantic evaluations. Therefore, what is modulated by timeshift is a representation structure shared across different semantic tasks.

\textbf{Obs. \ding{184} The empirical peaks are broadly consistent with the theoretical predictions. }In the left panel of Figure~\ref{fig:rq1_global_semantic_structure}, the dashed lines indicate the theoretically predicted optimal loc values. It can be seen that, under all three latent-dimension settings, the empirical peaks are close to the predicted positions, and the drift directions are fully consistent. This suggests that the observed drift is not an arbitrary empirical hyperparameter effect, but is instead consistent with the theoretical analysis in Appendix~\ref{sec:global_semantic_structure}, namely that shared latent structures lead to dimension-dependent timeshift calibration.

Implication~\ref{imp:rq1_global_semantic_structure} provides a rigorous contrapositive statement for the existence of global semantic structures, while the experimental results, by providing reverse evidence for this contrapositive, offer strong empirical support for the existence of shared and semantically relevant global structures in the latent space of \method. This also provides supporting evidence for the first condition in Eq.~\eqref{eq:main_three_curves} of Section~\ref{sec:theory_advantage}, thereby further supporting the advantage of \method.

\begin{figure}[t]
    \centering
    \includegraphics[width=\textwidth]{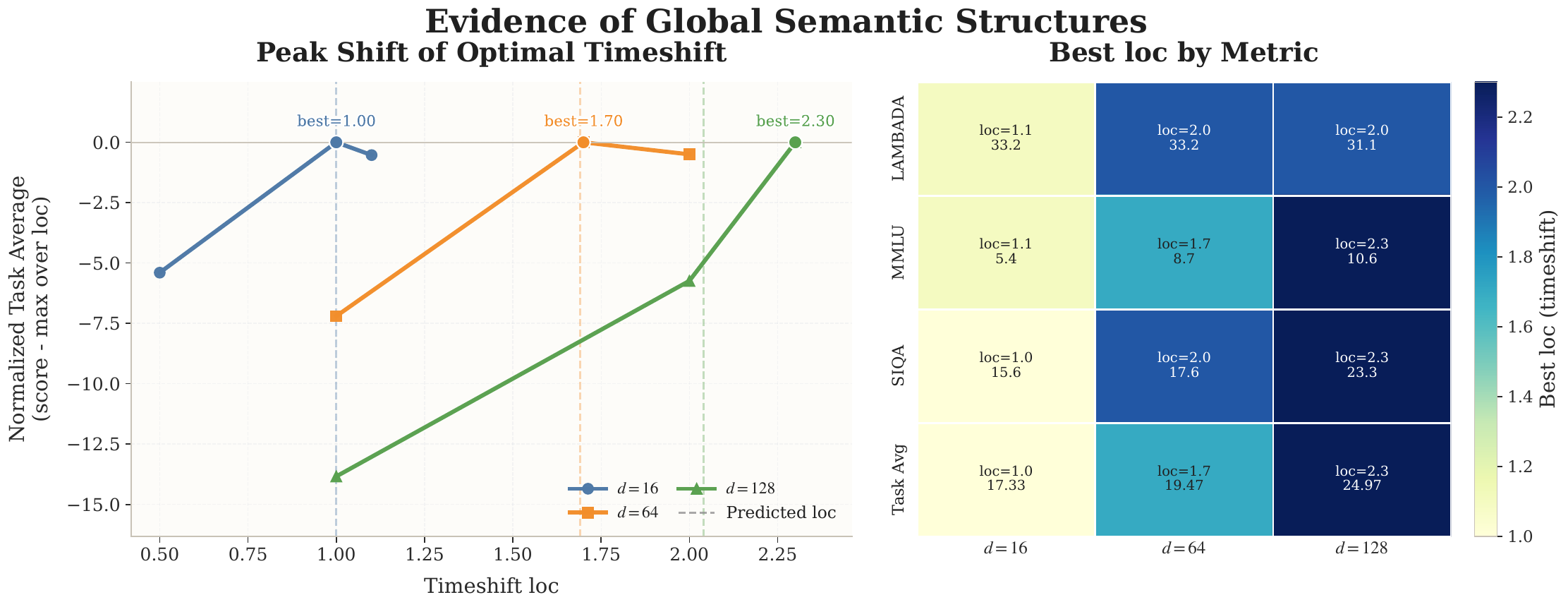}
    \caption{\textbf{Evidence of global semantic structures in the latent space.} Left: as latent dimension increases, optimal timeshifts move toward larger locations, with empirical peaks matching predictions. Right: most of the evaluation metrics consistently prefer larger best locations at higher dimensions. This stable cross-metric trend supports shared global semantic structures within the latent space.}
    \label{fig:rq1_global_semantic_structure}
\end{figure}

\subsection{Analysis of Different Latent Spaces in Cola DLM (RQ2)}
\label{exp:rq2}

In this section, we present a detailed empirical study of different latent spaces through both quantitative evaluation and visualization. We analyze the design of the latent space from three perspectives: whether it should be dynamic or static, what latent dimensionality is most appropriate, and how semantic smoothness contributes to latent-space quality. Based on these analyses, we identify the most effective directions for optimizing the latent space.

\begin{figure}[t]
    \centering
    \begin{subfigure}[t]{0.49\linewidth}
        \centering
        \includegraphics[width=\linewidth]{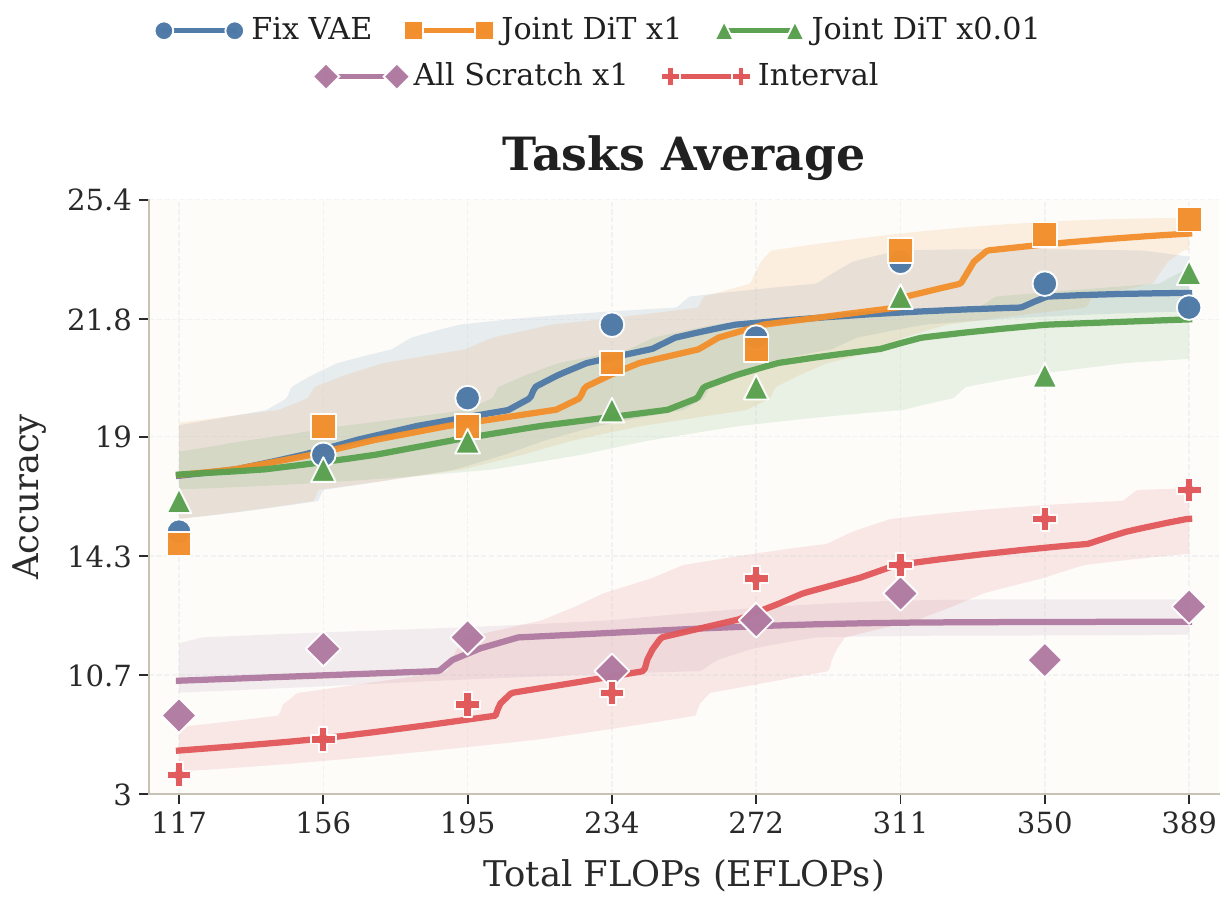}
        \label{fig:rq2_fix_vs_evolve_tasks_avg}
    \end{subfigure}
    \hfill
    \begin{subfigure}[t]{0.49\linewidth}
        \centering
        \includegraphics[width=\linewidth]{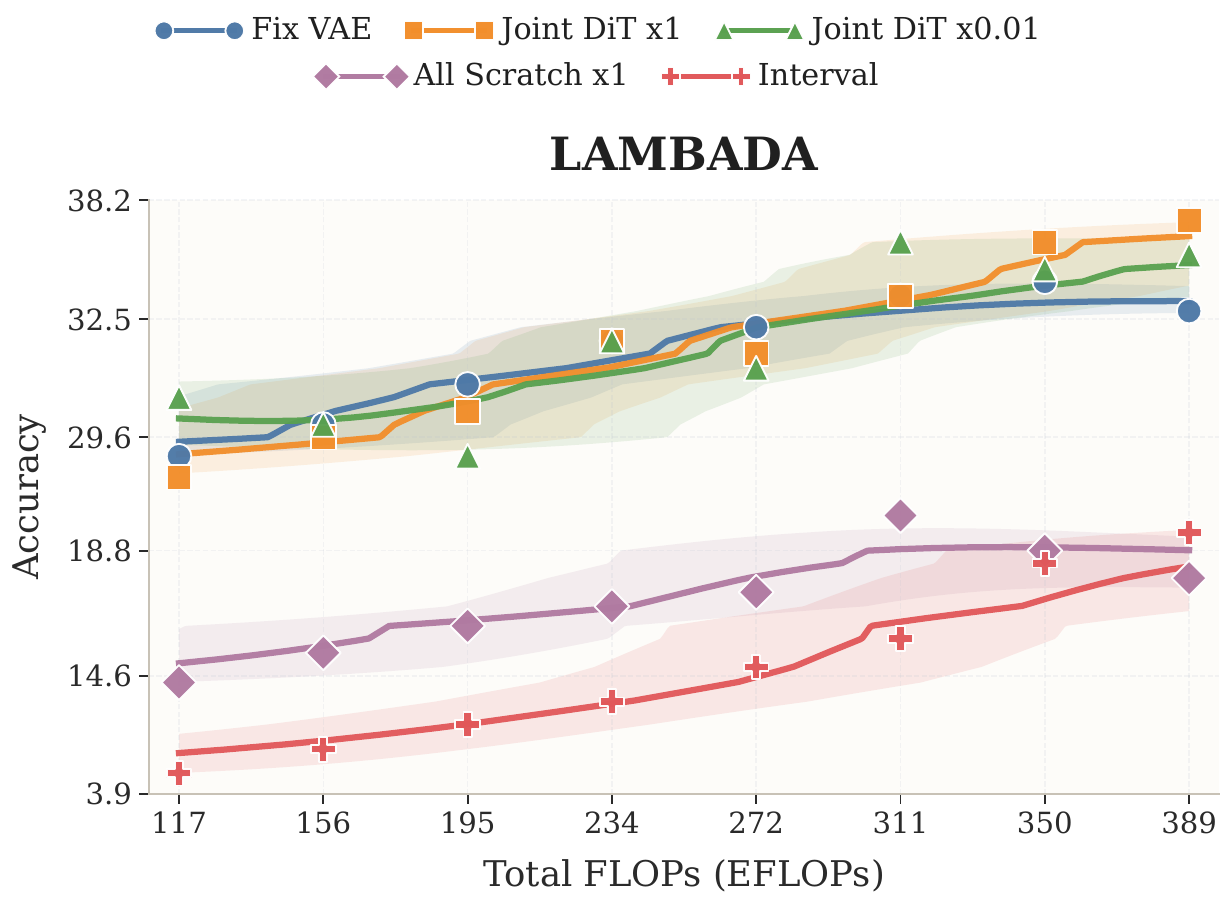}
        \label{fig:rq2_fix_vs_evolve_lambada}
    \end{subfigure}

    \vspace{0.5em}

    \begin{subfigure}[t]{0.49\linewidth}
        \centering
        \includegraphics[width=\linewidth]{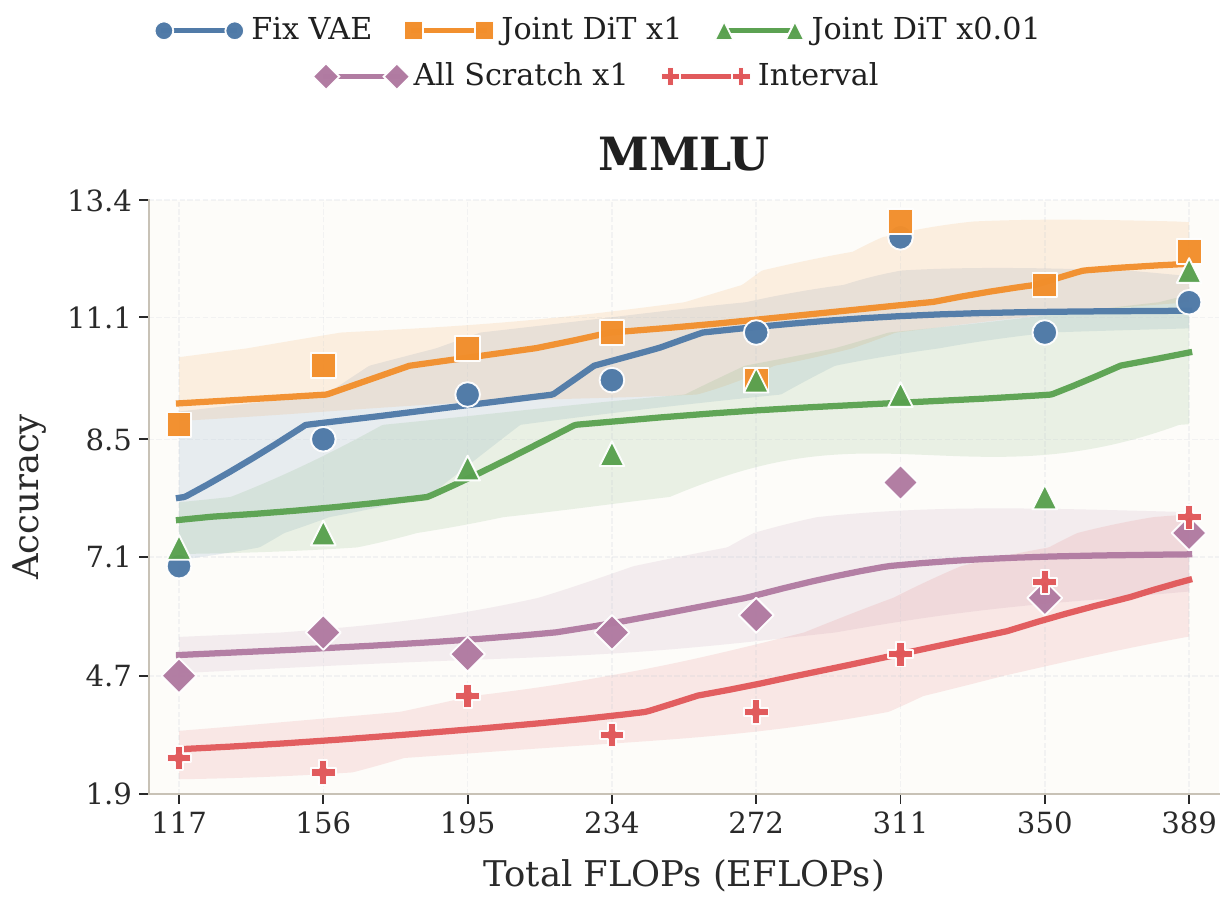}
        \label{fig:rq2_fix_vs_evolve_mmlu}
    \end{subfigure}
    \hfill
    \begin{subfigure}[t]{0.49\linewidth}
        \centering
        \includegraphics[width=\linewidth]{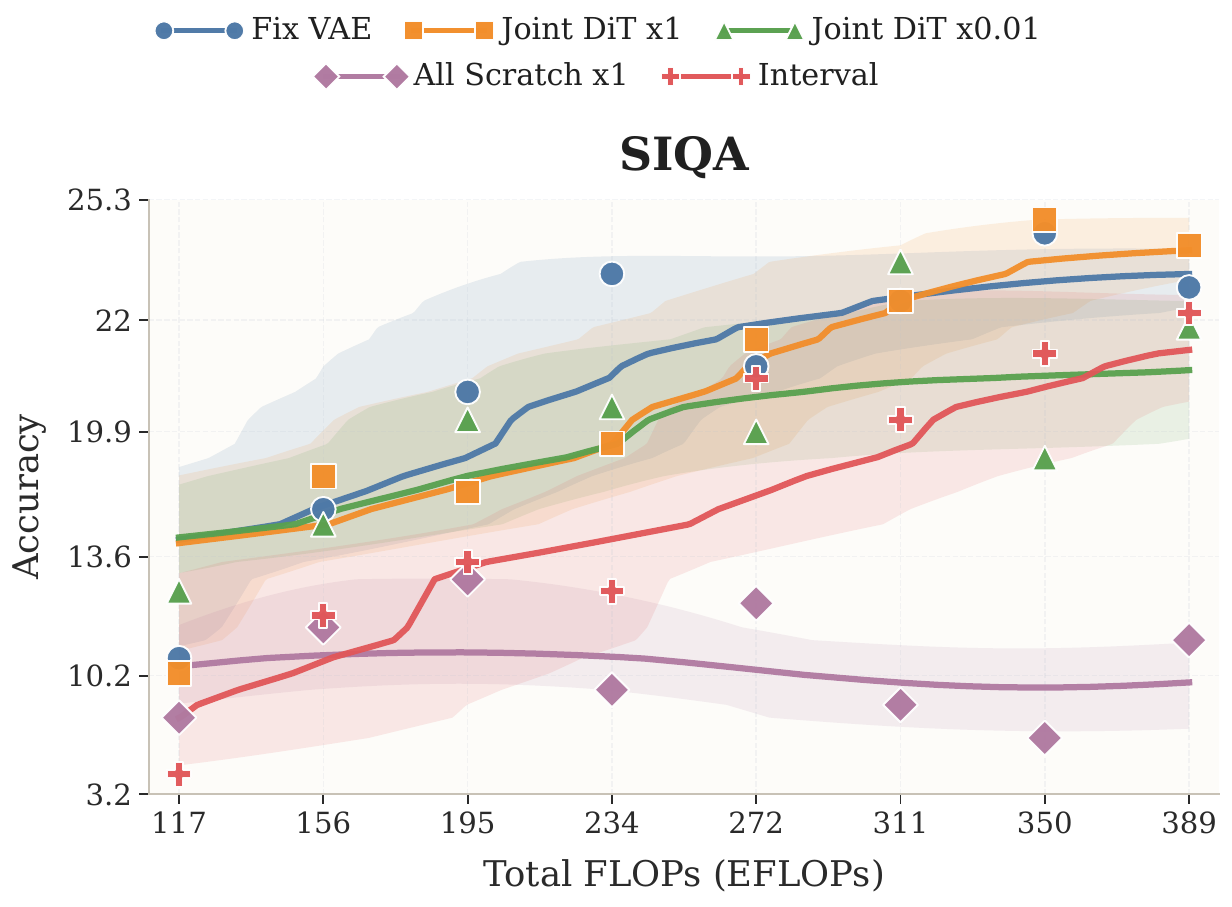}
        \label{fig:rq2_fix_vs_evolve_siqa}
    \end{subfigure}
    \caption{\textbf{Comparison between fixed and evolving latent spaces.} Across Task Average, LAMBADA, MMLU, and SIQA, joint evolution with DiT achieves the best overall scaling and performance when initialized from a stable pretrained VAE. Fixed spaces lead to earlier saturation, while training from scratch or interval updates remains less effective. This suggests the space should evolve with DiT, but from a stable initialization rather than scratch.}
    \label{fig:rq2_fix_vs_evolve}
\end{figure}

\begin{figure}[t]
    \centering
    \begin{subfigure}[t]{0.32\linewidth}
        \centering
        \includegraphics[width=\linewidth]{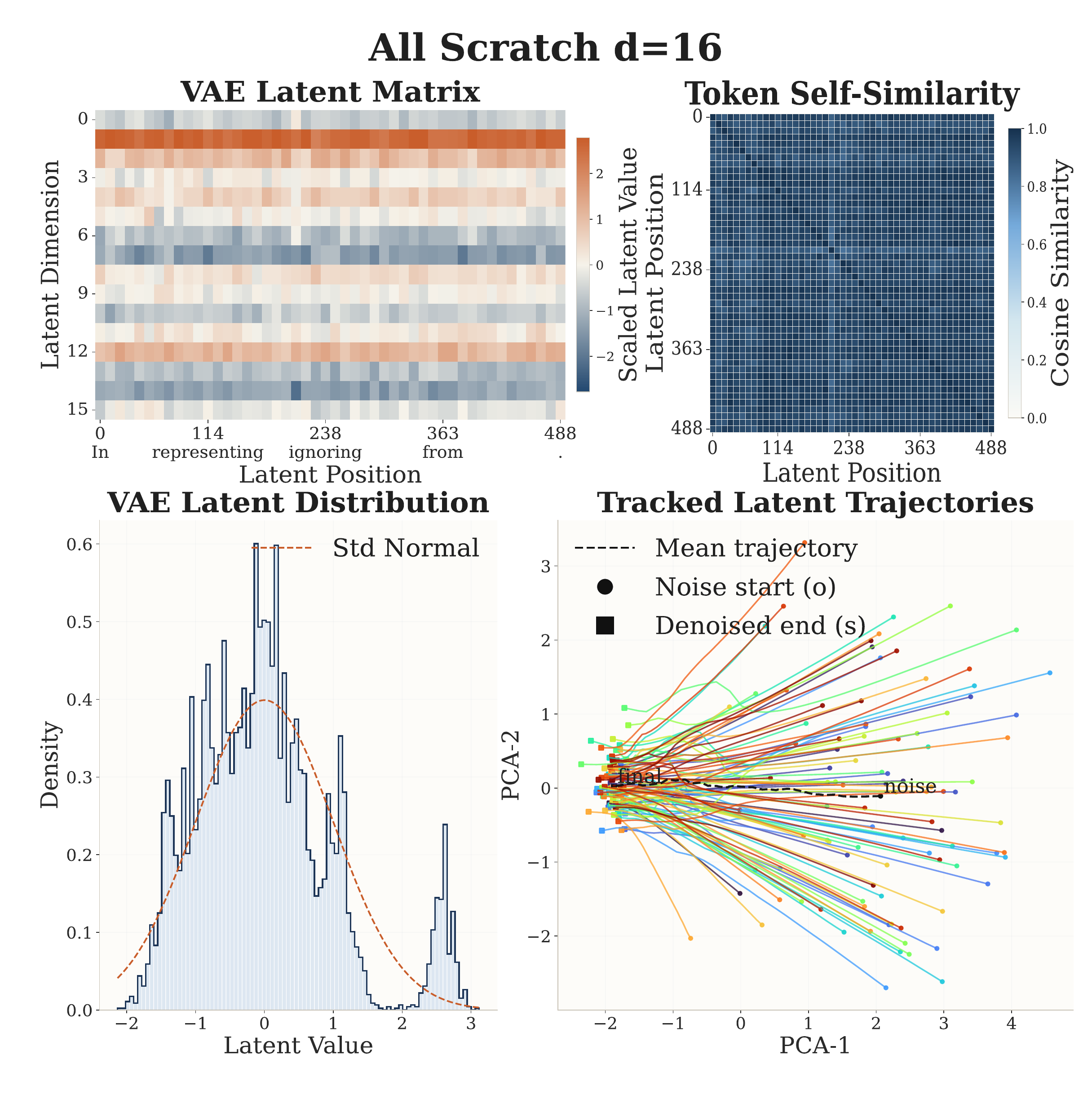}
        \label{fig:rq2_fix_vs_evolve_vis_all_sc_16}
    \end{subfigure}
    \hfill
    \begin{subfigure}[t]{0.32\linewidth}
        \centering
        \includegraphics[width=\linewidth]{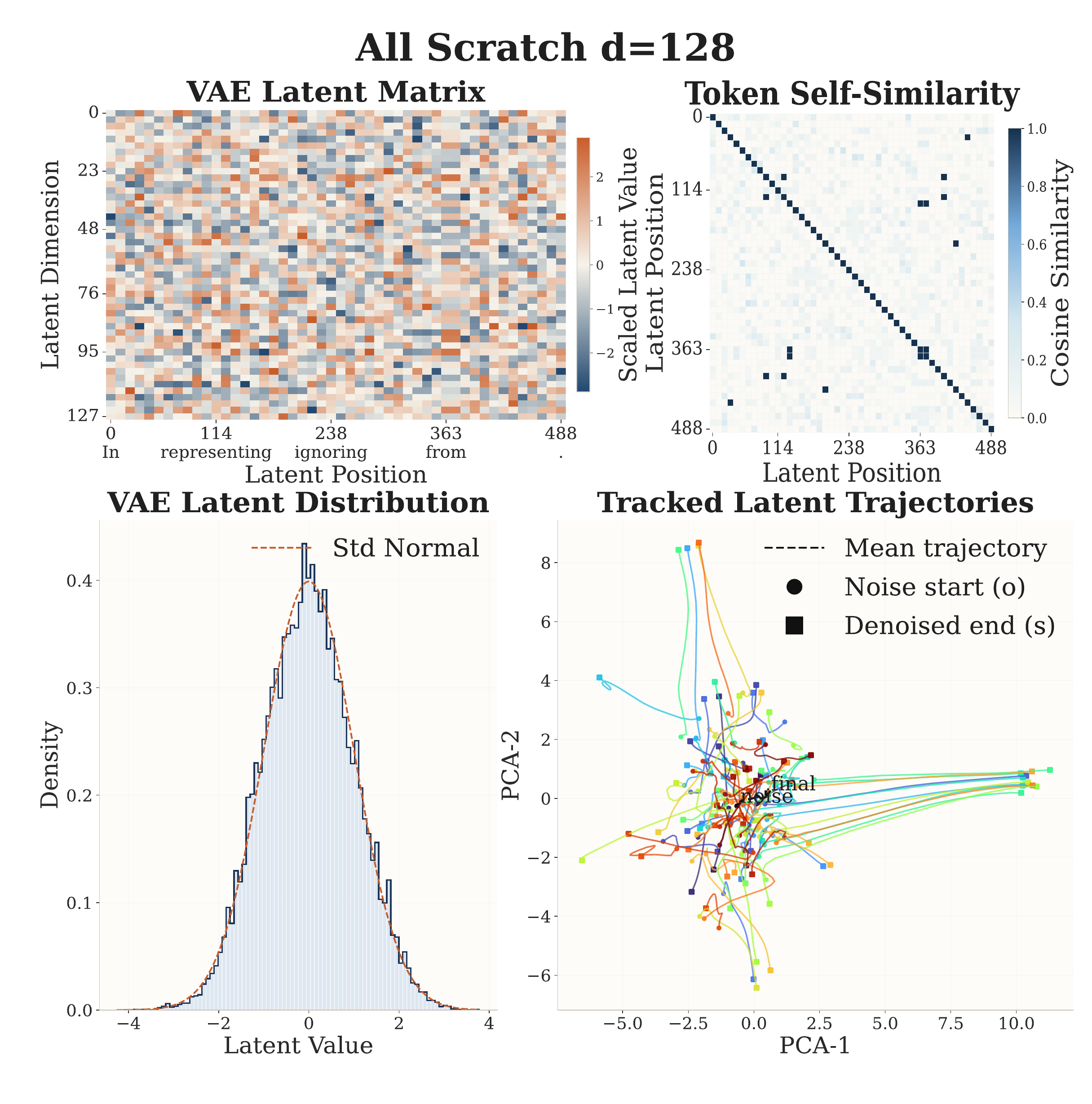}
        \label{fig:rq2_fix_vs_evolve_vis_all_sc_128}
    \end{subfigure}
    \hfill
    \begin{subfigure}[t]{0.32\linewidth}
        \centering
        \includegraphics[width=\linewidth]{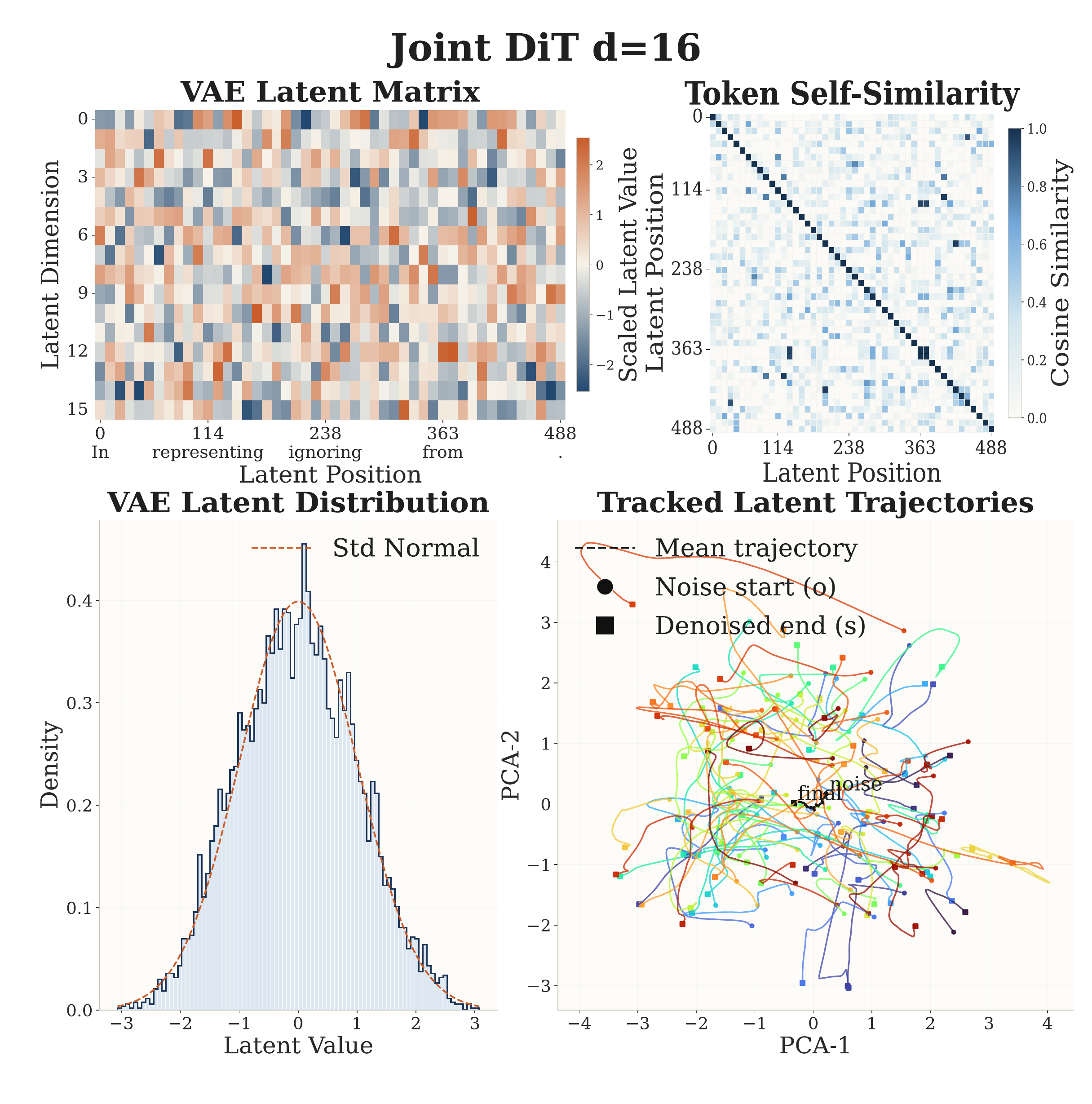}
        \label{fig:rq2_fix_vs_evolve_vis_dit_sc_16}
    \end{subfigure}
    \caption{\textbf{Visualization of latent spaces under different training strategies.} Joint optimization on a stable initialization yields a more structured, semantically organized latent space than training VAE and DiT from scratch. Increasing the latent dimension ($16$ to $128$) partially mitigates collapse but remains less structured than the stable-initialization approach.}
    \label{fig:rq2_fix_vs_evolve_visualize}
\end{figure}

\paragraph{Fixed vs. Evolving Latent Space.}
As shown in Figure~\ref{fig:rq2_fix_vs_evolve}, this section studies whether the latent space should evolve jointly with DiT during training. Under the same compute budget, we compare five strategies: fixing a pretrained VAE (Fix VAE); initializing the VAE from pretrained weights and jointly training it with DiT using a VAE learning rate equal to that of DiT or scaled to $0.01\times$ (Joint DiT x1 / Joint DiT x0.01); jointly training both VAE and DiT from random initialization with the same learning rate (All Scratch x1); and an interval-based strategy (Interval), where each $5$k-step cycle consists of $2$k steps of joint training followed by $3$k steps with the VAE frozen. The overall results suggest that the latent space should neither remain fully fixed nor be jointly optimized from scratch without constraint. Instead, the most effective strategy is to let it evolve together with DiT on top of a stable initialization.

\textbf{Obs. \ding{182} Joint DiT x1 shows the strongest scaling potential.}
At small compute budgets, Fix VAE and Joint DiT x1 are close, and Fix VAE is sometimes slightly better. As FLOPs increase, however, Joint DiT x1 improves more steadily and achieves the best final results on Task Avg, LAMBADA, MMLU, and SIQA, whereas Fix VAE gradually saturates. This indicates that a fixed latent space helps early stability but limits the performance ceiling, while continuous co-adaptation with DiT is more beneficial for scaling.

\textbf{Obs. \ding{183} The benefit of joint training depends on good initialization rather than trainability alone.}
All Scratch x1 performs consistently worse than the other methods across all metrics, and its gains remain limited throughout training. This suggests that the advantage of Joint DiT x1 does not come from making the latent space trainable by itself; it relies on starting from a meaningful pretrained latent space and then adapting it jointly with DiT.

\textbf{Obs. \ding{184} The latent-space visualization explains why All Scratch underperforms.}
Figure~\ref{fig:rq2_fix_vs_evolve_visualize} shows that All Scratch with $d=16$ yields a more collapsed and less structured latent space, with trajectories dominated by simple outward drift. Increasing the latent dimension to $128$ partially alleviates this issue, but the geometry still remains less organized than that of Joint DiT with stable initialization. In contrast, Joint DiT produces more heterogeneous latent patterns and richer trajectories, suggesting a more structured and semantically usable space.

\textbf{Obs. \ding{185} Effective latent evolution requires both continuous participation and sufficient update strength.}
Joint DiT x0.01 and Interval are both better than All Scratch x1, but still clearly worse than Joint DiT x1 in overall trend and final performance. This shows that partial latent participation is not enough: overly weak updates slow adaptation, while periodic freezing disrupts co-evolution with DiT. A better strategy is to update the latent space continuously and strongly, while keeping the initialization stable.

Overall, Figure~\ref{fig:rq2_fix_vs_evolve} and Figure~\ref{fig:rq2_fix_vs_evolve_visualize} consistently show that the best latent-space strategy is neither to keep it fixed nor to train it from scratch, but to let it evolve jointly with DiT on top of a good initialization. When the VAE and DiT are trained jointly, the VAE is exposed to more data and can fit $\pdata(z|x)$ more accurately. This further verifies the last condition in Eq.~\eqref{eq:main_three_curves} of Section~\ref{sec:theory_advantage}, and provides strong support for the potential advantage of \method. Additional results are in Appendix~\ref{app:fix_vs_evolve_fulu}.

\begin{wraptable}{r}{0.55\textwidth}
    \vspace{-0.8em}
    \centering
    \caption{\textbf{Dimensionality of the latent space under 117 EFLOPs.} Larger latent dimensions improve the overall average under the all-scratch setting with $\mathrm{loc}=1$.}
    \label{tab:latent_dim_half}
    \small
    \setlength{\tabcolsep}{3.5pt}
    \renewcommand{\arraystretch}{1.08}
    \begin{tabular}{lcccc}
        \toprule
        Method & Lambada & MMLU & SIQA & Avg. \\
        \midrule
        All Scratch, $d=16$, loc=1  & 14.3 & \underline{6.9} & 4.9 & 8.7 \\
        All Scratch, $d=64$, loc=1  & \textbf{20.9} & 5.4 & \underline{7.6} & \underline{11.3} \\
        All Scratch, $d=128$, loc=1 & \underline{18.5} & \textbf{8.1} & \textbf{8.9} & \textbf{11.8} \\
        \bottomrule
    \end{tabular}
    \vspace{-0.8em}
\end{wraptable}

\paragraph{Dimensionality of the Latent Space. }
We next study how the latent dimensionality affects both performance and latent-space quality. Table~\ref{tab:latent_dim_half} compares All Scratch models with different latent dimensions under the same EFLOPs budget ($117$), Figure~\ref{fig:rq2_fix_vs_evolve_visualize} provides the corresponding latent-space visualization, and Figure~\ref{fig:rq1_global_semantic_structure} shows how the optimal timeshift changes with dimension. Taken together, these results suggest that increasing the latent dimension partially alleviates collapse and improves semantic capacity, but it also changes the effective noise calibration of the latent space.

\textbf{Obs. \ding{182} Increasing the latent dimension improves the overall semantic capacity under the same compute budget.}
As shown in Table~\ref{tab:latent_dim_half}, the average score increases from $8.7$ at $d=16$ to $11.3$ at $d=64$, and further to $11.8$ at $d=128$. The improvement is most evident on MMLU and SIQA. Although LAMBADA peaks at $d=64$, the overall trend still suggests that a larger latent space carries stronger semantic capacity under the same compute budget.

\textbf{Obs. \ding{183} A larger latent dimension partially alleviates latent-space collapse, but does not fully solve it.}
Figure~\ref{fig:rq2_fix_vs_evolve_visualize} shows that increasing the dimension from $16$ to $128$ makes the latent space less collapsed and more dispersed. However, the resulting geometry still remains clearly less structured than Joint DiT with stable initialization. This indicates that increasing dimensionality is helpful, but cannot by itself replace proper latent-space formation.

\begin{figure}[t]
    \centering
    \begin{subfigure}[t]{0.49\linewidth}
        \centering
        \includegraphics[width=\linewidth]{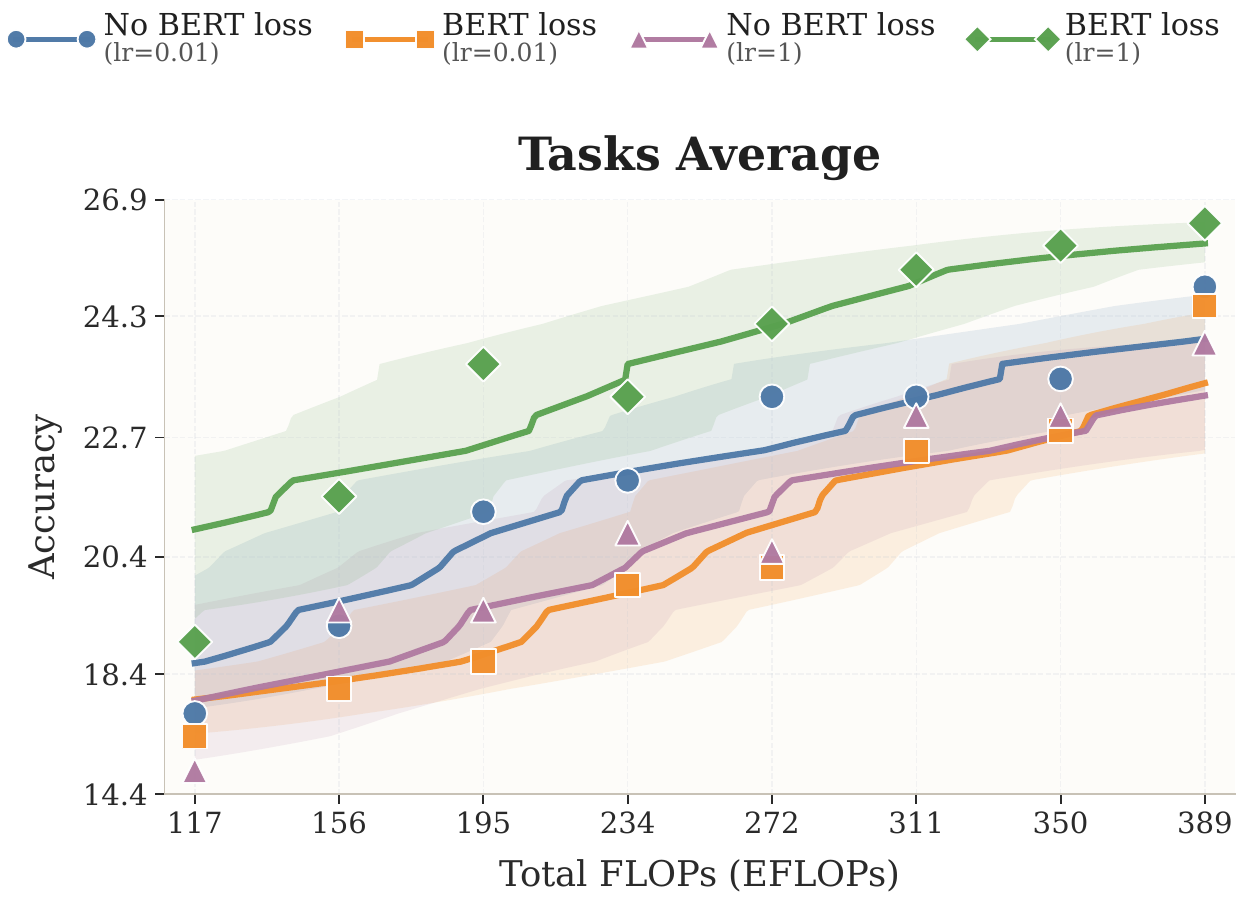}
        \label{fig:rq2_semantic_tasks_avg}
    \end{subfigure}
    \hfill
    \begin{subfigure}[t]{0.49\linewidth}
        \centering
        \includegraphics[width=\linewidth]{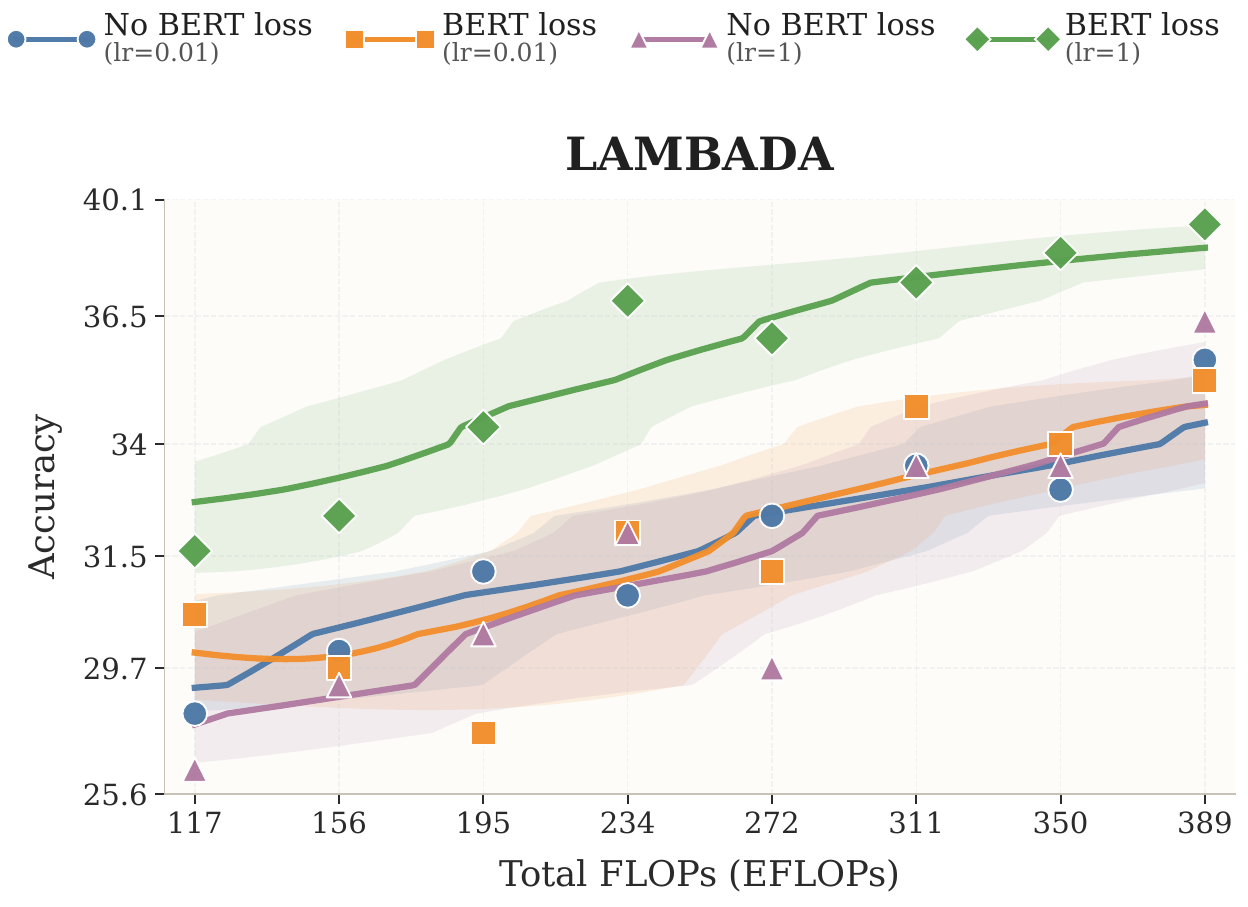}
        \label{fig:rq2_semantic_lambada}
    \end{subfigure}

    \vspace{0.5em}

    \begin{subfigure}[t]{0.49\linewidth}
        \centering
        \includegraphics[width=\linewidth]{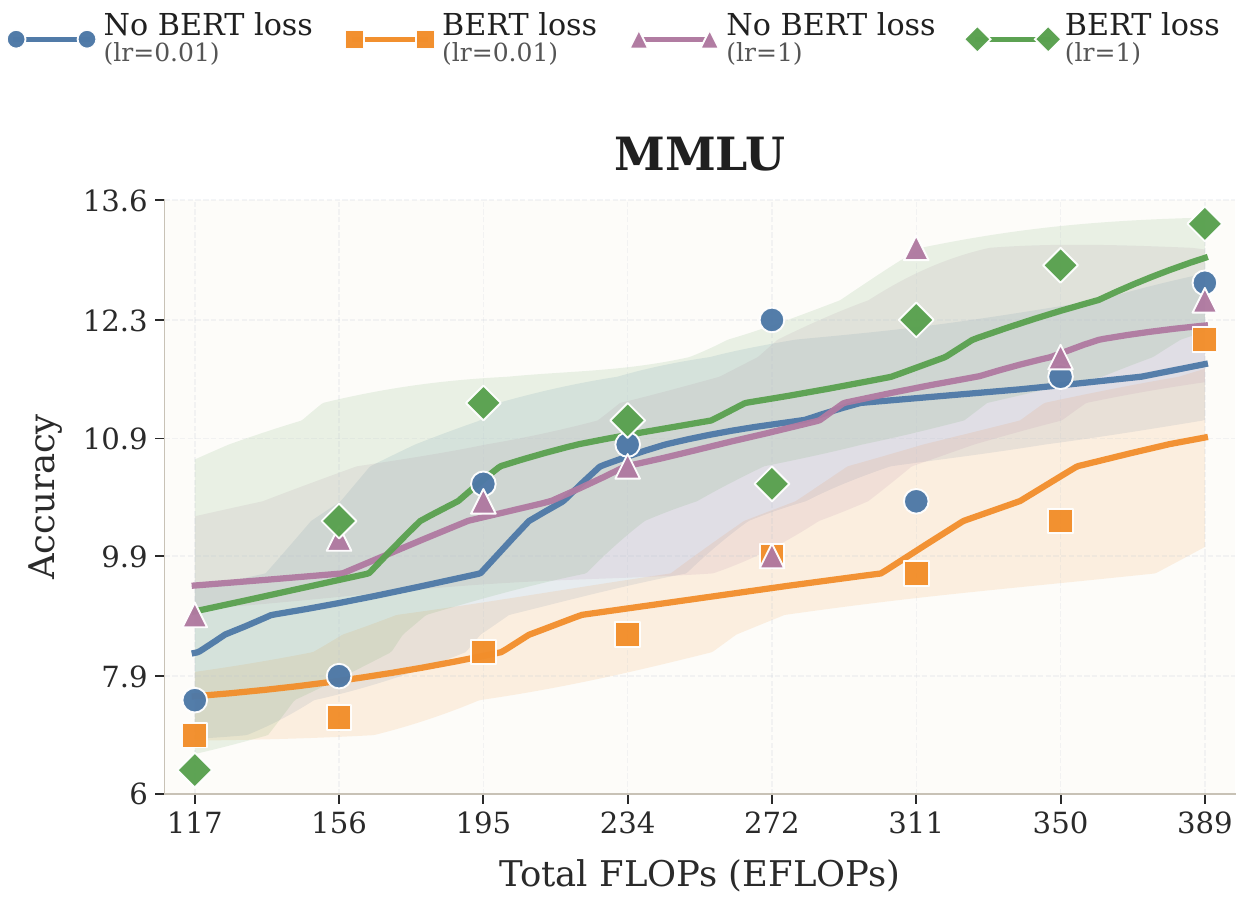}
        \label{fig:rq2_semantic_mmlu}
    \end{subfigure}
    \hfill
    \begin{subfigure}[t]{0.49\linewidth}
        \centering
        \includegraphics[width=\linewidth]{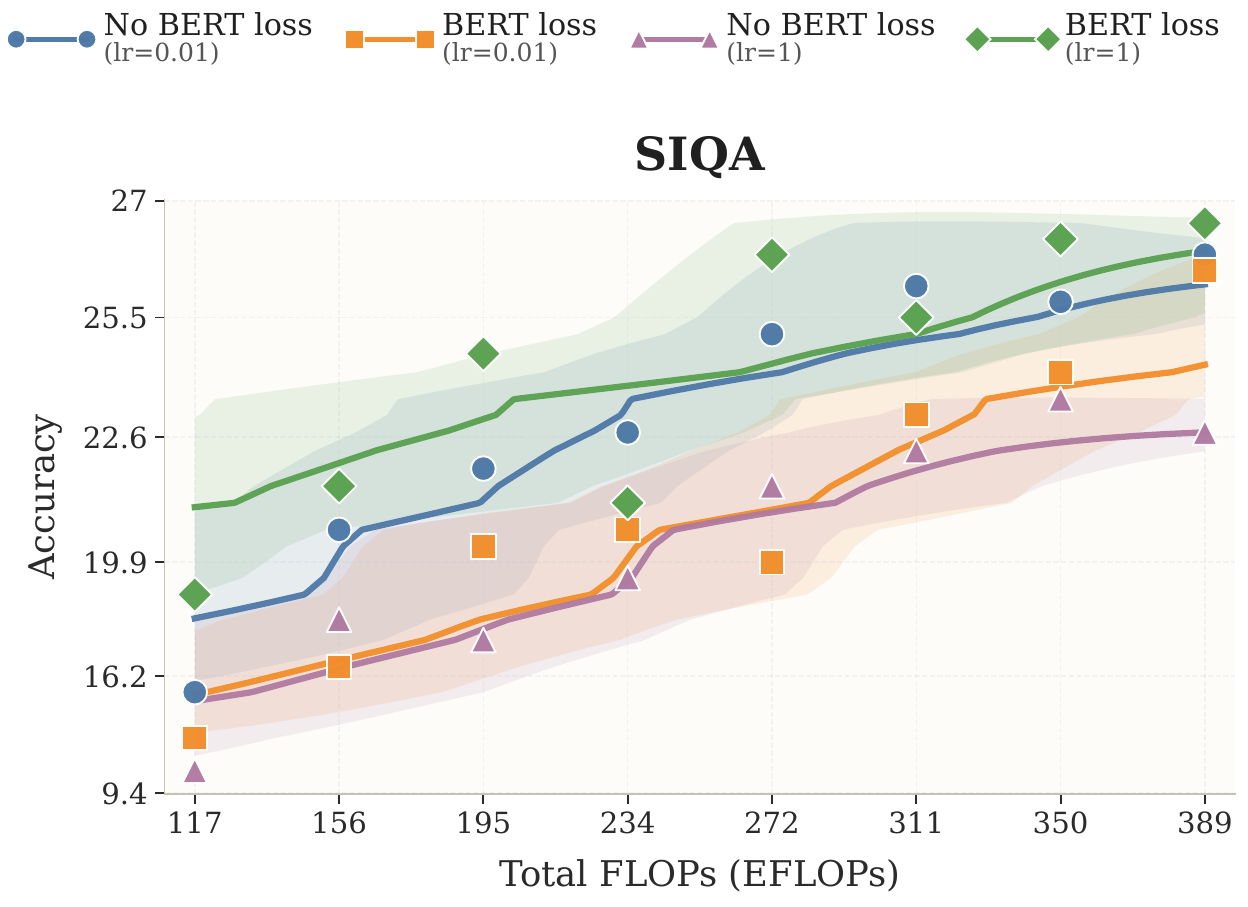}
        \label{fig:rq2_semantic_siqa}
    \end{subfigure}
    \caption{\textbf{Effect of semantic smoothness in the latent space under the Joint DiT setting.} Adding a BERT-style loss consistently improves performance, with larger gains during active latent updates ($\mathrm{lr}=1$). This suggests semantic smoothness benefits latent-space quality, especially when evolving jointly with DiT.}
    \label{fig:rq2_semantic_importance}
\end{figure}

\textbf{Obs. \ding{184} The effect of latent dimensionality is not only geometric, but also dynamical.}
Figure~\ref{fig:rq1_global_semantic_structure} shows that the best timeshift systematically shifts toward larger loc values as the latent dimension increases. This means that increasing the latent dimension does not merely enlarge the space; it also changes the denoising scale at which semantic information is best recovered. Therefore, the benefit of a higher-dimensional latent space depends not only on improved geometry, but also on proper noise calibration.

Overall, Table~\ref{tab:latent_dim_half}, Figure~\ref{fig:rq2_fix_vs_evolve_visualize}, and Figure~\ref{fig:rq1_global_semantic_structure} present a consistent picture: increasing the latent dimension improves latent-space quality and downstream performance, but the gain is only partial, and its full benefit still depends on proper training dynamics and timeshift calibration.

\paragraph{Semantic Importance of the Latent Space. }
As shown in Figure~\ref{fig:rq2_semantic_importance}, all results in this subsection are obtained under the Joint DiT setting, where the VAE is initialized from pretrained weights and jointly optimized with DiT. We further compare whether to add a BERT-style loss in VAE training, which encourages the latent space to preserve smoother local semantics. Here, the reported lr denotes the VAE learning-rate ratio relative to DiT. The results show that such semantic smoothness is important for downstream performance, especially when the latent space is allowed to evolve more actively.

\textbf{Obs. \ding{182} Adding BERT loss consistently improves performance when the latent space is actively updated.}
When the VAE learning-rate ratio is $1$, BERT loss gives the best overall results across nearly the entire training range. In Figure~\ref{fig:rq2_semantic_importance}, the BERT-loss curve consistently outperforms its no-BERT counterpart on Task Average, LAMBADA, MMLU, and SIQA, and also achieves the best final performance. This indicates that encouraging masked-token recoverability makes the latent space more semantically useful for downstream prediction.

\textbf{Obs. \ding{183} Strong latent evolution is effective only with semantic guidance.}
When the VAE learning-rate ratio is $0.01$, adding BERT loss brings only limited gains, whereas its advantage becomes clear and stable when the ratio is increased to $1$. At the same time, simply increasing the VAE update strength without BERT loss does not reliably improve performance and is even weaker at several later-stage points. This shows that trainability alone is not sufficient: when the latent space evolves more actively, its updates must also be constrained toward a semantically smoother organization.

Overall, Figure~\ref{fig:rq2_semantic_importance} shows that semantic smoothness is an important property of a useful latent space. It not only improves final performance, but also makes joint latent evolution substantially more effective and stable. These results suggest that the latent should be compact but semantically sufficient, consistent with Eq.~\eqref{eq:main_avg_elbo} and Eq.~\eqref{eq:main_three_curves}. The BERT-style loss helps retain useful semantics under the bottleneck.

\begin{table}[t]
    \centering
    \caption{\textbf{Performance under different VAE logSNR settings.} VAE logSNR strongly affects downstream performance. A learnable setting gives the best overall results, while fixed logSNR = 1.5 is the strongest fixed alternative.}
    \label{tab:fixlogsnr_two_budgets}
    \small
    \setlength{\tabcolsep}{5pt}
    \renewcommand{\arraystretch}{1.10}
    \begin{tabular*}{\textwidth}{@{\extracolsep{\fill}}lcccc@{\hspace{1.2em}}cccc}
        \toprule
        \textbf{Compute Budget}
        & \multicolumn{4}{c}{\textbf{EFLOPs = 77.86}}
        & \multicolumn{4}{c}{\textbf{EFLOPs = 116.78}} \\
        \cmidrule(r){1-1} \cmidrule(lr){2-5} \cmidrule(lr){6-9}
        \textbf{Method}
        & \textbf{Lambada} & \textbf{MMLU} & \textbf{SIQA} & \textbf{Avg.}
        & \textbf{Lambada} & \textbf{MMLU} & \textbf{SIQA} & \textbf{Avg.} \\
        \midrule
        Fixed VAE logSNR = 1.0 & 27.1 & 5.7 & 11.3 & 14.70 & 30.4 & 7.7 & 18.4 & 18.83 \\
        Fixed VAE logSNR = 1.5 & 29.5 & \underline{7.8} & \textbf{17.5} & \underline{18.27} & \underline{33.8} & 8.0 & \textbf{23.6} & \underline{21.80} \\
        Fixed VAE logSNR = 2.0 & \underline{30.9} & 5.1 & 14.3 & 16.77 & 32.7 & \underline{9.7} & 19.5 & 20.63 \\
        Learnable VAE logSNR ($\approx 4.5$) & \textbf{32.6} & \textbf{7.9} & \underline{16.2} & \textbf{18.90} & \textbf{34.6} & \textbf{10.1} & \underline{21.6} & \textbf{22.1} \\
        \bottomrule
    \end{tabular*}
\end{table}

\paragraph{Smoothness of the Latent Space. }
Table~\ref{tab:fixlogsnr_two_budgets} compares different VAE logSNR settings at two compute budgets. The results show that the VAE logSNR is an important factor for latent-space smoothness and downstream performance. Under the current setup, learning the VAE logSNR gives the strongest overall results, while fixing the VAE logSNR at 1.5 is the most competitive fixed alternative. The VAE logSNR formula is given in Appendix~\ref{app:vae_logsnr_calc}.

\textbf{Obs. \ding{182} Learning the VAE logSNR gives the strongest overall performance under the current setup.}
At both 77.86 and 116.78 EFLOPs, the learnable VAE logSNR setting achieves the best Task Average in Table~\ref{tab:fixlogsnr_two_budgets}. It also gives the best LAMBADA results at both checkpoints and the best MMLU result at the higher compute budget. This suggests that keeping the VAE logSNR learnable is currently the strongest overall choice, likely because it allows a more flexible smoothness profile during latent-space training.

\textbf{Obs. \ding{183} Fixing the VAE logSNR at 1.5 is the strongest fixed alternative.}
Although the learnable VAE logSNR setting ranks first on average, fixing the VAE logSNR at 1.5 remains very close at both compute budgets. It also consistently achieves the best SIQA results and stays competitive on the other tasks. This indicates that a properly chosen fixed VAE logSNR can already provide a strong balance between semantic preservation and optimization stability.

\textbf{Obs. \ding{184} The current results favor a learnable VAE logSNR, while still leaving room for further study of fixed settings.}
The advantage of the learnable VAE logSNR over fixing the VAE logSNR at 1.5 is consistent but not large, suggesting that the current conclusion is clear but not yet definitive. Since Table~\ref{tab:fixlogsnr_two_budgets} only reports two compute budgets, the scaling behavior of different VAE logSNR settings remains open and deserves more systematic study.

Overall, Table~\ref{tab:fixlogsnr_two_budgets} shows that the VAE logSNR is an important factor in shaping latent-space smoothness and downstream performance. Under the current setup, a learnable VAE logSNR is the strongest overall choice, while fixing the VAE logSNR at 1.5 stands out as a highly competitive fixed alternative.

\subsection{Ablation on the Diffusion Process in Cola DLM (RQ3)}
\label{exp:rq3}
In this section, we systematically study the training and inference design choices of the DiT module through ablations. By combining quantitative results with visualizations, we further analyze the mechanisms behind the observed optimization trends. On the training side, we investigate DiT models with different block sizes and examine the effect of different noise training schedules on downstream performance. On the inference side, we study the impact of the number of denoising steps and the choice of Classifier-Free Guidance (CFG) scales.

\begin{figure}[t]
    \centering
    \begin{subfigure}[t]{0.49\linewidth}
        \centering
        \includegraphics[width=\linewidth]{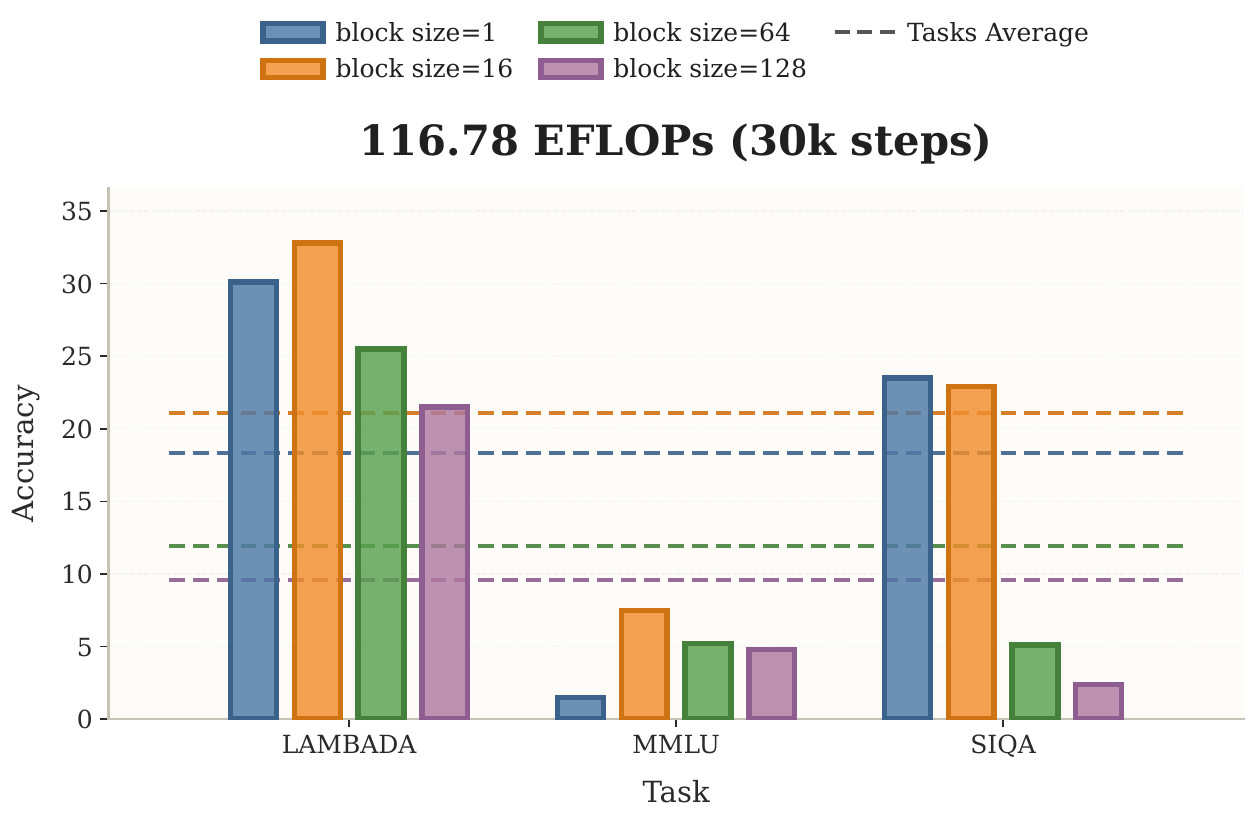}
        \label{fig:rq3_block_size_30k}
    \end{subfigure}
    \hfill
    \begin{subfigure}[t]{0.49\linewidth}
        \centering
        \includegraphics[width=\linewidth]{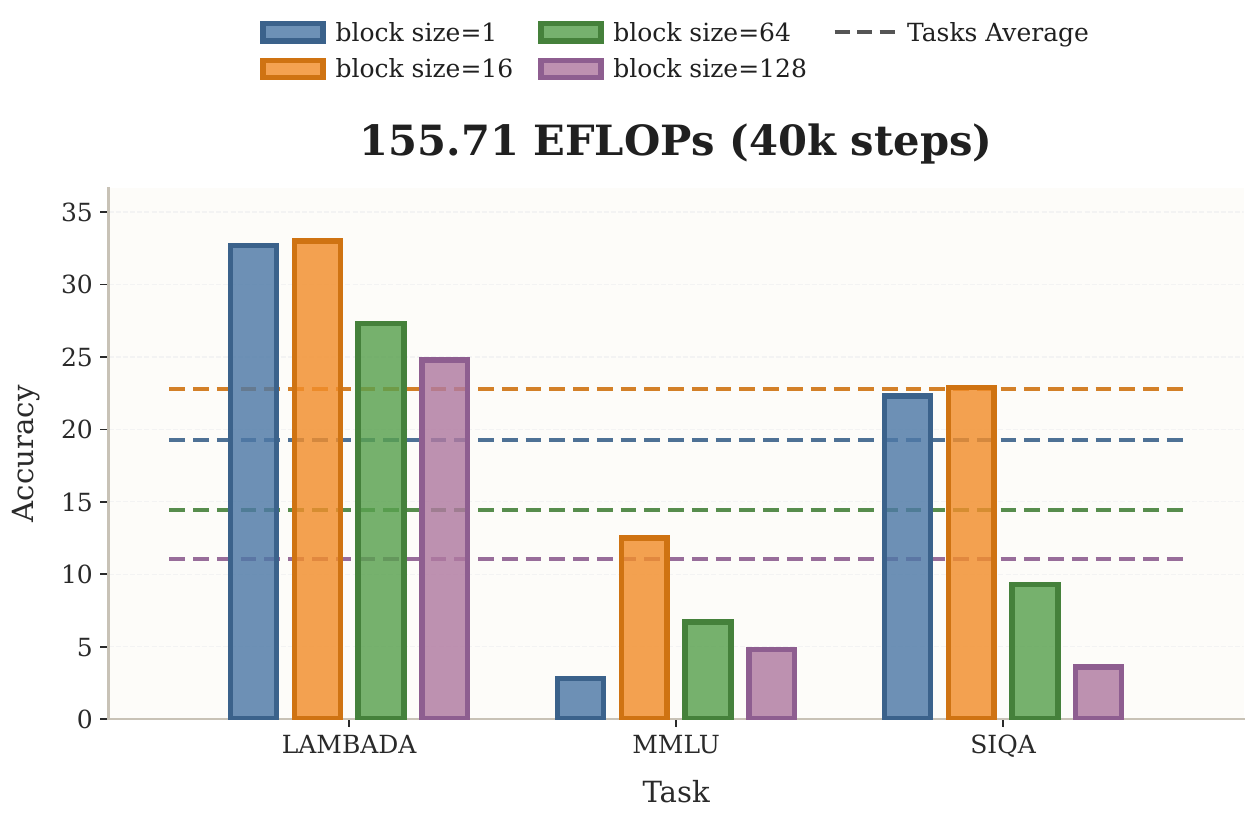}
        \label{fig:rq3_block_size_40k}
    \end{subfigure}
    \caption{\textbf{Impact of DiT block size.} A moderate block size (especially 16) achieves the best overall performance. Overly large blocks degrade results, while size 1 remains competitive but weaker than 16.}
    \label{fig:rq3_block_size}
\end{figure}

\subsubsection{Training Stage}
\paragraph{DiT Block Size.}As shown in Figure~\ref{fig:rq3_block_size}, all results in this subsection are obtained under the Joint DiT setting: the VAE is initialized from pretrained weights, the VAE and DiT are jointly optimized with the same learning rate, and the training noise schedule uses loc$=1$. We compare four DiT block sizes at two training checkpoints to study how the local processing granularity affects downstream performance. The results show that block size has a clear effect under the current setting, and that a moderate block size works best.

\textbf{Obs. \ding{182} Block size $16$ gives the best overall performance at both checkpoints.}
At both 30K and 40K checkpoints, block size $16$ achieves the highest Task Average in Figure~\ref{fig:rq3_block_size}. It also delivers the strongest or near-strongest results on all three benchmarks, especially on LAMBADA and MMLU. This suggests that, under the current setup, a moderate block size provides a favorable trade-off between local modeling capacity and semantic aggregation.

\textbf{Obs. \ding{183} Larger block sizes are generally less effective under the current setting.}
When the block size is increased from $16$ to $64$ and $128$, performance drops clearly on all three tasks at both checkpoints, with especially visible degradation on SIQA and MMLU. This suggests that overly coarse block partitioning may weaken useful semantic interactions inside the latent sequence. At the same time, since the training noise schedule is fixed to loc$=1$ here, we do not exclude the possibility that different block sizes may favor different noise calibrations.

\textbf{Obs. \ding{184} Block size $1$ is competitive but still weaker than block size $16$.}
Block size $1$ remains a relatively strong baseline and generally outperforms block sizes $64$ and $128$. However, it is still below block size $16$ in Task Average at both checkpoints, and is notably weaker on MMLU. This suggests that fully fine-grained, completely causal processing is not necessarily the optimal way to model text in this setting, and that some degree of local grouping can be beneficial.

Overall, Figure~\ref{fig:rq3_block_size} shows a clear pattern: under the current setting with loc$=1$, DiT block size should be neither too small nor too large. A moderate block size, especially $16$, provides the most effective balance and leads to the best overall performance in the current experiments.

\begin{figure}[t]
    \centering
    \begin{subfigure}[t]{0.49\linewidth}
        \centering
        \includegraphics[width=\linewidth]{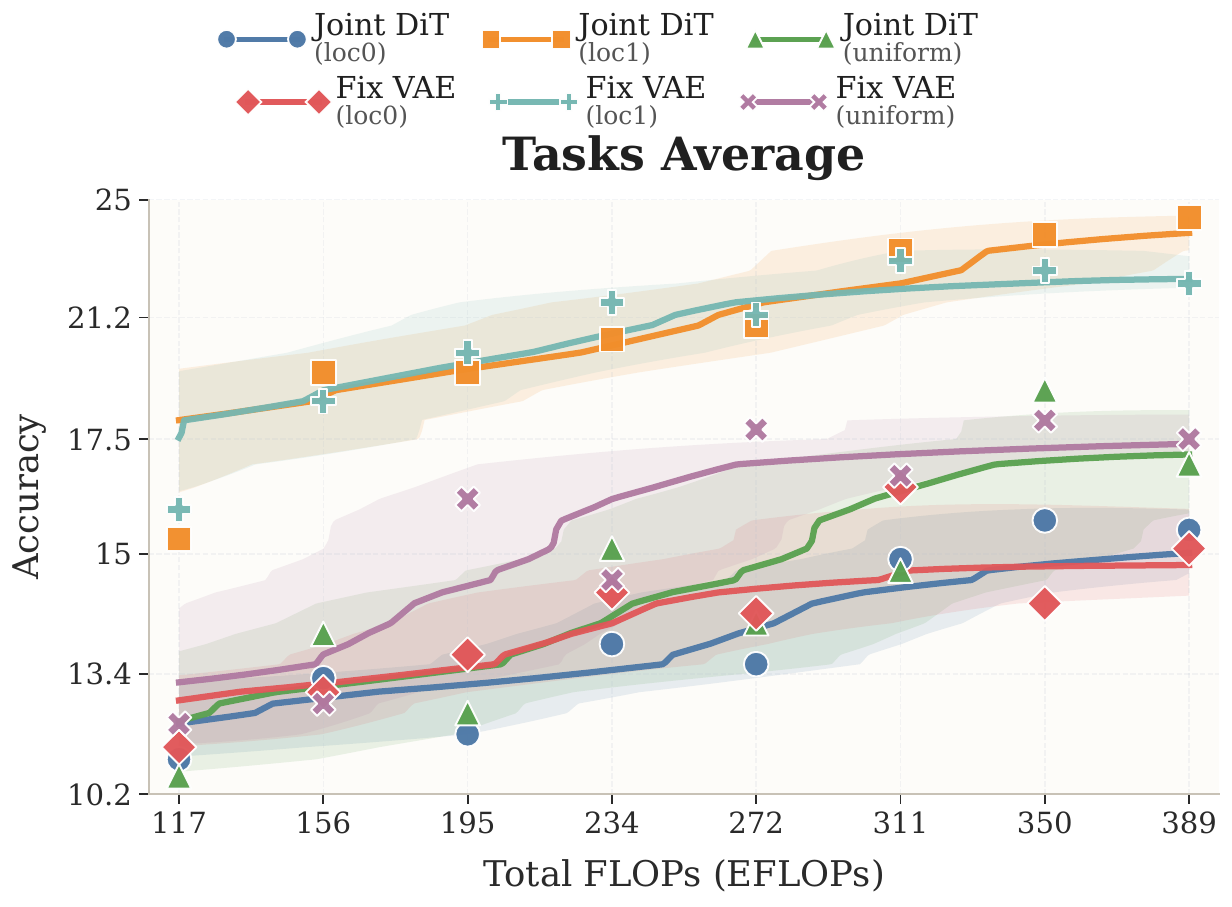}
        \label{fig:rq3_noise_schedule_joint_tasks_avg}
    \end{subfigure}
    \hfill
    \begin{subfigure}[t]{0.49\linewidth}
        \centering
        \includegraphics[width=\linewidth]{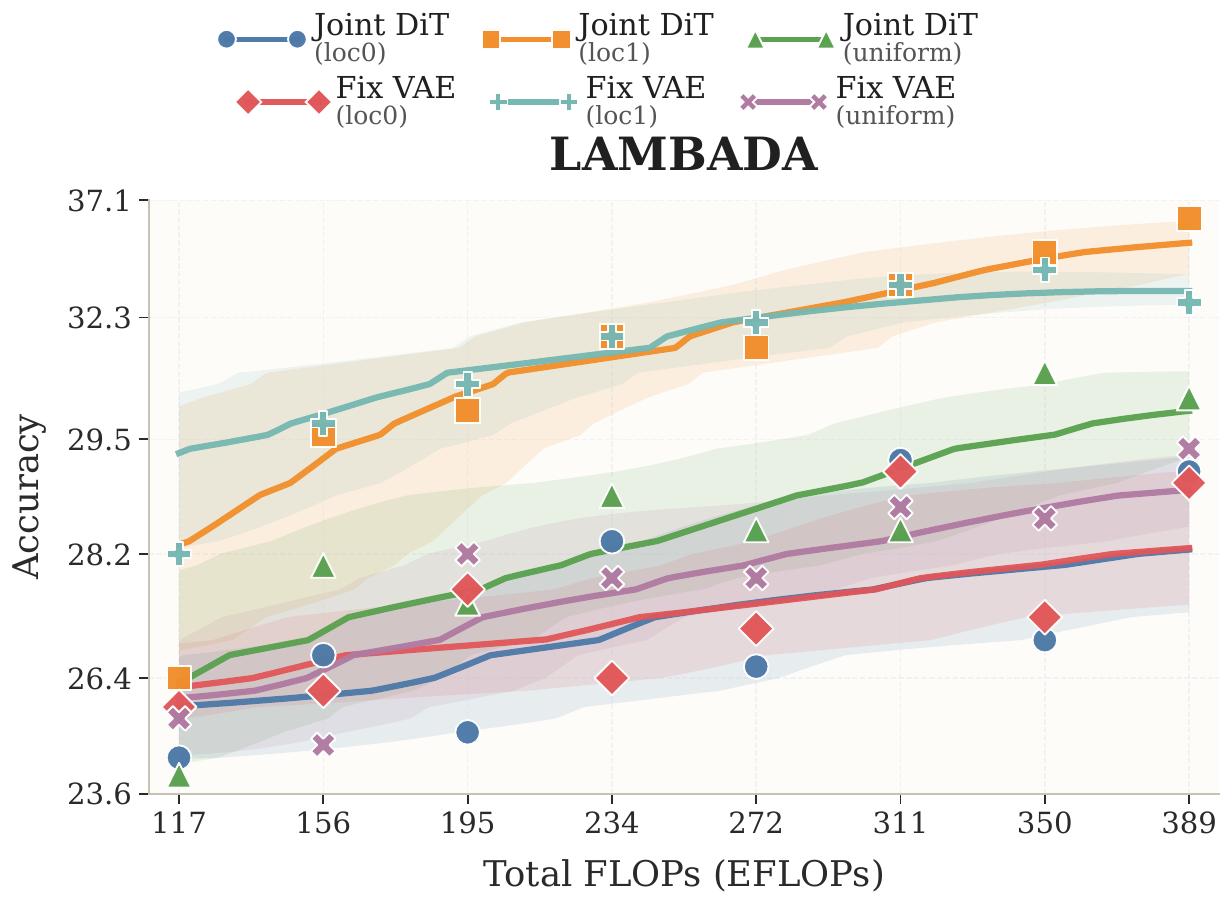}
        \label{fig:rq3_noise_schedule_joint_lambada}
    \end{subfigure}

    \vspace{0.5em}

    \begin{subfigure}[t]{0.49\linewidth}
        \centering
        \includegraphics[width=\linewidth]{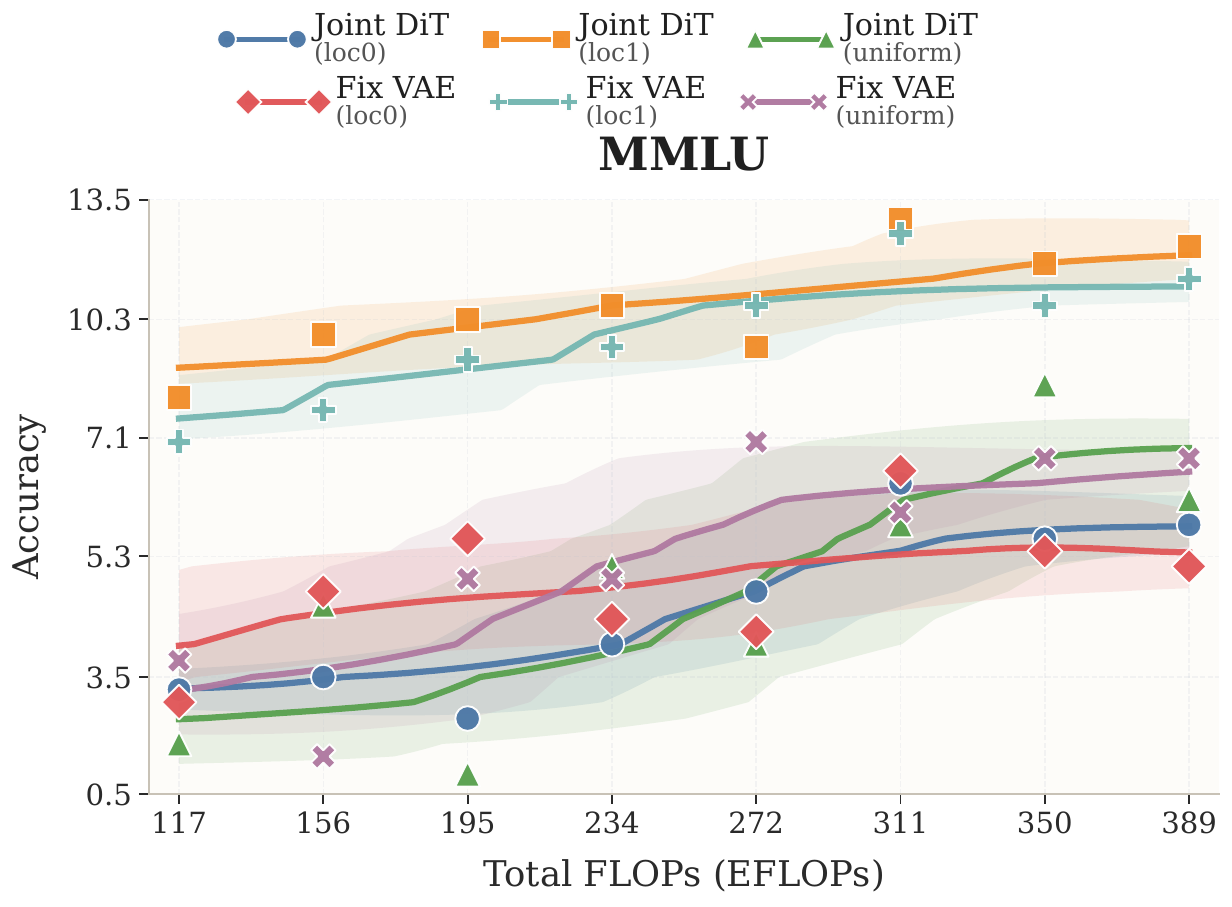}
        \label{fig:rq3_noise_schedule_joint_mmlu}
    \end{subfigure}
    \hfill
    \begin{subfigure}[t]{0.49\linewidth}
        \centering
        \includegraphics[width=\linewidth]{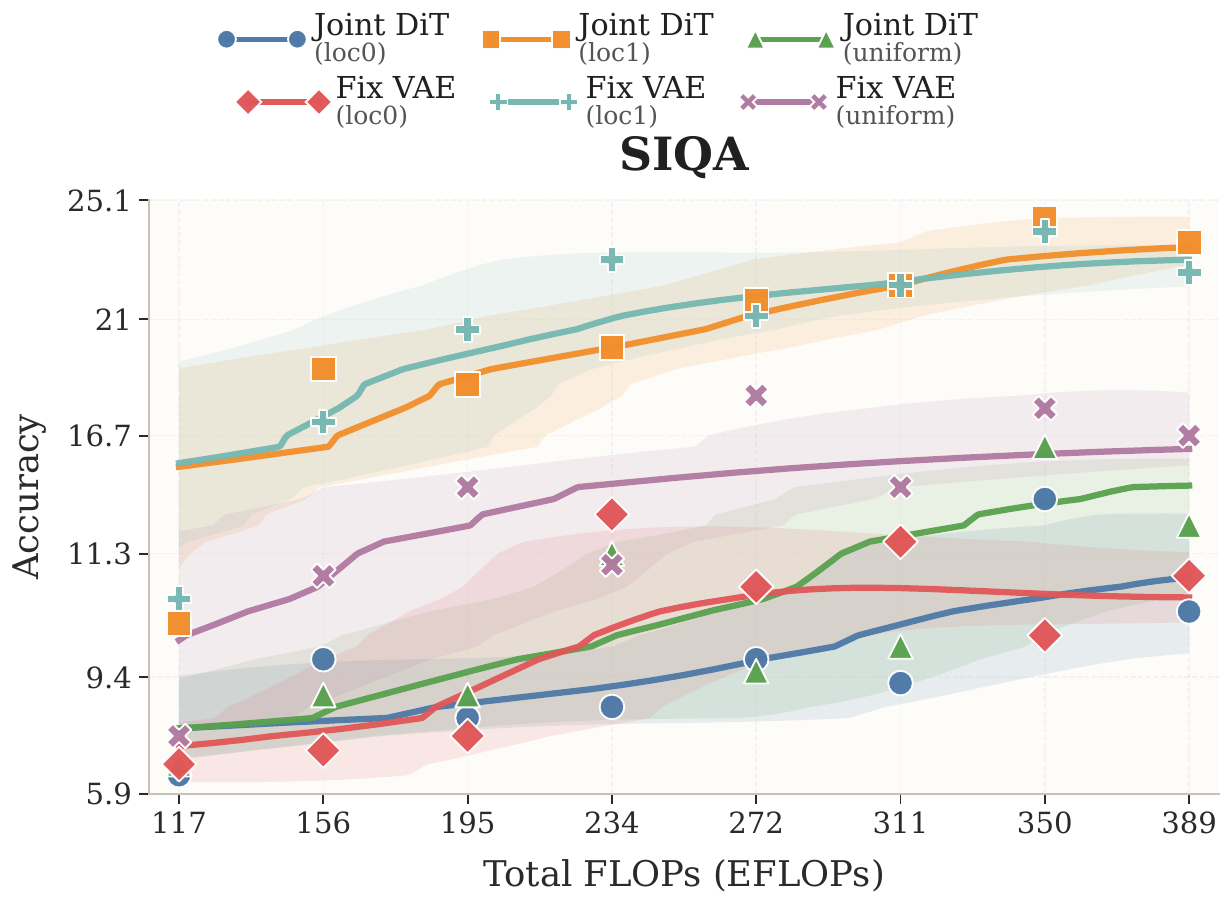}
        \label{fig:rq3_noise_schedule_joint_siqa}
    \end{subfigure}
    \caption{\textbf{Noise-schedule ablation.} Across all the tasks, $\mathrm{loc}=1$ gives the strongest overall performance, especially under Joint DiT, while uniform schedules are generally weaker. This suggests that noise-schedule calibration is important and becomes more beneficial when the latent space evolves jointly with DiT.}
    \label{fig:rq3_noise_schedule_joint}
\end{figure}

\paragraph{Noise Schedule. }
As shown in Figures~\ref{fig:rq3_noise_schedule_bar} and~\ref{fig:rq3_noise_schedule_joint}, all results in this subsection are obtained with latent dimension $d=16$ under the Joint DiT setting: the VAE is initialized from pretrained weights, and the VAE and DiT are jointly optimized with the same learning rate. We vary the schedule location parameter to study how noise calibration affects downstream performance, and include Fix VAE curves in Figure~\ref{fig:rq3_noise_schedule_joint} as references. From the information-theoretic analysis in Appendix~\ref{app:noise_schedule_logsnr_fm}, changing the schedule location is not merely changing a training-time heuristic: it effectively shifts the logSNR trajectory of the denoising process, and therefore changes how much semantic information remains available in the latent at different timesteps. The timestep shift formula and visualizations are provided in Appendix~\ref{app:timestep_shift_calc} and~\ref{app:timestep_shift_vis}.

\implicationbox{imp:rq3_noise_schedule}{
If the schedule location shifts the logSNR curve, then it also shifts the effective semantic-information regime seen by the DiT during denoising. Therefore, the best noise schedule is the one whose logSNR trajectory is best aligned with the latent space and the semantic scale to be recovered, rather than a universally fixed timestep parameterization.
}

\textbf{Obs. \ding{182} A moderate schedule location around loc$=1.0$ gives the best overall performance under the current setting.}
Figure~\ref{fig:rq3_noise_schedule_bar} shows that loc$=1.0$ achieves the highest Task Average at both the 30K and 40K checkpoints. It also gives the best or near-best results on the three tasks, with especially clear gains on MMLU and SIQA. From the information-theoretic view developed in Appendix~\ref{app:noise_schedule_logsnr_fm}, this suggests that loc$=1.0$ places the denoising trajectory in a more suitable effective logSNR range for semantic recovery, whereas both smaller and larger shifts move the model away from that regime.

\textbf{Obs. \ding{183} Proper noise calibration is especially important for Joint DiT.}
Figure~\ref{fig:rq3_noise_schedule_joint} further shows that Joint DiT with loc$=1$ is the strongest trainable setting across Task Average, LAMBADA, MMLU, and SIQA, whereas Joint DiT with loc$=0$ or a uniform schedule remains clearly weaker throughout training. Moreover, Joint DiT with loc$=1$ eventually matches or surpasses the corresponding Fix VAE baselines, while the mismatched schedules do not. This indicates that joint latent evolution becomes effective only when the denoising logSNR trajectory is aligned with the semantic structure of the evolving latent space.

\textbf{Obs. \ding{184} The effect of noise schedule should be understood through semantic-information calibration rather than as an isolated hyperparameter effect.}
Appendix~\ref{sec:global_semantic_structure} further implies that schedule location, latent dimension, and VAE logSNR all act on the same core object, namely the effective mutual-information curve of the semantic variable along diffusion time. From this perspective, the sensitivity observed here is not accidental: changing the noise schedule changes where the model spends its denoising capacity on the semantic-information axis. This also helps explain why different latent dimensions, different VAE smoothness settings, and potentially different DiT block sizes need not share the same optimal schedule.

Overall, Figures~\ref{fig:rq3_noise_schedule_bar} and~\ref{fig:rq3_noise_schedule_joint} show that the noise schedule is a key component of the training setup. Under the current Joint DiT setting with $d=16$, a properly calibrated schedule, especially loc$=1.0$, is important not only for stable optimization, but more fundamentally for aligning denoising with the effective semantic-information regime of the latent space. As implied by Eq.~\eqref{eq:main_avg_elbo}, this will further improve the average ELBO and is therefore theoretically well founded.

\begin{figure}[t]
    \centering
    \begin{subfigure}[t]{0.49\linewidth}
        \centering
        \includegraphics[width=\linewidth]{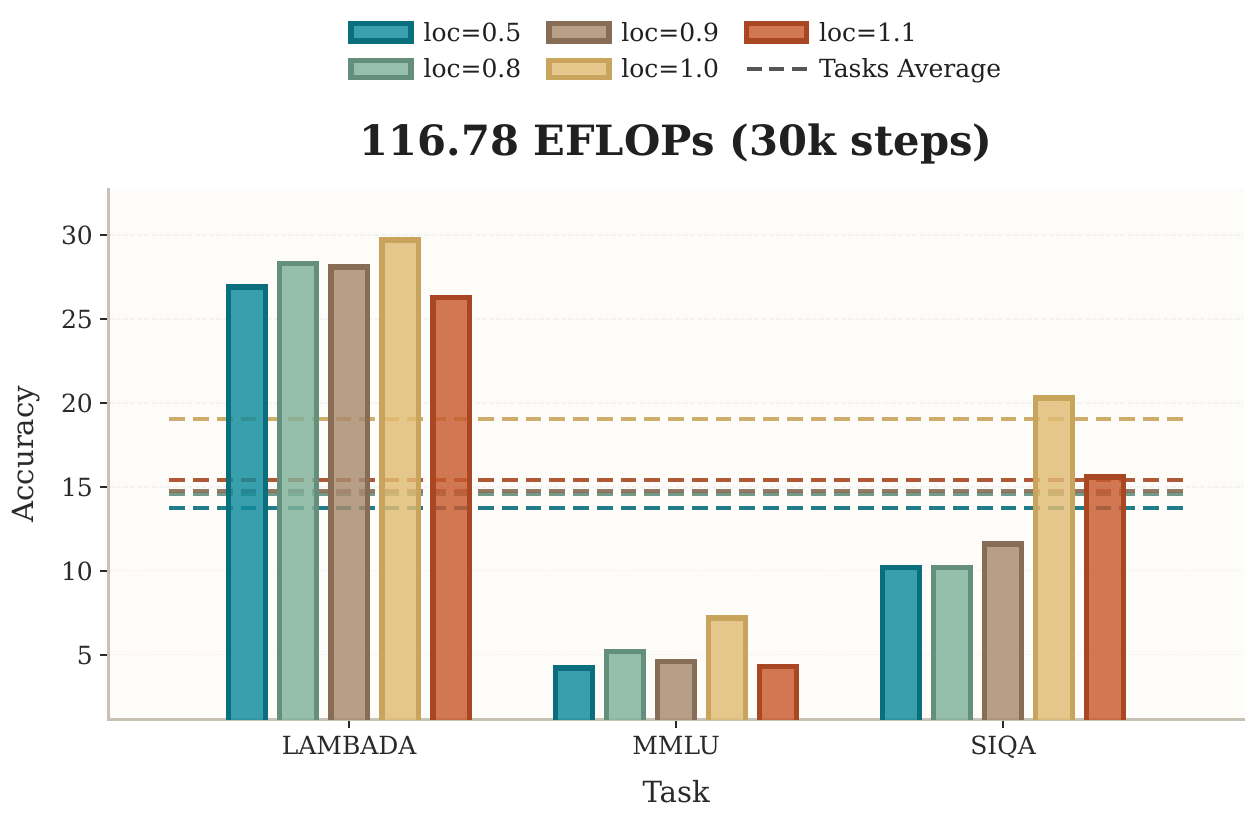}
        \label{fig:rq3_noise_schedule_bar_30k}
    \end{subfigure}
    \hfill
    \begin{subfigure}[t]{0.49\linewidth}
        \centering
        \includegraphics[width=\linewidth]{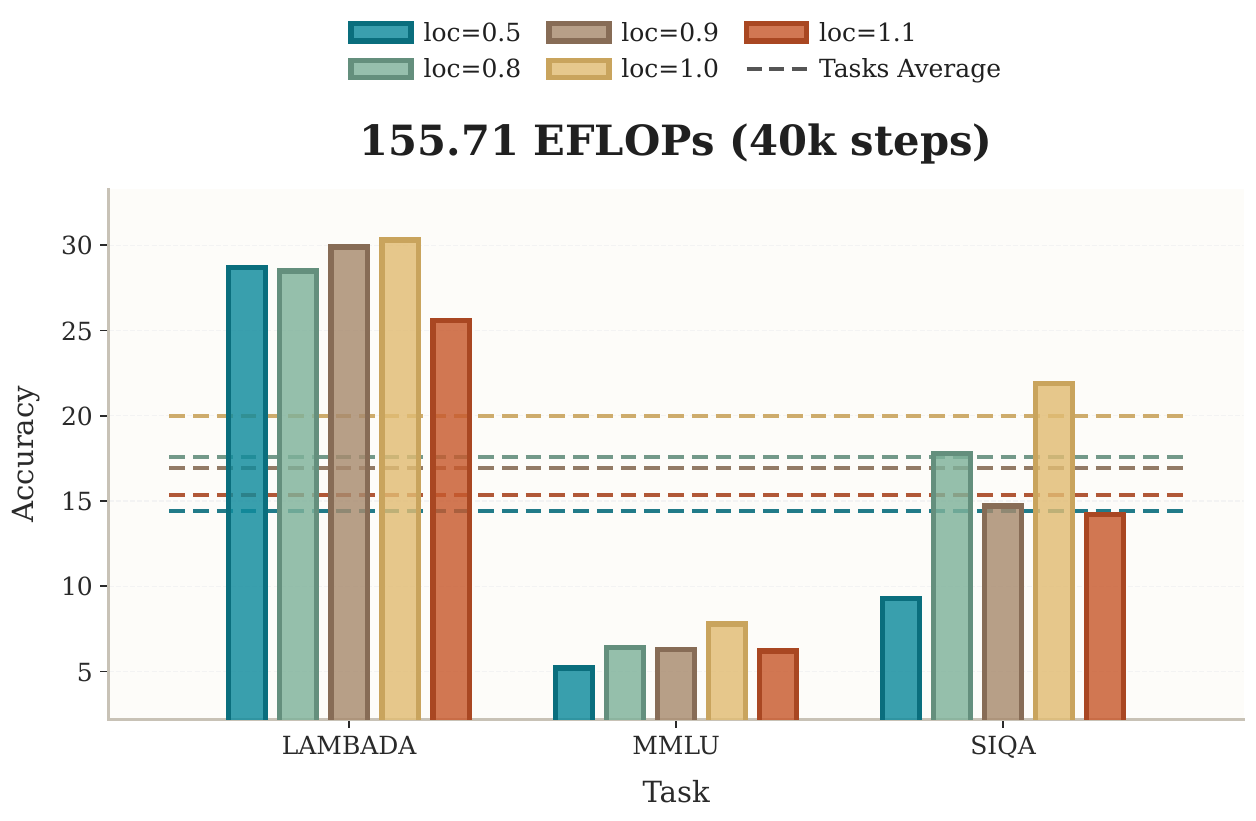}
        \label{fig:rq3_noise_schedule_bar_40k}
    \end{subfigure}
    \caption{\textbf{Noise-schedule comparison at different training checkpoints.} At both checkpoints, $\mathrm{loc}=1.0$ achieves the best Task Average and the most balanced overall performance across tasks. This indicates that the preferred schedule location is stable across training.}
    \label{fig:rq3_noise_schedule_bar}
\end{figure}

\subsubsection{Inference Stage}

\paragraph{Denoising Steps. }
As shown in Figure~\ref{fig:rq3_inferstep}, all results in this subsection are obtained under the Joint DiT setting: the VAE is initialized from pretrained weights, and the VAE and DiT are jointly optimized with the same learning rate. We vary the number of denoising steps at inference time to study the efficiency--performance trade-off. The results show that increasing the number of steps is highly beneficial in the low-step regime, while the gain quickly saturates as the inference budget becomes larger.

\textbf{Obs. \ding{182} Increasing denoising steps yields a clear early improvement.}
From 1--2 steps to 4--8 steps, all tasks improve substantially. The gain is especially large on LAMBADA, while SQuAD, SIQA, and Task Average also increase sharply. This indicates that very few denoising steps are insufficient for stable semantic recovery.

\textbf{Obs. \ding{183} Performance saturates after a moderate number of steps.}
After roughly 16--32 steps, the Task Average becomes nearly flat, and the marginal gain from additional steps is very limited. A similar saturation pattern is also visible on SIQA and SQuAD. This suggests that most useful denoising progress is already completed within a moderate inference budget.

\textbf{Obs. \ding{184} Most of the practical gain is achieved with only 8--10 denoising steps.}
From an efficiency perspective, 8--10 steps already recover most of the final performance. Since our DiT uses a block size of 16, this means 16 tokens can be generated with only 8--10 sequential denoising iterations, corresponding to an idealized 1.6--2.0$\times$ reduction in sequential generation depth compared with AR decoding.

\begin{figure}[t]
    \centering
    \begin{subfigure}[t]{0.47\columnwidth}
        \centering
        \includegraphics[width=\linewidth]{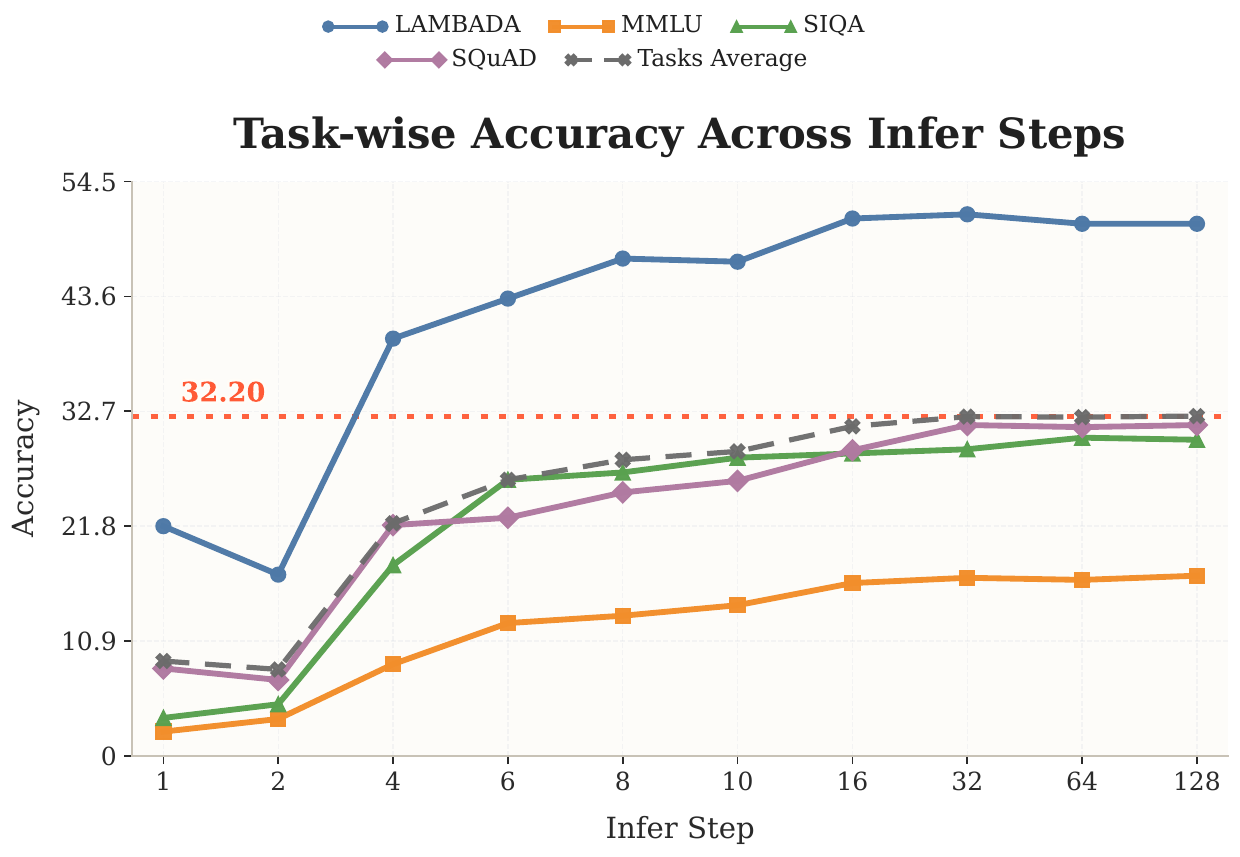}
        \caption{\textbf{Impact of denoising steps}}
        \label{fig:rq3_inferstep}
    \end{subfigure}
    \hfill
    \begin{subfigure}[t]{0.52\columnwidth}
        \centering
        \includegraphics[width=\linewidth]{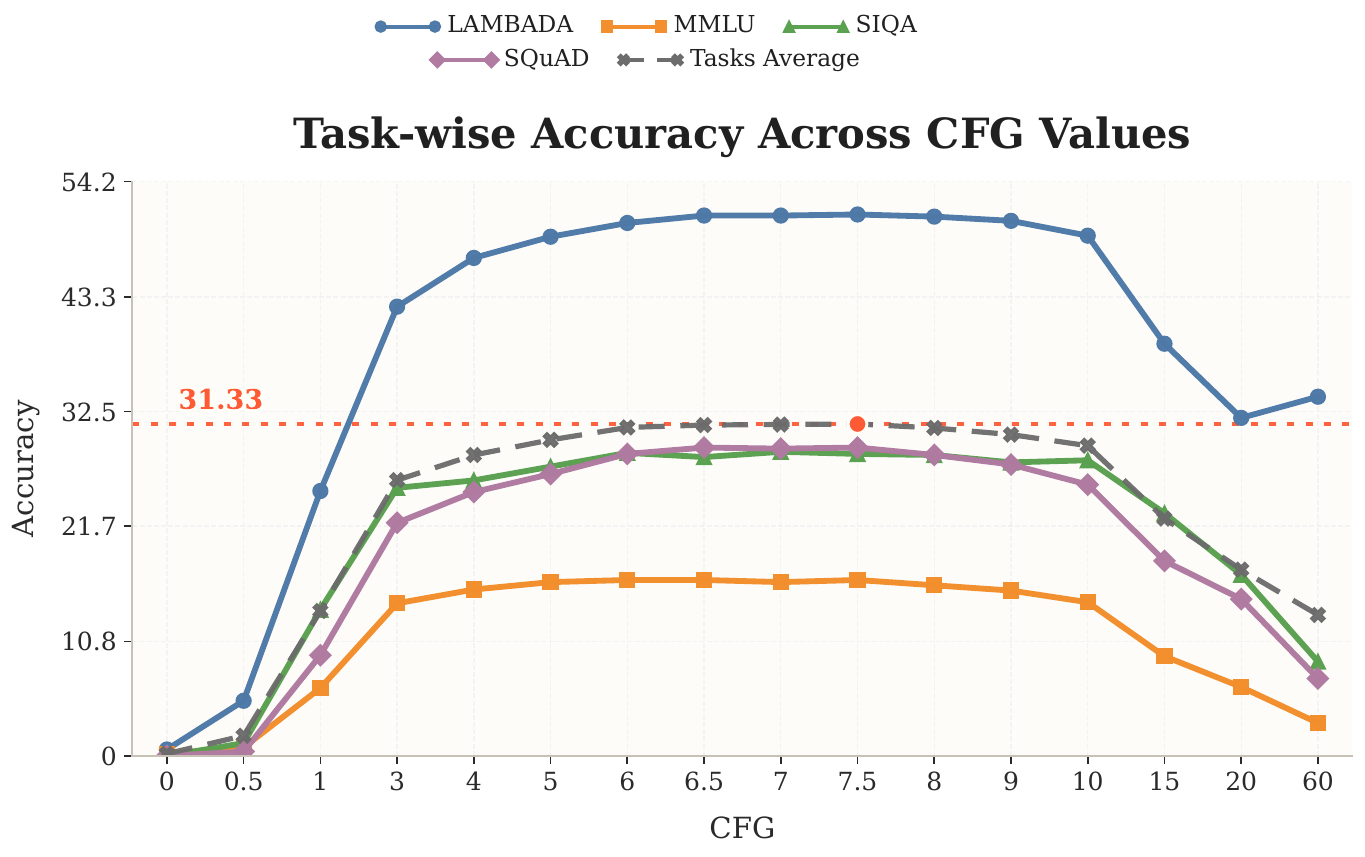}
        \caption{\textbf{Impact of CFG scales}}
        \label{fig:rq3_cfg}
    \end{subfigure}
    \caption{\textbf{Impact of inference-time hyperparameters.} Increasing denoising steps brings clear early gains but quickly saturates, while a moderate CFG value achieves the best overall performance.}
    \label{fig:rq3_infer_cfg}
\end{figure}

Overall, Figure~\ref{fig:rq3_inferstep} shows that denoising steps are important, but more is not always better. Under the Joint DiT setting, a moderate number of inference steps, around 10--32, already provides a strong trade-off between accuracy and efficiency.

\paragraph{Classifier-Free Guidance (CFG) Scales. }
As shown in Figure~\ref{fig:rq3_cfg}, all results in this subsection are obtained under the Joint DiT setting: the VAE is initialized from pretrained weights, and the VAE and DiT are jointly optimized with the same learning rate. We vary the Classifier-Free Guidance (CFG) scale at inference time to study how guidance strength affects downstream performance. The results show a clear non-monotonic pattern: increasing CFG is helpful at first, but overly large values significantly hurt performance.

\textbf{Obs. \ding{182} A moderate CFG scale gives the best overall performance.}
The Task Average rises rapidly as CFG increases from $0$ to around $3$--$6$, and then stays near its best region for a moderate range of values. This indicates that an appropriate amount of guidance substantially improves conditional denoising and semantic recovery.

\textbf{Obs. \ding{183} Excessive guidance leads to clear degradation.}
After the moderate optimum region, all task curves begin to decline as CFG becomes larger. The drop is especially pronounced beyond CFG $\approx 10$, and becomes severe at very large values such as $20$ and $60$. This shows that overly strong guidance distorts the denoising trajectory rather than improving it.

Overall, Figure~\ref{fig:rq3_cfg} shows that CFG is an important inference-time hyperparameter. Under the Joint DiT setting, a moderate CFG scale provides the best trade-off, while both weak guidance and excessive guidance lead to inferior results.

\subsection{Comparison of Scaling Performance (RQ4)}
\label{exp:rq4}

\begin{figure}[t]
    \centering
    \includegraphics[width=1\textwidth]{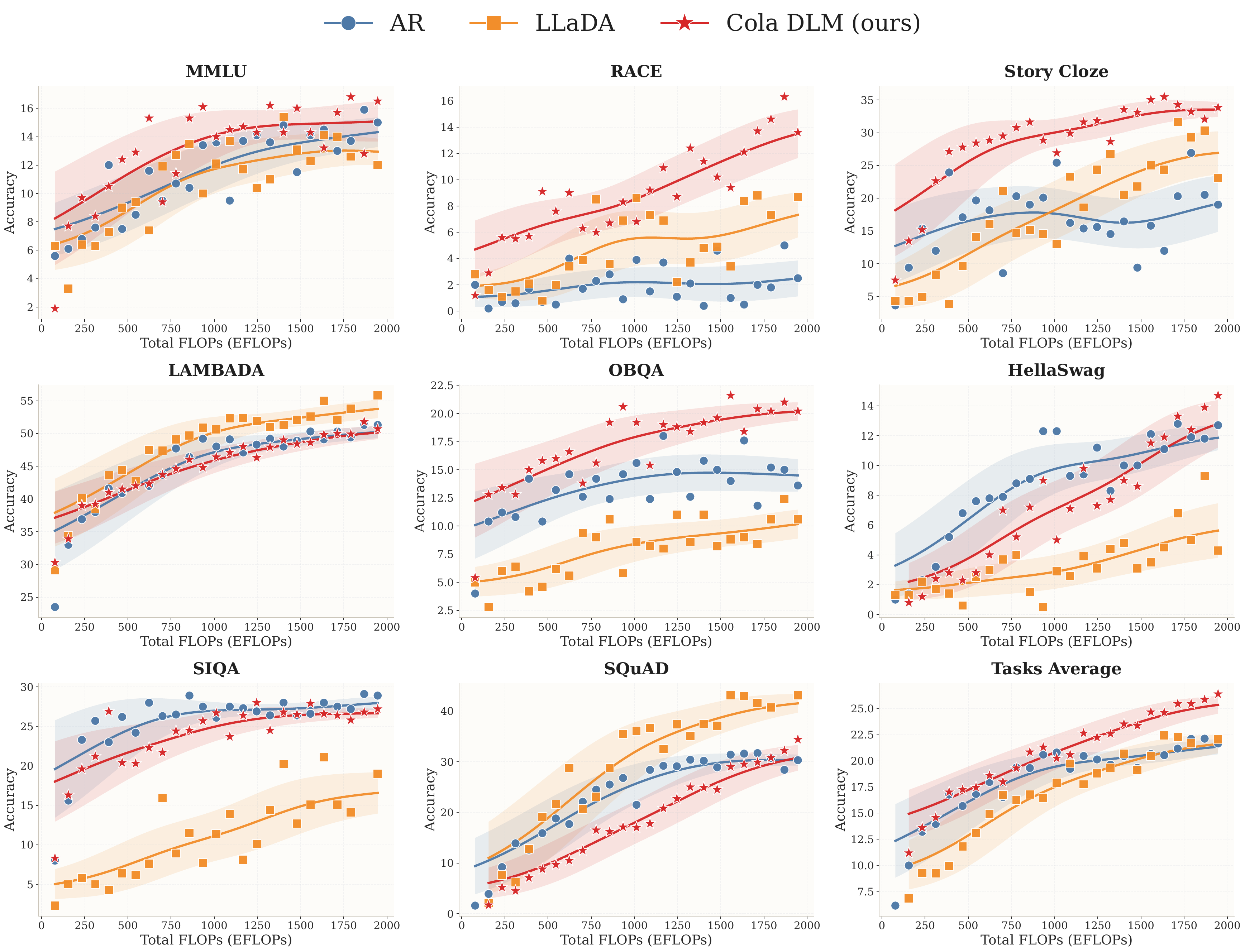}
    \caption{\textbf{Overall scaling performance under a unified few-shot generative evaluation protocol.} Across eight benchmarks and Task Average, \method exhibits strong scaling dynamics, ultimately reaching the best average performance. It should be noted that the lower absolute accuracy observed on specific multiple-choice tasks is an anticipated consequence of the rigorous generative evaluation paradigm; nevertheless, the underlying scaling trends are robustly preserved. These findings imply that continuous latent prior modeling possesses significant scaling potential, rendering the current performance a conservative measure of its true capacity.}
    \label{fig:scaling_exp_fig}
\end{figure}

In this section, we compare the scaling behavior of \method with strictly matched AR and LLaDA baselines under the best configuration identified by the previous tuning experiments. Specifically, \method uses latent dimension $d=16$, block size $16$, joint VAE--DiT training with a VAE/DiT learning-rate ratio of $1$, BERT loss, and a logit-normal training noise schedule with $\mathrm{loc}=1$; at inference time, we use $16$ denoising steps and CFG $=7$. The AR and LLaDA baselines are matched in scale, with the non-embedding backbone controlled at $1.8$B parameters, and LLaDA uses a denoising length equal to the generation length during inference.

It is also worth noting that the absolute scores in Figure~\ref{fig:scaling_exp_fig} are relatively low mainly on the multiple-choice benchmarks. This is because, for a fair comparison, all models are evaluated under a unified few-shot generative protocol rather than standard likelihood-based classification: LAMBADA and SQuAD are evaluated as generative tasks, while the remaining benchmarks are multiple-choice tasks but are also cast into few-shot generation. As discussed in Section~\ref{dis:ppl}, likelihood estimation can be substantially misaligned with the actual generation quality of \method. Therefore, although the absolute values on multiple-choice tasks are lower than those in conventional discriminative evaluation, the relative scaling trends remain informative and fair under this fully matched protocol.

As shown in Figure~\ref{fig:scaling_exp_fig}, \method exhibits strong overall scaling behavior, with increasingly encouraging gains as the compute budget grows.

\textbf{Obs. \ding{182} \method shows one of the strongest overall scaling trends.}
On Task Average, \method improves steadily across the full compute range and reaches the best final performance. AR remains competitive at smaller budgets, and LLaDA also shows clear early gains, but the curve of \method rises more persistently toward the high-compute regime. This suggests that \method already exhibits highly competitive, and at larger budgets stronger, scaling potential under the current matched setting.

\textbf{Obs. \ding{183} The scaling advantage of \method is especially clear on reasoning-intensive and global-semantic tasks.}
On MMLU, RACE, Story Cloze, and OBQA, \method maintains a strong upward trend and achieves the best or near-best performance across a wide compute range. The gains are particularly visible at medium-to-large budgets, indicating that continuous latent prior modeling is well suited to tasks that rely more on global semantic organization and holistic answer formation.

\textbf{Obs. \ding{184} On generative tasks, \method also shows encouraging scaling behavior.}
For LAMBADA and SQuAD, the scaling trends remain clear under the unified generative evaluation protocol. On LAMBADA, \method improves steadily with compute and remains close to AR at larger budgets, while SQuAD shows a particularly clear gain with scale, where \method eventually surpasses AR and continues to approach the strong performance region of LLaDA. These results suggest that, on generation-oriented evaluation, \method already demonstrates scaling behavior comparable to strong baselines, with encouraging headroom as compute increases.

\textbf{Obs. \ding{185} The current result is a conservative estimate of the scaling potential of \method.}
The present comparison is conducted under a relatively conservative configuration of \method. Earlier ablations already show that increasing the latent dimension from $16$ to $128$ can improve semantic capacity, and the analysis of logSNR also suggests that the current setting still leaves additional room for scaling. Therefore, Figure~\ref{fig:scaling_exp_fig} should be viewed as evidence that \method already scales well under a restrained setting, rather than as the upper bound of its capability.

Overall, Figure~\ref{fig:scaling_exp_fig} supports a consistent conclusion: under a strictly matched comparison and a unified generative evaluation protocol, \method exhibits scaling behavior that is fully competitive with strong AR and diffusion-based baselines, and on several tasks already shows particularly encouraging late-stage gains. Together with the remaining optimization headroom in latent-space design, these results provide supportive evidence that continuous latent prior modeling is a promising scaling direction for language modeling.

%% file: sections/discussion.tex
\section{Discussion}
In this section, we further examine several additional properties and extensions of \method. We focus on the structural gap between likelihood-oriented estimation and generation quality, analyze how different conditioning and padding strategies affect block-causal \method in the first generation block, and present a preliminary exploration of VAE-based text compression for faster generation. Finally, we highlight the broader potential of \method for combining with other continuous modalities.

\definecolor{gentokenmatch}{RGB}{0,204,0}
\definecolor{gentokenmismatch}{RGB}{204,0,0}
\newcommand{\gttok}[1]{{\color{red!80!black}\bfseries\fontsize{11.5}{13.5}\selectfont #1}}
\newcommand{\genmatch}[1]{{\color{gentokenmatch}\textbf{#1}}}
\newcommand{\genmismatch}[1]{{\color{gentokenmismatch}\textbf{#1}}}

\subsection{The Structural Gap Between Likelihood-Oriented Estimation and Generation Quality}
\label{dis:ppl}

\begin{figure}[t]
    \centering
    \includegraphics[width=\textwidth]{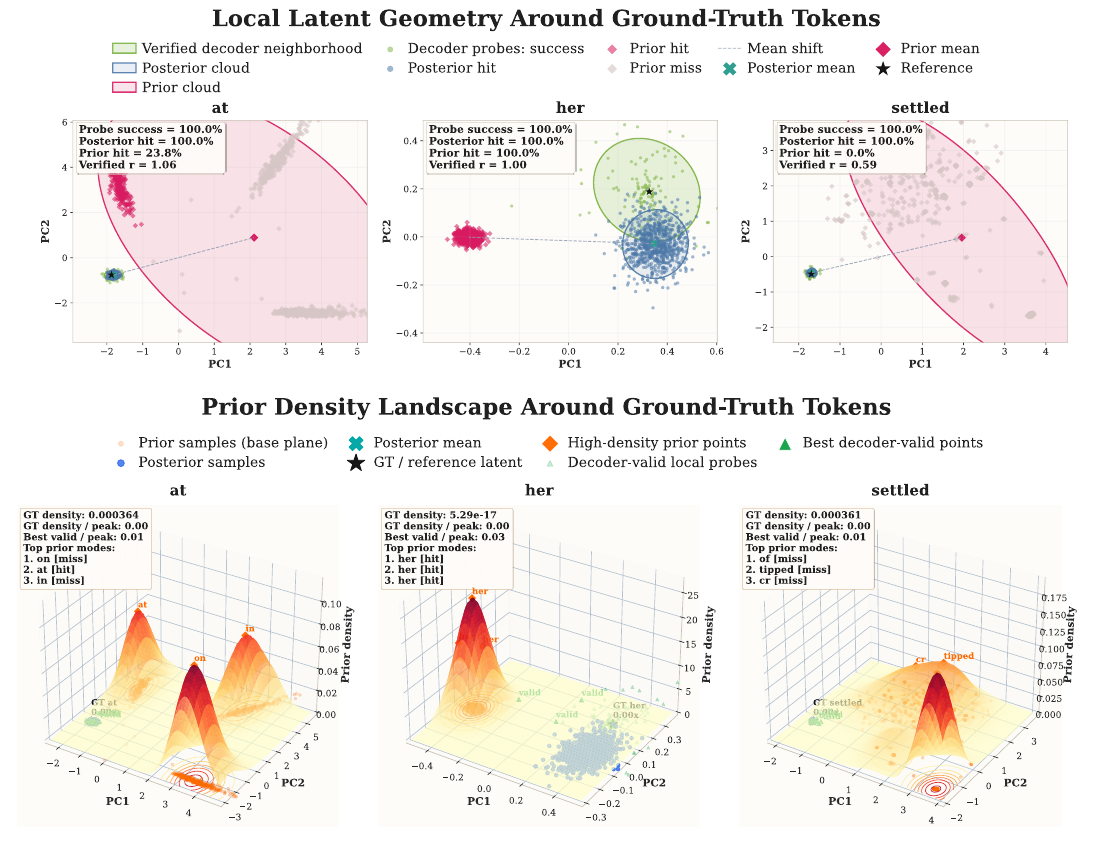}
    \caption{\textbf{A local view of the mismatch between likelihood-oriented estimation and generation quality.} Top: local latent geometry around representative ground-truth tokens. Bottom: corresponding prior-density landscapes. High decoder probe success and posterior hit contrast with sharply varying prior hit and density alignment. Thus, good generation relies on covering decoder-valid regions, while likelihood estimation also demands precise local calibration around the gold posterior.}
    \label{fig:discussion_likelihood_geometry}
\end{figure}

This section studies a central phenomenon in continuous latent language models: generation quality can already be reasonable while likelihood-oriented PPL remains poor. The key reason is that these two metrics target different properties. Generation only requires the prior mass to reach semantically decoder-valid regions, whereas likelihood-oriented estimation additionally requires accurate local probability calibration around the posterior neighborhood of the ground-truth target.

Let
\[
x=(x^{\mathrm{pre}},x^{\mathrm{res}}),
\]
where $x^{\mathrm{pre}}$ is the prefix, $x^{\mathrm{res}}$ is the response, and $c$ denotes the conditional information induced by the prefix. The exact conditional marginal is
\begin{equation}
p(x^{\mathrm{res}}\mid c)
=
\int p_{\theta}(x^{\mathrm{res}}\mid z,c)\,p_{\psi}(z\mid c)\,\dd z,
\label{eq:resp_conditional_marginal_main_final}
\end{equation}
while the practically accessible quantity is the local score
\begin{equation}
\mathcal S_{\mathrm{resp}}(x)
=
\E_{q_{\phi}(z\mid x,c)}
\Big[
\log p_{\theta}(x^{\mathrm{res}}\mid z,c)
+\log p_{\psi}(z\mid c)
-\log q_{\phi}(z\mid x,c)
\Big].
\label{eq:resp_conditional_score_main_final}
\end{equation}
The mismatch between these two quantities is the starting point of our analysis.

\implicationbox{imp:main_gap_final}{
In continuous latent language models, good generation and good likelihood-oriented estimation are not equivalent. Generation depends on whether the prior reaches semantically valid latent regions, whereas likelihood-oriented estimation additionally depends on local density calibration around the gold posterior neighborhood.
}

This distinction is directly supported by Figure~\ref{fig:discussion_likelihood_geometry} and Table~\ref{tab:discussion_likelihood_example}. In Figure~\ref{fig:discussion_likelihood_geometry}, decoder probe success and posterior hit are consistently high, showing that the decoder can reliably recover the ground-truth token inside the posterior neighborhood. However, the prior hit rates vary sharply, indicating that the main issue is not decoder failure but prior misalignment around the gold latent region. Table~\ref{tab:discussion_likelihood_example} shows the same pattern at the token level: for \textbf{at}, the likelihood-derived PPL improves dramatically from $1.15\times10^6$ to $641.57$ and then $245.36$, while the generated token deteriorates from \textit{on} to \textit{in} and then to a comma. Similarly, for \textbf{her}, smaller likelihood-derived PPL under fixed VAE logSNR does not recover the correct token. Thus, lower likelihood-derived PPL does not necessarily imply better generation.

\begin{table}[t]
    \centering
    \footnotesize
    \setlength{\tabcolsep}{3.5pt}
    \renewcommand{\arraystretch}{1.12}

    \begin{tabularx}{\textwidth}{@{}X@{}}
        \toprule
        \textbf{Sample text.} At dawn the research vessel Meridian slipped out of the harbor and followed a chain of islands that looked like dark brushstrokes on the horizon. Mira stood \gttok{at} the bow with a notebook pressed against \gttok{her} jacket, listening as the engine \gttok{settled} into a steady hum and the crew argued about \ldots \\
        \bottomrule
    \end{tabularx}

    \vspace{0.55em}

    \begin{tabularx}{\textwidth}{@{}c c c c c c c c@{}}
        \toprule
        \makecell[c]{Ground-truth\\token}
        & \makecell[c]{Posterior\\$\log p(z \mid x)$}
        & \makecell[c]{Prior\\$\log p(z)$}
        & \makecell[c]{Decoder\\$\log p(x \mid z)$}
        & \makecell[c]{Likelihood-derived\\PPL $\downarrow$}
        & \makecell[c]{Generated\\token}
        & \makecell[c]{Gen. PPL $\downarrow$\\of \textbf{generated}}
        & \makecell[c]{Gen. PPL $\downarrow$\\of \textbf{ground-truth}} \\
        \midrule

        \multicolumn{8}{@{}l@{}}{\textbf{\textit{Direct training (unfixed VAE logSNR; measured effective VAE logSNR $\approx 4.5$)}}} \\
        \midrule
        \textbf{at} & 18.70 & 4.74 & $-3.62 \times 10^{-4}$ & $1.15 \times 10^{6}$ & \genmismatch{on} & \textbf{3.83} & 6.90 \\
        \textbf{her} & 22.61 & 2.26 & $-8.58 \times 10^{-6}$ & $6.93 \times 10^{8}$ & \genmatch{her} & \textbf{1.12} & 1.12 \\
        \textbf{settled} & 32.02 & 18.45 & $-8.54 \times 10^{-3}$ & $7.88 \times 10^{5}$ & \genmismatch{of} & \textbf{5.05} & $2.58 \times 10^{3}$ \\
        \midrule

        \multicolumn{8}{@{}l@{}}{\textbf{\textit{Fixed VAE logSNR = 1.0}}} \\
        \midrule
        \textbf{at} & -15.81 & -22.28 & $-2.10 \times 10^{-5}$ & \underline{641.57} & \genmismatch{in} & \underline{6.69} & 6.90 \\
        \textbf{her} & -11.79 & -15.68 & $-1.91 \times 10^{-5}$ & \textbf{48.98} & \genmismatch{the} & \underline{14.17} & 1.12 \\
        \textbf{settled} & -12.56 & -21.51 & $-1.40 \times 10^{-3}$ & \textbf{7740} & \genmismatch{of} & \textbf{5.05} & $2.58 \times 10^{3}$ \\
        \midrule

        \multicolumn{8}{@{}l@{}}{\textbf{\textit{Fixed VAE logSNR = 1.5}}} \\
        \midrule
        \textbf{at} & -8.04 & -13.54 & $-7.06 \times 10^{-5}$ & \textbf{245.36} & \genmismatch{,} & 90.94 & 6.90 \\
        \textbf{her} & -7.48 & -15.91 & $-4.96 \times 10^{-5}$ & \underline{$4.57 \times 10^{3}$} & \genmismatch{the} & \underline{14.17} & 1.12 \\
        \textbf{settled} & -6.12 & -15.71 & $-8.77 \times 10^{-5}$ & \underline{$1.46 \times 10^{4}$} & \genmismatch{went} & 460.16 & $2.58 \times 10^{3}$ \\
        \midrule

        \multicolumn{8}{@{}l@{}}{\textbf{\textit{Fixed VAE logSNR = 2.0}}} \\
        \midrule
        \textbf{at} & 11.50 & -6.41 & $-5.91 \times 10^{-5}$ & $5.99 \times 10^{7}$ & \genmismatch{behind} & 114.91 & 6.90 \\
        \textbf{her} & 13.74 & 0.21 & $-8.58 \times 10^{-5}$ & $7.53 \times 10^{5}$ & \genmismatch{it} & $3.07 \times 10^{3}$ & 1.12 \\
        \textbf{settled} & 10.30 & -6.57 & $-4.03 \times 10^{-3}$ & $2.14 \times 10^{7}$ & \genmismatch{came} & \underline{442.42} & $2.58 \times 10^{3}$ \\
        \bottomrule
    \end{tabularx}

    \caption{\textbf{Token-level comparison across direct training and fixed VAE logSNR settings.} Across the same target tokens, lower likelihood-derived PPL does not necessarily lead to better generation. This token-level evidence further illustrates the structural gap between likelihood-oriented estimation and generation quality.}
    \label{tab:discussion_likelihood_example}
\end{table}

\implicationbox{imp:semantic_vs_prob_smoothing_final}{
In \method, generation quality is more related to semantic smoothness of the latent space, whereas likelihood-oriented PPL is more sensitive to probability-space smoothness shaped by the VAE logSNR. Since these two forms of smoothness are different, generation and PPL need not be aligned.
}

The fixed VAE logSNR settings in Table~\ref{tab:discussion_likelihood_example} should therefore be interpreted as changing the smoothness of the latent probability space rather than merely tuning a hyperparameter. Lower VAE logSNR corresponds to a flatter local density landscape, which tends to reduce pointwise density mismatch and improve likelihood-derived PPL. However, such smoothing can also blur local semantic structure and bias the model toward generic but semantically suboptimal continuations, such as \textit{in}, \textit{the}, or \textit{went}. By contrast, direct training yields much worse likelihood-derived PPL, but sometimes preserves more appropriate semantic behavior, such as correctly generating \textbf{her}.

\obsbox{
\textbf{Summary.}
In \method, generation quality primarily reflects the semantic smoothness of the latent space, whereas likelihood-oriented PPL is more sensitive to the smoothness of the latent probability space shaped by the VAE logSNR. As a result, better generation does not necessarily imply better PPL, and vice versa.
}

\begin{figure}[t]
    \centering
    \includegraphics[width=1\textwidth]{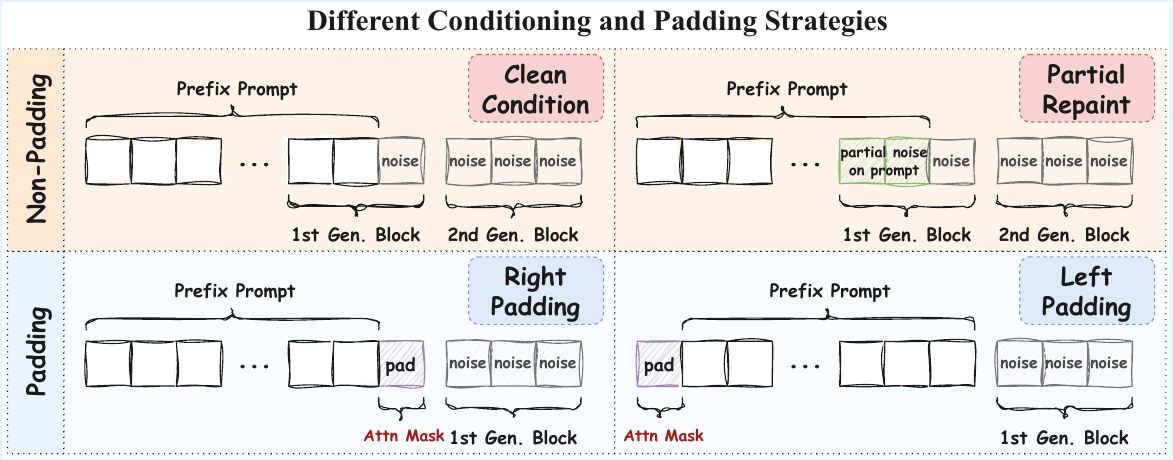}
    \caption{\textbf{Different conditioning and padding strategies in the first generation block.} The first generation block is a mixed region that contains both known prompt latents and unknown latents to be generated. Clean condition repaint keeps the known region fixed as a stable condition throughout denoising, partial repaint injects timestep-matched noisy guidance only during part of the trajectory, and left/right padding instead modify the layout of the known region without explicit repaint correction.}
    \label{fig:cond_strategies}
\end{figure}

\begin{table}[t]
    \centering
    \caption{\textbf{Impact of first-block conditioning strategies.} Clean condition repaint performs best, indicating strong, persistent conditioning is optimal for the first block's mixed denoising. Conversely, partial repaint is much weaker, reducing $m$ degrades performance, and increasing $t$ yields no stable gains. Left and right padding outperform partial repaint but remain inferior to clean conditioning.}
    \label{tab:first_block_conditioning}
    \small
    \setlength{\tabcolsep}{3.8pt}
    \renewcommand{\arraystretch}{1.10}
    \begin{tabular*}{\linewidth}{@{\extracolsep{\fill}}lccccccccc}
        \toprule
        \multirow{2}{*}{\textbf{Task}}
        & \multicolumn{3}{c}{\textbf{Partial repaint ($t=1$)}}
        & \multicolumn{3}{c}{\textbf{Partial repaint ($t=3$)}}
        & \multirow{2}{*}{\textbf{Clean cond.}}
        & \multirow{2}{*}{\textbf{Left pad.}}
        & \multirow{2}{*}{\textbf{Right pad.}} \\
        \cmidrule(lr){2-4} \cmidrule(lr){5-7}
        & \textbf{$m=1.0$} & \textbf{$m=0.7$} & \textbf{$m=0.3$}
        & \textbf{$m=1.0$} & \textbf{$m=0.7$} & \textbf{$m=0.3$}
        & \multicolumn{3}{c}{} \\
        \midrule
        \textbf{Lambada} & 8.5 & 8.5 & 6.6 & 7.0 & 7.3 & 5.6 & \textbf{37.1} & 24.6 & \underline{24.7} \\
        \textbf{MMLU}    & 7.9 & 7.9 & 7.8 & 7.6 & 6.7 & 7.0 & \textbf{11.9} & 8.4  & \underline{11.5} \\
        \textbf{SIQA}    & 8.8 & 8.7 & 8.2 & 13.3 & 13.0 & 12.0 & \textbf{24.8} & \underline{14.9} & 13.8 \\
        \midrule
        \textbf{Avg.}    & 8.4 & 8.4 & 7.5 & 9.3 & 9.0 & 8.2 & \textbf{24.6} & 16.0 & \underline{16.7} \\
        \bottomrule
    \end{tabular*}
\end{table}

\subsection{Impact of Conditioning and Padding Strategies in the First Generation Block}
\label{sec:first_block_conditioning}

In the first generation block, the input contains both known prompt latents and unknown latents to be generated. Figure~\ref{fig:cond_strategies} illustrates four representative strategies for handling this mixed region. Partial repaint injects timestep-matched noisy guidance on the known region, where $t$ controls the number of repaint repetitions and $m$ controls the fraction of the denoising trajectory that receives such guidance. Clean condition repaint instead keeps the known region fixed as clean guidance throughout denoising. By contrast, left and right padding do not explicitly repaint the known region, but only change its positional layout relative to the generated region. Notably, under the random-length setting, all aforementioned conditioning modes maintain strict consistency between training and inference.

As shown in Table~\ref{tab:first_block_conditioning}, clean condition repaint consistently achieves the best performance across all tasks. In contrast, partial repaint is substantially weaker, and reducing $m$ generally further degrades performance, indicating that shortening the guided portion makes the known region harder to preserve. Increasing the repaint repetitions from $t=1$ to $t=3$ also does not bring stable gains, suggesting that repeated early corrections cannot compensate for weak conditioning. Left and right padding are often stronger than most partial repaint settings because they avoid explicitly re-noising the known region, but still remain clearly below clean condition repaint. This suggests that positional layout alone is insufficient: padding does not provide a stable condition throughout denoising, and may further complicate the block-causal attention pattern.

Overall, these results show that the key challenge of the first generation block is to preserve the prompt-conditioned region while generating the remaining unknown part. For this mixed denoising problem, strong and persistent conditioning is more effective than partial noisy correction or positional layout alone. More details are provided in Appendix~\ref{app:first_block_conditioning_detail}.

\subsection{Compression of the Latent Space}
\label{dis:latent_compression}

\begin{table}[t]
    \centering
    \caption{\textbf{Performance under different sample labels and VAE patch sizes.} Patch size 2 is overall weaker, but this gap stems mainly from the Prompt Len Mod1 case (indivisible lengths). On Prompt Len Mod0, patch size 2 becomes competitive and even outperforms size 1. This suggests the weakness arises from boundary misalignment rather than latent compression itself.}
    \label{tab:sample_label_vae_patch}
    \small
    \setlength{\tabcolsep}{5pt}
    \renewcommand{\arraystretch}{1.10}
    \begin{tabular*}{\textwidth}{@{\extracolsep{\fill}}lcccccc}
        \toprule
        \textbf{Sample Label}
        & \multicolumn{2}{c}{\textbf{Overall}}
        & \multicolumn{2}{c}{\textbf{Prompt Len Mod0}}
        & \multicolumn{2}{c}{\textbf{Prompt Len Mod1}} \\
        \cmidrule(r){1-1} \cmidrule(lr){2-3} \cmidrule(lr){4-5} \cmidrule(l){6-7}
        \textbf{VAE Patch Size}
        & \textbf{p1} & \textbf{p2}
        & \textbf{p1} & \textbf{p2}
        & \textbf{p1} & \textbf{p2} \\
        \midrule
        \textbf{Lambada} & \textbf{31.10} & 17.40 & 32.11 & \textbf{34.55} & \textbf{30.12} & 0.79 \\
        \textbf{MMLU}    & \textbf{5.40}  & 3.90  & 6.89  & \textbf{7.68}  & \textbf{3.86}  & 0.00 \\
        \textbf{SIQA}    & \textbf{11.10} & 6.10  & \textbf{12.92} & 12.13 & \textbf{9.26} & 0.00 \\
        \midrule
        \textbf{Avg.}    & \textbf{15.87} & 9.13  & 17.31 & \textbf{18.12} & \textbf{14.41} & 0.26 \\
        \bottomrule
    \end{tabular*}
\end{table}

This section discusses whether compressing the text sequence in the VAE is beneficial for \method. We train two Text VAEs with the same latent dimensionality ($d=128$) but different patch sizes: $p1$ maps each token to one latent, while $p2$ compresses every two tokens into one latent. All other settings follow Section~\ref{exp:rq4}: the DiT block size is $16$, the training noise schedule uses logit-normal sampling with $\mathrm{loc}=1$ and $\mathrm{scale}=0$, and inference uses $16$ denoising steps with CFG $=7.0$. In Table~\ref{tab:sample_label_vae_patch}, \textit{Overall} reports the full evaluation result, while \textit{Prompt Len Mod0} and \textit{Prompt Len Mod1} group samples by whether the prompt length is divisible by $2$.

\implicationbox{imp:latent_compression_main}{
The weakness of patch size $2$ does not mainly come from compression itself, but from the boundary case where the prompt length is not divisible by the patch size. Once the latent grouping is well aligned with the text sequence, compression can instead become beneficial.
}

At the overall level, $p2$ is much worse than $p1$. However, the parity split shows that this gap is almost entirely caused by \textit{Prompt Len Mod1}. For odd-length prompts, $p2$ nearly collapses on all tasks, whereas on \textit{Prompt Len Mod0}, namely the even-length case seen by the patching rule, $p2$ becomes competitive and even slightly surpasses $p1$ on average. This suggests that the current failure is not evidence against latent compression itself, but against a compression scheme that does not robustly handle non-divisible sequence boundaries.

The reason is likely that, under patch size $2$, odd-length prompts necessarily involve padding or incomplete token groups during compression. If this boundary pattern is not properly learned, the compressed prompt latent becomes semantically shifted. In \method, this issue is particularly severe because the prompt latent is the clean condition for subsequent block-wise prior generation rather than a weak auxiliary representation. Once the prompt-side latent is biased, the error propagates through denoising and finally harms conditional decoding, which naturally explains the near-zero performance on \textit{Mod1}.

By contrast, the \textit{Mod0} result is encouraging. It shows that when the latent grouping is semantically valid, compressing two tokens into one latent does not necessarily hurt generation and may even help it. This is consistent with the core idea of \method: the latent space is not intended to preserve a token-aligned recovery path, but to provide a lower-rate representation for global semantic organization, while the decoder handles local realization. Under this view, moderate compression can be beneficial because each latent summarizes a larger textual span and thus better matches the role of the prior.

This also makes latent compression attractive from the efficiency perspective. Under the same DiT block size, one denoising block corresponds to $\text{patch size} \times \text{block size}$ text tokens after decoding. Therefore, with block size $16$, patch size $1$ covers $16$ text tokens per block, while patch size $2$ covers $32$. If the boundary issue can be resolved, larger patch sizes may improve both semantic abstraction and generation efficiency.

\obsbox{
\textbf{Summary.}
Table~\ref{tab:sample_label_vae_patch} suggests that latent compression is a promising direction for \method. Its current limitation mainly comes from unstable handling of non-divisible sequence boundaries, while the aligned even-length case already shows that compressed latents can support both stronger semantic abstraction and faster generation.
}

\subsection{Robustness of VAE Latent Reconstruction}
\label{dis:vae_recon_robustness}

\begin{wrapfigure}{r}{0.52\textwidth}
    \centering
    \vspace{-0.5em}
    \includegraphics[width=0.50\textwidth]{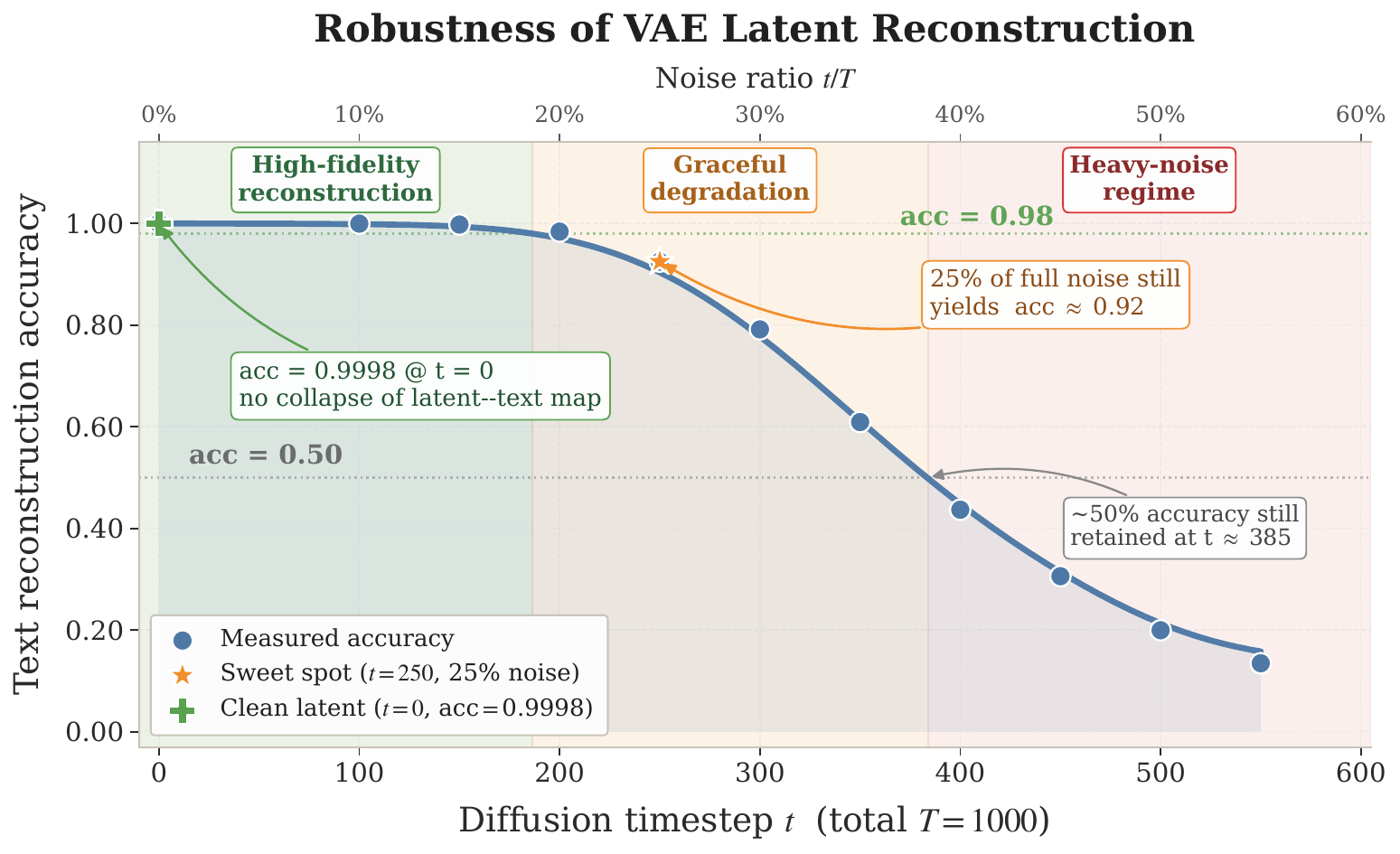}
    \vspace{-0.8em}
    \caption{\textbf{Robustness of VAE latent reconstruction.} The VAE preserves near-perfect reconstruction at low noise and degrades gracefully under stronger perturbations, indicating a stable latent--text mapping.}
    \label{fig:vae_recon_robustness}
    \vspace{-1.0em}
\end{wrapfigure}

We further analyze the robustness of the VAE latent space from the reconstruction perspective. As shown in Figure~\ref{fig:vae_recon_robustness}, the VAE achieves nearly perfect reconstruction at $t=0$, indicating that the learned latent--text mapping remains highly faithful and does not collapse. Moreover, the reconstruction accuracy stays very high throughout the low-noise regime, and still remains around $0.92$ at $t=250$, before degrading more noticeably under heavier noise.

These results suggest that the latent space learned by the VAE is not merely a fragile compressed code, but a stable and broadly usable intermediate representation for text. In particular, the graceful degradation pattern indicates that semantic information is not destroyed abruptly by small or moderate perturbations, which further supports the view that the VAE latent space in \method is sufficiently robust to serve as the semantic interface for subsequent prior modeling.

\subsection{Towards a Unified Approach with Image Modalities}

\begin{figure}[tp]
    \centering
    \includegraphics[width=\textwidth]{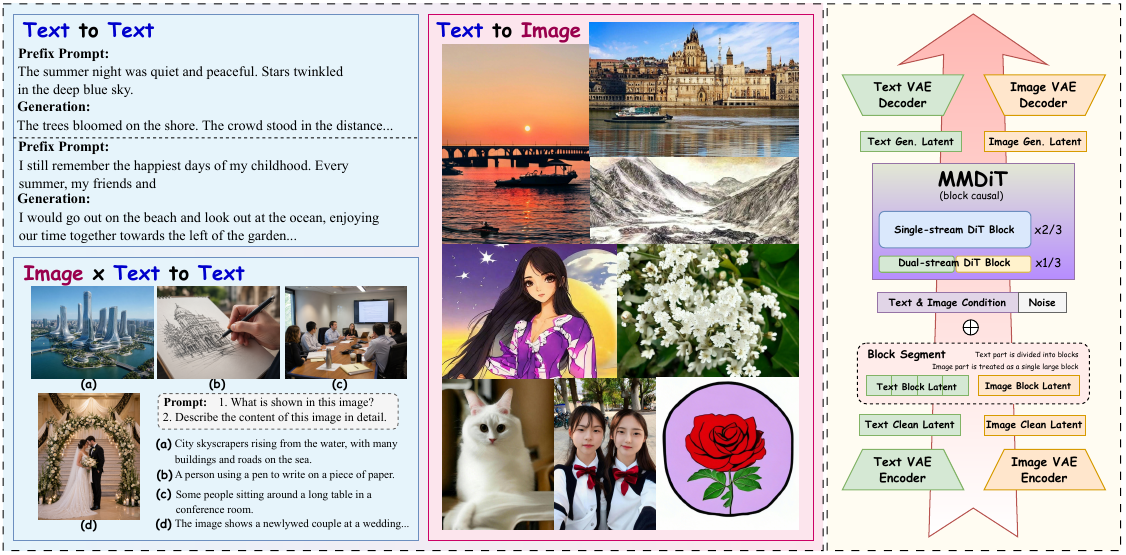}
    \caption{\textbf{Preliminary qualitative examples of unified text-image modeling.}
    Left: text-only continuation and image-conditioned text generation. Middle: text-to-image results with only pretraining. Right: a schematic extension of \method, where text and image are mapped into modality-specific continuous latents and modeled by a shared block-causal prior.}
    \label{fig:unified_samples}
\end{figure}

A broader implication of \method is that it provides a natural bridge from discrete text to continuous multimodal modeling. The key idea of unified modeling is not merely to place text and image into one backbone, but to map heterogeneous observations into a shared continuous latent interaction space, where higher-level semantics can be organized under common dynamics.

A natural extension of \method follows the same probabilistic decomposition as in the text-only setting. Let $x^{\text{text}}$ and $x^{\text{img}}$ denote the text and image observations, and let their modality-specific latent variables be
\[
z^{\text{text}}_0 \sim q_{\phi_{\text{text}}}(z \mid x^{\text{text}}), 
\qquad
z^{\text{img}}_0 \sim q_{\phi_{\text{img}}}(z \mid x^{\text{img}}).
\]
We then define a joint latent state
\[
\tilde z_0 = \bigl(z^{\text{text}}_0, z^{\text{img}}_0\bigr),
\]
and model the unified generative process as
\[
p(x^{\text{text}}, x^{\text{img}}, \tilde z_0)
=
p_{\theta}\!\left(x^{\text{text}}, x^{\text{img}} \mid \tilde z_0\right)
\, p_{\psi}(\tilde z_0).
\]
Under this view, modality-specific VAE encoders and decoders are responsible for surface-level representation and realization, while the shared prior models the higher-level semantic structure and cross-modal dependency in latent space.

This perspective is consistent with the central modeling principle of \method. In Cola DLM, diffusion is not used for token-level observation recovery, but for latent prior transport:
\[
z_1 \sim p_1, 
\qquad
z_0 = \Phi^{\psi}_{0\leftarrow1}(z_1),
\qquad
x \sim p_{\theta}(x \mid z_0).
\]
In the unified setting, the same idea extends to the multimodal latent state $\tilde z_0$: the shared block-causal MMDiT prior transports and organizes the joint latent semantics, while the modality-specific decoders handle the final text or image realization. Therefore, continuity is introduced at the level of prior modeling rather than direct token or pixel recovery.

From the ELBO viewpoint, the benefit of such a decomposition is also conceptually clear. A unified latent-variable objective takes the form
\[
\mathbb{E}[\mathcal{L}_{\mathrm{ELBO}}]
=
\mathbb{E}_{q}\!\left[\log p_{\theta}(x^{\text{text}}, x^{\text{img}} \mid \tilde z_0)\right]
- I\!\left((X^{\text{text}}, X^{\text{img}});\tilde Z_0\right)
- \mathrm{KL}\!\left(\bar q(\tilde z_0)\,\|\,p_{\psi}(\tilde z_0)\right),
\]
which shows the same division of labor as in the text-only case: the latent variable carries compressed global semantics, while the decoder is responsible for modality-specific realization. In this sense, unified modeling is not simply parameter sharing across modalities, but a shared semantic prior over heterogeneous observations.

Figure~\ref{fig:unified_samples} presents a preliminary prototype of this idea. In the current design, the text sequence is divided into blocks, while the image latent is treated as a single large block. Specifically, the image representation is obtained via an Image VAE trained on internal multi-resolution data (256 / 384 / 640 / 1024), with a spatial downsampling factor of 16 and 64 latent channels, providing a compact yet expressive latent space for visual content. The shared block-causal MMDiT prior operates over both text blocks and image latents, supporting intra-modal processing as well as cross-modal interaction. Within a unified framework, this enables text-to-text continuation, image-conditioned text generation, and text-to-image generation. We jointly optimize these tasks on internal image–text pairs during training. For the text-to-image task, we first train on 256-resolution data for 80k steps with a global batch size of approximately 3k, and then continue training on 640-resolution data for 10k steps with a global batch size of approximately 1k. For image-conditioned text generation, we adopt the same batch size configuration and train for approximately 50k steps. More result samples are provided at~\ref{app:unified_results}.

These results should be interpreted primarily as qualitative evidence of feasibility. As the current prototype remains at an early stage of training, and our experiments are limited to moderate pretraining on in-house 256 and 640 resolution data, without extensive high-quality data curation or supervised fine-tuning. The goal of this section is not to present a mature multimodal system. Rather, it is to demonstrate that the hierarchical latent-prior formulation of \method naturally extends beyond text-only generation. In future work, we plan to conduct more comprehensive unified multimodal training. More broadly, these findings suggest that decoupling global latent organization from modality-specific realization may offer a structurally clean and scalable path toward more native unified generative models.

\obsbox{
\textbf{Summary.}
These preliminary results suggest that \method naturally extends to unified text--image modeling. A shared block-causal prior organizes global and cross-modal semantics, while modality-specific decoders handle final realization. Although still early-stage and qualitative, this prototype already shows a promising bridge from language generation to native multimodal generative modeling.
}

%% file: sections/conclusion.tex
\section{Limitations \& Future Prospects}

Although this paper has provided initial evidence for the feasibility, competitiveness, and promising scaling potential of \method for text generation in continuous latent space, we view it as a starting point for further exploration rather than a finished endpoint. First, at the scale and evaluation level, the current results reveal encouraging trends, but the experiments are still conducted at a relatively controlled scale and mainly serve to clarify the key properties of the framework. It is therefore natural and important to further examine its upper bound under larger model sizes, longer training, and more substantial compute budgets. Second, at the model-design level, our analyses show that the training strategy of the Text VAE, the text compression scheme, the choice of latent dimensionality, the semantic smoothness of the latent space, and the joint calibration of VAE logSNR, DiT block size, and noise schedule all affect the semantic organization of the latent space and the final generation quality. In particular, the experiments suggest that stronger latent representations usually require better-aligned noise calibration, indicating substantial room for further optimization. Finally, at the framework level, the main value of \method lies not merely in the denoising process itself, but in its decomposition of text generation into global semantic prior modeling and local textual realization. This opens the door to exploring stronger latent modules, such as AE \citep{bank2023autoencoders} and RAE \citep{zheng2025diffusion}, as well as more flexible prior-learning approaches, such as drifting-model-based \citep{deng2026generative} distribution matching for continuous priors. More broadly, following the idea of unified continuous latent-space modeling, the framework may also be extended to continuous modalities such as images, further advancing unified generation.

\section{Conclusion}

In conclusion, this paper presents \method, a hierarchical continuous latent diffusion language model that decomposes text generation into global semantic prior modeling in latent space and local textual realization through conditional decoding, thereby providing a principled alternative to strictly token-level language modeling. Across the full study, both the theoretical analysis and the experiments consistently suggest that text generation can benefit from hierarchical information decomposition: we find evidence of shared global semantic structure in latent space, identify effective design choices for latent-space formation and diffusion modeling, and show that under strictly matched comparisons, \method exhibits strong generation quality and encouraging scaling behavior. More broadly, our results indicate that for this class of models, generation-oriented evaluation and scaling trends may be more informative than likelihood alone, while the continuous latent formulation also offers a concrete path toward more native unified modeling across discrete text and continuous modalities.

%% file: sections/afterword.tex
\section{Afterword: Research Objectives and Significance}

Viewed from a broader perspective, this study is not only concerned with proposing an alternative architecture for text generation, but also with clarifying a more general picture of learning in which \emph{representation}, \emph{objective}, and \emph{environment} must be understood jointly. From this perspective, the three themes of this work are closely connected rather than independent. The first concerns how text should be represented and generated. The second concerns what kinds of objectives and evaluation criteria are genuinely aligned with such representations. The third concerns the kind of environment in which a model should ultimately learn if the goal is more general multimodal intelligence.

A useful starting point is to formalize learning itself as a model--environment interaction system. Let the environment be
\begin{equation}
\mathcal E = (\Omega, \mathcal O, \mathcal A, \mathcal T, \mathcal F, \mathcal G),
\label{eq:after_env}
\end{equation}
where $\Omega$ is the environment state space, $\mathcal O$ is the observation space, $\mathcal A$ is the action or output space, $\mathcal T$ is the state transition mechanism, $\mathcal F$ is the feedback generation mechanism, and $\mathcal G$ is the rule that converts feedback into optimization signals. Importantly, the notion of environment is understood here in a broad sense: it includes not only the external world, but also the data distribution presented to the model, task formats, supervision protocols, and even the loss rules by which feedback is transformed into gradients.

Let the model be denoted by $M_\theta$, with internal state space $\mathcal H$, state update map $U_\theta$, and policy or generation map $\Pi_\theta$. At interaction step $t$, the closed-loop system can be written as
\begin{align}
o_t &\sim P_{\mathcal E}(\cdot \mid \omega_t), \\
h_t &= U_\theta(h_{t-1}, o_t), \\
a_t &\sim \Pi_\theta(\cdot \mid h_t), \\
\xi_t &\sim \mathcal F(\cdot \mid \omega_t, o_t, a_t), \\
\omega_{t+1} &\sim \mathcal T(\cdot \mid \omega_t, o_t, a_t, \xi_t), \\
\ell_t &= \mathcal G(\omega_t, o_t, a_t, \xi_t).
\end{align}
The overall learning objective is therefore
\begin{equation}
\mathcal J(\theta; \mathcal E)
=
\mathbb E_{\tau \sim P(\tau \mid \theta, \mathcal E)}
\left[
\sum_{t=1}^{T} \gamma^{t-1} \ell_t
\right],
\label{eq:after_general_objective}
\end{equation}
where $\tau$ denotes a complete interaction trajectory and $\gamma$ is the discount factor.

This formalization shows directly that learning is never an isolated question of model structure alone. Rather, it is jointly determined by three factors: first, the state space in which the model absorbs and organizes information; second, the kind of feedback through which the environment defines improvement; and third, the actual structure that generates observations, transitions, and feedback. In this work, these three aspects correspond precisely to the three recurring themes of the paper: how text should be represented, which metrics are aligned with the true learning objective, and what kind of environment unified models are ultimately meant to enter.

\subsection{Rethinking Text Modeling Paradigms: From State Space in the System to Hierarchical Text Generation}

From a system-level perspective, the central question of text modeling is not merely which generation order to adopt, but rather in what kind of \emph{state} text should be represented within the learning system. Mainstream autoregressive language models bind the state tightly to the surface token prefix, and generation is therefore written as
\begin{equation}
p_{\mathrm{AR}}(x)=\prod_{t=1}^{n}p_\theta(x_t\mid x_{<t}).
\label{eq:after_ar_factorization}
\end{equation}
This factorization is highly effective, but it fundamentally corresponds to a strong modeling assumption: both global semantics and local realization are propagated through the same token-level conditional chain. In other words, it assumes that the surface string itself is the most natural and primary state space.

The route explored in this paper instead reconsiders text generation from the level of the state space itself. If text indeed contains a low-dimensional yet sufficiently useful global semantic structure, then a more natural approach is not to place the entire burden of generation on a token-level chain factorization, but to introduce latent variables explicitly and model high-level semantic organization separately from local textual realization. Correspondingly, the core factorization of \method is
\begin{equation}
p(x,z_0)=p_\theta(x\mid z_0)p_\psi(z_0),
\qquad
p(x)=\int p_\theta(x\mid z_0)p_\psi(z_0)\,dz_0,
\label{eq:after_hierarchical_factorization}
\end{equation}
where $z_0$ is a continuous latent variable, $p_\psi(z_0)$ is the latent prior, and $p_\theta(x\mid z_0)$ is the conditional decoder. The crucial change here is not merely the introduction of latent variables, but the redefinition of the role of state in the system: the path no longer acts directly on observation recovery, but instead organizes global semantics in latent space first, after which the decoder carries out local textual realization.

This point can be stated compactly through the information decomposition of the average ELBO. Let
\begin{equation}
q(x,z_0):=p_{\mathrm{data}}(x)q_\phi(z_0\mid x),
\end{equation}
then
\begin{equation}
\mathbb E_{p_{\mathrm{data}}(x)}[\mathcal L_{\mathrm{ELBO}}(x)]
=
\mathbb E_{q(x,z_0)}[\log p_\theta(x\mid z_0)]
-
I_q(X;Z_0)
-
\mathrm{KL}\!\big(\bar q_\phi(z_0)\,\|\,p_\psi(z_0)\big),
\label{eq:after_avg_elbo}
\end{equation}
where $\bar q_\phi(z_0)$ is the aggregated posterior. This decomposition shows that hierarchical latent-space modeling breaks the text problem into three coupled but analytically distinguishable components: conditional realization, information compression, and prior matching. The latent variable is therefore not merely a continuous surrogate for discrete tokens, but an explicit intermediate state through which global semantic organization can be separated from local textual realization and modeled on its own terms.

From this perspective, compression must also be reconsidered. Prior work has emphasized the connection between compression and intelligence \citep{huang2024compression}, while recent explorations of generation closer to raw data forms in images and videos, such as pixel-space modeling \citep{deng2026generative}, further suggest that compression should not be equated with harmful information deletion. The key question is not whether every local detail is preserved, but whether the model can extract and organize structural information that is genuinely effective and generalizable. If text indeed admits a hierarchical structure in which high-level semantics and low-level realization are relatively separable, then reinterpreting text generation through informational hierarchy is not merely a change of method, but a theoretical re-evaluation of text modeling itself.

Accordingly, the first theme of this paper is not to reject autoregression, but to point out that autoregression occupies only one self-consistent, rather than unique, corner of the design space. If the data truly contains a hierarchy between low-dimensional global semantics and high-dimensional local realization, then organizing semantics first in a latent state and realizing text through conditional decoding may be closer to the true generative mechanism. Text generation should therefore not be understood solely as next-token fitting over discrete strings, but more generally as a systematic problem of how information is represented, compressed, and organized hierarchically.

\subsection{Understanding the Continuous Extension of Discrete Text: From Objective Mismatch to a Shift in Evaluation Emphasis}

Once the state space of the system is changed, the issue at the objective level changes accordingly. For conventional autoregressive language models, the training objective and evaluation quantities are naturally well aligned: maximum likelihood training directly corresponds to probability fitting over text, and likelihood and perplexity therefore have a clear and stable interpretation. In hierarchical continuous latent-space models, however, the actual training path is no longer direct token-level maximum likelihood, but a hierarchical objective jointly composed of reconstruction, latent prior learning, and representation regularization.

This can be seen from the relation between the ELBO and the true marginal likelihood:
\begin{equation}
-\mathcal L_{\mathrm{ELBO}}(x)
=
-\log p_{\theta,\psi}(x)
+
\mathrm{KL}\!\big(
q_\phi(z_0\mid x)\,\|\,p_{\theta,\psi}(z_0\mid x)
\big).
\label{eq:after_inference_gap}
\end{equation}
This shows that even at the level of the ELBO, the training objective is already separated from the true log-likelihood by a variational inference gap. Furthermore, in the actual training of \method, the model must jointly learn latent reconstruction, continuous prior fitting, and representation stabilization. The quantity being optimized is therefore not a single token-level likelihood in the classical sense.

For this reason, the mismatch should not be interpreted as a failure to learn, but rather as evidence that the model is learning something different. For autoregressive models and other paradigms that directly fit discrete distributions, likelihood and perplexity remain highly informative because they are naturally aligned with the training objective. For hierarchical continuous latent-space models, by contrast, the central issue is no longer whether local discrete distributions are fitted as sharply as possible, but whether higher-level semantic structures are effectively organized, whether the latent prior is well learned, and whether the final generations satisfy the actual task requirements.

From the perspective of systematic modeling, this phenomenon is in fact expected: when the state space expands from surface tokens to hierarchical latent variables, the optimization target correspondingly shifts from precise fitting of local discrete distributions to the organization of higher-level semantic structure, stable latent prior learning, and satisfaction of the true generative objective. For this route, generation-oriented metrics are therefore often more closely aligned with what the model is actually trained to do than perplexity. More importantly, model potential is often reflected more clearly in scaling behavior than in any single static likelihood value: what matters is whether capability continues to improve steadily as model size, data, and compute increase, rather than whether local fit under a particular pointwise metric is better.

This can also be connected to the perspective of the three governing curves developed in the theoretical analysis of the paper. For \method, the applicability of this route is not determined by a single likelihood value, but by whether three conditions hold simultaneously: the representation rate--distortion curve is already favorable at relatively low rate, the approximation error of the latent prior continues to decrease, and the inference gap remains controllable. In other words, the advantage of this route is not guaranteed automatically by latent variables or flow-based modeling themselves; it depends on whether the data truly contains a compressible global semantic structure, and whether the model can learn, fit, and realize that structure in a stable manner.

The second theme of this paper is therefore not merely that perplexity is inadequate, but that evaluation language itself must change once representation and objective have changed. For this class of models, generation quality and scaling behavior are often closer to the model's true capability and long-term potential than traditional perplexity.

\subsection{Exploring Unified Models: Model-Environment Interaction and the Value of Multimodal Unification}

If we return again to the model--environment formalization in Eq.~\eqref{eq:after_env} and Eq.~\eqref{eq:after_general_objective}, the third theme becomes more natural. The importance of unified models does not lie merely in placing multiple modalities within a single parameterized network, but in changing the structure of the environment in which the model learns. In the real world, observations, transitions, and feedback are usually not generated independently across modalities; rather, they are often jointly determined by a shared latent state. A more general learning system therefore requires not a set of isolated modality interfaces placed side by side, but unified representations that can enter the same interaction state and share the same dynamical constraints.

This is closely related to two broader views of intelligence. One influential view understands intelligence as a collection of skills across tasks \citep{chollet2019measure}. Under this view, a system becomes more capable because it can solve problems across more domains and under more diverse forms of supervision and interaction. The recent development of large language models partly reflects this tendency. A representative example is the progress of code agents that can operate within command-line environments. In such environments, the observation space, action space, and feedback mechanism are unusually well aligned with discrete symbolic representations. Interaction trajectories are easy to record, and correctness is often straightforward to verify, so these environments provide dense and precise learning signals.

Another view, closer to the world-model perspective, holds that intelligence consists in acquiring an internal model of the structure and dynamics of the world. Recent work on world models \citep{tu2025playerone} moves in this direction by seeking to learn richer environmental dynamics, thereby supporting stronger generalization and more realistic interaction. From this perspective, the question is not only how many tasks a model can solve, but whether it learns in an environment whose structure is rich enough to induce the right abstractions. The environment therefore becomes central: a model can only internalize the regularities that are actually present in the observations, transitions, and feedback it encounters.

This can also be written more formally. Let the observation at step $t$ be multimodal,
\begin{equation}
o_t=\big(o_t^{(1)},o_t^{(2)},\dots,o_t^{(M)}\big),
\qquad
o_t^{(m)}\in\mathcal O^{(m)},
\label{eq:after_multimodal_obs}
\end{equation}
and suppose there exists a joint latent state
\begin{equation}
z_t=\Phi\!\left(o_t^{(1)},\dots,o_t^{(M)}\right),
\label{eq:after_joint_latent}
\end{equation}
such that feedback and transition depend primarily on this joint state rather than on marginal factorizations over modalities:
\begin{equation}
\xi_t,\omega_{t+1}
\sim
p\big(\xi_t,\omega_{t+1}\mid z_t,a_t\big).
\label{eq:after_joint_feedback}
\end{equation}
If the true environmental dynamics satisfy
\begin{equation}
p(\xi_t,\omega_{t+1}\mid o_t,a_t)
\neq
\prod_{m=1}^{M}
p_m\big(\xi_t^{(m)},\omega_{t+1}^{(m)} \mid o_t^{(m)}, a_t^{(m)}\big),
\label{eq:after_nonseparable_env}
\end{equation}
then the learning problem is structurally non-separable across modalities. In such a case, treating each modality as an independent channel and only combining them superficially is generally insufficient. The theoretical significance of unified models lies precisely in the fact that the environment itself is non-separable in the sense of Eq.~\eqref{eq:after_nonseparable_env}: the regularities that determine useful feedback are joint regularities rather than regularities defined on the marginal distribution of each modality.

This clarifies why multimodal unification is not merely an engineering convenience. Its purpose is not simply to process multiple data types with one backbone, but to allow the model to learn in an environment whose observation, transition, and supervision structure more faithfully reflects the coupled regularities of the real world. In such an environment, both inputs and outputs may be multimodal; useful feedback may depend on how different modalities constrain each other jointly; and the learned internal state should ideally reflect these joint constraints.

This also explains why text has long been the most difficult component in unified models. Images and videos naturally operate in continuous spaces, whereas text is a prototypically discrete modality. If they are to enter a common interaction state and share latent dynamics, a severe representational mismatch immediately arises. This is precisely one of the central obstacles repeatedly identified in recent unified-model research \citep{deng2025emerging}. In this sense, the significance of \method lies not only in proposing another text generator, but in providing a natural interface through which discrete text can enter a continuous latent space.

If discrete text is mapped into a continuous latent variable through
\begin{equation}
z^{\mathrm{text}}\sim q_\phi(z\mid x^{\mathrm{text}}),
\qquad
x^{\mathrm{text}}\sim p_\eta(x\mid z^{\mathrm{text}}),
\label{eq:after_text_bridge}
\end{equation}
then text acquires an interface compatible with other continuous modalities. One may then define a unified interaction state
\begin{equation}
\tilde z_t
=
\Psi\!\left(
z_t^{\mathrm{text}},
z_t^{\mathrm{img}},
z_t^{\mathrm{vid}},
\dots
\right),
\label{eq:after_unified_latent}
\end{equation}
and perform state evolution, decision making, and feedback modeling at this level. Equations~\eqref{eq:after_text_bridge}--\eqref{eq:after_unified_latent} formalize why \method may matter beyond text generation itself: its role is not only to generate text through a different path, but to provide a bridge through which an intrinsically discrete modality can participate in a continuous multimodal interaction state. In other words, it reduces the structural mismatch that otherwise prevents text from naturally entering a shared continuous environment.

This is why the broader significance of \method is better understood through model--environment interaction than through single-modality benchmarks alone. If learning is viewed as the optimization of Eq.~\eqref{eq:after_general_objective} in richer and more realistic environments, then unified models matter because they expand the environments in which the model can learn. If text is to participate fully in such environments, then a bridge such as that in Eq.~\eqref{eq:after_text_bridge} becomes especially desirable. In this sense, \method is not merely an alternative text generator; it can also be understood as a candidate mechanism for aligning discrete text with continuous multimodal learning systems.

\subsection{The Three Themes Under a Unified Perspective}

In summary, the three themes of this paper are not separate supplementary discussions, but three manifestations of the same systematic problem. The first concerns the \emph{representation level}: whether text should be modeled entirely on the token surface, or whether higher-level semantics can be organized in an independent latent state. The second concerns the \emph{objective level}: once the model is trained through latent transport, reconstruction, and regularization rather than direct token-level maximum likelihood, which metrics remain genuinely aligned with the learning problem. The third concerns the \emph{environment level}: if learning is ultimately model--environment interaction, then what kind of environment future models should inhabit, and what representational interfaces are needed for different modalities to become compatible within it.

From this perspective, autoregressive language modeling occupies a self-consistent corner of the design space: representation is tightly bound to surface tokens, the training objective is direct likelihood maximization, and the environment is largely symbolic and text-centered. The route explored in this work changes all three assumptions simultaneously. It introduces a hierarchical latent-variable representation for text, thereby changing the representational assumption; it moves optimization away from direct token-level likelihood, thereby weakening the central interpretive role of perplexity; and it provides a continuous interface for discrete text, thereby making text potentially more compatible with multimodal environments that are more naturally expressed in continuous latent space.

We therefore hope that the contribution of this work is not only a viable alternative path for text generation, but also a more systematic way of thinking that jointly considers representation, objective alignment, and environment design. More broadly, we hope it encourages future research to rethink text, images, videos, and other modalities not as isolated domains that must be solved separately, but as components of a larger learning system in which unified representation, unified objectives, and unified environments may become increasingly central to the development of more general multimodal intelligence.

%% file: sections/appendix.tex
\section{Strict Probabilistic Formulation and Training Objective Details of Cola DLM}
\label{app:prob_model}

This appendix provides the strict probabilistic formulation of \method and clarifies the role of each training component. We first define the latent-variable model and the CNF prior precisely, then derive the marginal likelihood, the ELBO, and its population form, and finally explain the exact role of Flow Matching in the overall theoretical framework.

\subsection{Probabilistic Objects, CNF Prior, and Generative Distribution}
\label{app:prob_objects}

Let $x\in\mathcal X$ denote a discrete text sequence, and let $z_0\in\mathbb R^d$ denote its continuous latent variable. \method consists of three probabilistic components:
\begin{align}
q_{\phi}(z_0\mid x) &\quad \text{encoder (approximate posterior)}, \\
p_{\theta}(x\mid z_0) &\quad \text{decoder (conditional text generator)}, \\
p_{\psi}(z_0) &\quad \text{latent prior (defined by a continuous flow)}.
\end{align}
Here, $q_{\phi}(z_0\mid x)$ is used only for variational inference during training. The actual generative distribution is defined by $p_{\theta}(x\mid z_0)$ and $p_{\psi}(z_0)$.

To ensure that the flow map and density-transport formula are well defined, we impose the standard regularity assumption on the CNF vector field.

\begin{assumption}[Regularity of the CNF vector field]
\label{ass:app_cnf_regular}
For each parameter $\psi$, the vector field $\vv_{\psi}(z,t)$ is locally Lipschitz in $z$, measurable in $t$, and satisfies a linear-growth condition ensuring the existence and uniqueness of global ODE solutions. In addition, $\nabla\!\cdot\vv_{\psi}(z,t)$ is integrable along the trajectories of interest. Then, for any $0\le s\le t\le 1$, the flow map $\Phi^{\psi}_{t\leftarrow s}$ exists, is unique, and is invertible.
\end{assumption}

Let the base distribution be the standard Gaussian
\begin{equation}
z_1\sim p_1(z_1)=\Ncal(0,I).
\label{eq:app_base_dist}
\end{equation}
Given a time-dependent vector field $\vv_{\psi}(z_t,t)$, consider the ODE
\begin{equation}
\frac{\dd z_t}{\dd t}=\vv_{\psi}(z_t,t),
\qquad t\in[0,1].
\label{eq:app_ode_def}
\end{equation}
By Assumption~\ref{ass:app_cnf_regular}, this ODE induces a flow map
\begin{equation}
z_t=\Phi^{\psi}_{t\leftarrow s}(z_s).
\label{eq:app_flow_map}
\end{equation}
Therefore, a prior sample can be written as
\begin{equation}
z_0=\Phi^{\psi}_{0\leftarrow 1}(z_1),
\qquad z_1\sim p_1.
\label{eq:app_reverse_flow_sampling}
\end{equation}
Equivalently, the prior distribution is the pushforward of the base distribution under the reverse flow:
\begin{equation}
p_{\psi}=(\Phi^{\psi}_{0\leftarrow 1})_{\sharp}p_1.
\label{eq:app_prior_pushforward}
\end{equation}
This shows that $p_{\psi}(z_0)$ is not an abstract symbol, but a rigorously defined probability distribution induced jointly by $(p_1,\vv_{\psi})$.

Hence, the generative process of \method can be written as
\begin{equation}
z_1\sim p_1(z_1),
\qquad
z_0=\Phi^{\psi}_{0\leftarrow 1}(z_1),
\qquad
x\sim p_{\theta}(x\mid z_0).
\label{eq:app_generative_process}
\end{equation}
The corresponding joint and marginal distributions are
\begin{equation}
p(x,z_0)=p_{\theta}(x\mid z_0)p_{\psi}(z_0),
\qquad
p(x)=\int p_{\theta}(x\mid z_0)p_{\psi}(z_0)\,\dd z_0.
\label{eq:app_joint_and_marginal}
\end{equation}
If the base variable $z_1$ is introduced explicitly, the marginal can also be written as
\begin{equation}
p(x)=\int p_1(z_1)\,p_{\theta}\bigl(x\mid \Phi^{\psi}_{0\leftarrow 1}(z_1)\bigr)\,\dd z_1.
\label{eq:app_marginal_via_z1}
\end{equation}
Therefore, \method is fundamentally a hierarchical latent-variable language model, rather than a model that directly defines a reverse process in token space.

\subsection{Explicit Density of the CNF Prior and the Outer ELBO}
\label{app:cnf_and_elbo}

Let $Z_t$ denote the random variable induced by the ODE at time $t$, with density $p_t$. Under Assumption~\ref{ass:app_cnf_regular}, $p_t$ satisfies the continuity equation
\[
\partial_t p_t(z)+\nabla\!\cdot\bigl(p_t(z)\,\vv_{\psi}(z,t)\bigr)=0.
\]
Along a characteristic trajectory $t\mapsto z_t$, the instantaneous change-of-variables formula gives
\begin{equation}
\frac{\dd}{\dd t}\log p_t(z_t)
=
-\nabla\!\cdot \vv_{\psi}(z_t,t).
\label{eq:app_instantaneous_cov}
\end{equation}
Integrating from $0$ to $1$, and noting that $p_0=p_{\psi}$, yields
\begin{equation}
\log p_{\psi}(z_0)
=
\log p_1(z_1)
+
\int_0^1 \nabla\!\cdot \vv_{\psi}(z_t,t)\,\dd t,
\qquad z_1=\Phi^{\psi}_{1\leftarrow 0}(z_0).
\label{eq:app_prior_log_density}
\end{equation}
Thus, the CNF prior admits an explicitly defined log-density rather than an implicit prior that exists only through sampling.

The flow map can also be written as a deterministic Markov kernel. Define
\begin{equation}
K^{\psi}_{t\leftarrow s}(\dd z_t\mid z_s)
:=
\delta_{\Phi^{\psi}_{t\leftarrow s}(z_s)}(\dd z_t).
\label{eq:app_cnf_markov_kernel}
\end{equation}
By the semigroup property of ODE flows, $K^{\psi}_{t\leftarrow s}$ satisfies the Chapman--Kolmogorov relation and therefore forms a deterministic Markov semigroup. However, the Markov structure here comes from prior transport, not from an observation-dependent corruption--recovery mechanism.

We now introduce the variational posterior $q_{\phi}(z_0\mid x)$ to handle the latent integral in the marginal likelihood. For any $q_{\phi}(z_0\mid x)$ satisfying the required support condition,
\begin{align}
\log p(x)
&=
\log \int q_{\phi}(z_0\mid x)
\frac{p_{\theta}(x\mid z_0)p_{\psi}(z_0)}{q_{\phi}(z_0\mid x)}\,\dd z_0 \\
&\ge
\E_{q_{\phi}(z_0\mid x)}\Bigl[
\log p_{\theta}(x\mid z_0)
+\log p_{\psi}(z_0)
-\log q_{\phi}(z_0\mid x)
\Bigr].
\label{eq:app_outer_elbo}
\end{align}
Define the outer ELBO as
\begin{equation}
\Lcal_{\mathrm{ELBO}}(x)
:=
\E_{q_{\phi}(z_0\mid x)}\Bigl[
\log p_{\theta}(x\mid z_0)
+\log p_{\psi}(z_0)
-\log q_{\phi}(z_0\mid x)
\Bigr].
\label{eq:app_elbo_def}
\end{equation}
Then
\begin{equation}
\log p(x)\ge \Lcal_{\mathrm{ELBO}}(x).
\label{eq:app_elbo_lower_bound}
\end{equation}
More importantly, the exact decomposition between the true log-likelihood and the ELBO is
\begin{equation}
\log p(x)
=
\Lcal_{\mathrm{ELBO}}(x)
+
\KL\bigl(q_{\phi}(z_0\mid x)\,\|\,p(z_0\mid x)\bigr),
\label{eq:app_elbo_identity}
\end{equation}
where the true posterior is
\begin{equation}
p(z_0\mid x)=\frac{p_{\theta}(x\mid z_0)p_{\psi}(z_0)}{p(x)}.
\label{eq:app_true_posterior}
\end{equation}
Equation~\eqref{eq:app_elbo_identity} makes the distinction explicit: the true likelihood is always $\log p(x)$, whereas the training objective directly optimized is its variational lower bound $\Lcal_{\mathrm{ELBO}}(x)$.

Substituting the explicit CNF density formula \eqref{eq:app_prior_log_density} into the outer ELBO gives
\begin{align}
\Lcal_{\mathrm{ELBO}}(x)
=
\E_{q_{\phi}(z_0\mid x)}\Bigl[
&\log p_{\theta}(x\mid z_0)
+\log p_1(z_1)
+\int_0^1 \nabla\!\cdot \vv_{\psi}(z_t,t)\,\dd t \\
&-\log q_{\phi}(z_0\mid x)
\Bigr],
\qquad z_1=\Phi^{\psi}_{1\leftarrow 0}(z_0).
\label{eq:app_elbo_with_cnf_prior}
\end{align}
This is the strict training objective of \method in the setting where the decoder can compute $\log p_{\theta}(x\mid z_0)$.

\subsection{Average ELBO, Aggregated Posterior, and Prior Matching}
\label{app:avg_elbo}

To analyze the nature of prior learning, define the aggregated posterior
\begin{equation}
\bar q_{\phi}(z_0)
:=
\int q_{\phi}(z_0\mid x)\,\pdata(x)\,\dd x,
\label{eq:app_agg_posterior}
\end{equation}
and the joint distribution
\begin{equation}
q(x,z_0):=\pdata(x)q_{\phi}(z_0\mid x).
\label{eq:app_joint_qxz}
\end{equation}
When the encoder $q_{\phi}$ and decoder $p_{\theta}$ are fixed, the dependence of the data-averaged ELBO on $\psi$ appears only through the prior term:
\begin{align}
\E_{\pdata(x)}\bigl[\Lcal_{\mathrm{ELBO}}(x)\bigr]
&=
\E_{q(x,z_0)}\bigl[\log p_{\theta}(x\mid z_0)-\log q_{\phi}(z_0\mid x)\bigr]
+\E_{q(x,z_0)}\bigl[\log p_{\psi}(z_0)\bigr] \\
&=
C(\phi,\theta)
+
\E_{z_0\sim \bar q_{\phi}}\bigl[\log p_{\psi}(z_0)\bigr],
\label{eq:app_avg_elbo_prior_part}
\end{align}
where $C(\phi,\theta)$ is independent of $\psi$.

Using
\[
\E_{z_0\sim \bar q_{\phi}}[\log p_{\psi}(z_0)]
=
-\KL\bigl(\bar q_{\phi}(z_0)\,\|\,p_{\psi}(z_0)\bigr)
+
\E_{z_0\sim \bar q_{\phi}}[\log \bar q_{\phi}(z_0)],
\]
we obtain
\begin{equation}
\max_{\psi}\;\E_{\pdata(x)}\bigl[\Lcal_{\mathrm{ELBO}}(x)\bigr]
\quad\Longleftrightarrow\quad
\min_{\psi}\;\KL\bigl(\bar q_{\phi}(z_0)\,\|\,p_{\psi}(z_0)\bigr).
\label{eq:app_prior_matching_problem}
\end{equation}
Therefore, after fixing the encoder and decoder, the strict prior-learning subproblem is to fit the aggregated posterior $\bar q_{\phi}$ with $p_{\psi}$.

The average ELBO can be further written in information-decomposition form. Adding and subtracting $\log \bar q_{\phi}(z_0)$ in the prior term yields
\begin{align}
\E_{\pdata(x)}[\Lcal_{\mathrm{ELBO}}(x)]
&=
\E_{q(x,z_0)}[\log p_{\theta}(x\mid z_0)]
- I_q(X;Z_0)
- \KL\bigl(\bar q_{\phi}(z_0)\,\|\,p_{\psi}(z_0)\bigr),
\label{eq:app_avg_elbo_decomp}
\end{align}
where
\begin{equation}
I_q(X;Z_0)
:=
\E_{\pdata(x)}\Bigl[\KL\bigl(q_{\phi}(z_0\mid x)\,\|\,\bar q_{\phi}(z_0)\bigr)\Bigr]
\label{eq:app_mutual_info}
\end{equation}
is the amount of information retained in the latent variable.

Equation~\eqref{eq:app_avg_elbo_decomp} reveals the three theoretical roles of the encoder. It determines the target distribution $\bar q_{\phi}(z_0)$ that the prior must fit, the compression strength through $I_q(X;Z_0)$, and therefore the division of labor between global semantics in latent space and local realization in the decoder.

\subsection{The Role of Flow Matching: A Prior Solver Rather Than the Model Definition}
\label{app:fm_role}

Direct optimization of $\log p_{\psi}(z_0)$ usually requires repeated ODE solves and divergence estimation, which is expensive in practice. For this reason, Flow Matching is often used to learn the vector field of the same prior more efficiently. It should therefore be viewed as a prior solver, not as the likelihood objective itself.

To do so, one first specifies a bridge path from the aggregated posterior to the base distribution. Let
\begin{equation}
z_0\sim \bar q_{\phi}(z_0),
\qquad
z_1\sim p_1(z_1),
\qquad
t\sim U[0,1].
\label{eq:app_fm_sampling}
\end{equation}
Consider the conditional path
\begin{equation}
z_t=(1-\alpha(t))z_0+\alpha(t)z_1,
\qquad
\alpha(0)=0,
\ \alpha(1)=1.
\label{eq:app_general_fm_path}
\end{equation}
Here $\alpha$ can be any sufficiently smooth monotone function; linear interpolation corresponds to $\alpha(t)=t$. The target velocity field of this path is
\begin{equation}
u_t(z_0,z_1)
:=
\frac{\dd z_t}{\dd t}
=
\dot\alpha(t)(z_1-z_0).
\label{eq:app_fm_target_velocity}
\end{equation}
The Flow Matching loss is then
\begin{equation}
\Lcal_{\mathrm{FM}}(\psi;\phi)
:=
\E_{t,\,z_0\sim \bar q_{\phi},\,z_1\sim p_1}
\left[
\left\|\vv_{\psi}(z_t,t)-u_t(z_0,z_1)\right\|_2^2
\right].
\label{eq:app_fm_loss}
\end{equation}

For a fixed path construction, the pointwise optimal solution is the conditional expectation
\begin{equation}
\vv_{\psi}^{\star}(z,t)
=
\E\bigl[u_t(z_0,z_1)\mid z_t=z,\,t\bigr].
\label{eq:app_fm_optimal_field}
\end{equation}
This is simply a conditional-mean regression result, but it makes the role of Flow Matching precise: under ideal conditions, the learned vector field corresponds to a continuous transport connecting $\bar q_{\phi}$ and $p_1$, which in turn defines the prior $p_{\psi}(z_0)$.

Hence, two objectives must be distinguished. The first,
\begin{equation}
\max_{\psi}\;\E_{z_0\sim \bar q_{\phi}}[\log p_{\psi}(z_0)],
\label{eq:app_strict_prior_nll}
\end{equation}
is the strict probabilistic objective for prior learning. The second,
\begin{equation}
\min_{\psi}\;\Lcal_{\mathrm{FM}}(\psi;\phi),
\label{eq:app_fm_solver_obj}
\end{equation}
is a practical solver for learning the vector field of the same prior. They serve the same prior-modeling problem, but they are not the same mathematical object. In particular, $\Lcal_{\mathrm{FM}}$ cannot be identified term-by-term with $-\log p_{\psi}(z_0)$ in the ELBO.

\subsection{Strict Form of the Training Objective}
\label{app:training_objective}

If the decoder can compute $\log p_{\theta}(x\mid z_0)$ exactly, then the strict training objective is to maximize \eqref{eq:app_elbo_def}, or equivalently to minimize the negative ELBO. Define
\begin{align}
L_{\mathrm{REC}}(x)
&:= -\E_{q_{\phi}(z_0\mid x)}\bigl[\log p_{\theta}(x\mid z_0)\bigr],
\label{eq:app_l_rec}\\
L_{\mathrm{PRIOR}}(x)
&:= -\E_{q_{\phi}(z_0\mid x)}\bigl[\log p_{\psi}(z_0)\bigr],
\label{eq:app_l_prior}\\
L_{\mathrm{REG}}(x)
&:= \E_{q_{\phi}(z_0\mid x)}\bigl[\log q_{\phi}(z_0\mid x)\bigr],
\label{eq:app_l_reg}
\end{align}
so that the total single-sample loss becomes
\begin{equation}
L_{\mathrm{Total}}^{\mathrm{strict}}(x)
:=
L_{\mathrm{REC}}(x)+L_{\mathrm{PRIOR}}(x)+L_{\mathrm{REG}}(x)
=
-\Lcal_{\mathrm{ELBO}}(x).
\label{eq:app_l_total_strict}
\end{equation}
Taking expectation over the data distribution gives the strict training objective
\begin{equation}
\min_{\phi,\theta,\psi}\;\E_{\pdata(x)}\bigl[L_{\mathrm{Total}}^{\mathrm{strict}}(x)\bigr]
\quad\Longleftrightarrow\quad
\max_{\phi,\theta,\psi}\;\E_{\pdata(x)}\bigl[\Lcal_{\mathrm{ELBO}}(x)\bigr].
\label{eq:app_strict_training_objective}
\end{equation}
Therefore, the objective that truly corresponds to maximizing a lower bound of the marginal likelihood is always the ELBO; Flow Matching is only an efficient implementation for learning the prior.

\section{Sampling, Inference, and Probability Estimation Details}
\label{app:inference_and_likelihood}

This appendix supplements the strict formulation of sampling, continuation, unconditional likelihood estimation, and conditional probability estimation in \method. Since the main text only presents the overall framework, we provide the details here.

\subsection{Sampling and Inference Procedure}
\label{app:sampling}

To generate a new sample from \method, one should follow the generative process defined by the model, rather than the approximate posterior path used during training. First, sample a latent seed from the simple base distribution:
\begin{equation}
z_1\sim p_1(z_1)=\Ncal(0,I).
\label{eq:app_sampling_step1}
\end{equation}
Second, integrate the vector field backward to obtain a latent sample:
\begin{equation}
\frac{\dd z_t}{\dd t}=\vv_{\psi}(z_t,t),
\qquad t:1\rightarrow 0,
\label{eq:app_sampling_reverse_ode}
\end{equation}
so that starting from $z(1)=z_1$ yields $z(0)=z_0$. In practice, Euler, Heun, or higher-order ODE solvers can be used. This step generates global latent semantics rather than tokens themselves.

Third, the decoder generates text conditioned on the latent:
\begin{equation}
x\sim p_{\theta}(x\mid z_0).
\label{eq:app_sampling_step_decoder}
\end{equation}
If the decoder is autoregressive, this step is conditional autoregressive sampling; if the decoder is non-autoregressive or bidirectional, it becomes the corresponding conditional recovery procedure. Hence, sampling in \method is inherently two-stage: one first samples a global latent, and then realizes the text conditioned on it.

If a task provides a prefix $x^{\mathrm{pre}}$, inference first encodes the prefix into a conditional latent:
\begin{equation}
z^{\mathrm{pre}}\sim q_{\phi}(z^{\mathrm{pre}}\mid x^{\mathrm{pre}}),
\label{eq:app_prefix_encode}
\end{equation}
and then generates the response latent block by block. Let the previously generated blocks be $\hat z_0^{(<b)}$. The $b$-th response block is generated by
\begin{equation}
\hat z_0^{(b)}
=
\Phi^{\psi}_{0\leftarrow 1}\!\bigl(\epsilon^{(b)};z^{\mathrm{pre}},\hat z_0^{(<b)}\bigr),
\qquad
\epsilon^{(b)}\sim\Ncal(0,I).
\label{eq:app_block_generation}
\end{equation}
The response is then decoded as
\begin{equation}
\hat x^{\mathrm{res}}
\sim
p_{\theta}\!\bigl(x^{\mathrm{res}}\mid x^{\mathrm{pre}},z^{\mathrm{pre}},\hat z_0^{(1:B)}\bigr).
\label{eq:app_conditional_decode}
\end{equation}
This is exactly the prefix-encoding, blockwise-generation, and conditional-decoding procedure described in the main text.

\subsection{Unconditional Log-Likelihood and PPL Estimation}
\label{app:uncond_likelihood}

For a dataset $\mathcal D$, if one directly uses the ELBO as a lower bound on $\log p(x)$, one may define an ELBO-based perplexity as
\begin{equation}
\mathrm{PPL}_{\mathrm{ELBO}}
:=
\exp\left(
-\frac{\sum_{x\in\mathcal D}\Lcal_{\mathrm{ELBO}}(x)}{\sum_{x\in\mathcal D}|x|}
\right),
\label{eq:app_ppl_elbo}
\end{equation}
where $|x|$ denotes sequence length. Since $\Lcal_{\mathrm{ELBO}}(x)\le \log p(x)$, $\mathrm{PPL}_{\mathrm{ELBO}}$ should be interpreted as an upper-bound-type estimate of the true perplexity.

If the decoder can compute $\log p_{\theta}(x\mid z_0)$, one can also obtain tighter likelihood estimates through importance sampling or IWAE. Let
\begin{equation}
z_0^{(k)}\sim q_{\phi}(z_0\mid x),
\qquad
\log w^{(k)}
=
\log p_{\theta}(x\mid z_0^{(k)})
+\log p_{\psi}(z_0^{(k)})
-\log q_{\phi}(z_0^{(k)}\mid x).
\label{eq:app_iwae_weight}
\end{equation}
Then the ELBO-style and IWAE-style estimators are
\begin{equation}
\log \widehat p_{\mathrm{ELBO},K}(x)
=
\frac{1}{K}\sum_{k=1}^{K}\log w^{(k)},
\qquad
\log \widehat p_{\mathrm{IWAE},K}(x)
=
\log\!\left(\frac{1}{K}\sum_{k=1}^{K}e^{\log w^{(k)}}\right).
\label{eq:app_p_estimators}
\end{equation}
The IWAE estimator is generally tighter. When IWAE is used for scoring, the corresponding perplexity is
\begin{equation}
\mathrm{PPL}_{\mathrm{IWAE}}
:=
\exp\left(
-\frac{\sum_{x\in\mathcal D}\log \widehat p_{\mathrm{IWAE},K}(x)}{\sum_{x\in\mathcal D}|x|}
\right).
\label{eq:app_ppl_iwae}
\end{equation}
If the decoder itself can provide only a lower bound, then evaluation can report only bound-based PPL rather than exact perplexity.

\subsection{Conditional Probability Estimation and Block-Level Conditional Scoring}
\label{app:conditional_likelihood}

In continuation, multiple-choice evaluation, and candidate-response ranking, conditional probability is more important than the unconditional score. Let a full sample be decomposed as
\begin{equation}
x=(x^{\mathrm{pre}},x^{\mathrm{res}}),
\label{eq:app_prefix_response_partition}
\end{equation}
where $x^{\mathrm{pre}}$ is the given prefix and $x^{\mathrm{res}}$ is the response to be evaluated. The exact conditional log-probability satisfies
\begin{equation}
\log p(x^{\mathrm{res}}\mid x^{\mathrm{pre}})
=
\log p(x^{\mathrm{pre}},x^{\mathrm{res}})
-
\log p(x^{\mathrm{pre}}).
\label{eq:app_conditional_identity}
\end{equation}
Hence one can apply the same unconditional estimator separately to the joint sample and the prefix sample, and then take the difference to obtain a plug-in conditional estimator:
\begin{equation}
\widehat{\log p}_{\mathrm{cond}}^{\mathrm{ELBO}}(x^{\mathrm{res}}\mid x^{\mathrm{pre}})
:=
\Lcal_{\mathrm{ELBO}}(x^{\mathrm{pre}},x^{\mathrm{res}})
-
\Lcal_{\mathrm{ELBO}}(x^{\mathrm{pre}}),
\label{eq:app_conditional_elbo_plugin}
\end{equation}
and
\begin{equation}
\widehat{\log p}_{\mathrm{cond},K}^{\mathrm{IWAE}}(x^{\mathrm{res}}\mid x^{\mathrm{pre}})
:=
\log \widehat p_{\mathrm{IWAE},K}(x^{\mathrm{pre}},x^{\mathrm{res}})
-
\log \widehat p_{\mathrm{IWAE},K}(x^{\mathrm{pre}}).
\label{eq:app_conditional_iwae_plugin}
\end{equation}
It should be emphasized that both \eqref{eq:app_conditional_elbo_plugin} and \eqref{eq:app_conditional_iwae_plugin} are plug-in estimators. They are obtained by substituting unconditional estimators into the conditional identity, and they do not automatically inherit a lower-bound property from the unconditional estimators. Strictly speaking, they should therefore be interpreted as practical estimators of conditional log-likelihood.

For a block-causal prior, the conditional estimator can also be written in a blockwise form aligned with the architecture. Suppose the latent sequence is partitioned into prefix blocks and response blocks:
\begin{equation}
z=(z^{(1)},\dots,z^{(B)},z^{(B+1)},\dots,z^{(B+R)}).
\label{eq:app_block_partition}
\end{equation}
If the prior admits a block-level causal factorization, then
\begin{equation}
p_{\psi}(z)
=
p_{\psi}(z^{(1)})
\prod_{b=2}^{B+R}p_{\psi}(z^{(b)}\mid z^{(<b)}).
\label{eq:app_block_causal_prior}
\end{equation}
Conditioned on the prefix latent $z^{\mathrm{pre}}:=z^{(1:B)}$, the conditional marginal of the response is
\begin{equation}
p(x^{\mathrm{res}}\mid x^{\mathrm{pre}},z^{\mathrm{pre}})
=
\int p_{\theta}(x^{\mathrm{res}}\mid x^{\mathrm{pre}},z^{\mathrm{pre}},z^{\mathrm{res}})
\,p_{\psi}(z^{\mathrm{res}}\mid z^{\mathrm{pre}})
\,\dd z^{\mathrm{res}},
\label{eq:app_conditional_marginal_given_prefix_latent}
\end{equation}
where $z^{\mathrm{res}}:=z^{(B+1:B+R)}$. For any approximate posterior $q_{\phi}(z^{\mathrm{res}}\mid x^{\mathrm{pre}},x^{\mathrm{res}},z^{\mathrm{pre}})$, one has the conditional ELBO
\begin{align}
\log p(x^{\mathrm{res}}\mid x^{\mathrm{pre}},z^{\mathrm{pre}})
\ge
&E_{q_{\phi}(z^{\mathrm{res}}\mid x^{\mathrm{pre}},x^{\mathrm{res}},z^{\mathrm{pre}})}\Bigl[
\log p_{\theta}(x^{\mathrm{res}}\mid x^{\mathrm{pre}},z^{\mathrm{pre}},z^{\mathrm{res}}) \\
&\qquad\qquad +\log p_{\psi}(z^{\mathrm{res}}\mid z^{\mathrm{pre}})
-\log q_{\phi}(z^{\mathrm{res}}\mid x^{\mathrm{pre}},x^{\mathrm{res}},z^{\mathrm{pre}})
\Bigr].
\label{eq:app_conditional_elbo_given_prefix_latent}
\end{align}
Marginalizing over the true posterior of the prefix latent then yields
\begin{equation}
p(x^{\mathrm{res}}\mid x^{\mathrm{pre}})
=
\int p(x^{\mathrm{res}}\mid x^{\mathrm{pre}},z^{\mathrm{pre}})
\,p(z^{\mathrm{pre}}\mid x^{\mathrm{pre}})\,\dd z^{\mathrm{pre}}.
\label{eq:app_conditional_with_prefix_posterior}
\end{equation}
Thus, a strict conditional probability requires not only marginalization over the response latent, but also an additional integration over the posterior uncertainty of the prefix latent. In practice, two approximations are common: one uses the joint-minus-prefix estimator in \eqref{eq:app_conditional_identity}, while the other approximates $p(z^{\mathrm{pre}}\mid x^{\mathrm{pre}})$ with the encoder mean or a single encoder sample and then evaluates \eqref{eq:app_conditional_elbo_given_prefix_latent}.

If the response contains only one new block, the above expression reduces to a single-block response score:
\begin{align}
\mathcal S_{\mathrm{resp}}(x^{(B+1)};x^{(1:B)},z^{(1:B)})
:=
&E_{q_{\phi}(z^{(B+1)}\mid x^{(1:B+1)},z^{(1:B)})}\Bigl[
\log p_{\theta}(x^{(B+1)}\mid x^{(1:B)},z^{(1:B+1)}) \\
&\qquad\qquad +\log p_{\psi}(z^{(B+1)}\mid z^{(1:B)})
-\log q_{\phi}(z^{(B+1)}\mid x^{(1:B+1)},z^{(1:B)})
\Bigr].
\label{eq:app_single_block_score}
\end{align}
This form is especially useful for scoring multiple-choice options, ranking candidate responses, and analyzing whether the separation between candidates is driven mainly by the prior term or the decoder term.

\subsection{Numerical Evaluation of the CNF Prior and Estimation Algorithms}
\label{app:algorithms}

To evaluate $\log p_{\psi}(z_0^{(k)})$ in \eqref{eq:app_iwae_weight}, one solves the augmented ODE
\begin{equation}
\frac{\dd}{\dd t}
\begin{bmatrix}
z_t \\
\ell_t
\end{bmatrix}
=
\begin{bmatrix}
\vv_{\psi}(z_t,t) \\
\nabla\!\cdot \vv_{\psi}(z_t,t)
\end{bmatrix},
\qquad
\begin{bmatrix}
z_0 \\
\ell_0
\end{bmatrix}
=
\begin{bmatrix}
z_0^{(k)} \\
0
\end{bmatrix}.
\label{eq:app_augmented_ode}
\end{equation}
Integrating from $t=0$ to $t=1$ gives $(z_1^{(k)},\ell_1^{(k)})$, and therefore
\begin{equation}
\log p_{\psi}(z_0^{(k)})
=
\log p_1(z_1^{(k)})+\ell_1^{(k)}.
\label{eq:app_prior_eval_again}
\end{equation}
In high dimensions, the divergence can be approximated by the Hutchinson trace estimator:
\begin{equation}
\nabla\!\cdot \vv_{\psi}(z_t,t)
=
\mathrm{Tr}\!\left(\frac{\partial \vv_{\psi}(z_t,t)}{\partial z_t}\right)
=
\E_{\veps}\left[\veps^{\top}
\frac{\partial \vv_{\psi}(z_t,t)}{\partial z_t}
\veps\right].
\label{eq:app_hutchinson}
\end{equation}
In practice, one usually fixes the same random vector $\veps$ within a single ODE solve so that the induced dynamics remain consistent along that trajectory.

We now present two estimators corresponding to the above formulas. To avoid extra package dependencies, the algorithms are written in directly compilable pseudo-code style.

\textbf{Algorithm A.1: Unconditional Log-Likelihood Estimation (ELBO / IWAE)}
\begin{quote}
\textbf{Input:} sequence $x$; encoder $q_{\phi}$; decoder $p_{\theta}$; CNF prior $p_{\psi}$; sample number $K$; estimation mode $\texttt{mode}\in\{\texttt{ELBO},\texttt{IWAE}\}$.\\
\textbf{Procedure:}
\begin{enumerate}[leftmargin=2.2em]
    \item For each $k=1,\dots,K$, sample $z_0^{(k)}\sim q_{\phi}(z_0\mid x)$.
    \item Starting from $(z_0^{(k)},0)$, solve \eqref{eq:app_augmented_ode} to $t=1$ and obtain $(z_1^{(k)},\ell_1^{(k)})$.
    \item Compute $\log p_{\psi}(z_0^{(k)})=\log p_1(z_1^{(k)})+\ell_1^{(k)}$.
    \item Compute $\log p_{\theta}(x\mid z_0^{(k)})$ and $\log q_{\phi}(z_0^{(k)}\mid x)$.
    \item Construct the importance weight
    \[
    \log w^{(k)}
    \gets
    \log p_{\theta}(x\mid z_0^{(k)})
    +\log p_{\psi}(z_0^{(k)})
    -\log q_{\phi}(z_0^{(k)}\mid x).
    \]
    \item If $\texttt{mode}=\texttt{ELBO}$, return
    \[
    \widehat{\log p}(x)=\frac{1}{K}\sum_{k=1}^{K}\log w^{(k)};
    \]
    otherwise return
    \[
    \widehat{\log p}(x)=\log\!\left(\frac{1}{K}\sum_{k=1}^{K}e^{\log w^{(k)}}\right).
    \]
\end{enumerate}
\textbf{Output:} an unconditional log-likelihood estimate $\widehat{\log p}(x)$ for sequence $x$.
\end{quote}

\textbf{Algorithm A.2: Conditional Log-Likelihood Estimation (Joint Minus Prefix)}
\begin{quote}
\textbf{Input:} prefix $x^{\mathrm{pre}}$; response $x^{\mathrm{res}}$; sample number $K$; estimation mode $\texttt{mode}\in\{\texttt{ELBO},\texttt{IWAE}\}$.\\
\textbf{Procedure:}
\begin{enumerate}[leftmargin=2.2em]
    \item Construct the full sequence $x^{\mathrm{full}}=(x^{\mathrm{pre}},x^{\mathrm{res}})$.
    \item Call Algorithm A.1 to estimate $\widehat{\log p}(x^{\mathrm{full}})$.
    \item Call Algorithm A.1 to estimate $\widehat{\log p}(x^{\mathrm{pre}})$.
    \item Compute
    \[
    \widehat{\log p}(x^{\mathrm{res}}\mid x^{\mathrm{pre}})
    \gets
    \widehat{\log p}(x^{\mathrm{full}})-\widehat{\log p}(x^{\mathrm{pre}}).
    \]
\end{enumerate}
\textbf{Output:} a conditional log-likelihood estimate $\widehat{\log p}(x^{\mathrm{res}}\mid x^{\mathrm{pre}})$.
\end{quote}

Algorithm A.1 applies to both ELBO-style and IWAE-style estimators. Algorithm A.2 treats it as a black-box unconditional likelihood estimator and combines two calls through the conditional identity. If the task focuses only on the local score of a single newly generated block, then \eqref{eq:app_single_block_score} can be used directly to rank candidate responses.

\section{Comparison Under a Unified Markov Path View}
\label{app:comparison}

This appendix further develops the mathematical comparison behind the main-text discussion of unified text modeling. The central question is not which method is more diffusion-like at a superficial level, but rather: into what state space and along what type of path does each method factorize the text distribution, and what role does that path play in generation?

\subsection{A Unified Framework for Process-Based Generative Models}
\label{app:unified_process}

To establish a common language, consider the abstract definition of a process-based generative model. Let
\begin{equation}
Y=(Y_t)_{t\in\mathcal T},
\qquad Y_t\in\mathcal S,
\label{eq:app_unified_path_process}
\end{equation}
where $\mathcal S$ is the state space and $\mathcal T$ is either discrete time or continuous time. If there exists a family of transition kernels
\begin{equation}
K_{t\leftarrow s}(\dd y_t\mid y_s),
\qquad s\le t,
\label{eq:app_unified_transition_kernel}
\end{equation}
satisfying the Chapman--Kolmogorov relation, then $Y$ forms a Markov process. If one additionally specifies an emission kernel
\begin{equation}
e_{\theta}(x\mid y_{t_\star}),
\label{eq:app_unified_emission}
\end{equation}
then the text distribution can be written uniformly as
\begin{equation}
p_{\theta}(x)=\int e_{\theta}(x\mid y_{t_\star})\,p_{\theta}(\dd Y).
\label{eq:app_unified_path_marginal}
\end{equation}

Within this framework, the difference among AR, LLaDA, Plaid, and \method is not simply whether each admits a stepwise sampling algorithm. The essential distinction is that they factorize the text distribution over different state spaces and assign fundamentally different roles to the underlying paths. More specifically, they differ at least in the following six aspects: state space, state trajectory object, transition mechanism, the role of the path in generation, the point at which the observation is read out, and the tightness of the training objective relative to the true log-likelihood.

This unified view shows that the key question is not which method uses diffusion, but rather what type of distributional complexity is decomposed into what kind of path over what kind of state space, and whether that path is an observation-dependent corruption path or an observation-independent latent prior path.

\subsection{Difference from Autoregressive Models: Why Diffusion}
\label{app:why_diffusion}

AR models text directly via the chain rule:
\begin{equation}
p_{\eta}^{\mathrm{AR}}(x)=\prod_{i=1}^{L} p_{\eta}(x_i\mid x_{<i}).
\label{eq:app_ar_factorization}
\end{equation}
Its negative log-likelihood is
\begin{equation}
-\log p_{\eta}^{\mathrm{AR}}(x)=\sum_{i=1}^{L}-\log p_{\eta}(x_i\mid x_{<i}).
\label{eq:app_ar_nll}
\end{equation}
The advantage of AR is therefore exact token-level likelihood and the most direct training and evaluation protocol. Its limitation comes from the fixed left-to-right conditioning order.

Under the unified path framework, AR can be written using prefix states
\begin{equation}
S_i:=x_{1:i},
\qquad S_0:=\varnothing.
\label{eq:app_prefix_state_def}
\end{equation}
Then $(S_i)_{i=0}^{L}$ forms a time-inhomogeneous Markov chain, whose one-step transition kernel is exactly the AR conditional distribution:
\begin{equation}
\mathbb P(S_i=s_i\mid S_{0:i-1})
=
\mathbb P(S_i=s_i\mid S_{i-1}=s_{i-1})
=
p_{\eta}(x_i\mid x_{<i}).
\label{eq:app_ar_as_markov_chain}
\end{equation}
Thus, the core inductive bias of AR is not merely Markovianity, but the fact that conditioning is restricted to the unique prefix filtration
\begin{equation}
\sigma(X_1)\subset \sigma(X_{1:2})\subset \cdots \subset \sigma(X_{1:L}).
\label{eq:app_ar_filtration}
\end{equation}

By contrast, the model distribution of \method is not defined through a fixed token order, but by latent marginalization:
\begin{equation}
p^{\mathrm{Cola\ DLM}}(x)
=
\int p_{\theta}(x\mid z_0)\,p_{\psi}(z_0)\,\dd z_0.
\label{eq:app_ldlm_factorization}
\end{equation}
If the decoder itself is not left-to-right autoregressive, the resulting model distribution is not tied to any pre-specified token ordering. In this sense, the non-AR nature of \method comes from two levels: the prior $p_{\psi}(z_0)$ is defined by continuous transport rather than token-wise chain factorization, and if the decoder is also non-AR, then generation becomes a two-stage procedure of first sampling global semantic latent variables and then realizing text conditionally.

From the optimization viewpoint, both families minimize
\begin{equation}
\KL\bigl(\pdata(x)\,\|\,p_{\mathrm{model}}(x)\bigr),
\label{eq:app_common_goal}
\end{equation}
but they parameterize $p(x)$ differently. AR uses a chain-rule factorization, whereas \method uses latent marginalization. Therefore, the strict meaning of \emph{why diffusion} here is not to avoid maximum likelihood, but to improve the fit of $p_{\psi}(z_0)$ to the aggregated posterior through a more expressive continuous prior family.

This can be quantified exactly. For any two candidate priors $p_a$ and $p_b$,
\begin{equation}
\E_{\bar q_{\phi}(z_0)}[\log p_b(z_0)-\log p_a(z_0)]
=
\KL\bigl(\bar q_{\phi}(z_0)\,\|\,p_a(z_0)\bigr)
-
\KL\bigl(\bar q_{\phi}(z_0)\,\|\,p_b(z_0)\bigr).
\label{eq:app_prior_improvement_identity}
\end{equation}
Hence, whenever the flow/CNF prior is closer to $\bar q_{\phi}(z_0)$ than a simple Gaussian prior, it improves the average ELBO.

\subsection{Difference from Discrete Denoising Language Models: Why Continuous}
\label{app:why_continuous}

Discrete denoising language models such as LLaDA operate directly in token space. A typical training objective can be written as
\begin{equation}
\Lcal_{\mathrm{MDM}}(\theta)
:=
\E_{t\sim U[0,1],\,x_0,\,x_t}
\left[
\frac{1}{t}
\sum_{i=1}^{L}
\mathbf 1[x_t^i=M]\bigl(-\log p_{\theta}(x_0^i\mid x_t)\bigr)
\right],
\label{eq:app_llada_obj}
\end{equation}
and satisfies
\begin{equation}
-\E_{\pdata(x_0)}\log p_{\theta}(x_0)
\le
\Lcal_{\mathrm{MDM}}(\theta).
\label{eq:app_llada_bound}
\end{equation}
Thus, LLaDA learns a token-space distribution through reverse recovery under random mask ratios. Its advantage is that it has no explicit encoder and no latent bottleneck; its limitation is that it always remains in a discrete state space.

More precisely, the forward masking process of LLaDA can be embedded into a continuous-time Markov chain. For a single token, let the state space be $\mathcal V\cup\{M\}$ and define the generator
\begin{equation}
G_{\tau}(a,M)=\lambda(\tau),
\qquad
G_{\tau}(a,a)=-\lambda(\tau),
\qquad
G_{\tau}(M,M)=0,
\label{eq:app_llada_ctmc_generator}
\end{equation}
where $a\in\mathcal V$. Then the transition probability from any non-mask state to the absorbing mask state is
\begin{equation}
\mathbb P(X_{\tau}=M\mid X_0=a)
=
1-\exp\Bigl(-\int_0^{\tau}\lambda(u)\,\dd u\Bigr).
\label{eq:app_llada_mask_probability}
\end{equation}
If the right-hand side is identified with $t\in[0,1]$, one recovers the forward kernel in which each token is masked with probability $t$, as used in \eqref{eq:app_llada_obj}. Therefore, LLaDA can be understood rigorously as a reverse recovery problem over a CTMC with discrete states and absorption into the mask state.

By contrast, \method models $p_{\psi}(z_0)$ in continuous latent space rather than performing discrete diffusion in token space. This continuous formulation provides three strict advantages. First, continuous-time change-of-variables yields an explicit log-density, namely \eqref{eq:app_prior_log_density}. Second, vector-field regression in Euclidean space allows the use of Flow Matching to learn prior geometry directly. Third, continuous latent variables naturally support interpolation, gradient guidance, and representation editing.

Accordingly, the meaning of \emph{why continuous} is not that continuous representations are inherently superior to discrete ones. Rather, compared with a discrete token-space corruption model, continuous latent space permits a more direct description of probability geometry and reversible transport. Whether this yields actual gains still depends on whether the continuous structure aligns with the low-dimensional semantic organization of the data.

\subsection{Difference from Plaid: Why Latent}
\label{app:why_latent}

Plaid is also a continuous method, but it makes continuous a token-aligned representation that stays close to the original sequence, rather than an explicit semantic latent variable. A typical forward process can be written as
\begin{equation}
q(h_0\mid x)=\Ncal(\widetilde x,\sigma^2(0)I),
\qquad
q(h_t\mid h_s)=\Ncal\bigl(h_s,\sigma^2(t)-\sigma^2(s)\bigr),
\label{eq:app_plaid_forward}
\end{equation}
where $\widetilde x=\mathrm{Embed}(x)$. Hence, as the noise level approaches zero, the state of Plaid still remains tightly attached to the observation space. By contrast, \method explicitly separates global semantic representation from textual realization through
\begin{equation}
p(x)=\int p_{\theta}(x\mid z_0)p_{\psi}(z_0)\,\dd z_0.
\label{eq:app_px_independent_of_encoder}
\end{equation}

This distinction is not merely about where continuity is introduced; it is about how the model is factorized. If one further lets
\begin{equation}
q_{\phi}(z_0\mid x)=\Ncal(E(x),\sigma_0^2I),
\qquad
\sigma_0^2\to 0,
\label{eq:app_plaid_as_limit}
\end{equation}
and $E(x)$ is nearly invertible, then \method degenerates to the limiting case where $z_0$ is almost just a continuousized representation of the original sequence. This limit clarifies the structural relation: Plaid is closer to the near-observation, non-compressive limit of \method, while the genuinely new ingredient of \method is the latent decomposition itself.

The mathematical core of this difference is captured by the information decomposition of the average ELBO in \eqref{eq:app_avg_elbo_decomp}. That decomposition shows that representation learning in \method splits into three parts: the conditional reconstruction term controls textual realization given the latent, the mutual information $I_q(X;Z_0)$ controls how much global information about the text is retained, and the KL between the aggregated posterior and the prior controls the difficulty of prior matching. Precisely because this explicit latent variable is introduced, the model can encode global planning, topic, and style constraints into $z_0$, while delegating local token realization to the decoder.

Therefore, the strict answer to \emph{why latent} is that a latent variable is not merely a continuous relaxation of a discrete object. It is an intermediate random variable that explicitly participates in marginalization, allowing the text generation problem to be decomposed into prior modeling plus conditional realization. This decomposition yields statistical benefits only when the data indeed contain low-rate but high-value semantic structure; this issue is analyzed further in Appendix~\ref{app:advantages}.

\subsection{Why Cola DLM Is Not Merely a Latent Reverse-Diffusion LM}
\label{app:not_simple_latent_diffusion}

To avoid conceptual confusion, we now make the distinction more precise. Suppose a narrow latent diffusion LM is defined as follows: one first specifies a data-conditioned forward noising process on a latent state,
\begin{equation}
q(z_{0:T}\mid x),
\label{eq:app_narrow_latent_diffusion_def}
\end{equation}
and then learns its reverse process to gradually recover a latent state used for text generation. The core structure of such a model is therefore the reversal of a data-conditioned forward corruption process.

Under this narrow definition, \method is not simply a latent reverse-diffusion LM. First, the continuous path in \method lies in prior transport:
\begin{equation}
z_1\sim p_1(z_1),
\qquad
z_0=\Phi^{\psi}_{0\leftarrow 1}(z_1),
\qquad
x\sim p_{\theta}(x\mid z_0).
\label{eq:app_ldlm_prior_path_again}
\end{equation}
This path first describes how to generate a semantic prior sample in latent space from a simple Gaussian. During generation, it does not depend on any specific observation $x$. By contrast, the forward paths in Plaid and LLaDA describe how a given sample $x$ is perturbed into a high-noise state. Thus, although all three can be described as path models, the path in \method is prior transport, whereas the paths in Plaid and LLaDA are data corruption.

Second, in \method the encoder appears only in variational inference:
\begin{equation}
\log p(x)
\ge
\E_{q_{\phi}(z_0\mid x)}\Bigl[
\log p_{\theta}(x\mid z_0)+\log p_{\psi}(z_0)-\log q_{\phi}(z_0\mid x)
\Bigr].
\label{eq:app_encoder_only_in_elbo}
\end{equation}
It does not participate in the generation procedure \eqref{eq:app_ldlm_prior_path_again}, and is therefore part of the variational posterior. In Plaid and LLaDA, however, the forward process $q$ is part of the model definition because the reverse model is trained against that corruption law. In short, the former $q$ belongs to inference, whereas the latter $q$ belongs to model construction.

Finally, the marginal distribution of \method is always defined by the hierarchical latent-variable identity
\begin{equation}
p(x)=\int p_{\theta}(x\mid z_0)p_{\psi}(z_0)\,\dd z_0.
\label{eq:app_ldlm_hierarchical_identity}
\end{equation}
Strictly speaking, \method is therefore first and foremost a hierarchical latent-variable language model with a CNF prior. Flow or diffusion is only a way to make the prior family more expressive, not the entirety of the model itself.

\section{Theoretical Advantages, Applicability, and Limitations}
\label{app:advantages}

We finally compare \method with AR, LLaDA, and Plaid under a unified theoretical criterion. The goal is not to argue that a more complex model should be stronger, but to ask under what conditions \method is expected to be better, and under what conditions it is expected to be worse.

\subsection{A Unified Population-Level Comparison Criterion}
\label{app:population_risk}

For any model class $\mathcal M$, define its irreducible approximation error with respect to the true data distribution $\pdata(x)$ as
\begin{equation}
\mathcal E(\mathcal M)
:=
\inf_{p\in\mathcal M}\KL\bigl(\pdata(x)\,\|\,p(x)\bigr).
\label{eq:app_model_class_error}
\end{equation}
This quantity measures the best unavoidable mismatch of the model class even after ignoring finite-sample effects and optimization error.

For any model trained by exact negative log-likelihood, one has the standard identity
\begin{equation}
\inf_{p\in\mathcal M}\E_{\pdata(x)}[-\log p(x)]
=
H(\pdata)+\mathcal E(\mathcal M),
\label{eq:app_population_risk_exact_nll}
\end{equation}
where $H(\pdata):=-\E_{\pdata(x)}\log \pdata(x)$ is the data entropy. Therefore, for exact-NLL models, comparing theoretical limits is equivalent to comparing $\mathcal E(\mathcal M)$.

AR belongs to this class, and its population risk is
\begin{equation}
\E_{\pdata(x)}[-\log p^{\mathrm{AR}}_{\eta}(x)]
=
H(\pdata)+\KL\bigl(\pdata\,\|\,p^{\mathrm{AR}}_{\eta}\bigr).
\label{eq:app_ar_population_risk}
\end{equation}
By contrast, \method is trained through the ELBO, so for each sample
\begin{equation}
-\Lcal_{\mathrm{ELBO}}(x)
=
-\log p_{\theta,\psi}(x)
+
\KL\bigl(q_{\phi}(z_0\mid x)\,\|\,p_{\theta,\psi}(z_0\mid x)\bigr).
\label{eq:app_ldlm_elbo_gap_per_sample}
\end{equation}
Taking expectation over the data distribution gives
\begin{align}
\E_{\pdata(x)}[-\Lcal_{\mathrm{ELBO}}(x)]
&=
H(\pdata)
+\KL\bigl(\pdata\,\|\,p_{\theta,\psi}\bigr)
+\mathcal G^{\mathrm{infer}}_{\mathrm{Cola\ DLM}}(\phi,\theta,\psi),
\label{eq:app_ldlm_population_risk}
\end{align}
where
\begin{equation}
\mathcal G^{\mathrm{infer}}_{\mathrm{Cola\ DLM}}(\phi,\theta,\psi)
:=
\E_{\pdata(x)}\KL\bigl(q_{\phi}(z_0\mid x)\,\|\,p_{\theta,\psi}(z_0\mid x)\bigr)
\ge 0
\label{eq:app_ldlm_inference_gap}
\end{equation}
is the inference gap.

Similarly, the variational objective of Plaid can be written as true NLL plus a reverse/path-posterior gap, and the mask-denoising objective of LLaDA can be written as true token-space NLL plus a surrogate gap. Hence all four model classes can be unified as
\begin{equation}
\text{population risk}
=
H(\pdata)+\text{model mismatch}+\text{additional gap induced by the training objective}.
\label{eq:app_unified_population_risk}
\end{equation}
Therefore, theoretical superiority does not follow from model complexity alone. It depends on whether the total statistical burden is smaller.

Define the total statistical burden of \method as
\begin{equation}
\mathfrak R_{\mathrm{Cola\ DLM}}
:=
\mathcal E(\mathcal M_{\mathrm{Cola\ DLM}})
+
\inf_{\phi,\theta,\psi}\mathcal G^{\mathrm{infer}}_{\mathrm{Cola\ DLM}}(\phi,\theta,\psi),
\label{eq:app_ldlm_total_burden}
\end{equation}
and define analogous quantities for AR, Plaid, and LLaDA. Then, at the population level, \method is better than a comparison class if and only if its total statistical burden is smaller. For example,
\begin{equation}
\mathfrak R_{\mathrm{Cola\ DLM}}<\mathfrak R_{\mathrm{AR}}
\iff
\text{\method is better than AR at the population level}.
\label{eq:app_ldlm_better_than_ar}
\end{equation}
This gives a unified and rigorous comparison criterion.

\subsection{Rate-Distortion and the Applicability Boundary of the Latent Bottleneck}
\label{app:rate_distortion}

The unified criterion still leaves a deeper question unanswered: why can latent decomposition reduce model mismatch in some cases but hurt modeling in others? The answer comes from conditional entropy, mutual information, and rate-distortion structure.

For any encoder $q_{\phi}(z_0\mid x)$ and decoder $p_{\theta}(x\mid z_0)$, under the joint distribution $q(x,z_0)=\pdata(x)q_{\phi}(z_0\mid x)$ one always has
\begin{equation}
\E_{q(x,z_0)}[-\log p_{\theta}(x\mid z_0)]
\ge
H_q(X\mid Z_0).
\label{eq:app_reconstruction_lower_bound}
\end{equation}
By the mutual-information identity
\begin{equation}
H_q(X)=H_q(X\mid Z_0)+I_q(X;Z_0),
\label{eq:app_entropy_mi_identity}
\end{equation}
together with $q(x)=\pdata(x)$, it follows that
\begin{equation}
H_q(X\mid Z_0)=H(\pdata)-I_q(X;Z_0).
\label{eq:app_conditional_entropy_as_entropy_minus_mi}
\end{equation}
Therefore, the reconstruction cost satisfies
\begin{equation}
\E_{q(x,z_0)}[-\log p_{\theta}(x\mid z_0)]
\ge
H(\pdata)-I_q(X;Z_0).
\label{eq:app_reconstruction_bound_by_mi}
\end{equation}
This shows that if $I_q(X;Z_0)$ is small, then the latent retains little information about the text and the lower bound on the best achievable reconstruction loss rises. If $I_q(X;Z_0)$ is large enough, then the decoder can realize the text more easily given the latent.

To formalize this further, define the representation rate-distortion function
\begin{equation}
\mathcal D(R)
:=
\inf_{q(z_0\mid x):\,I_q(X;Z_0)\le R}
\ \inf_{p_{\theta}(x\mid z_0)}
\E_{q(x,z_0)}[-\log p_{\theta}(x\mid z_0)].
\label{eq:app_representation_rate_distortion}
\end{equation}
It represents the minimum achievable average reconstruction cost when the latent channel is allowed to transmit at most $R$ nats of information about the text.

If, for a given data distribution, $\mathcal D(R)$ falls quickly to a low level already at small $R$, then the data admit a low-rate but high-value latent representation, and a latent bottleneck is more likely to help. If substantial reduction in reconstruction error occurs only when $R$ approaches $H(X)$, then the data are nearly incompressible and a latent bottleneck is more likely to become an additional burden.

Therefore, whether \method benefits from the bottleneck is ultimately governed by the rate-distortion curve of the data itself. If global topics, plans, style, and constraints occupy a relatively low information rate yet strongly determine the final realization, then encoding them into $z_0$ is efficient. If nearly every token-level detail is irreducible, then latent compression does not bring genuine statistical benefit.

\subsection{Structured Generation Regime: When Cola DLM Better Matches the True Mechanism}
\label{app:structured_generation}

To formalize the above intuition one step further, consider the following structured generation assumption: there exists a global variable $G$ such that the true data distribution can be written as
\begin{equation}
\pdata(x)=\int p^{\star}(x\mid g)\,p^{\star}(g)\,\dd g,
\label{eq:app_global_local_factorization}
\end{equation}
where $G$ determines topic, plan, discourse structure, style, or other global constraints, and $p^{\star}(x\mid g)$ realizes these global factors into concrete text. If, in addition,
\begin{equation}
H(X\mid G)\ll H(X),
\qquad
\dim(G)\ll \dim(E(X)),
\label{eq:app_good_latent_regime}
\end{equation}
then the data are in a favorable regime dominated by low-dimensional but high-value latent factors.

In this regime, the inductive bias of \method is better aligned with the true generation mechanism. The encoder recovers posterior information about $G$ from $x$, the prior learns a global semantic distribution close to $p^{\star}(g)$, and the decoder learns textual realization conditioned on the global factor. In other words, the potential advantage of \method does not come from being more complicated, but from splitting a hard problem into two better-matched subproblems:
\begin{equation}
\text{learning }\pdata(x)
\quad\leadsto\quad
\text{learning }p^{\star}(g)
\text{ and }
p^{\star}(x\mid g).
\label{eq:app_global_local_problem_split}
\end{equation}

If the encoder can recover information about $G$ with a small inference gap, $p_{\psi}(z_0)$ can approximate the aggregated latent distribution corresponding to $G$, and $p_{\theta}(x\mid z_0)$ can approximate the conditional realization well, then the statistical burden of \method is determined primarily by the modeling complexity of the low-dimensional global factor rather than that of the full surface sequence. This is the fundamental source of its possible advantage over AR, LLaDA, and Plaid.

\subsection{Three Curves, Theoretical Advantages, and Limitations}
\label{app:three_curves}

In the opposite regime, where \eqref{eq:app_good_latent_regime} does not hold, \method may systematically underperform more direct modeling routes. Typical unfavorable cases include the absence of a low-rate sufficient representation, settings where local surface form itself carries the main semantics, or regimes where a strong token-space model is already adequate. In such cases, the additional costs of \method mainly come from three sources:
\begin{equation}
\begin{aligned}
\text{(i)}\;& \text{inference gap } \mathcal G^{\mathrm{infer}}_{\mathrm{Cola\ DLM}}, \\
\text{(ii)}\;& \text{the elevated reconstruction lower bound } H(X \mid Z_0) \\
&\text{induced by the bottleneck}, \\
\text{(iii)}\;& \text{the optimization complexity of jointly training encoder,} \\
&\text{prior, and decoder}.
\end{aligned}
\label{eq:app_ldlm_three_costs}
\end{equation}
Hence, the success or failure of \method can be reduced to a competition among three curves: the representation rate-distortion curve $\mathcal D(R)$, the prior-approximation curve, and the inference-gap curve $\mathcal G^{\mathrm{infer}}_{\mathrm{Cola\ DLM}}$. Only when all three favor \method does latent decomposition become a true theoretical advantage.

These conclusions can be summarized more directly as follows.

\textbf{Theoretical Advantages.}
Relative to AR, if the decoder is non-autoregressive, then \method is not constrained by a single left-to-right ordering and can determine global semantics in latent space before realizing text, which gives it a potentially more natural inductive bias for infilling, constraint satisfaction, and global planning. Relative to discrete DLMs, \method has the geometric structure of continuous latent space and therefore supports explicit log-densities, vector-field regression, interpolation, and latent editing. Relative to Plaid, \method explicitly introduces semantic compression, allowing the model to learn a low-dimensional global representation. If the data indeed exhibit a hierarchical structure of low-dimensional global semantics and high-dimensional local realization, this inductive bias is closer to the true mechanism. Finally, the information decomposition of the average ELBO makes it natural to analyze the representation, prior, and decoder as separate components.

\textbf{Theoretical Limitations.}
\method inherently suffers from a variational gap:
\begin{equation}
\log p(x)-\Lcal_{\mathrm{ELBO}}(x)
=
\KL\bigl(q_{\phi}(z_0\mid x)\,\|\,p(z_0\mid x)\bigr).
\label{eq:app_variational_gap}
\end{equation}
If the encoder is poor, the model pays a systematic likelihood penalty. The latent bottleneck also introduces the risk of posterior collapse: if the decoder is too strong, the latent dimension is too small, or the regularization is too strong, then $I_q(X;Z_0)$ may become too small and the encoder may stop transmitting useful information. In addition, accurate evaluation of $\log p_{\psi}(z_0)$ often requires ODE integration and divergence estimation, making evaluation more expensive than in AR models. Importantly, the potential advantage of \method does not follow automatically from using latent variables and flows. It relies on the inductive assumption that the data contain a favorable low-dimensional semantic latent structure. This requires experiments to verify.

\section{Existence of Global Semantic Structures}
\label{sec:global_semantic_structure}

The goal of this section is to establish a falsifiable null hypothesis and, based on it, show that if the latent representation consists only of purely local and separable components, then the optimal timestep shift should not exhibit a systematic drift as the latent dimension varies. We then compare the experimentally observed shift drift against this null hypothesis, thereby obtaining a more robust conclusion regarding the existence of shared structures. To avoid conflating conclusions that can be established rigorously with structural models introduced for interpretation, we first formulate an abstract null hypothesis.

\begin{assumption}[Null hypothesis of purely separable representations]
\label{ass:global_semantic_null}
Let $d$ denote the latent dimension, let $\delta$ denote the timestep shift, and let $\mathcal J_d(\delta)$ denote the overall performance objective under a fixed semantic metric at latent dimension $d$. Assume that the following conditions hold:
\begin{enumerate}
    \item \textbf{Local separability:} there is no cross-dimensional joint structure in the latent representation that affects semantics, and the overall semantic performance is obtained by independently aggregating the contributions of local dimensions;
    \item \textbf{Homogeneity:} each local dimension shares the same optimal temporal scale with respect to shift; changing $d$ only changes the number of independent dimensions and does not alter the functional form of each local subproblem;
    \item \textbf{Additive aggregation:} the overall objective can be written as
    \begin{equation}
    \mathcal J_d(\delta)=\sum_{i=1}^{d} j_i(\delta)
    \qquad\text{or}\qquad
    \mathcal J_d(\delta)=\frac{1}{d}\sum_{i=1}^{d} j_i(\delta),
    \label{eq:separable_objective}
    \end{equation}
    where all $j_i$ attain their optimum at the same point. More strongly, if homogeneity holds exactly, then there exists a common function $j$ and constants $a_d>0,b_d$ such that
    \begin{equation}
    \mathcal J_d(\delta)=a_d\,j(\delta)+b_d.
    \label{eq:affine_common_profile}
    \end{equation}
\end{enumerate}
\end{assumption}

Assumption~\ref{ass:global_semantic_null} does not presuppose any specific network architecture. Rather, it describes a theoretical limit in which semantics is carried entirely by mutually independent local dimensions. Under this limit, changing the latent dimension can at most modify the scale of the overall objective, but should not change the location of the optimal shift.

\begin{proposition}[Dimension-invariance of the optimal shift under the separable null]
\label{prop:shift_invariance_under_separable_null}
Under Assumption~\ref{ass:global_semantic_null}, let
\[
\delta_d^{\star}\in\arg\max_{\delta}\mathcal J_d(\delta)
\]
denote the optimal timestep shift at latent dimension $d$. Then for any two dimensions $d_1,d_2$, we have
\begin{equation}
\arg\max_{\delta}\mathcal J_{d_1}(\delta)
=
\arg\max_{\delta}\mathcal J_{d_2}(\delta).
\label{eq:shift_invariance_result}
\end{equation}
In particular, if the maximizer is unique, then
\begin{equation}
\delta_{d_1}^{\star}=\delta_{d_2}^{\star}.
\label{eq:shift_unique_invariance}
\end{equation}
Therefore, under purely local and separable representations, the optimal timestep shift should not drift systematically with the latent dimension.
\end{proposition}

\begin{proof}
By \eqref{eq:affine_common_profile}, for any $d$,
\[
\mathcal J_d(\delta)=a_d\,j(\delta)+b_d,
\qquad a_d>0.
\]
Since $b_d$ is independent of $\delta$ and $a_d$ only induces a positive rescaling, it follows that
\[
\arg\max_{\delta}\mathcal J_d(\delta)
=
\arg\max_{\delta}j(\delta).
\]
Hence the optimal point is independent of $d$, which proves \eqref{eq:shift_invariance_result} and \eqref{eq:shift_unique_invariance}.
\end{proof}

The key conclusion of Proposition~\ref{prop:shift_invariance_under_separable_null} is that what can be established rigorously is that if no shared or joint structure affecting semantics exists, then the optimal shift should not drift systematically with the latent dimension.

\begin{corollary}[Systematic drift refutes the purely separable null]
\label{cor:drift_refutes_separable_null}
If experiments show that, as the latent dimension $d$ varies, the optimal shift $\delta^{\star}(d)$ exhibits a stable, monotonic, and reproducible systematic drift, and if this drift cannot be explained by changes in parameter count, insufficient training, optimization instability, or evaluation noise, then Assumption~\ref{ass:global_semantic_null} does not hold.
\end{corollary}

\begin{proof}
By Proposition~\ref{prop:shift_invariance_under_separable_null}, if Assumption~\ref{ass:global_semantic_null} holds, then the optimal shift must be identical across all latent dimensions. If experiments repeatedly show that the optimal shift drifts systematically with $d$, then this contradicts the necessary consequence of the null hypothesis. Therefore, the null hypothesis is rejected.
\end{proof}

The boundary of the conclusion in Corollary~\ref{cor:drift_refutes_separable_null} is clear: it directly rules out the hypothesis that the latent space is purely local and separable. Consequently, the strictly derivable conclusion is that some cross-dimensional shared or joint structure must exist in the latent representation. If one further wishes to interpret this structure as a global semantic factor, additional modeling assumptions and experimental evidence are required.

\paragraph{An equivalent formulation from an information-theoretic perspective}

The above conclusion admits a consistent restatement in information-theoretic terms. Let $s$ denote a latent variable associated with high-level semantics, and let $z$ denote the VAE latent. Then the diffusion forward process can be written as
\begin{equation}
z_t=\alpha_t z+\sigma_t\epsilon,
\qquad
\epsilon\sim\mathcal N(0,I),
\label{eq:diffusion_forward_latent}
\end{equation}
where $\alpha_t$ and $\sigma_t$ are determined by the noise schedule. If we further decompose the latent as
\begin{equation}
z=\phi(s)+u,
\label{eq:latent_signal_noise_decomposition}
\end{equation}
where $\phi(s)$ denotes the effective signal induced by the semantic variable and $u$ denotes encoding residuals or posterior uncertainty, then the observation received by the DiT at timestep $t$ becomes
\begin{equation}
z_t=\alpha_t\phi(s)+\underbrace{\alpha_t u+\sigma_t\epsilon}_{\text{total noise}}.
\label{eq:effective_noisy_channel}
\end{equation}
Hence, what matters to the DiT is not the raw timestep itself, but rather how much effective information about $s$ is still preserved in the latent at that timestep.

Under Assumption~\ref{ass:global_semantic_null}, if semantics is carried entirely by local independent dimensions, namely
\[
s=(s_1,\dots,s_d),\qquad z=(z_1,\dots,z_d),
\]
and if the channels are independent and identically distributed across dimensions, then the additivity of mutual information yields
\begin{equation}
I(s;z_t)=\sum_{i=1}^{d} I(s_i;z_{t,i}).
\label{eq:mi_additive_separable}
\end{equation}
If each local channel has the same information curve with respect to shift, then changing $d$ only rescales the total mutual information linearly and does not change which shift places the effective information in the range most suitable for semantic recovery. From this information-theoretic perspective, the essence of Proposition~\ref{prop:shift_invariance_under_separable_null} is that, under purely separable representations, varying the dimension does not alter the shift corresponding to the optimal semantic information regime.

Conversely, if the optimal shift drifts systematically with $d$, then the purely additive information structure corresponding to \eqref{eq:mi_additive_separable} can no longer hold. This indicates that the effective semantic information in the latent is not merely the sum of a set of completely homogeneous local subchannels, but is more likely to depend on some shared structure that only emerges after cross-dimensional aggregation.

\paragraph{A shared-factor model and the interpretation of shift scaling laws}

To further explain why systematic drift may exhibit a pattern similar to the resolution-dependent shift in Stable Diffusion, we introduce a stronger but analytically convenient structural model. Suppose that each latent dimension can be written as
\begin{equation}
z_i=A_i g+\xi_i,
\qquad i=1,\dots,d,
\label{eq:shared_factor_model}
\end{equation}
where $g\in\mathbb R^r$ denotes a low-dimensional shared factor, $A_i$ is a linear observation matrix, and $\xi_i$ is local independent noise. Here, $g$ may be understood as an abstract shared semantic state, without assuming that it is a single scalar quantity.

In the simplest setting, assume that $A_i\equiv A$ and that $\xi_i\sim\mathcal N(0,\sigma_{\xi}^2 I)$ are independent across dimensions. Standard results from linear Gaussian models imply that, as the latent dimension $d$ increases, the posterior precision for $g$ grows linearly, while the posterior covariance shrinks as $d^{-1}$. Equivalently, the effective signal-to-noise ratio for recovering the shared factor satisfies
\begin{equation}
\mathrm{SNR}_{\mathrm{eff}}(d)\propto d.
\label{eq:effective_snr_shared_factor}
\end{equation}

From an information-theoretic viewpoint, this means that the recoverable information about the shared factor $g$ increases with $d$. More specifically, if the total effective noise variance at diffusion timestep $t$ is denoted by $\sigma_{\mathrm{eff}}^2(t)$, then the mutual information about $g$ can be written as
\begin{equation}
I(g;z_t)
=
\frac{1}{2}\log\det\!\Bigl(
I+\alpha_t^2\Sigma_g^{1/2}A^\top
\bigl(\sigma_{\mathrm{eff}}^2(t)I\bigr)^{-1}
A\Sigma_g^{1/2}
\Bigr).
\label{eq:mi_shared_factor_general}
\end{equation}
Under a scalar or isotropic approximation, this further reduces to
\begin{equation}
I(g;z_t)
\approx
\frac{r}{2}\log\!\Bigl(1+c\cdot d\cdot \mathrm{SNR}_{\mathrm{eff}}(t)\Bigr),
\label{eq:mi_shared_factor_scaling}
\end{equation}
where the constant $c>0$ is determined by the observation matrix and the prior covariance of the shared factor.

If the effect of a shift $\delta$ is equivalent to translating the log-SNR of the denoising process, and if semantic recovery performance is primarily determined by whether the effective log-SNR of the shared factor falls within an appropriate range, then the shift required to maintain optimal semantic recovery must compensate for the dimension dependence induced by \eqref{eq:effective_snr_shared_factor} or \eqref{eq:mi_shared_factor_scaling}. This leads to a logarithmic relation of the form
\begin{equation}
\delta^{\star}(d)=a\log d+b,
\label{eq:log_scaling_shift_shared_factor}
\end{equation}
where the sign of $a$ depends on the specific time parameterization and the convention used to define the shift.

\begin{remark}[Structural homology with resolution-dependent shift]
\label{rem:resolution_shift_homology}
The significance of \eqref{eq:log_scaling_shift_shared_factor} is not that it uniquely identifies a true global variable, but rather that once a shared factor is repeatedly observed across multiple latent dimensions, the mechanism
\[
\text{increasing dimension} \;\Rightarrow\; \text{stronger effective SNR} \;\Rightarrow\; \text{logarithmic compensation in the optimal shift}
\]
arises naturally. This is structurally homologous to the intuition underlying resolution-dependent timestep shift in Stable Diffusion. The difference is that, in the present case, what is shared is a joint structure within the latent representation, whereas in the image case the shared content comes from more globally distributed information under higher resolution.
\end{remark}

It should be emphasized that the main role of the shared-factor model is to explain the form of the systematic drift, rather than to replace the rigorous contradiction argument above. The latter answers which hypotheses must fail if drift is observed, whereas the former explains what kind of structure such drift is most consistent with.

\paragraph{Relation to VAE logSNR}

The above analysis also explains another empirical phenomenon: even when the latent dimension is fixed, the optimal timestep shift may still change when the logSNR of the VAE posterior is altered. The reason is that the VAE is not a noiseless encoder. In \eqref{eq:latent_signal_noise_decomposition}, the term $u$ precisely represents posterior uncertainty. If the timestep $t$ is fixed, then the total noise covariance with respect to the semantic variable is
\begin{equation}
\Sigma_{\mathrm{noise}}(t)
=
\alpha_t^2\Sigma_u+\sigma_t^2 I.
\label{eq:effective_noise_covariance}
\end{equation}
When the VAE logSNR decreases, this typically implies that the posterior variance increases, that is, $\Sigma_u$ becomes larger. Consequently, at the same raw timestep, the effective semantic information preserved in the latent,
\begin{equation}
I(s;z_t)
=
\frac{1}{2}
\log\det\!\Bigl(
I+\alpha_t^2\Sigma_s
(\alpha_t^2\Sigma_u+\sigma_t^2 I)^{-1}
\Bigr),
\label{eq:semantic_mi_general}
\end{equation}
decreases overall. In other words, in a smoother latent space with larger posterior variance, the same diffusion timestep corresponds to a lower effective semantic SNR. Therefore, in order for training and sampling to remain within a similar optimal semantic information regime, the timestep shift must also be recalibrated.

This shows that changing the latent dimension and changing the VAE logSNR are, on the surface, two different operations, but they act on the same core object from an information-theoretic perspective, namely the effective mutual information curve $I(s;z_t)$ of the semantic variable along diffusion time. The former mainly changes how shared information is aggregated and amplified across dimensions, whereas the latter mainly changes the effective noise floor at each timestep. Both can cause the optimal information regime to drift along the raw timestep axis.

\paragraph{Why this further points to semantic structure}

Corollary~\ref{cor:drift_refutes_separable_null} alone only implies the existence of shared or joint structure; it does not yet establish that this structure necessarily carries semantic information. To strengthen the conclusion to shared structure that affects semantics, one needs task-level selective evidence. A particularly informative situation is the following:
\begin{enumerate}
    \item reconstruction quality remains relatively stable across different shifts;
    \item semantic downstream metrics are highly sensitive to the shift;
    \item and the optimal shift exhibits systematic drift with the latent dimension.
\end{enumerate}

If these three conditions hold simultaneously, then it is no longer plausible to explain the phenomenon merely by saying that the shift changes the difficulty of local token realization. A more reasonable interpretation is that the shift is adjusting the noise calibration of the latent prior with respect to some shared high-level structure, and that this shared structure is precisely involved in semantic organization.

From an information-theoretic standpoint, this means that changing the shift does not substantially alter the decoder's ability to recover local details, but does substantially change the preservation regime of the shared information required by high-level tasks during the diffusion process. Therefore, what is being recalibrated is more likely not local reconstruction error, but the semantic component of mutual information at a higher level.

In other words, if reconstruction remains consistently strong while semantic metrics vary substantially, then the main quantity being affected is not the decoder's local realization capability, but rather the alignment of the latent prior with high-level shared information. Combined with the contradiction result above, this yields a stronger yet still robust conclusion: the latent space contains shared or joint structures that affect semantics.

\begin{remark}[Conclusions that are rigorously supported in this section]
\label{rem:final_supported_claim}
Under a broad class of purely local and separable latent hypotheses, the optimal timestep shift should not drift systematically with the latent dimension. The stable drift observed in experiments refutes this null hypothesis and therefore implies the existence of cross-dimensional shared or joint structures in the latent space. Furthermore, if this drift is manifested primarily in semantic metrics rather than reconstruction metrics, then one can further conclude that such shared or joint structures indeed participate in high-level semantic organization.
\end{remark}

These conclusions are sufficient to support the central hypothesis of this work: if such global or semi-global semantic shared structures do exist in the latent space, then calibrating the noise schedule at the level of the latent prior may offer a more substantial advantage than methods based purely on local modeling.

\section{Theoretical Analysis of Likelihood Estimation Bias}
\label{sec:likelihood_mismatch_appendix}

This section provides a more complete theoretical analysis of the phenomenon discussed in Section~\ref{dis:ppl} of the main text. The central question is why \method can already generate reasonably high-quality text in free-generation settings, while its PPL / conditional score estimated via the ELBO or a plug-in conditional estimator remains substantially inflated. We argue that this is not an incidental engineering artifact, but rather a structural consequence of the misalignment among the training objective, the underlying probabilistic object, and the local geometry in continuous latent generative models.

To avoid conflating the conditional marginal probability of an entire response with the local prior density evaluated at a particular latent point, we explicitly distinguish the following two objects throughout. Let
\[
x=(x^{\mathrm{pre}},x^{\mathrm{res}}),
\]
where $x^{\mathrm{pre}}$ denotes the prefix / context and $x^{\mathrm{res}}$ denotes the response, and let $c$ denote the conditional information induced by the prefix. Then the exact conditional marginal probability is
\begin{equation}
p(x^{\mathrm{res}}\mid c)
=
\int p_{\theta}(x^{\mathrm{res}}\mid z,c)\,p_{\psi}(z\mid c)\,\dd z,
\label{eq:resp_conditional_marginal_appendix}
\end{equation}
whereas the local conditional score corresponding to the response-level score in the main text is
\begin{equation}
\mathcal S_{\mathrm{resp}}(x)
=
\E_{q_{\phi}(z\mid x,c)}
\Big[
\log p_{\theta}(x^{\mathrm{res}}\mid z,c)
+\log p_{\psi}(z\mid c)
-\log q_{\phi}(z\mid x,c)
\Big].
\label{eq:resp_conditional_score_appendix}
\end{equation}
Equation~\eqref{eq:resp_conditional_marginal_appendix} is the exact conditional marginal likelihood, whereas Equation~\eqref{eq:resp_conditional_score_appendix} is a local estimator that is more directly accessible in analysis. The gap between them is precisely the probabilistic foundation of the discussion that follows.

\subsection{Why the Flow Matching Training Objective Is Naturally Misaligned with Conditional PPL}

\subsubsection{Flow Matching optimizes vector-field regression rather than gold-specific local log-density}

With the encoder and decoder fixed, the prior-learning subproblem can be understood as fitting $p_{\psi}(z\mid c)$ to the conditional latent distribution induced jointly by the data and the posterior. In practice, however, Flow Matching does not directly maximize the conditional log-density at a particular gold latent point. Instead, it learns a time-dependent vector field $v_{\psi}(z,t,c)$. Its typical objective can be written as
\begin{equation}
\mathcal L_{\mathrm{FM}}(\psi)
=
\E\Big[
\|v_{\psi}(Z_t,t,c)-U^{\star}\|^2
\Big],
\label{eq:flow_matching_objective_appendix}
\end{equation}
where $U^{\star}$ is the target velocity field induced by a bridging path or a reference flow. Equation~\eqref{eq:flow_matching_objective_appendix} shows that Flow Matching directly regresses a dynamical system, rather than the local conditional log-density at a specific gold latent point.

More importantly, the pointwise optimum of the squared loss is not the most likely velocity field in a mode-seeking sense, but the conditional mean velocity field.

\begin{proposition}[The pointwise optimum of Flow Matching is the conditional mean velocity]
\label{prop:fm_conditional_mean_appendix}
For fixed $(z,t,c)$, define
\[
Y:=U^{\star}\mid (Z_t=z,t,c),
\qquad
f:=v_{\psi}(z,t,c).
\]
Then the conditional risk
\begin{equation}
\mathcal R(f;z,t,c)
:=
\E\bigl[\|f-Y\|^2\mid Z_t=z,t,c\bigr]
\label{eq:fm_conditional_risk_appendix}
\end{equation}
admits the unique minimizer
\begin{equation}
f^{\star}(z,t,c)=\E[Y\mid Z_t=z,t,c].
\label{eq:fm_conditional_mean_appendix}
\end{equation}
\end{proposition}

\begin{proof}
For any $f$, we have
\[
f-Y=(f-\E[Y\mid Z_t=z,t,c])-(Y-\E[Y\mid Z_t=z,t,c]).
\]
Expanding the square, taking the conditional expectation, and using the fact that the cross term vanishes, we obtain
\[
\mathcal R(f;z,t,c)
=
\|f-\E[Y\mid Z_t=z,t,c]\|^2
+\E\bigl[\|Y-\E[Y\mid Z_t=z,t,c]\|^2\mid Z_t=z,t,c\bigr].
\]
The second term is independent of $f$, so the unique optimum is attained at
\[
f=\E[Y\mid Z_t=z,t,c].
\]
\end{proof}

The meaning of Proposition~\ref{prop:fm_conditional_mean_appendix} is direct: \textbf{Flow Matching learns the conditional mean dynamics, rather than a gold-specific local calibration specialized to the posterior tube of a particular response.} As long as the conditional response distribution under a given context is itself multimodal or broad-peaked, Flow Matching is more naturally biased toward learning an average transport map that moves total probability mass into a reasonable region, rather than automatically guaranteeing local density calibration around a particular sample.

\subsubsection{Why the prior mean may be far from the ground truth while generation remains good}

This point can be formalized under multimodal conditional distributions. For a fixed context $c$, suppose
\begin{equation}
p_{\psi}(z\mid c)
=
\sum_{m=1}^{M}\pi_m(c)\,\Ncal(\mu_m,\Sigma_m),
\qquad
\sum_{m=1}^{M}\pi_m(c)=1,
\label{eq:prior_mixture_appendix}
\end{equation}
and for a particular sample, its posterior is mainly concentrated on the $m^{\star}$-th mode, namely
\begin{equation}
q_{\phi}(z\mid x,c)\approx \Ncal(\mu_{m^{\star}},\Sigma_q).
\label{eq:posterior_single_mode_appendix}
\end{equation}
Then the global mean of the prior is
\begin{equation}
\bar\mu_p(c)=\sum_{m=1}^{M}\pi_m(c)\mu_m.
\label{eq:prior_global_mean_appendix}
\end{equation}

\begin{corollary}[Mean displacement under multimodal conditional distributions]
\label{cor:multimodal_mean_gap_appendix}
Under the setting of Equations~\eqref{eq:prior_mixture_appendix}--\eqref{eq:prior_global_mean_appendix},
\begin{equation}
\|\bar\mu_p(c)-\mu_{m^{\star}}\|
\le
\sum_{m\neq m^{\star}}\pi_m(c)\,\|\mu_m-\mu_{m^{\star}}\|.
\label{eq:multimodal_mean_gap_bound_appendix}
\end{equation}
If the modes are well separated, then the prior mean may deviate substantially from the latent of a particular ground-truth sample. Nevertheless, as long as the major mass of these modes lies within a decoder-good region, free-generation quality can still remain good.
\end{corollary}

\begin{proof}
By
\[
\bar\mu_p(c)-\mu_{m^{\star}}
=
\sum_{m=1}^{M}\pi_m(c)(\mu_m-\mu_{m^{\star}})
=
\sum_{m\neq m^{\star}}\pi_m(c)(\mu_m-\mu_{m^{\star}})
\]
and the triangle inequality, the result follows immediately.
\end{proof}

Therefore, a large distance between the prior mean and the ground-truth latent does not automatically imply that the prior has failed to learn the distribution. A more plausible explanation is often that \textbf{the context-level conditional distribution is itself multimodal, while the posterior selects only one narrower gold-specific mode among them.}

\subsection{Why This Further Leads to Good Generation but Poor PPL}

\subsubsection{Generation quality and conditional PPL correspond to two different geometric objects}

Fix a context $c$, let $r(x^{\mathrm{res}};c)$ denote some external generation-quality function, and let $\tau$ be a threshold. Define the decoder-good region by
\begin{equation}
A_{\mathrm{good}}^{\tau}(c)
:=
\left\{
z:
\E_{x^{\mathrm{res}}\sim p_{\theta}(\cdot\mid z,c)}
\bigl[r(x^{\mathrm{res}};c)\bigr]\ge \tau
\right\},
\label{eq:good_region_appendix}
\end{equation}
and for a particular sample $x$ and a constant $\rho>0$, define its posterior typical set by
\begin{equation}
A_x^{\rho}
:=
\left\{
z:\ q_{\phi}(z\mid x,c)\ge \rho
\right\}.
\label{eq:gold_tube_appendix}
\end{equation}
If $A_x^{\rho}$ is narrow and sharp, we may refer to it as the \textbf{gold tube} of that sample. In general, $A_x^{\rho}$ is only a more local and sharper subset inside $A_{\mathrm{good}}^{\tau}(c)$, namely
\begin{equation}
A_x^{\rho}\subseteq A_{\mathrm{good}}^{\tau}(c)
\qquad\text{or at least}\qquad
A_x^{\rho}\ \text{is a more local subregion inside}\ A_{\mathrm{good}}^{\tau}(c).
\label{eq:gold_tube_inside_good_region_appendix}
\end{equation}

This distinction shows that good generation and good conditional PPL are not the same requirement.

\begin{proposition}[Good generation and poor conditional score can hold simultaneously]
\label{prop:good_gen_bad_score_appendix}
Fix a context $c$. Suppose there exist constants $\alpha,\Delta>0$ and thresholds $\tau,\rho$ such that
\begin{align}
p_{\psi}\bigl(A_{\mathrm{good}}^{\tau}(c)\mid c\bigr) &\ge \alpha,
\label{eq:coverage_condition_appendix}
\\
q_{\phi}(A_x^{\rho}\mid x,c) &\approx 1,
\label{eq:tube_mass_condition_appendix}
\\
\log q_{\phi}(z\mid x,c)-\log p_{\psi}(z\mid c) &\ge \Delta,
\qquad \forall z\in A_x^{\rho},
\label{eq:local_gap_condition_appendix}
\end{align}
and the decoder reconstruction term is uniformly bounded on $A_x^{\rho}$:
\begin{equation}
\log p_{\theta}(x^{\mathrm{res}}\mid z,c)\le B,
\qquad \forall z\in A_x^{\rho}.
\label{eq:decoder_bound_appendix}
\end{equation}
Then:
\begin{enumerate}
    \item under free sampling from the prior, at least an $\alpha$ fraction of the mass falls into the decoder-good region, so generation quality may already be good;
    \item but the conditional score satisfies
    \begin{equation}
    \mathcal S_{\mathrm{resp}}(x)\le B-\Delta+o(1),
    \label{eq:conditional_score_upper_bound_appendix}
    \end{equation}
    and therefore the conditional ELBO or plug-in PPL may still be substantially biased whenever $\Delta$ is sufficiently large.
\end{enumerate}
\end{proposition}

\begin{proof}
The first claim follows directly from Equation~\eqref{eq:coverage_condition_appendix}. For the second claim, rewrite Equation~\eqref{eq:resp_conditional_score_appendix} as
\[
\mathcal S_{\mathrm{resp}}(x)
=
\E_{q_{\phi}}
\Big[
\log p_{\theta}(x^{\mathrm{res}}\mid z,c)
-\bigl(\log q_{\phi}(z\mid x,c)-\log p_{\psi}(z\mid c)\bigr)
\Big].
\]
By Equation~\eqref{eq:tube_mass_condition_appendix}, the mass of $q_{\phi}$ outside $A_x^{\rho}$ is $o(1)$. Combining this with Equations~\eqref{eq:local_gap_condition_appendix} and \eqref{eq:decoder_bound_appendix} yields
\[
\mathcal S_{\mathrm{resp}}(x)\le B-\Delta+o(1).
\]
\end{proof}

The essence of Proposition~\ref{prop:good_gen_bad_score_appendix} is that \textbf{generation quality mainly requires the prior mass to enter $A_{\mathrm{good}}^{\tau}(c)$, whereas conditional PPL additionally requires the prior to assign sufficiently high local density on the posterior tube of a particular gold response.} The former is a coverage requirement, whereas the latter is a calibration requirement.

\subsubsection{Why good reconstruction does not imply good PPL}

Equation~\eqref{eq:resp_conditional_score_appendix} can also be written as
\begin{equation}
\mathcal S_{\mathrm{resp}}(x)
=
\underbrace{
\E_{q_{\phi}(z\mid x,c)}[\log p_{\theta}(x^{\mathrm{res}}\mid z,c)]
}_{R(x;c)}
-
\underbrace{
\KL\bigl(q_{\phi}(z\mid x,c)\,\|\,p_{\psi}(z\mid c)\bigr)
}_{G(x;c)}.
\label{eq:recon_minus_gap_appendix}
\end{equation}

\begin{proposition}[Good reconstruction does not imply good conditional PPL]
\label{prop:reconstruction_not_enough_appendix}
If there exists a sequence of samples $\{x_n\}$ such that
\[
R(x_n;c_n)\to R_{\max},
\]
while at the same time
\[
G(x_n;c_n)\to +\infty
\qquad\text{or remains significantly positive},
\]
then the corresponding conditional score may still be poor.
\end{proposition}

\begin{proof}
Equation~\eqref{eq:recon_minus_gap_appendix} directly gives
\[
\mathcal S_{\mathrm{resp}}(x_n)=R(x_n;c_n)-G(x_n;c_n).
\]
Therefore, as long as the posterior--prior gap remains large, the conditional score can still be low even when reconstruction is already strong.
\end{proof}

This shows that, in continuous latent language models, good reconstruction near the ground truth is not sufficient to imply good PPL. What truly determines PPL also includes the degree of local matching between the posterior and the prior.

\subsection{Even If the Prior Mean Is Close to the Ground Truth, PPL May Still Be Poor}

The previous subsection clarified the distinction between large-region coverage and gold-tube calibration. However, even if the problem is further restricted to the case where the center of the prior is already approximately aligned with the target, PPL may still remain poor. The reason is that \textbf{local covariance, orientation, curvature, and volume terms also enter the conditional score.}

\subsubsection{Under a local Gaussian approximation, center alignment is necessary but not sufficient}

Suppose that, in a neighborhood of a gold sample,
\begin{equation}
q_{\phi}(z\mid x,c)\approx \Ncal(\mu_q,\Sigma_q),
\qquad
p_{\psi}(z\mid c)\approx \Ncal(\mu_p,\Sigma_p).
\label{eq:local_gaussian_pair_appendix}
\end{equation}
Then the corresponding KL divergence is
\begin{equation}
\KL(q_{\phi}\|p_{\psi})
=
\frac12
\Big[
\operatorname{tr}(\Sigma_p^{-1}\Sigma_q)
+(\mu_p-\mu_q)^{\top}\Sigma_p^{-1}(\mu_p-\mu_q)
-d
+\log\frac{\det\Sigma_p}{\det\Sigma_q}
\Big].
\label{eq:gaussian_kl_appendix}
\end{equation}

\begin{proposition}[Center alignment alone is insufficient to guarantee good PPL]
\label{prop:center_alignment_not_sufficient_appendix}
Under Equation~\eqref{eq:local_gaussian_pair_appendix}, even if $\mu_p\approx\mu_q$, as long as $\Sigma_p$ and $\Sigma_q$ are mismatched in scale, orientation, or volume, $\KL(q_{\phi}\|p_{\psi})$ may still be substantially large, and hence the conditional score may still be poor.
\end{proposition}

\begin{proof}
Equation~\eqref{eq:gaussian_kl_appendix} shows that, beyond the center-discrepancy term, there are also the explicit scale- and volume-related terms
\[
\operatorname{tr}(\Sigma_p^{-1}\Sigma_q)
\quad\text{and}\quad
\log\frac{\det\Sigma_p}{\det\Sigma_q}.
\]
Even if the center-discrepancy term is small, the KL divergence can remain significantly positive whenever these latter terms are mismatched.
\end{proof}

Therefore, in continuous density models, it is incorrect to claim that it suffices to move the prior mean close to the ground-truth latent. PPL requires a stronger form of local probabilistic calibration: not only must the centers be close, but the local covariance, local volume, and local orientation must also be matched.

\subsubsection{An overly sharp posterior further amplifies this issue}

If the posterior is highly concentrated, so that $\Sigma_q$ is small, then the scale and volume terms in Equation~\eqref{eq:gaussian_kl_appendix} become even more sensitive. Intuitively, the gold posterior is no longer a thick region, but rather resembles a very thin tube. In this regime, even a slight deviation of the prior can induce a substantially amplified local log-density gap. Therefore,
\begin{equation}
\text{good generation} \ \Longrightarrow\ \text{the prior hits a relatively large decoder-good region},
\end{equation}
but
\begin{equation}
\text{good PPL} \ \Longrightarrow\ \text{the prior must additionally hit a much narrower gold posterior tube}.
\end{equation}
This also explains why, in continuous latent language models, PPL behaves more like a \emph{local density calibration} metric rather than merely a \emph{generation quality} metric.

\subsection{Why PPL and Generation Are Usually More Strongly Correlated in Discrete Models Such as AR / LLaDA}

\subsubsection{AR: training, evaluation, and generation share the same discrete probabilistic object}

For autoregressive models, the conditional distribution admits the exact chain decomposition
\begin{equation}
\log p_{\eta}^{\mathrm{AR}}(x^{\mathrm{res}}\mid x^{\mathrm{pre}})
=
\sum_{i=1}^{N_{\mathrm{res}}}
\log p_{\eta}(x_i^{\mathrm{res}}\mid x^{\mathrm{pre}},x_{<i}^{\mathrm{res}}),
\label{eq:ar_factorization_appendix}
\end{equation}
and therefore the PPL is
\begin{equation}
\mathrm{PPL}_{\mathrm{AR}}(x)
=
\exp\!\left(
-\frac{1}{N_{\mathrm{res}}}
\sum_{i=1}^{N_{\mathrm{res}}}
\log p_{\eta}(x_i^{\mathrm{res}}\mid x^{\mathrm{pre}},x_{<i}^{\mathrm{res}})
\right).
\label{eq:ar_ppl_appendix}
\end{equation}
The next-token cross-entropy minimized during training is precisely the negative log of the terms in Equation~\eqref{eq:ar_ppl_appendix}, and generation also directly queries these same conditional probabilities. Hence, in AR models,
\begin{equation}
\text{training objective}
=
\text{the object evaluated by PPL}
=
\text{the object used during generation}.
\label{eq:ar_alignment_appendix}
\end{equation}
This is why PPL and generation quality are usually most strongly correlated in autoregressive models.

\subsubsection{Discrete models optimize the probability mass of the correct event more directly}

For a discrete classifier, if the gold token at a given position is $y$ and the model outputs a probability vector $\pi\in\Delta^{V-1}$, then the cross-entropy is
\begin{equation}
\mathcal L_{\mathrm{CE}}=-\log \pi_y.
\label{eq:discrete_ce_appendix}
\end{equation}
What is directly optimized here is the probability mass assigned to the correct event itself. In contrast, \method requires the computation of an integral over the continuous latent space:
\begin{equation}
p(x^{\mathrm{res}}\mid c)
=
\int \exp\!\Big(
\log p_{\theta}(x^{\mathrm{res}}\mid z,c)+\log p_{\psi}(z\mid c)
\Big)\,\dd z.
\label{eq:continuous_integral_appendix}
\end{equation}
Accordingly, its value depends not only on whether the mode center is correctly located, but also explicitly on local volume and curvature. If one performs a Laplace approximation of
\[
F_{x,c}(z):=\log p_{\theta}(x^{\mathrm{res}}\mid z,c)+\log p_{\psi}(z\mid c)
\]
at its maximizer, then
\begin{equation}
\log p(x^{\mathrm{res}}\mid c)
\approx
F_{x,c}(z_{x,c}^{\star})
+\frac{d}{2}\log(2\pi)
-\frac12\log\det H_{x,c},
\label{eq:laplace_appendix}
\end{equation}
where
\[
H_{x,c}:=-\nabla_z^2F_{x,c}(z_{x,c}^{\star}).
\]
This shows that \textbf{conditional likelihood in continuous latent models depends not only on mode height, but also on local geometry.}

Therefore, the more accurate statement is not merely that discrete models are inherently simpler, but rather that \textbf{training, evaluation, and generation in discrete models all act more directly on the same discrete probabilistic object, whereas continuous latent models are additionally mediated by latent-space integration, posterior--prior matching, and decoder-emission compatibility.}

\subsubsection{LLaDA: the main discrepancy arises at the discrete observation level rather than the latent calibration level}

LLaDA is not identical to AR, but compared with \method, its recovery object is still the discrete observed state itself, rather than first encoding the sample-dependent posterior into a continuous latent and then requiring an independent conditional prior to match it pointwise. Therefore, even though LLaDA may also exhibit a surrogate gap, its main discrepancy is still concentrated within the discrete token space, rather than in the posterior--prior density collision characteristic of \method. For this reason, the correlation between PPL and generation is usually still stronger in LLaDA than in continuous latent models.

\section{A Unified Theoretical View of Noise Schedule, logSNR, and Flow Matching}
\label{app:noise_schedule_logsnr_fm}

This section consolidates the preceding discussion of noise schedule, Flow Matching, latent-dimension-dependent shift drift, and semantic information calibration into a single theoretical chain. The key question is not why a particular empirical hyperparameter works, but rather: \emph{what mathematical object the noise schedule controls in \method, and how this object enters the training objective and ultimately affects semantic modeling quality}. More specifically, we establish the following four points:
\begin{enumerate}
    \item the noise schedule and the logSNR curve are two equivalent parameterizations of the same forward noising process;
    \item in Flow Matching, the effect of logSNR on the loss arises both through the change of sampling measure inside the expectation and through the scale change of the supervised target velocity field;
    \item uniform-timestep training is generally not equivalent to uniform-logSNR training;
    \item the essential role of the optimal timestep shift is not to translate an abstract time label, but to recalibrate training and sampling toward a more appropriate \emph{effective semantic information regime}.
\end{enumerate}

\subsection{Noise Schedule and logSNR as Equivalent Parameterizations}
\label{app:noise_schedule_logsnr_equiv}

We begin with the standard forward process used throughout the main text and appendix:
\begin{equation}
z_t=\alpha_t z+\sigma_t \epsilon,
\qquad
\epsilon\sim \Ncal(0,I),
\label{eq:app_ns_forward_process}
\end{equation}
where $\alpha_t$ and $\sigma_t$ are determined by the noise schedule. The signal-to-noise ratio is naturally defined by
\begin{equation}
\mathrm{SNR}(t)=\frac{\alpha_t^2}{\sigma_t^2},
\qquad
\lambda(t):=\log \mathrm{SNR}(t)=\log \frac{\alpha_t^2}{\sigma_t^2}.
\label{eq:app_logsnr_def}
\end{equation}
Under the standard normalization
\begin{equation}
\alpha_t^2+\sigma_t^2=1,
\label{eq:app_snr_normalization}
\end{equation}
one has
\begin{equation}
\alpha_t^2=\mathrm{sigmoid}(\lambda(t)),
\qquad
\sigma_t^2=\mathrm{sigmoid}(-\lambda(t)).
\label{eq:app_alpha_sigma_from_logsnr}
\end{equation}
Therefore, specifying $\lambda(t)$ uniquely determines $(\alpha_t,\sigma_t)$, and conversely $(\alpha_t,\sigma_t)$ uniquely determines $\lambda(t)$. In this precise sense,
\begin{equation}
\boxed{
\text{noise schedule}
\Longleftrightarrow
\text{logSNR curve}
}
\label{eq:app_schedule_logsnr_equiv}
\end{equation}
so the two are not related through a secondary transformation of the loss; they are simply two equivalent parameterizations of the same forward noising geometry.

Now consider a shifted logSNR curve
\begin{equation}
\lambda_{\delta}(t)=\lambda(t)+\delta.
\label{eq:app_shifted_logsnr}
\end{equation}
The corresponding signal and noise coefficients become
\begin{equation}
\alpha_{t,\delta}^2=\mathrm{sigmoid}(\lambda(t)+\delta),
\qquad
\sigma_{t,\delta}^2=\mathrm{sigmoid}(-\lambda(t)-\delta).
\label{eq:app_shifted_alpha_sigma}
\end{equation}
Hence the role of timestep shift is not merely to translate the raw timestep axis, but rather to remap the same raw timestep to a different logSNR regime:
\begin{equation}
\boxed{
\text{the same raw timestep } t
\text{ is reassigned to a different logSNR interval.}
}
\label{eq:app_shift_meaning}
\end{equation}

\subsection{Reparameterization of the Flow Matching Objective}
\label{app:flow_matching_reparam}

As discussed above, the strict probabilistic objective of \method is always ELBO maximization together with prior matching, while Flow Matching is the practical mechanism used to learn the same prior family. To understand how the noise schedule enters training, it is therefore sufficient to analyze the mean-squared vector-field regression objective.

To avoid a notation clash with the coefficients $(\alpha_t,\sigma_t)$ in Equation~\eqref{eq:app_ns_forward_process}, let the interpolation coefficient of the bridge path be denoted by $\beta(t)$. Consider
\begin{equation}
z_t=(1-\beta(t))z_0+\beta(t)z_1,
\qquad
u_t(z_0,z_1)=\frac{\dd z_t}{\dd t}=\dot\beta(t)(z_1-z_0),
\label{eq:app_bridge_path_beta}
\end{equation}
and write the single-block Flow Matching objective as
\begin{equation}
\Lcal_{\mathrm{FM}}
=
\E_{t,z_0,z_1}
\Big[
\|v_{\psi}(z_t,t)-u_t(z_0,z_1)\|_2^2
\Big],
\qquad
t\sim U[0,1].
\label{eq:app_fm_loss_single_block}
\end{equation}
Equivalently,
\begin{equation}
\Lcal_{\mathrm{FM}}
=
\int_0^1 \ell(t)\,\dd t,
\qquad
\ell(t):=
\E_{z_0,z_1}
\Big[
\|v_{\psi}(z_t,t)-u_t(z_0,z_1)\|_2^2
\Big].
\label{eq:app_fm_loss_integral_t}
\end{equation}

Now let
\begin{equation}
\lambda=\Lambda(t):=\log \mathrm{SNR}(t),
\label{eq:app_lambda_as_reparam}
\end{equation}
and assume that $\Lambda$ is strictly monotone and differentiable, so that the inverse mapping $t=T(\lambda)=\Lambda^{-1}(\lambda)$ exists. By the change-of-variables formula,
\begin{equation}
\dd t=\left|\frac{\dd t}{\dd \lambda}\right|\dd \lambda,
\label{eq:app_dt_dlambda}
\end{equation}
which yields
\begin{equation}
\Lcal_{\mathrm{FM}}
=
\int
\ell(T(\lambda))
\left|\frac{\dd t}{\dd \lambda}\right|
\dd \lambda.
\label{eq:app_fm_loss_integral_lambda_first}
\end{equation}
Therefore, if training samples $t$ uniformly, then the induced density on the logSNR axis is
\begin{equation}
p_{\lambda}(\lambda)
=
p_t(T(\lambda))
\left|\frac{\dd t}{\dd \lambda}\right|
=
\left|\frac{\dd t}{\dd \lambda}\right|,
\label{eq:app_induced_density_lambda}
\end{equation}
where we used $p_t(t)\equiv 1$.

\begin{proposition}[Uniform timestep is generally not equivalent to uniform logSNR]
\label{prop:app_uniform_t_not_uniform_logsnr}
Let $\lambda=\Lambda(t)$ be a strictly monotone differentiable function, and suppose training uses $t\sim U[0,1]$. Then the sampling density on the logSNR axis is
\begin{equation}
p_{\lambda}(\lambda)=\left|\frac{\dd t}{\dd \lambda}\right|.
\label{eq:app_lambda_density_again}
\end{equation}
Therefore, unless $\Lambda(t)$ is affine, uniform-timestep training is not equivalent to uniform-logSNR training.
\end{proposition}

\begin{proof}
By Equation~\eqref{eq:app_fm_loss_integral_lambda_first}, the pushforward of the uniform measure on $t$ to the $\lambda$ axis has density
\[
p_{\lambda}(\lambda)=p_t(T(\lambda))\left|\frac{\dd t}{\dd \lambda}\right|.
\]
Since $p_t(t)\equiv 1$, Equation~\eqref{eq:app_lambda_density_again} follows immediately. The density $p_{\lambda}(\lambda)$ is constant if and only if $\frac{\dd \lambda}{\dd t}$ is constant, namely if and only if $\Lambda(t)$ is affine. Hence the two training schemes are not equivalent in general.
\end{proof}

The above proposition shows only that the \emph{sampling measure} inside the expectation changes after reparameterization by logSNR. This is not yet the whole story. The supervised target velocity field itself also rescales under the new time coordinate. Let
\begin{equation}
z_{\lambda}:=z_{T(\lambda)},
\qquad
u_{\lambda}:=\frac{\dd z_{\lambda}}{\dd \lambda}.
\label{eq:app_zlambda_ulambda}
\end{equation}
Then, by the chain rule,
\begin{equation}
u_t
=
\frac{\dd z_t}{\dd t}
=
\frac{\dd z_{\lambda}}{\dd \lambda}\cdot \frac{\dd \lambda}{\dd t}
=
\dot\lambda(t)\,u_{\lambda},
\qquad
\dot\lambda(t):=\frac{\dd \lambda}{\dd t}.
\label{eq:app_ut_scaling_with_lambda}
\end{equation}
Substituting this identity into the loss gives
\begin{equation}
\Lcal_{\mathrm{FM}}
=
\int
\E_{z_0,z_1}
\Big[
\|v_{\psi}(z_{\lambda},T(\lambda))
-
\dot\lambda(T(\lambda))u_{\lambda}\|_2^2
\Big]
\left|\frac{\dd t}{\dd \lambda}\right|
\dd \lambda.
\label{eq:app_fm_loss_lambda_with_ut}
\end{equation}
Hence logSNR affects the Flow Matching loss in two distinct ways:
\begin{enumerate}
    \item it changes the sampling weight assigned to different noise regimes inside the expectation;
    \item it rescales the numerical magnitude of the supervised target velocity field.
\end{enumerate}

If we further define the $\lambda$-time vector field by
\begin{equation}
\tilde v_{\psi}(z,\lambda)
:=
\frac{1}{\dot\lambda(T(\lambda))}
 v_{\psi}(z,T(\lambda)),
\label{eq:app_vtilde_lambda}
\end{equation}
then Equation~\eqref{eq:app_fm_loss_lambda_with_ut} can be rewritten as
\begin{equation}
\Lcal_{\mathrm{FM}}
=
\int
w_{\mathrm{eff}}(\lambda)
\E_{z_0,z_1}
\Big[
\|\tilde v_{\psi}(z_{\lambda},\lambda)-u_{\lambda}\|_2^2
\Big]
\dd \lambda,
\label{eq:app_fm_loss_lambda_weighted}
\end{equation}
where the effective weight is
\begin{equation}
w_{\mathrm{eff}}(\lambda)
=
\left|\frac{\dd \lambda}{\dd t}\right|.
\label{eq:app_effective_weight_lambda}
\end{equation}
This yields the following summary.

\begin{remark}[How logSNR enters the Flow Matching loss]
\label{rem:app_logsnr_enters_fm}
In the Flow Matching training of \method, the effect of logSNR is not merely that one samples different time indices. It changes the training geometry jointly through
\begin{enumerate}
    \item the induced sampling measure on the logSNR axis, $p_{\lambda}(\lambda)=\left|\frac{\dd t}{\dd \lambda}\right|$;
    \item the time-scale rescaling of the supervised target, $u_t=\dot\lambda(t)u_{\lambda}$.
\end{enumerate}
Therefore, the noise schedule is not an external hyperparameter detached from the loss. It directly determines which logSNR regimes are emphasized during training and how difficult the regression task is within those regimes.
\end{remark}

\subsection{Equivalent Derivation Under the Signal--Noise Parameterization}
\label{app:signal_noise_parametrization}

For consistency with Equation~\eqref{eq:app_ns_forward_process}, it is useful to derive the same scaling relation directly from the signal--noise coefficients $(\alpha_t,\sigma_t)$. By Equation~\eqref{eq:app_alpha_sigma_from_logsnr},
\[
\alpha_t^2=\mathrm{sigmoid}(\lambda(t)),
\qquad
\sigma_t^2=\mathrm{sigmoid}(-\lambda(t)).
\]
Differentiating $\alpha_t^2$ with respect to $t$ yields
\begin{equation}
\frac{\dd}{\dd t}\alpha_t^2
=
\alpha_t^2\sigma_t^2\dot\lambda(t),
\label{eq:app_d_alpha_sq}
\end{equation}
and therefore
\begin{equation}
\dot\alpha_t
=
\frac{1}{2}\alpha_t\sigma_t^2\dot\lambda(t).
\label{eq:app_d_alpha}
\end{equation}
Similarly,
\begin{equation}
\dot\sigma_t
=
-\frac{1}{2}\sigma_t\alpha_t^2\dot\lambda(t).
\label{eq:app_d_sigma}
\end{equation}
Differentiating Equation~\eqref{eq:app_ns_forward_process} then gives
\begin{equation}
u_t
:=
\frac{\dd z_t}{\dd t}
=
\dot\alpha_t z+\dot\sigma_t \epsilon
=
\frac{1}{2}\dot\lambda(t)
\Bigl(
\alpha_t\sigma_t^2 z-\sigma_t\alpha_t^2\epsilon
\Bigr).
\label{eq:app_ut_from_alpha_sigma}
\end{equation}
If we define
\begin{equation}
u_{\lambda}
:=
\frac{\dd z_t}{\dd \lambda}
=
\frac{1}{2}
\Bigl(
\alpha_t\sigma_t^2 z-\sigma_t\alpha_t^2\epsilon
\Bigr),
\label{eq:app_ulambda_from_alpha_sigma}
\end{equation}
then we again obtain
\begin{equation}
u_t=\dot\lambda(t)\,u_{\lambda}.
\label{eq:app_ut_equals_lambda_dot_ulambda}
\end{equation}
Thus, the same conclusion holds whether one starts from an abstract bridge-path parameterization or directly from the signal--noise parameterization of diffusion.

\subsection{The Optimal Schedule as Calibration of the Effective Semantic Information Regime}
\label{app:optimal_schedule_semantic_calibration}

Showing that the schedule affects the loss is not yet sufficient. The more important question is: \emph{what meaningful object in the latent space is ultimately being calibrated?}

Following the notation in Appendix~\ref{sec:global_semantic_structure}, let $s$ denote a high-level semantic variable, and decompose the latent representation as
\begin{equation}
z=\phi(s)+u,
\label{eq:app_semantic_signal_plus_residual}
\end{equation}
where $\phi(s)$ is the semantic signal and $u$ represents posterior uncertainty or encoding residuals. Then the noisy latent at diffusion time $t$ is
\begin{equation}
z_t
=
\alpha_t\phi(s)+\alpha_t u+\sigma_t\epsilon.
\label{eq:app_noisy_latent_semantic}
\end{equation}
From the viewpoint of the semantic variable $s$, the total noise covariance is therefore
\begin{equation}
\Sigma_{\mathrm{noise}}(t)
=
\alpha_t^2\Sigma_u+\sigma_t^2 I,
\label{eq:app_noise_covariance_effective}
\end{equation}
where $\Sigma_u$ is the covariance of the posterior uncertainty term. If $\Sigma_s$ denotes the covariance of the semantic signal $\phi(s)$, then the effective information about $s$ preserved in the latent at time $t$ can be written as
\begin{equation}
I(s;z_t)
=
\frac{1}{2}
\log\det\!\Bigl(
I+\alpha_t^2\Sigma_s
(\alpha_t^2\Sigma_u+\sigma_t^2 I)^{-1}
\Bigr).
\label{eq:app_effective_semantic_mi}
\end{equation}
This makes clear that what matters to the DiT is not the raw timestep itself, but how much effective semantic information remains in the latent at that timestep. In this sense,
\begin{equation}
\boxed{
\text{the schedule effectively controls the curve }
 t\mapsto I(s;z_t).
}
\label{eq:app_schedule_controls_semantic_mi}
\end{equation}
Accordingly, the role of timestep shift should be interpreted not as the tuning of an isolated numerical hyperparameter, but as the recalibration of training and sampling toward a different effective semantic-information regime.

More abstractly, if the shifted logSNR curve is given by $\lambda_{\delta}(t)=\lambda(t)+\delta$, then the optimal shift can be expressed as
\begin{equation}
\delta^{\star}
=
\arg\max_{\delta}
\mathrm{Perf}\Bigl(
I_{\mathrm{eff},\delta}(t;d,\Sigma_u,\mathcal G,B,\vartheta)
\Bigr),
\label{eq:app_abstract_optimal_shift}
\end{equation}
where $d$ denotes the latent dimension, $\Sigma_u$ the posterior uncertainty, $\mathcal G$ the latent geometry or semantic smoothness, $B$ the block size, and $\vartheta$ the collection of model parameters. Thus, the problem of selecting the optimal noise schedule is best understood as an \emph{effective semantic information calibration} problem rather than an isolated numerical tuning problem.

\subsection{Why Latent Dimension and VAE logSNR Shift the Optimum}
\label{app:dimension_and_vae_logsnr_shift}

The viewpoint developed above is consistent with Proposition~\ref{prop:shift_invariance_under_separable_null} and Corollary~\ref{cor:drift_refutes_separable_null}. If the latent representation were purely local, separable, and homogeneous across dimensions, then the optimal shift would not systematically drift with dimension. When a stable drift is observed empirically, the natural interpretation is that latent dimension and VAE posterior logSNR both modify the effective semantic-information curve in Equation~\eqref{eq:app_effective_semantic_mi}.

\paragraph{Latent dimension.}
Under the shared-factor model
\begin{equation}
z_i=A_i g+\xi_i,
\qquad
i=1,\dots,d,
\label{eq:app_shared_factor_model_again}
\end{equation}
if different latent dimensions repeatedly observe the same low-dimensional shared factor $g$, then standard linear-Gaussian inference implies that the effective recovery signal-to-noise ratio scales as
\begin{equation}
\mathrm{SNR}_{\mathrm{eff}}(d)\propto d.
\label{eq:app_snr_eff_dim}
\end{equation}
Consequently, the recoverable information about the shared factor satisfies the approximation
\begin{equation}
I(g;z_t)
\approx
\frac{r}{2}\log\!\Bigl(1+c\cdot d\cdot \mathrm{SNR}_{\mathrm{eff}}(t)\Bigr).
\label{eq:app_mi_shared_factor_again}
\end{equation}
This means that increasing the latent dimension can strengthen the recoverable information about shared semantic structure. In order to keep training in a comparable semantic-recovery regime, the timestep shift must compensate for this dimensional effect, which naturally leads to a logarithmic law of the form
\begin{equation}
\delta^{\star}(d)=a\log d+b.
\label{eq:app_delta_log_d_again}
\end{equation}

\paragraph{VAE posterior logSNR.}
On the other hand, lowering the posterior logSNR of the VAE typically corresponds to increasing posterior variance, that is, enlarging $\Sigma_u$. By Equation~\eqref{eq:app_effective_semantic_mi}, for fixed $t$, increasing
\[
\alpha_t^2\Sigma_u+\sigma_t^2 I
\]
reduces the effective semantic mutual information. Therefore, the same raw timestep corresponds to a lower effective semantic SNR in a smoother latent space with larger posterior variance. To align training and sampling with a comparable semantic-information regime, the optimal shift must change accordingly.

This shows that changing the latent dimension and changing the VAE posterior logSNR are, on the surface, two different operations, but from an information-theoretic perspective they act on the same core object: the effective semantic-information curve $I(s;z_t)$ along diffusion time. The former mainly changes how shared information is aggregated and amplified across dimensions, whereas the latter mainly changes the effective noise floor at each timestep.

\begin{remark}[Role of block size]
\label{rem:app_block_size_role}
Compared with latent dimension and posterior logSNR, the effect of block size is not yet characterized by a closed-form law such as Equation~\eqref{eq:app_delta_log_d_again}. Nevertheless, block size changes the granularity of within-block semantic aggregation, the strength of cross-block conditional dependence, and the intrinsic modeling difficulty within the same noise regime. It therefore interacts with the schedule indirectly through the shape of the effective semantic-information curve. The more precise statement is not that block size directly determines the schedule, but that block size and schedule are coupled through the effective geometry of the block-causal prior.
\end{remark}

\begin{remark}[Summary]
\label{rem:app_noise_schedule_summary}
In \method, the noise schedule is not an empirical hyperparameter external to the training objective. It first defines the forward noising geometry through $(\alpha_t,\sigma_t)$, or equivalently through the logSNR curve; it then enters the Flow Matching objective jointly through the sampling measure and the time-scale of the supervised target velocity field; and, at the semantic level, it calibrates not the raw timestep itself but the effective mutual-information regime of high-level latent semantics. Consequently, changes in latent dimension, posterior uncertainty, and latent geometry can all induce systematic drift in the optimal timestep shift by altering the same effective semantic-information curve.
\end{remark}

\section{Experiments Details}

\subsection{Dataset Introduction}
\label{app:dataset}

\paragraph{LAMBADA.}
LAMBADA \citep{paperno2016lambada} is a long-context word prediction benchmark designed to evaluate whether a model can use broad discourse context for text continuation. The target word is usually difficult to infer from local context alone and instead requires understanding information distributed across the preceding passage. In our experiments, LAMBADA is used as a representative continuation-style benchmark to evaluate global semantic modeling and long-range contextual coherence.

\paragraph{MMLU.}
MMLU \citep{hendrycks2020measuring} is a multitask multiple-choice benchmark covering a wide range of subjects, including humanities, social sciences, STEM, and professional domains. It is widely used to assess broad factual knowledge and reasoning ability. In our experiments, MMLU serves as a representative knowledge-intensive benchmark for evaluating the quality of generated answers under multiple-choice settings.

\paragraph{SIQA.}
SIQA \citep{sap2019social} (Social IQa) is a multiple-choice commonsense reasoning benchmark focused on social situations, intentions, and reactions. Each example requires selecting the most plausible answer given a short social context. We use SIQA to evaluate whether the model can capture socially grounded semantics and perform coherent commonsense inference.

\paragraph{SQuAD.}
SQuAD \citep{rajpurkar2016squad} is a reading comprehension benchmark in which the model answers questions based on a given passage. In our evaluation, it is used in a generative setting to assess whether the model can produce accurate answers conditioned on contextual evidence. This benchmark complements multiple-choice tasks by testing open-form text generation ability.

\paragraph{Story Cloze.}
Story Cloze \citep{mostafazadeh2016corpus} is a story understanding benchmark that asks the model to choose the most plausible ending for a four-sentence story context. It mainly evaluates narrative coherence, causal reasoning, and commonsense understanding. We include it as a benchmark for discourse-level semantic consistency.

\paragraph{OBQA.}
OBQA \citep{mihaylov2018can} (OpenBookQA) is a multiple-choice question answering benchmark designed to test the combination of core science facts with broader commonsense knowledge. Many questions require multi-hop reasoning beyond direct fact recall. In our experiments, it is used to assess knowledge integration and reasoning under a constrained answer space.

\paragraph{RACE.}
RACE \citep{lai2017race} is a large-scale multiple-choice reading comprehension benchmark collected from English examinations for middle and high school students. The questions often require passage understanding, inference, and reasoning rather than simple span extraction. We use RACE to evaluate more challenging reading comprehension and answer selection ability.

\paragraph{HellaSwag.}
HellaSwag \citep{zellers2019hellaswag} is a multiple-choice benchmark for grounded commonsense reasoning and sentence continuation. The task is to select the most plausible continuation among several adversarially constructed candidates. It is widely regarded as a strong test of commonsense and semantic plausibility, and is included here to evaluate the model's ability to distinguish coherent continuations from distractors.

\paragraph{Usage in This Work.}
For the internal component analysis in Sections~\ref{exp:rq1}, \ref{exp:rq2}, and \ref{exp:rq3}, we conduct experiments on randomly sampled subsets from the test sets of LAMBADA, MMLU, and SIQA for efficiency and controlled comparison. For the broader external comparison in Section~\ref{exp:rq4}, we additionally evaluate on the test sets of SQuAD, Story Cloze, OBQA, RACE, and HellaSwag to provide a more comprehensive assessment across continuation, reading comprehension, knowledge-intensive reasoning, and commonsense reasoning tasks.

\subsection{Baseline Details}
\label{app:baseline}

For the internal component analysis in Sections~\ref{exp:rq1}, \ref{exp:rq2}, and \ref{exp:rq3}, the comparisons are conducted among different configurations of \method itself, so no external baseline is introduced in these sections. External baselines are mainly used in Section~\ref{exp:rq4} for scaling comparisons.

\paragraph{Autoregressive Baseline.}
We use a standard autoregressive language model as the primary reference baseline. The model is implemented based on the official LLaMA framework~\citep{touvron2023llama} and is trained from scratch with random initialization. It adopts the conventional next-token prediction objective and serves as a representative strong decoder-only autoregressive paradigm. To ensure a fair comparison, the autoregressive baseline is trained under settings that are strictly matched to \method in terms of data source, tokenizer, training budget, model scale, and optimization recipe wherever applicable.

\paragraph{Discrete Diffusion Baseline (LLaDA).}
As the discrete diffusion baseline, we use LLaDA~\citep{nie2025large}, implemented with its official codebase and also trained from scratch with random initialization. This baseline represents the discrete diffusion language modeling paradigm, in which generation is performed by iterative denoising in discrete token space. In the scaling experiments, LLaDA is trained under the same overall experimental protocol as the autoregressive baseline and \method, including matched training data, tokenizer, compute budget, and model size, so that the comparison focuses on the difference in modeling paradigm rather than differences in training resources or implementation details.

\paragraph{Fairness of Comparison.}
For all external comparisons in Section~\ref{exp:rq4}, the autoregressive baseline, LLaDA, and \method are trained independently under strictly matched settings. In particular, all models are trained from scratch rather than initialized from pretrained checkpoints, and we keep the training corpus, parameter scale, and compute budget aligned as closely as possible. Therefore, the reported differences mainly reflect the properties of the underlying modeling frameworks, rather than advantages introduced by data, initialization, or training recipes.

\subsection{Metrics and Evaluation}
\label{app:metrics}

To ensure the fairest possible comparison across model families, all models are evaluated under a unified few-shot generative protocol. As discussed in Section~\ref{dis:ppl}, we do not use perplexity as the primary metric for the main comparison, since likelihood-oriented estimates can exhibit a substantial mismatch with the actual generation quality of \method. Instead, all models are prompted to directly generate task answers under fixed templates, and performance is computed after unified post-processing and answer matching.

Across all tasks, we use manually written in-context demonstrations that are shared across model families and do not overlap with the evaluation set. Unless otherwise noted, the templates adopt 2-shot prompting. For multiple-choice tasks, the options are always presented with uppercase markers such as \texttt{(A)}, \texttt{(B)}, \texttt{(C)}, and \texttt{(D)}, with the number of options adjusted dynamically for each benchmark. Importantly, the model is always asked to generate the \emph{option text itself}, rather than the option label, so that all tasks follow the same generative evaluation protocol. Each prompt ends with a task-specific anchor such as \texttt{Answer:} or \texttt{End:}, without adding an extra trailing newline, so that the model directly continues from the same completion point.

For all multiple-choice benchmarks, the generated output is truncated at the first newline and normalized before scoring. We then compare the normalized prediction against all candidate option texts and map it to the best-matching candidate. A prediction is counted as correct if and only if the selected candidate matches the gold answer. For continuation-style generation such as LAMBADA, the model directly continues the given prefix and the first generated word is used as the prediction. For SQuAD, we follow the same unified generative protocol and evaluate the generated short answer after standard normalization. Unless otherwise stated, the templates, demonstrations, truncation rules, and answer-matching rules are kept identical for AR, LLaDA, and \method; the only difference is the native generation procedure used by each model.

\paragraph{Summary of prompt formats.}
Table~\ref{tab:prompt_summary} summarizes the prompt format used for each benchmark.

\begin{table}[t]
\centering
\small
\begin{tabular}{lcccc}
\toprule
Task & Format & \# Shots & \# Options & Generation Anchor \\
\midrule
MMLU        & multiple choice             & 2 & 4 & \texttt{Answer:} \\
RACE        & reading comprehension MCQ   & 2 & 4 & \texttt{Answer:} \\
Story Cloze & story ending selection      & 2 & 2 & \texttt{End:} \\
LAMBADA     & continuation / last-word prediction & 0 & -- & direct continuation \\
OBQA        & multiple choice             & 2 & 4 & \texttt{Answer:} \\
HellaSwag   & situational continuation MCQ & 2 & 4 & \texttt{Answer:} \\
SIQA        & social commonsense MCQ      & 2 & 3 & \texttt{Answer:} \\
SQuAD       & extractive QA generation    & 1 & -- & \texttt{Answer:} \\
\bottomrule
\end{tabular}
\caption{\textbf{Unified prompt formats used in evaluation.} All models are evaluated with the same prompt templates and answer-matching rules.}
\label{tab:prompt_summary}
\end{table}

\paragraph{MMLU.}
MMLU is evaluated as a 4-way multiple-choice generation task. The prompt contains two in-context examples, followed by the test question and its four candidate options. The model is required to generate the text of the correct option.

\begin{promptbox}[MMLU prompt template]
Question: Which gas do plants absorb from the air during photosynthesis?
(A) Oxygen
(B) Carbon dioxide
(C) Nitrogen
(D) Hydrogen
Answer: Carbon dioxide

Question: A triangle has angles 50 degrees and 60 degrees. What is the third angle?
(A) 60 degrees
(B) 70 degrees
(C) 80 degrees
(D) 90 degrees
Answer: 70 degrees

Question: {question}
(A) {choices[0]}
(B) {choices[1]}
(C) {choices[2]}
(D) {choices[3]}
Answer:
\end{promptbox}

\paragraph{RACE.}
RACE is also evaluated as a 4-way multiple-choice generation task. Here, the dataset field already contains the article together with the question, so the full input is placed after \texttt{Article:} and the candidate options are listed below it.

\begin{promptbox}[RACE prompt template]
Read the following article and answer the question.

Article: Mary went to the store to buy some fruits. She bought five apples and two oranges. She paid 5 dollars in total. What did Mary buy?
Options:
(A) Bananas
(B) Apples and oranges
(C) Grapes
(D) Watermelon
Answer: Apples and oranges

Article: {question}
Options:
(A) {choices[0]}
(B) {choices[1]}
(C) {choices[2]}
(D) {choices[3]}
Answer:
\end{promptbox}

\paragraph{Story Cloze.}
Story Cloze is formatted as a story-completion benchmark with two candidate endings. To better match the semantics of narrative continuation, the prompt uses \texttt{End:} instead of \texttt{Answer:} as the completion anchor.

\begin{promptbox}[Story Cloze prompt template]
Story: I wanted to make an omelet. I cracked two eggs into a bowl and whisked them. Then I poured them into a hot pan.
(A) I ate a delicious omelet for breakfast.
(B) I decided to order a pizza instead.
End: I ate a delicious omelet for breakfast.

Story: The runner tied his shoes tight. He sprinted as fast as he could during the race. He crossed the finish line first.
(A) He was sad that he lost the race.
(B) He won the gold medal.
End: He won the gold medal.

Story: {question}
(A) {choices[0]}
(B) {choices[1]}
End:
\end{promptbox}

\paragraph{LAMBADA.}
LAMBADA is treated as open-ended continuation rather than multiple choice. Since the task is essentially final-word prediction given a long context, we do not use in-context demonstrations. The model directly continues the provided passage prefix, and the first generated word is taken as the prediction.

\begin{promptbox}[LAMBADA prompt template]
{question}
\end{promptbox}

\paragraph{OBQA.}
OBQA is evaluated with the same 2-shot generative multiple-choice protocol as MMLU. The model is required to generate the text of the correct answer option rather than its label.

\begin{promptbox}[OBQA prompt template]
Question: Which tool is best for tightening a screw?
(A) spoon
(B) hammer
(C) screwdriver
(D) paintbrush
Answer: screwdriver

Question: What do plants absorb from the air during photosynthesis?
(A) carbon dioxide
(B) oxygen
(C) helium
(D) salt
Answer: carbon dioxide

Question: {question}
(A) {choices[0]}
(B) {choices[1]}
(C) {choices[2]}
(D) {choices[3]}
Answer:
\end{promptbox}

\paragraph{HellaSwag.}
HellaSwag is evaluated as a 4-way situational continuation task. We use \texttt{Context:} as the input field name to emphasize that the model must choose the most plausible continuation for the given context.

\begin{promptbox}[HellaSwag prompt template]
Context: The girl puts the bread into the toaster and pushes the lever down. The bread
(A) becomes a slice of pizza.
(B) starts to toast and turn brown.
(C) disappears immediately.
(D) turns into a glass of water.
Answer: starts to toast and turn brown.

Context: The goalkeeper sees the ball coming towards the net. He dives and
(A) catches the ball with his hands.
(B) starts dancing in the field.
(C) opens a laptop to check email.
(D) runs away from the stadium.
Answer: catches the ball with his hands.

Context: {question}
(A) {choices[0]}
(B) {choices[1]}
(C) {choices[2]}
(D) {choices[3]}
Answer:
\end{promptbox}

\paragraph{SIQA.}
SIQA is a 3-way social commonsense multiple-choice benchmark. The template is identical in spirit to the other multiple-choice prompts, except that only three options are presented.

\begin{promptbox}[SIQA prompt template]
Question: Jordan wanted to tell a joke to his friends. What does Jordan need to do before this?
(A) ignore his friends
(B) think of a funny story
(C) leave the room
Answer: think of a funny story

Question: Kai helped his neighbor carry heavy groceries inside. How would the neighbor feel?
(A) angry
(B) grateful
(C) scared
Answer: grateful

Question: {question}
(A) {choices[0]}
(B) {choices[1]}
(C) {choices[2]}
Answer:
\end{promptbox}

\paragraph{SQuAD.}
SQuAD is evaluated as short-answer generation under a standard extractive QA format. The prompt provides a context paragraph and a question, and the model is asked to generate the answer phrase directly. In our unified protocol, the generated answer is normalized before exact-match evaluation.

\begin{promptbox}[SQuAD prompt template]
Context: The Normans (Norman: Nourmands; French: Normands; Latin: Normanni) were the people who in the 10th and 11th centuries gave their name to Normandy, a region in France. They were descended from Norse raiders and pirates from Denmark, Iceland and Norway.
Question: In what country is Normandy located?
Answer: France

Context: {context}
Question: {question}
Answer:
\end{promptbox}

\paragraph{Final remarks.}
These templates are kept identical across AR, LLaDA, and \method, so that the comparison focuses on the model family itself rather than on prompt engineering differences. In all cases, evaluation is performed after unified post-processing, including truncation at the first newline and text normalization before matching against the candidate options or reference answers.

\subsection{Setup Details}
\label{app:setup}

To ensure a fair comparison, we align AR, LLaDA, and \method on the same pretraining corpus, tokenizer family, sequence length, batch size, optimizer, learning-rate schedule, hardware setup, and evaluation protocol. As a result, all methods observe the same number of training tokens at matched optimization steps, making step-wise comparison directly meaningful. We also keep the core generation backbone at a comparable scale: AR and LLaDA use $\sim$1.8B non-embedding transformer parameters, while the DiT backbone in \method is also kept at $\sim$1.8B; the remaining parameter difference comes from the additional VAE required by the continuous latent formulation.

\begin{table*}[t]
    \centering
    \caption{\textbf{Controlled training and evaluation setup shared across AR, LLaDA, and \method.} These settings are matched unless method-specific generation mechanisms require otherwise.}
    \label{tab:setup_shared}
    \small
    \setlength{\tabcolsep}{7pt}
    \renewcommand{\arraystretch}{1.12}
    \begin{tabular*}{\textwidth}{@{\extracolsep{\fill}}ll}
        \toprule
        \textbf{Category} & \textbf{Setting} \\
        \midrule
        Tokenizer & OLMo 2 tokenizer \\
        Supervised data mixing & None \\
        Sequence length & 512 \\

        Global batch size & 1,408 \\
        Tokens per step & 720,896 \\

        Precision & bfloat16 \\
        Optimizer & AdamW \\
        Peak learning rate & $1.5 \times 10^{-4}$ \\
        Betas & $(0.9,\ 0.95)$ \\
        Weight decay & 0.01 \\
        Gradient clipping & 1.0 \\
        Warmup & 5K steps, linear warmup from $1\times10^{-6}$ to $1.5\times10^{-4}$ \\
        LR decay & Cosine decay to $1\times10^{-5}$ \\

        Random seed / data seed & 66 / 6198 \\

        Max evaluation new tokens & 32 \\

        \bottomrule
    \end{tabular*}
\end{table*}

\begin{table*}[t]
    \centering
    \caption{\textbf{Model capacity and architecture comparison.} The core generation backbone is kept in a comparable regime across methods. \method introduces an additional VAE module for continuous latent modeling.}
    \label{tab:setup_model_compare}
    \small
    \setlength{\tabcolsep}{6pt}
    \renewcommand{\arraystretch}{1.12}
    \begin{tabular*}{\textwidth}{@{\extracolsep{\fill}}lccc}
        \toprule
        \textbf{Model} & \textbf{AR} & \textbf{LLaDA} & \textbf{\method} \\
        \midrule
        Model type & Autoregressive LM & Masked diffusion LM & Continuous latent diffusion LM \\
        Total parameters & $\sim$2.2B & $\sim$2.2B & $\sim$2.3B \\
        Core non-embedding parameters & $\sim$1.8B & $\sim$1.8B & DiT $\sim$1.8B \\
        VAE parameters & -- & -- & $\sim$500M \\
        Embedding parameters & 410,738,688 & 410,738,688 & 308,054,016 \\
        Vocabulary size & 100,278 & 100,278 & 100,278 \\
        Main hidden dimension & 2,048 & 2,048 & DiT: 2,048 \\
        FFN dimension & 8,192 & 8,192 & DiT: 8,192 \\
        Number of backbone layers & 27 & 27 & DiT: 24 \\
        Attention heads & 16 & 16 & DiT: 16 \\
        Head dimension & 128 & 128 & 128 \\
        Positional encoding & RoPE & RoPE & RoPE \\
        Attention pattern & Causal & Bidirectional & VAE causal + DiT block causal \\
        Latent dimension & -- & -- & 16 \\
        VAE encoder / decoder blocks & -- & -- & 4 / 4 \\
        VAE hidden dimension & -- & -- & 1,536 \\
        VAE FFN dimension & -- & -- & 6,144 \\
        \bottomrule
    \end{tabular*}
\end{table*}

\begin{table*}[t]
    \centering
    \caption{\textbf{Implementation-level numerical setup comparison across AR, LLaDA, and \method.}
    We also verify that the three methods use closely aligned checkpoint precision, runtime autocast policy, and numerically sensitive operations.}
    \label{tab:setup_numerical_fairness}
    \small
    \setlength{\tabcolsep}{5pt}
    \renewcommand{\arraystretch}{1.12}
    \begin{tabular*}{\textwidth}{@{\extracolsep{\fill}}p{0.23\textwidth}p{0.24\textwidth}p{0.24\textwidth}p{0.24\textwidth}}
        \toprule
        \textbf{Aspect} & \textbf{\method} & \textbf{AR} & \textbf{LLaDA} \\
        \midrule
        
        Checkpoint save dtype 
        & fp32 
        & fp32 
        & fp32 \\
        
        Checkpoint load dtype 
        & fp32 
        & fp32 
        & fp32 \\
        
        Runtime weight dtype 
        & fp32 
        & fp32 
        & fp32 \\
        
        Model input type 
        & Latents kept in fp32 before entering DiT; cast under autocast when needed 
        & Token ids in integer format 
        & Token ids in integer format \\
        
        Autocast policy 
        & bfloat16 autocast for DiT/VAE forward 
        & bfloat16 autocast for generation forward 
        & bfloat16 autocast for diffusion generation forward \\
        
        Main matmul / attention compute 
        & bf16 
        & bf16 
        & bf16 \\
        
        Numerically sensitive ops 
        & LayerNorm / softmax in fp32 
        & LayerNorm / softmax in fp32 
        & LayerNorm / softmax in fp32 \\
        
        Sampling-time accumulation 
        & Iterative latent update kept in fp32; CFG combination applied during denoising 
        & Token decoding without iterative floating-point latent accumulation 
        & Sampling uses numerically stable probability computation for token selection \\
        \bottomrule
    \end{tabular*}
\end{table*}

In addition to the matched data, optimization, and model-scale settings in Tables~\ref{tab:setup_shared} and~\ref{tab:setup_model_compare}, we further verified implementation-level numerical consistency across the three methods. As shown in Table~\ref{tab:setup_numerical_fairness}, all models save and load checkpoints in fp32, keep runtime weights in fp32, and perform the main forward computation under bfloat16 autocast, while numerically sensitive operations such as normalization and softmax remain in fp32. Therefore, the precision regime is broadly aligned across methods. The main implementation difference is that \method contains an additional VAE encode/decode path, which is intrinsic to continuous latent modeling rather than an extra numerical advantage.

Overall, the comparison is controlled along the main axes that determine pretraining scale and optimization difficulty, including data, tokenizer family, token budget, optimizer, learning-rate schedule, hardware, and evaluation protocol. Therefore, the observed differences mainly reflect the modeling paradigm itself rather than mismatched training resources or evaluation settings.

\subsection{Why VAE Pretraining Does Not Compromise Fairness}
\label{app:fairness_embed_vae_latent_space}

\begin{figure}[t]
    \centering
    \includegraphics[width=\linewidth]{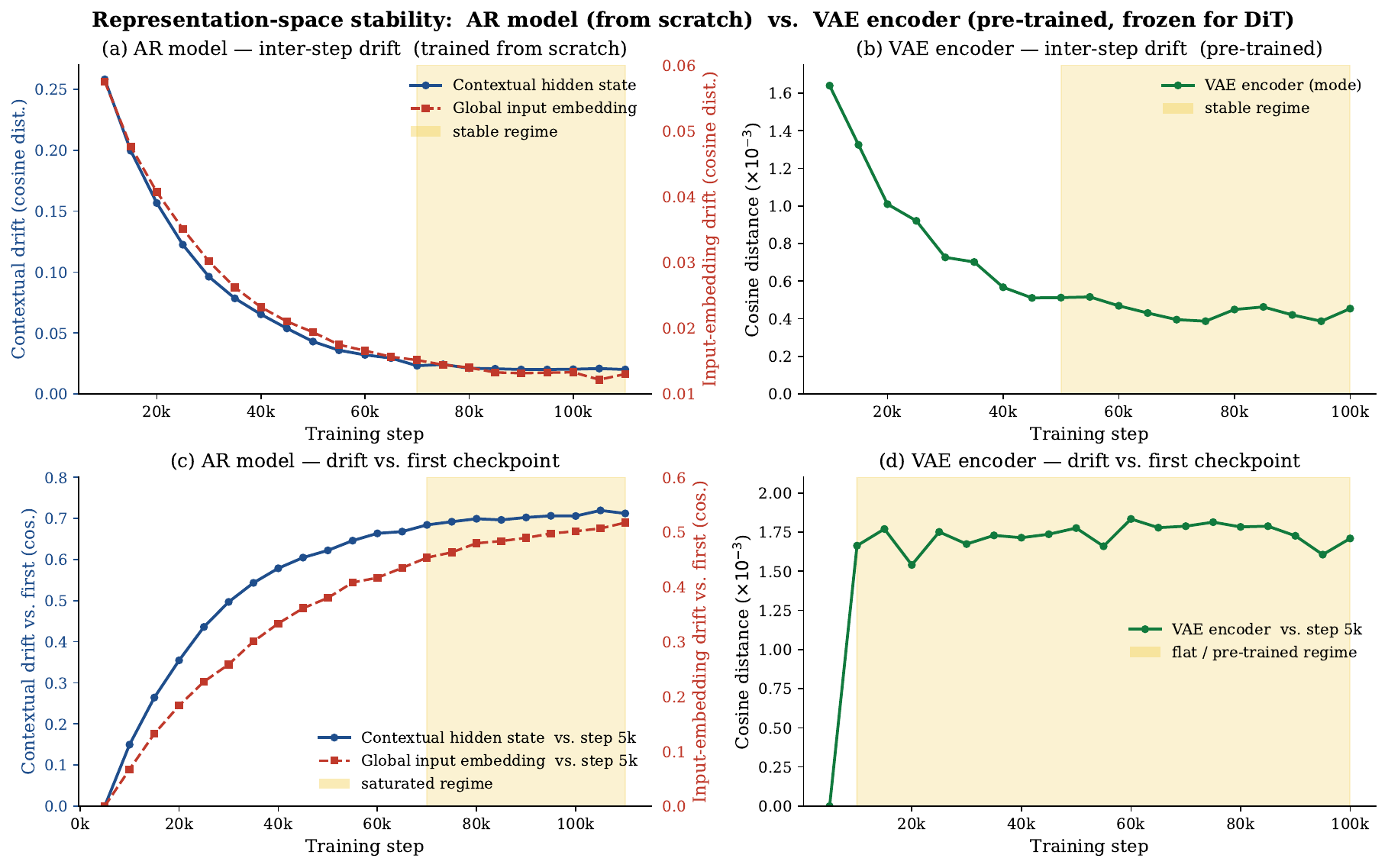}
    \caption{\textbf{Representation-space stability of AR training and pretrained VAE encoding.} Although the AR model learns its embedding and hidden-state geometry from scratch, and the VAE encoder starts from a pretrained state, both representations stabilize rapidly in the early stage of training. The inter-step drift in AR drops quickly and soon enters a stable regime, while the VAE encoder remains nearly flat throughout DiT training. This suggests that representation-space formation is largely an early-stage effect in both cases, and its influence on later generation performance is limited. The main generative burden therefore still lies in the subsequent backbone modeling, rather than in continued changes of the embedding or latent space itself.}
    \label{fig:fairness_embed_vae_latent}
\end{figure}

A potential concern is that Cola DLM includes an additional VAE pretraining stage, whereas the AR baseline is trained fully from scratch. We therefore clarify why this does not undermine the fairness of the comparison.

First, the pretrained VAE encoder is not itself the generative prior of Cola DLM. In our formulation, the generative model is defined by a decoder $p_\theta(x \mid z_0)$ and a latent prior $p_\psi(z_0)$, while $q_\phi(z_0 \mid x)$ is used only for variational inference and latent construction during training. At inference time, new response content is generated by the DiT prior block by block and then decoded into text. Therefore, the extra VAE stage should be understood as learning a stable text--latent interface, rather than pretraining an additional text generator.

Second, the overall comparison is already controlled along the main axes that determine pretraining scale and optimization difficulty. AR, LLaDA, and Cola DLM are aligned on the same training corpus, tokenizer family, sequence length, batch size, optimizer, learning-rate schedule, hardware setup, and evaluation protocol. We also keep the \emph{core generation backbone} in a comparable regime: AR and LLaDA use $\sim$1.8B non-embedding transformer parameters, while Cola DLM uses a $\sim$1.8B DiT prior; the remaining parameter difference comes from the additional VAE required by the continuous latent formulation. As a result, the comparison remains centered on the actual generation backbone rather than on mismatched training resources.

Third, Figure~\ref{fig:fairness_embed_vae_latent} provides an empirical perspective on representation-space stabilization. In the AR model, both the contextual hidden states and the global input embeddings show relatively large drift at the beginning of training, but quickly enter a stable regime. The inter-step drift becomes small after the early phase, and the distance to the first checkpoint also gradually saturates. The pretrained VAE encoder exhibits an even more explicit stabilization pattern: its inter-step drift drops rapidly and then remains nearly flat, while its distance to the early checkpoint changes very little after the initial stage. This indicates that, in both model families, the representation space becomes largely stable relatively early; once this happens, the dominant learning burden is no longer representation construction itself, but the modeling of generation dynamics on top of that space.

This observation is important for interpreting fairness. AR models also spend a nontrivial portion of training on forming and stabilizing their embedding and hidden-state geometry. Cola DLM makes this step explicit through VAE pretraining, but this does not mean that it receives a hidden pretrained generative advantage. Instead, it moves representation stabilization into a separate and modular stage, after which the main text-modeling burden still lies in the DiT prior and the conditional decoder. In this sense, VAE pretraining mainly amortizes representation learning, rather than replacing the need to learn the actual generative mechanism.

Moreover, the results in Figure~\ref{fig:rq2_fix_vs_evolve} already suggest that a pretrained encoder alone cannot explain the final gains. Although \textsc{Fix VAE} can be competitive at smaller compute budgets, it saturates earlier, whereas \textsc{Joint DiT x1} shows the strongest overall scaling trend. If the advantage mainly came from having a pretrained encoder by itself, the fixed-space variant should have remained dominant. Instead, the empirical pattern indicates that the key factor is the subsequent prior learning and co-adaptation with DiT, rather than a frozen pretrained representation alone.

Overall, we view the additional VAE stage as a method-specific representation-construction cost required by continuous latent modeling, analogous in spirit to the early embedding-space formation that AR models must also undergo during training. Once the representation space has stabilized, the comparison remains focused on the generative backbone under matched data, optimization, and evaluation settings. Therefore, the observed performance differences are more appropriately attributed to the modeling paradigm itself, rather than to an advantage from VAE pretraining.

\subsection{Fixed vs. Evolving Latent Space Additional Results}
\label{app:fix_vs_evolve_fulu}
\begin{figure}[t]
    \centering
    \begin{subfigure}[t]{0.49\linewidth}
        \centering
        \includegraphics[width=\linewidth]{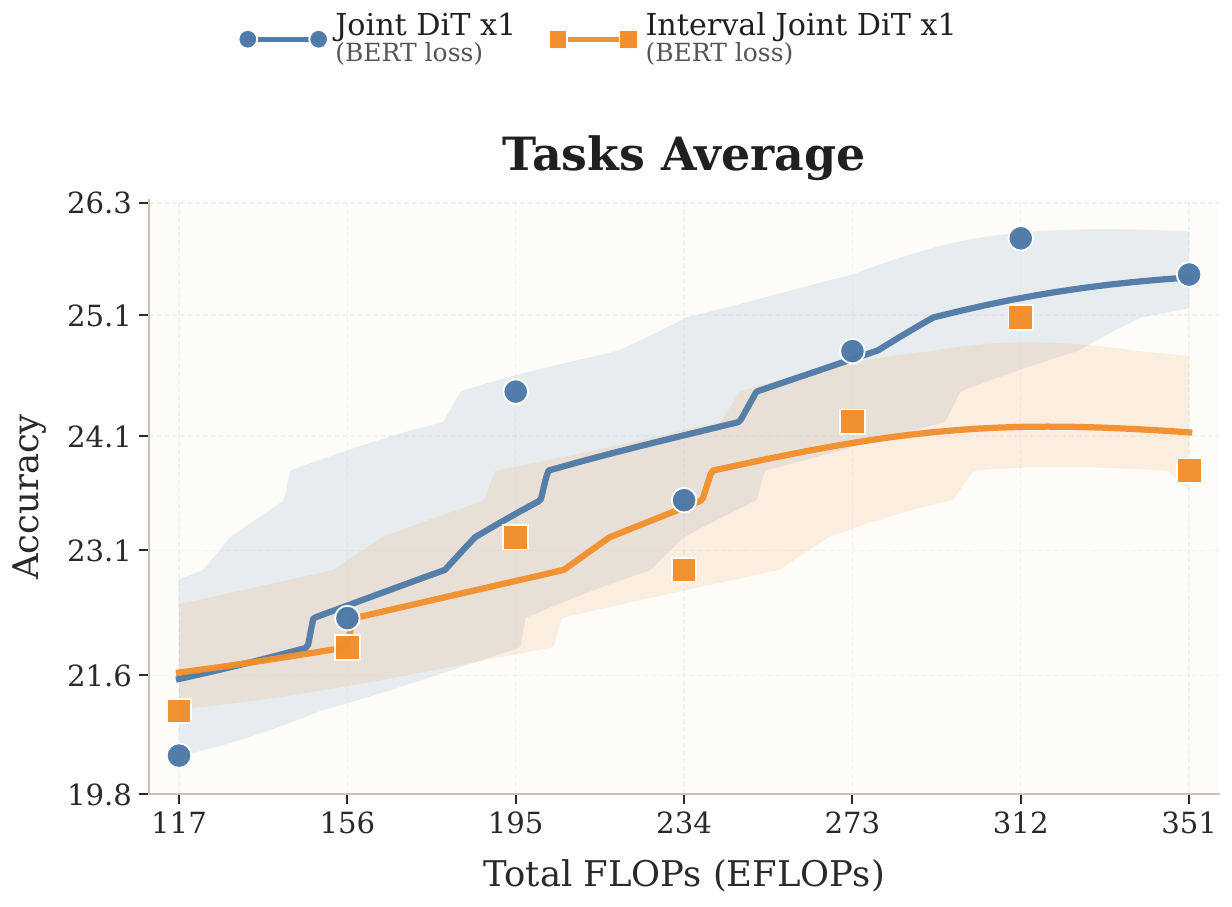}
        \label{fig:rq2_fix_vs_evolve_fulu_avg}
    \end{subfigure}
    \hfill
    \begin{subfigure}[t]{0.49\linewidth}
        \centering
        \includegraphics[width=\linewidth]{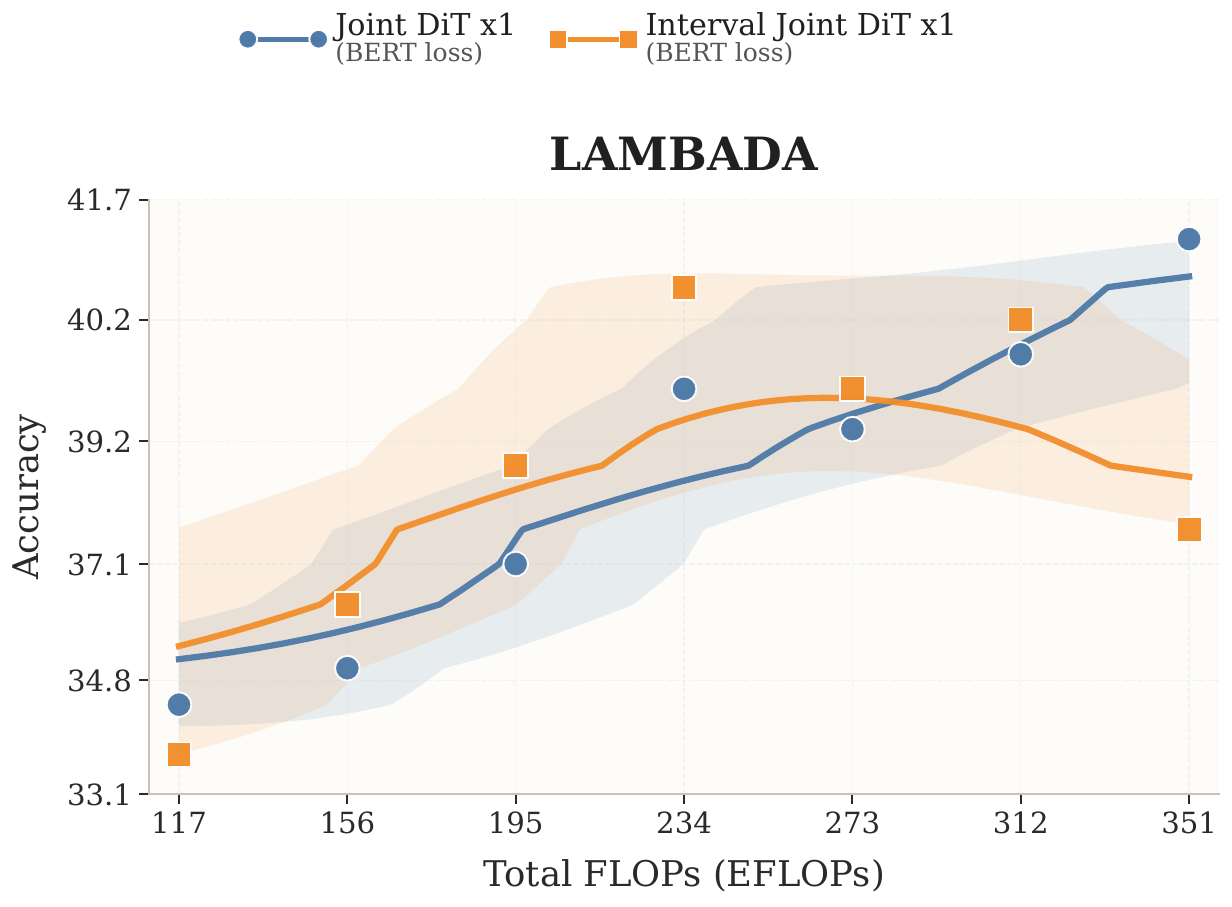}
        \label{fig:rq2_fix_vs_evolve_fulu_lambada}
    \end{subfigure}

    \vspace{0.5em}

    \begin{subfigure}[t]{0.49\linewidth}
        \centering
        \includegraphics[width=\linewidth]{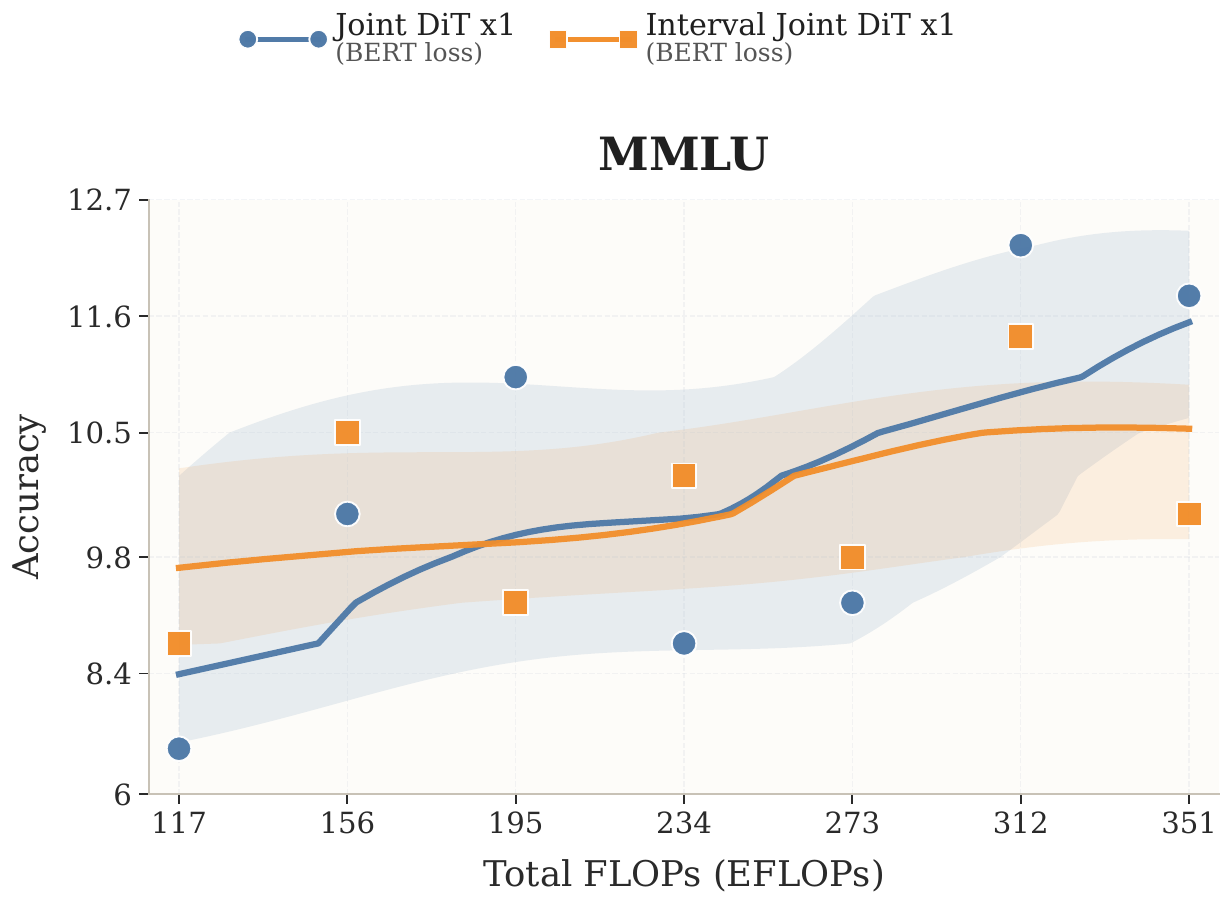}
        \label{fig:rq2_fix_vs_evolve_fulu_mmlu}
    \end{subfigure}
    \hfill
    \begin{subfigure}[t]{0.49\linewidth}
        \centering
        \includegraphics[width=\linewidth]{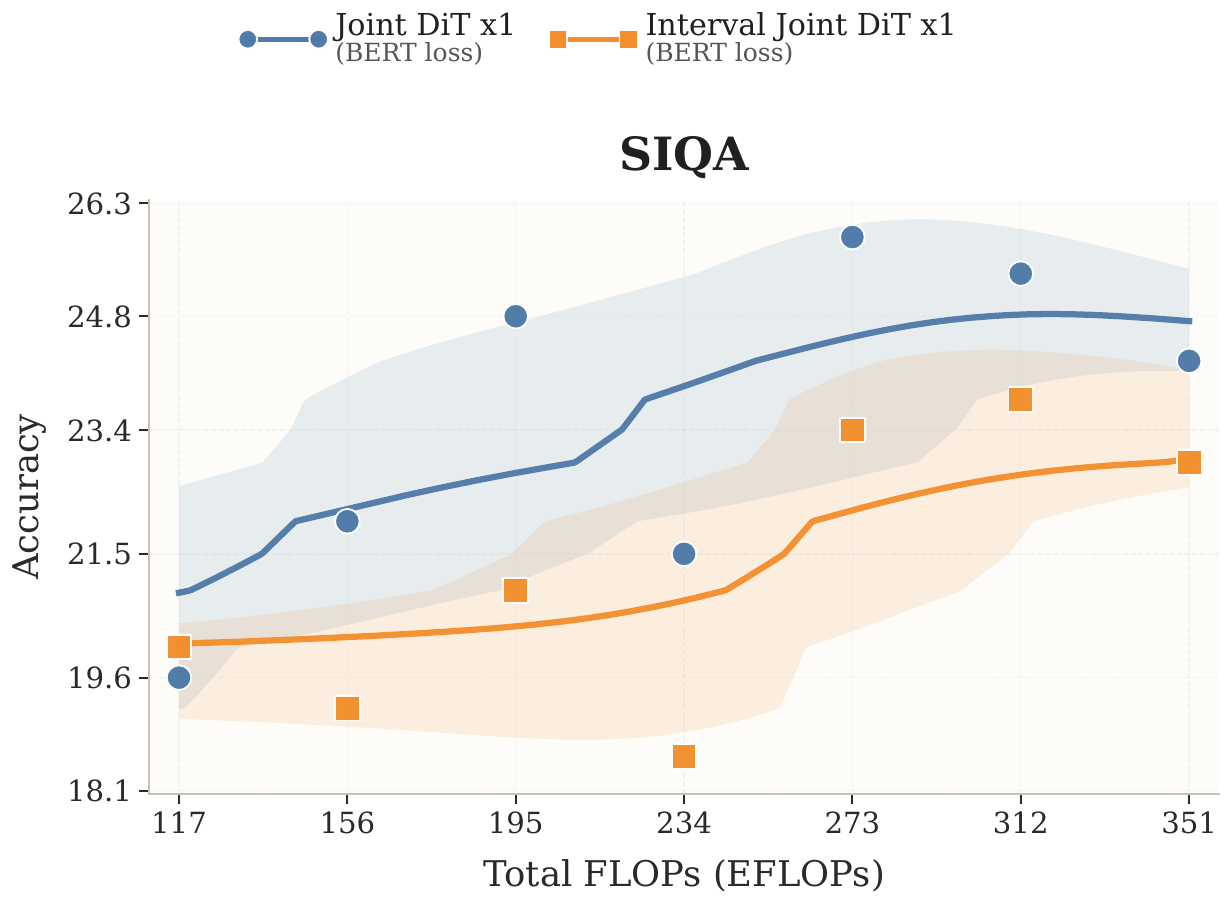}
        \label{fig:rq2_fix_vs_evolve_fulu_siqa}
    \end{subfigure}
    \caption{\textbf{Continuous co-evolution outperforms interval updates with pretrained VAE initialization.} Across Task Average, LAMBADA, MMLU, and SIQA, Joint DiT x1 still shows a stronger overall trend and better final performance than Interval, indicating that even with a stable pretrained VAE initialization, the latent space benefits more from sustained co-evolution with DiT than from periodic freezing.}
    \label{fig:rq2_fix_vs_evolve_fulu}
\end{figure}

As shown in Figure~\ref{fig:rq2_fix_vs_evolve_fulu}, we further compare continuous joint evolution and interval training under a stronger setting where the Interval variant is also initialized from a pretrained VAE. This setting should therefore be interpreted as isolating the effect of continuous versus intermittent latent updates after stable initialization, rather than the effect of initialization quality itself. Even under this stronger setup, Joint DiT x1 remains consistently better than Interval across Task Average, LAMBADA, MMLU, and SIQA. This suggests that the advantage of continuous latent evolution does not merely come from a better starting point. Instead, once a meaningful latent space has been established, the latent representation still benefits more from sustained co-evolution with DiT, while periodic freezing tends to interrupt the mutual adaptation between the evolving latent space and the DiT prior. And this results also further verify the last condition in Eq.~\eqref{eq:main_three_curves} of Section~\ref{sec:theory_advantage}, and provides strong support for the potential advantage of \method.

\subsection{VAE LogSNR Formula}
\label{app:vae_logsnr_calc}

For the VAE encoder posterior
\begin{equation}
q_{\phi}(z \mid x)=\mathcal{N}\!\left(\mu_{\phi}(x), \mathrm{diag}\!\left(\sigma_{\phi}^{2}(x)\right)\right),
\end{equation}
we define the \emph{VAE logSNR} as the log ratio between the signal power carried by the posterior mean and the noise power carried by the posterior variance:
\begin{equation}
\mathrm{logSNR}_{\mathrm{vae}}
=
\log
\frac{\mathbb{E}_{x,i}\!\left[\mu_{\phi,i}(x)^2\right]}
{\mathbb{E}_{x,i}\!\left[\sigma_{\phi,i}(x)^2\right]},
\label{eq:vae_logsnr_def}
\end{equation}
where \(i\) indexes all latent coordinates, including token/spatial positions and channels, and the expectation is taken over the data distribution and latent dimensions.

In practice, we estimate it empirically on a dataset or a mini-batch:
\begin{equation}
\widehat{\mathrm{logSNR}}_{\mathrm{vae}}
=
\log
\frac{
\frac{1}{N}\sum_{n=1}^{N}\frac{1}{D}\sum_{i=1}^{D}\mu_{n,i}^{2}
}{
\frac{1}{N}\sum_{n=1}^{N}\frac{1}{D}\sum_{i=1}^{D}\sigma_{n,i}^{2}
},
\label{eq:vae_logsnr_empirical}
\end{equation}
where \(N\) is the number of samples and \(D\) is the number of latent coordinates per sample. If the encoder predicts log-variance, we compute
\begin{equation}
\sigma_{n,i}^{2} = \exp\!\left(\log \sigma_{n,i}^{2}\right).
\end{equation}

A larger VAE logSNR indicates that the posterior mean dominates the posterior uncertainty, corresponding to a cleaner and more deterministic latent representation. Throughout the paper, when we refer to \emph{VAE logSNR}, we specifically mean the quantity defined in Eq.~\eqref{eq:vae_logsnr_def}.

\subsection{Noise Schedule Timestep Shift Visualization}
\label{app:timestep_shift_vis}
\begin{figure}[t]
    \centering
    \includegraphics[width=1\textwidth]{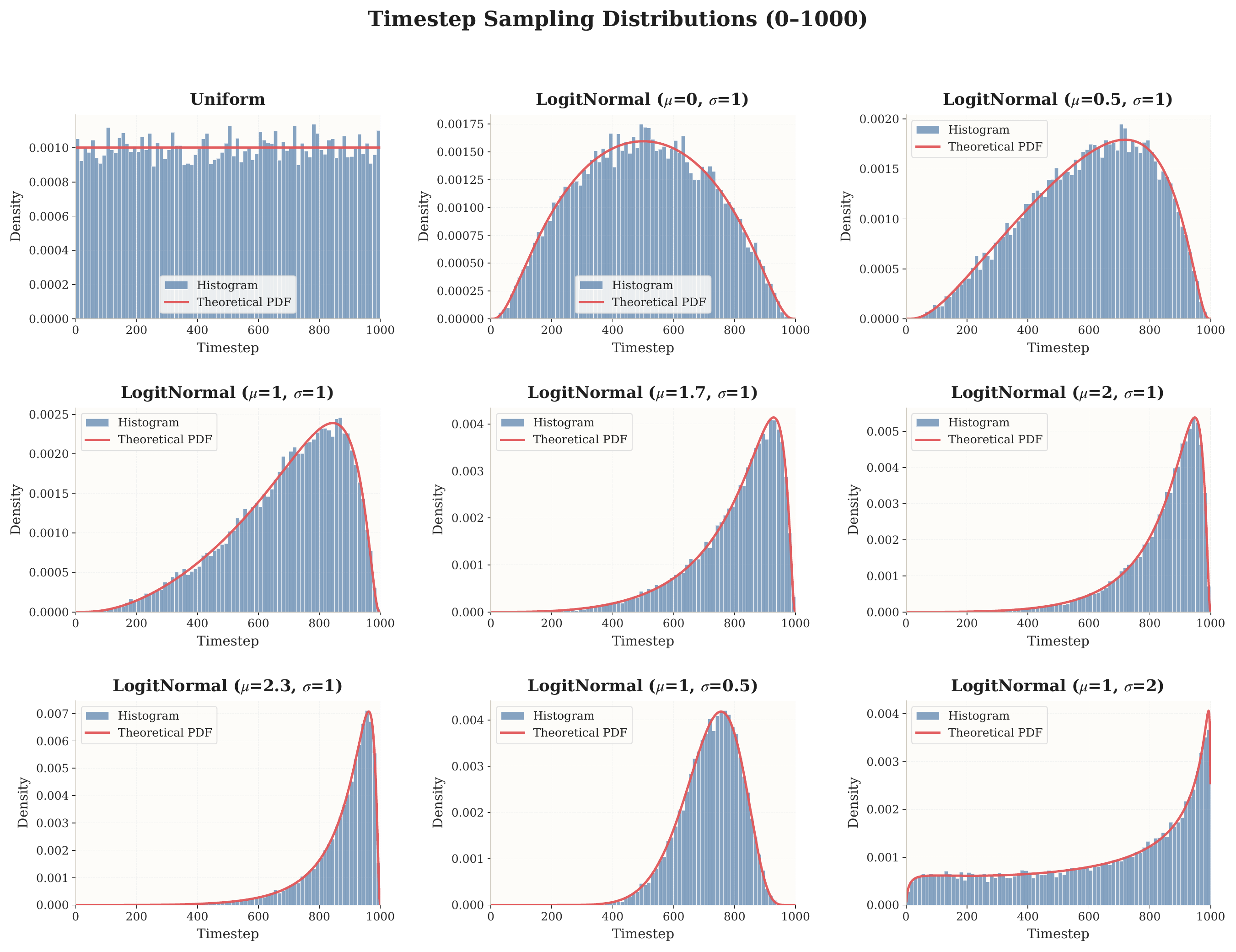}
    \caption{\textbf{Visualization of timestep-shifted sampling distributions.} The figure compares the uniform schedule with LogitNormal timestep-sampling distributions under different $\mu$ and $\sigma$. Increasing $\mu$ progressively shifts the sampling mass toward later timesteps, while changing $\sigma$ controls the sharpness of this concentration. This illustrates that timestep shift does not merely change a numerical sampling preference, but effectively changes the logSNR regime emphasized during training.}
    \label{fig:timestep_vis}
\end{figure}
As shown in Figure~\ref{fig:timestep_vis}, we visualize the timestep-sampling distributions induced by different timestep-shift settings. Compared with the uniform schedule, LogitNormal schedules place non-uniform emphasis on different parts of the denoising trajectory: increasing $\mu$ shifts the distribution toward later timesteps, while $\sigma$ controls how concentrated or dispersed the sampling mass is. From the perspective of the main text, this means that timestep shift effectively changes which semantic-information regime is more frequently emphasized during training, rather than simply tuning an isolated numerical hyperparameter.

\subsection{Timestep Shift Formula}
\label{app:timestep_shift_calc}

We implement timestep shift by replacing uniform timestep sampling with a LogitNormal distribution on the normalized timestep. Specifically, let
\begin{equation}
u \sim \mathcal{N}(\mu, \sigma^2),
\end{equation}
and map it to the unit interval by the logistic function:
\begin{equation}
s = \mathrm{sigmoid}(u) = \frac{1}{1+e^{-u}},
\qquad s \in (0,1).
\end{equation}
The continuous timestep is then defined as
\begin{equation}
t = T \cdot s = T \cdot \mathrm{sigmoid}(u),
\label{eq:timestep_shift_continuous}
\end{equation}
where \(T\) is the maximum diffusion timestep. In practice, a discrete training timestep is obtained by rounding or flooring \(t\) to the valid timestep grid.

Equivalently, the normalized timestep \(s=t/T\) follows a LogitNormal distribution:
\begin{equation}
s \sim \mathrm{LogitNormal}(\mu,\sigma^2).
\end{equation}
Its probability density function is
\begin{equation}
p(s;\mu,\sigma)
=
\frac{1}{\sigma \sqrt{2\pi}}
\frac{1}{s(1-s)}
\exp\!\left(
-\frac{(\log \frac{s}{1-s}-\mu)^2}{2\sigma^2}
\right),
\qquad s\in(0,1).
\label{eq:logitnormal_pdf}
\end{equation}
After rescaling to the timestep domain \(t\in(0,T)\), the corresponding density becomes
\begin{equation}
p(t;\mu,\sigma)
=
\frac{1}{T}\,
p\!\left(\frac{t}{T};\mu,\sigma\right).
\label{eq:timestep_shift_pdf}
\end{equation}

Under this parameterization, \(\mu\) mainly controls the location of the sampling mass: a larger \(\mu\) shifts probability toward later timesteps. In contrast, \(\sigma\) controls the spread of the distribution: a smaller \(\sigma\) produces a sharper concentration, while a larger \(\sigma\) yields a more dispersed schedule. Therefore, timestep shift directly changes the distribution of sampled logSNR regimes during training, rather than merely reweighting timestep indices numerically.

\section{Discussion}
\label{app:discussion}
\subsection{Understanding the Conditioning Strategies in the First Generation Block}
\label{app:first_block_conditioning_detail}

This subsection explains the conditioning strategies compared in Table~\ref{tab:first_block_conditioning}, and further analyzes them from the perspective of conditional Flow Matching in \method. The key issue arises in the \emph{first generation block}, where the model simultaneously sees a known prompt region and an unknown region to be generated. In other words, the input to this block is mixed: part of it should be preserved as condition, while the rest should be completed by the model. The methods in Table~\ref{tab:first_block_conditioning} differ precisely in how they handle this mixed denoising process.

\paragraph{What makes the first block special.}
Later generation blocks are purely generative: all latents in the block are unknown and are produced by denoising. The first generation block is different because it sits immediately after the prefix prompt. As a result, some positions in this block are already determined by the prompt, while the remaining positions are not. A good conditioning strategy should therefore achieve two goals at the same time: it should preserve the known region reliably, and it should allow the unknown region to grow coherently from that preserved context.

From the viewpoint of \method, this distinction is especially important because the prior is learned block by block through a conditional flow:
\[
p_\psi(z_0)=p_\psi(z_0^{(1)})\prod_{b=2}^{B} p_\psi(z_0^{(b)} \mid z_0^{(<b)}),
\]
and the corresponding conditional Flow Matching objective learns a vector field of the form
\[
v_\psi(z_t^{(b)}, t; z_0^{(<b)}).
\]
That is, the current noisy block is predicted under \emph{clean historical latent conditions}. At inference time, generation follows the same factorization by transporting a noise seed under the clean prefix and historical latent conditions. Therefore, the mathematically natural role of the known region in the first generation block is not that of a denoising target, but that of a fixed condition.

\paragraph{A decomposition of the first generation block.}
To make this point more explicit, let the first generation block be decomposed as
\[
z^{(1)} = (z_K, z_U),
\]
where \(z_K\) denotes the known region determined by the prefix, and \(z_U\) denotes the unknown region to be generated. In an ideal formulation, the task of the first generation block should be understood as
\[
\text{generate } z_U \text{ under the fixed condition } (z_{\mathrm{pre}}, z_K),
\]
rather than to jointly re-estimate both \(z_K\) and \(z_U\). Equivalently, the desired conditional prior is closer to
\[
p_\psi(z_U \mid z_{\mathrm{pre}}, z_K),
\]
where the known part serves as a stable boundary condition and only the unknown part is transported by the flow.

\paragraph{Partial repaint.}
The first family of methods in Table~\ref{tab:first_block_conditioning} is \emph{partial repaint}. These methods repeatedly revisit the early denoising steps and re-inject guidance into the known region during only part of the trajectory.

The parameter \(t\) controls how many repaint cycles are used inside each early denoising step. For example, \(t=1\) means that each step is executed once, while \(t=3\) means that the model denoises, re-noises, and denoises again multiple times before moving on. The parameter \(m\) controls how long the known region remains guided. When \(m=1.0\), guidance is applied throughout the whole denoising trajectory of the first block. When \(m=0.7\) or \(m=0.3\), the model only receives such guidance in the earlier portion of the trajectory, and then the remainder is left to evolve more freely.

The important point is that partial repaint does \emph{not} keep the known region perfectly fixed. Instead, it replaces the known region with a timestep-matched noisy surrogate during part of the trajectory and only nudges that region back toward the target at selected denoising steps. This makes the strategy softer and more flexible, but also fundamentally weaker as a conditioning mechanism.

\paragraph{A simple example for partial repaint.}
Consider a first generation block with four denoising transitions. If we use partial repaint with \(t=3\) and \(m=0.5\), then only the first half of the trajectory receives explicit guidance on the known region, while each of these early steps is revisited three times. Intuitively, the model repeatedly tries to correct the early transition, but once it enters the later half of the trajectory, the known region is no longer explicitly reinforced. The final result therefore depends on whether those earlier corrections are strong enough to keep the conditional structure stable.

\paragraph{Clean condition repaint.}
The method labeled \textbf{Clean cond.} in Table~\ref{tab:first_block_conditioning} takes a much stronger approach. Instead of providing only partial or noisy guidance, it keeps the known region directly aligned with the clean condition throughout the denoising process of the first block. In effect, the model always sees an exact and stable version of the prompt-conditioned part, while only the unknown region is allowed to change.

This strategy is much more conservative than partial repaint. It reduces ambiguity in the mixed region and makes the role of the first block very clear: preserve the prompt-conditioned part exactly, and generate the remaining part around it. This is especially useful when the boundary between known and unknown positions is semantically important.

\paragraph{A simple example for clean conditioning.}
Suppose the first few positions in the block correspond to prompt content and the remaining positions must be generated. Under clean condition repaint, the prompt side is always kept fixed, no matter how many denoising steps are taken. The model therefore never needs to re-discover or re-estimate that part. Its effort is focused entirely on making the unknown side compatible with a stable context. This is a much easier problem than simultaneously preserving the known side and generating the unknown side under drifting noise.

\paragraph{Left padding and right padding.}
The last two methods, \textbf{Left pad.} and \textbf{Right pad.}, should be understood as layout-based alternatives rather than repaint-based conditioning strategies. They change where the known condition is placed relative to the region to be generated, but they do not explicitly enforce the condition during denoising in the same way as repaint-based methods do.

Their role in the table is useful for interpretation. If padding alone were sufficient, then a simple positional arrangement would already match the performance of explicit conditioning. If not, then the gain must come from the conditioning mechanism itself rather than from layout alone.

\paragraph{A Flow Matching view of why clean conditioning works best.}
The advantage of clean condition repaint is not merely that it provides ``stronger guidance.'' More fundamentally, it is the only strategy among those compared that preserves the original conditional semantics of \method.

Under the standard Flow Matching formulation, let the target conditional velocity be denoted by \(u_t\), and let the optimal vector field be the conditional expectation
\[
v^*(z_t,t;c)=\mathbb{E}[u_t \mid z_t,t,c],
\]
where \(c\) denotes the conditioning information. For the first generation block, the ideal condition is
\[
c=(z_{\mathrm{pre}}, z_K),
\]
namely, the clean prefix latent together with the clean known region inside the first block. Under clean condition repaint, the model effectively solves the transport problem under exactly this condition. In the unknown region, this can be viewed as a conditional velocity regression of the form
\[
v^*_{\mathrm{clean}}(z_{U,t}, t; z_{\mathrm{pre}}, z_K)
=
\mathbb{E}[u_t^U \mid z_{U,t}, t, z_{\mathrm{pre}}, z_K].
\]
Although the implementation still predicts the vector field for the whole current block, the uncertainty that truly matters is concentrated on the unknown region, while the known region acts as a fixed boundary condition. This is precisely the role decomposition that the model is designed to exploit.

\paragraph{Why partial repaint is mathematically harder.}
Partial repaint changes this problem in a crucial way. Instead of conditioning on the true clean known region \(z_K\), it supplies a degraded and time-varying surrogate \(\tilde z_{K,t}\) during only part of the trajectory. The effective condition becomes
\[
\tilde c_t = (z_{\mathrm{pre}}, \tilde z_{K,t}),
\]
and the corresponding optimal regression target becomes
\[
v^*_{\mathrm{partial}}(z_{U,t}, t; z_{\mathrm{pre}}, \tilde z_{K,t})
=
\mathbb{E}[u_t^U \mid z_{U,t}, t, z_{\mathrm{pre}}, \tilde z_{K,t}].
\]
This is generally a more difficult regression problem. Since \(\tilde z_{K,t}\) is only a noisy surrogate of the true condition, the same \((z_{U,t}, t, \tilde z_{K,t})\) may correspond to a broader set of compatible clean targets. As a result, the conditional variance of the target velocity is larger. In terms of Bayes risk, one can compare
\[
\mathcal{R}(c)
=
\mathbb{E}\!\left[\mathrm{Var}(u_t^U \mid z_{U,t}, t, c)\right],
\qquad
\mathcal{R}(\tilde c_t)
=
\mathbb{E}\!\left[\mathrm{Var}(u_t^U \mid z_{U,t}, t, \tilde c_t)\right].
\]
Since \(\tilde c_t\) is a weaker and noisier version of \(c\), the second quantity is generally larger. Therefore, partial repaint increases the irreducible uncertainty of the conditional velocity regression itself. From the Flow Matching perspective, its weakness is not simply that the guidance is softer, but that it turns a clean conditional transport problem into a noisier mixed recovery problem.

\paragraph{Condition-role mismatch in partial repaint.}
There is also a structural mismatch. In \method, the flow path is used for \emph{prior transport}, not for observation recovery. Historical conditions are supposed to be stable anchors, while the current noisy block is the object being transported. Partial repaint violates this separation by converting the known region from a condition into a partially denoised state variable. In other words, it changes the task from
\[
\text{transport the unknown part under fixed condition}
\]
to
\[
\text{jointly maintain a noisy known part and transport the unknown part}.
\]
This role confusion is especially harmful in the first generation block, because the known and unknown regions are tightly coupled inside the same block. Once the known side drifts, the unknown side no longer grows from a stable semantic anchor.

\paragraph{Why errors accumulate along the flow trajectory.}
The weakness of partial repaint is further amplified by the continuous nature of flow-based generation. Let the ideal vector field under clean conditioning be \(v^*(z,t;c)\), and let the field induced by partial repaint be
\[
\tilde v(z,t) = v^*(z,t;c) + \delta(z,t),
\]
where \(\delta\) denotes the condition-induced bias. Since inference transports the latent by integrating the learned vector field along an ODE trajectory, this bias accumulates over time. Under standard stability assumptions, the trajectory error satisfies a bound of the form
\[
\|\hat z_t - z_t^*\|
\le
e^{Lt}\int_0^t \|\delta(z_s,s)\|\, ds,
\]
where \(L\) is a Lipschitz constant of the vector field. This expression makes the weakness of partial repaint especially clear.

First, reducing \(m\) makes performance worse because the explicit guidance only covers an earlier portion of the trajectory. Once guidance is removed, the later part evolves under a mismatched condition, so the induced bias continues to accumulate during the unguided interval. Second, increasing \(t\) from \(1\) to \(3\) does not provide stable gains because repeated corrections in the early stage cannot eliminate the structural mismatch in the later stage. In other words, repaint can repeatedly repair local early transitions, but it cannot turn a transient condition into a persistent one.

\paragraph{Why the first block is more sensitive than later blocks.}
Another reason for the strong gap is that the first generation block is not a purely unknown block. The known and unknown regions coexist and interact inside the same block. Therefore, any instability in the known region is immediately transmitted to the unknown region through the within-block interactions. In latent space, this effect is even stronger because the known region functions as a semantic anchor rather than a mere token-level hint. If that anchor is kept clean, the model only needs to fit the unknown part to a stable context. If that anchor becomes noisy and drifts over time, the model must simultaneously infer the boundary condition and generate the completion, which is a fundamentally harder problem.

\paragraph{Why padding is better than partial repaint but still worse than clean conditioning.}
Left and right padding do not explicitly re-noise the known region, so they avoid the most harmful failure mode of partial repaint. This is why they are often stronger than most partial repaint settings. However, padding alone still does not provide a fully stable conditioning mechanism throughout denoising. It changes the positional layout, but it does not lock the known region as an exact condition in the same way as clean condition repaint. Moreover, padding can also make the block-causal attention pattern more complicated: by reshaping the relative arrangement between the known and unknown regions, it may introduce a less natural or less aligned dependency structure for the model to follow during denoising. Therefore, padding can be helpful, but it cannot match the benefit of preserving a stable condition explicitly.

\paragraph{A unified intuition across all methods.}
All methods in Table~\ref{tab:first_block_conditioning} are trying to solve the same problem: the first generation block contains both a part that should remain faithful to the prompt and a part that should be newly generated. The difference lies in how strongly the known region is protected.

Partial repaint protects it only temporarily and only approximately. Clean condition repaint protects it strongly and continuously. Left and right padding do not explicitly protect it during denoising, but instead hope that a favorable layout will make the interaction easier for the model.

\paragraph{How to read Table~\ref{tab:first_block_conditioning}.}
The results show a very clear pattern. Clean condition repaint consistently performs best across all three tasks. This suggests that, for the mixed denoising problem in the first block, strong and persistent conditioning is much more effective than soft partial correction. By contrast, reducing \(m\) in partial repaint generally hurts performance, which indicates that shortening the guided portion of the trajectory makes the known region harder to preserve. Increasing the repaint cycles from \(t=1\) to \(t=3\) does not produce stable gains either, suggesting that repeatedly revisiting early denoising steps cannot by itself compensate for weak conditioning.

Another useful observation is that left and right padding are often stronger than partial repaint, but still remain clearly below clean condition repaint. This indicates that positional arrangement does matter, but it is not the main factor. In particular, although padding may alleviate some interference by changing the layout, it can also make the block-causal attention structure less clean and more difficult to optimize. The dominant advantage still comes from explicitly maintaining a stable condition in the known region throughout denoising.

\paragraph{Practical takeaway.}
The first generation block is not merely a smaller version of later generation blocks. It is a special mixed region where conditioning and generation must coexist. From the perspective of conditional Flow Matching, the key challenge is to preserve the known region as a stable condition while transporting only the unknown region. The experiments suggest that this regime benefits most from a strong and explicit preservation of the known condition. In practice, this is exactly what clean condition repaint provides, which explains why it is consistently the strongest strategy in Table~\ref{tab:first_block_conditioning}.

\subsection{More Unified Approach Results with Image Modalities}
\label{app:unified_results}
In this section, we present additional samples that demonstrate cross-modal generation across different input and output modalities. It is worth emphasizing that all text and image outputs are produced by a single unified model. All tasks are pretrained from scratch, and the model is trained jointly on multiple task types, with the relative data proportions adjusted across different training stages. Although the model is trained with only a very limited amount of data, the resulting samples already exhibit promising quality, suggesting that \method has substantial potential for multimodal understanding and generation.

\subsubsection{Text to Text}
In this subsection, we present examples of text-to-text generation. As shown in the samples, the generated text is generally coherent and fluent, despite the fact that the model is trained on only approximately 1B text tokens in total.

\begin{figure*}[t]
    \centering
    \fcolorbox{blue!50}{blue!5}{
    \begin{minipage}{0.97\textwidth}
    \vspace{0.4em}
    {\large\bfseries Text to Text}\par
    \vspace{0.6em}
    \footnotesize

    \begin{minipage}[t]{0.485\textwidth}

    \textbf{Prefix Prompt:} You now have the opportunity to really learn a lot. Just ask when you do not know something; the worst thing you can do is pretend that you understand.\\
    \textbf{Generation:} This is great for you, so do not be intimidated. Instead, show your openness and willingness to learn, and I am sure anyone on your team will be more than happy to help.

    \vspace{0.7em}
    \hrule
    \vspace{0.7em}

    \textbf{Prefix Prompt:} As an adult, your hair is often one of the first indicators used to describe you. You are often known as the friend with the curls. If you then find a stylist who understands and respects curls, you will eventually relay your entire hair history to her.\\
    \textbf{Generation:} But we all do this with a knowing look and a slightly sarcastic grin. This is a common, shared history, one that curly people seem to recognize in one another immediately.

    \vspace{0.7em}
    \hrule
    \vspace{0.7em}

    \textbf{Prefix Prompt:} Comparative Advantage: Nobel laureate Paul Samuelson was once challenged to name one proposition in the social sciences that is both true and non-trivial. Several years later, he gave the correct response: comparative advantage.\\
    \textbf{Generation:} That it is logically true need not be argued before a mathematician; that it is not trivial is attested by the thousands of intelligent men who have never been able to grasp the doctrine for themselves or to believe it after it was explained to them.

    \vspace{0.7em}
    \hrule
    \vspace{0.7em}

    \textbf{Prefix Prompt:} How do I draw a circle with arrows in PowerPoint? Numbered Circular Arrows in PowerPoint using SmartArt.\\
    \textbf{Generation:} Step 1: Insert the relevant SmartArt graphic. Go to Insert, then SmartArt, then Cycle, and choose the segmented cycle graphic. Step 2: Convert it to shapes and continue editing from there.

    \vspace{0.7em}
    \hrule
    \vspace{0.7em}

    \textbf{Prefix Prompt:} ``It is N'Nance,'' she said. That sent a small chill up my spine. She nodded at the horrified expression on my face. ``He has finally had his big accident.'' I knew what that meant. ``I will be right there.''\\
    \textbf{Generation:} I turned to my own patient and said, ``Athena, can you lock this down by yourself?'' The cybertank grumbled her engine for a second, then said, ``I can manage. I will hold things here until you get back.''

    \vspace{0.7em}
    \hrule
    \vspace{0.7em}

    \textbf{Prefix Prompt:} Subject 3: Drilling and Completion of Coal Bed Methane Wells. The drilling program for a CBM well is generally a less complex process than that for other types of unconventional reservoirs.\\
    \textbf{Generation:} This is true for several reasons: CBM reservoirs are generally shallow, commonly ranging from roughly 500 to 3000 feet in depth, and their operational requirements are often more straightforward than those of many other unconventional wells.

    \end{minipage}
    \hfill
    \begin{minipage}[t]{0.485\textwidth}

    \textbf{Prefix Prompt:} Here is why the Indian Army, Navy, and Air Force salute in a different way. Probably no one ever noticed that the salute styles of the Indian Army, Navy, and Air Force are different from each other.\\
    \textbf{Generation:} The gesture of respect comes with such modest reasons that you will be amazed at the chivalry of the defence forces. When they raise their hands for the salute, they do not just greet, but honor a long-standing tradition.

    \vspace{0.7em}
    \hrule
    \vspace{0.7em}

    \textbf{Prefix Prompt:} ``Dark Knight Begins!'' Walkthrough to the Batman animated movies. DC Universe is a world filled with iconic characters, and many fans spend hours watching its animated series.\\
    \textbf{Generation:} Since every kid grows up knowing Batman and Superman, these stories remain a natural entry point into the darker side of the DC universe, especially for viewers who want an accessible introduction to its animated films.

    \vspace{0.7em}
    \hrule
    \vspace{0.7em}

    \textbf{Prefix Prompt:} Eddie Huang has circled back to his roots to close out \textit{Fresh Off the Boat} Season 2, and Part 1 of the New York episode travels to the Bronx, which, he learns from some locals itching to get out, will probably never gentrify.\\
    \textbf{Generation:} Honestly, though, I am good with cake and aggressive pedicures, both of which the Bronx seems to have in spades. Check out \textit{Fresh Off the Boat} and see how the episode leans into the borough's character and humor.

    \vspace{0.7em}
    \hrule
    \vspace{0.7em}

    \textbf{Prefix Prompt:} Those who read Nietzsche in English translation have been lied to, deceived, seduced, and hoodwinked by dishonest translators and commentators.\\
    \textbf{Generation:} My intention here will be twofold. First, to correct some of the more serious misinterpretations in the introduction to the Cambridge University Press translation of Nietzsche's \textit{Morgenr\"othe}, and second, to clarify what is at stake in reading the text more carefully.

    \vspace{0.7em}
    \hrule
    \vspace{0.7em}

    \textbf{Prefix Prompt:} First off, thank Hashem that so few people in Israel have been hurt. It is not for lack of trying on the part of our foes.\\
    \textbf{Generation:} It is miraculous and wonderful. In reply to Dan, I am a physicist, and I also wondered when I first saw the statistics. Although I would be cautious about overinterpreting them, the numbers are certainly striking.

    \vspace{0.7em}
    \hrule
    \vspace{0.7em}

    \textbf{Prefix Prompt:} Europeans have a higher tolerance for government intrusiveness than do Americans. I do not think anything in that statement will surprise my readers.\\
    \textbf{Generation:} The contrast becomes even more visible in television, where assumptions about privacy, regulation, and public life often emerge in surprisingly casual ways and reveal broader cultural differences.

    \end{minipage}

    \vspace{0.5em}
    \end{minipage}}
    \caption{\textbf{Additional text-to-text examples.} We show more qualitative continuation results from the preliminary unified prototype. The examples span conversational, narrative, expository, technical, and article-style text, illustrating that the unified latent-prior formulation can support coherent continuation across heterogeneous textual styles.}
    \label{fig:app_unified_t2t}
\end{figure*}

\subsubsection{Text to Image}
In this subsection, we further evaluate the model's ability to perform text-to-image generation. 

\newcommand{\titosample}[2]{%
\begin{minipage}[t]{0.315\textwidth}
    \centering
    \includegraphics[width=\linewidth]{exp_fig/discussion/unified/t2i_samples/#1}
    \vspace{0.35em}

    {\scriptsize\raggedright\textbf{Prompt:} #2\par}
\end{minipage}
}

\begin{figure*}[t]
    \centering
    \titosample{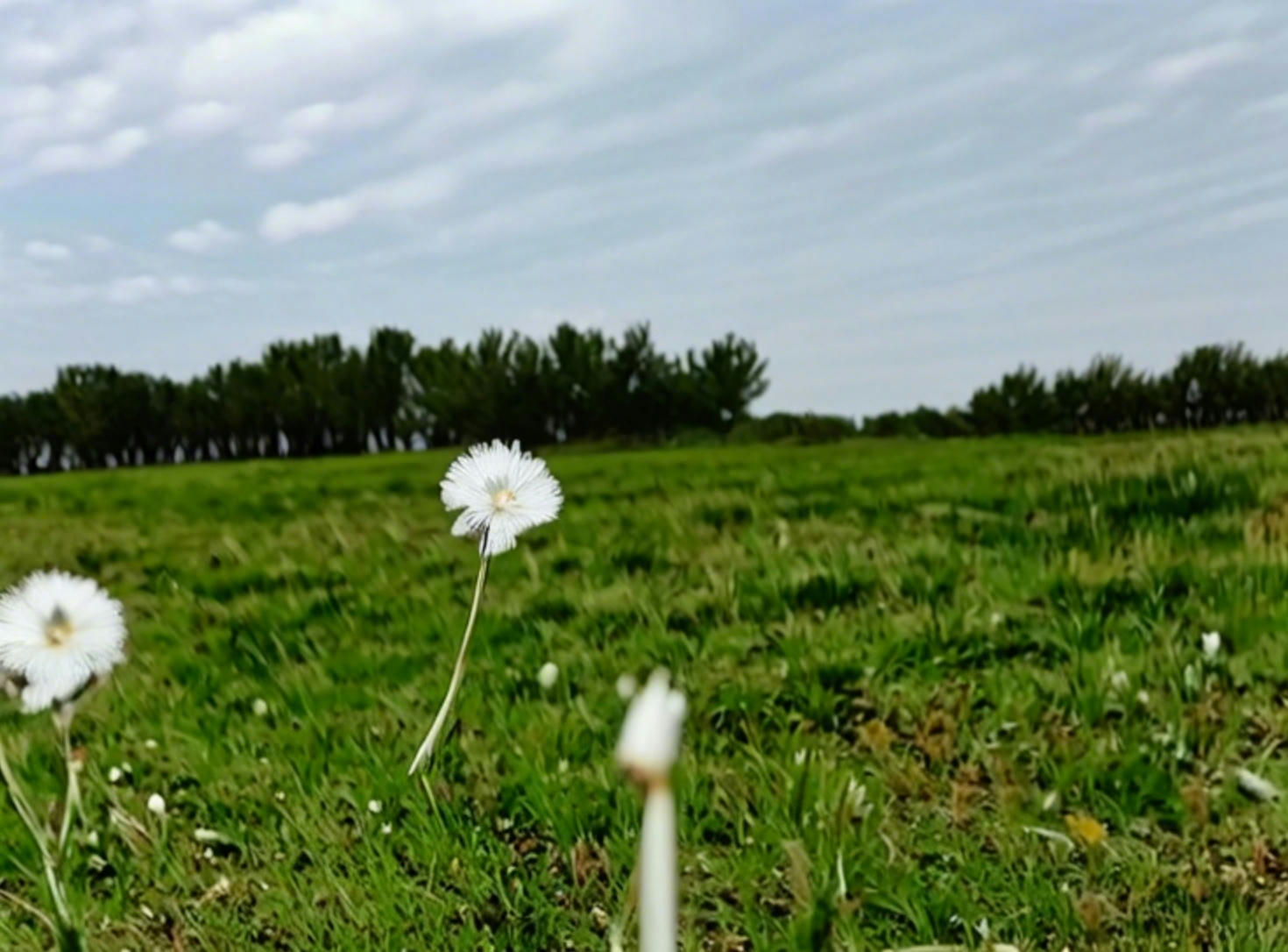}{A delicate white dandelion seed head stands above a green meadow, with other soft, blurred flowers nearby. A dark row of trees and a cloudy blue-gray sky give the open field a calm, breezy spring atmosphere.}
    \hfill
    \titosample{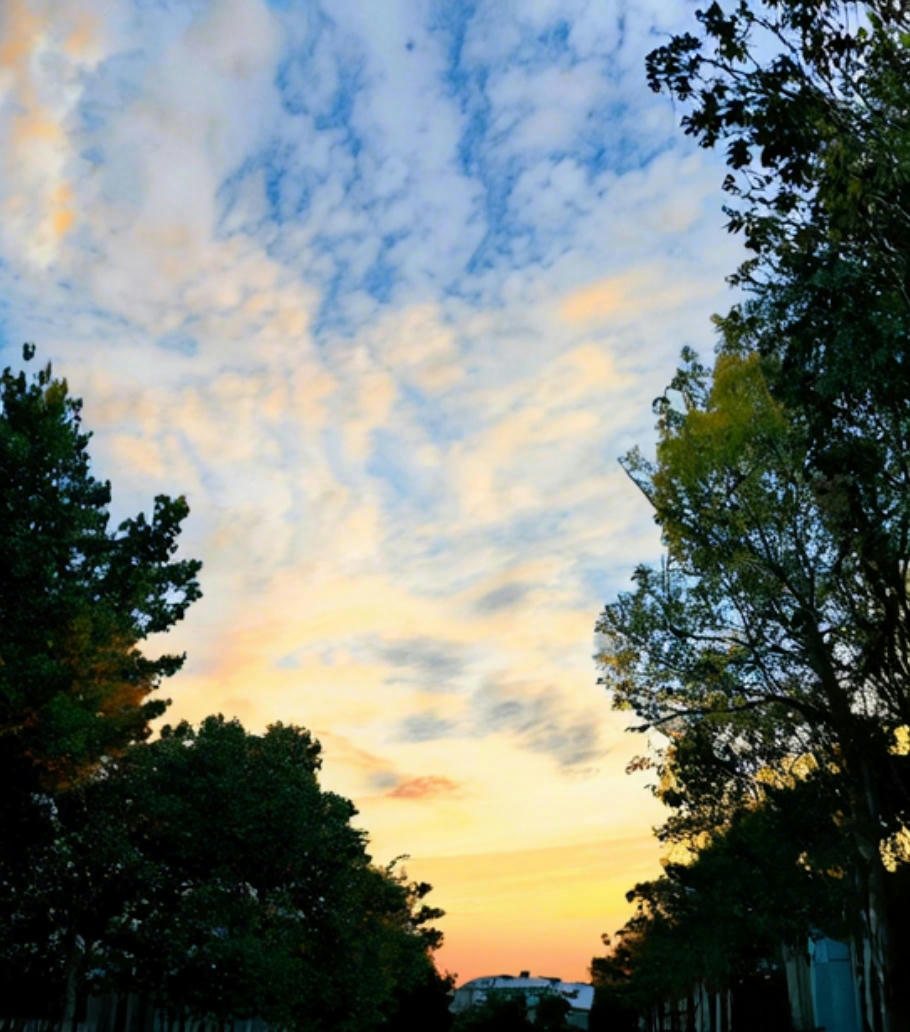}{Tall trees on both sides frame a dramatic sunset sky filled with glowing orange, pale yellow, and blue cloud textures. The dark silhouettes below make the scene feel quiet, spacious, and gently illuminated by evening light.}
    \hfill
    \titosample{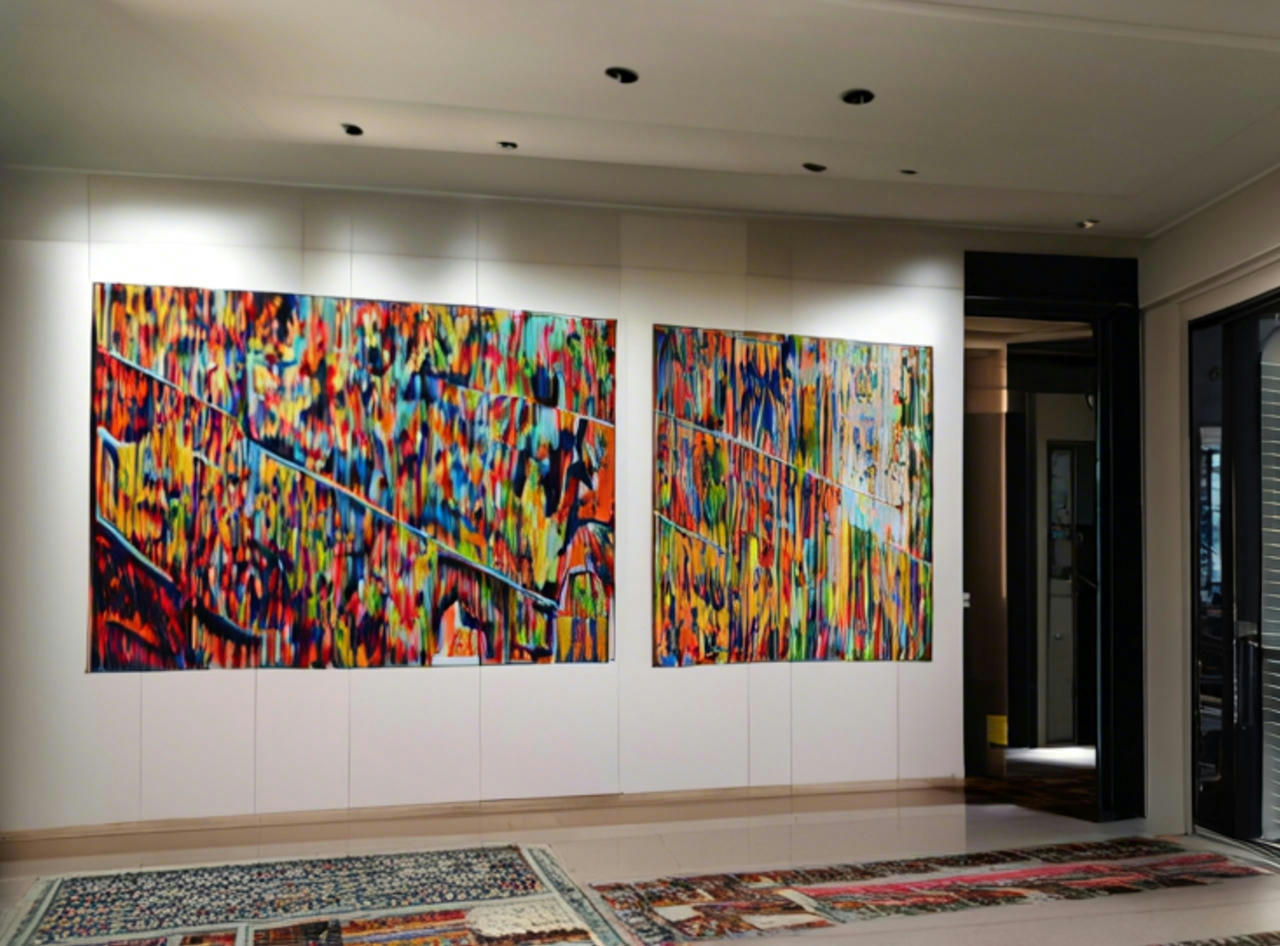}{Two large abstract paintings dominate a polished interior hallway, their vivid colors contrasting with the neutral wall panels and glossy floor. Ceiling spotlights, patterned rugs, and nearby doorways create the feeling of a private gallery or elegant office space.}

    \vspace{0.9em}

    \titosample{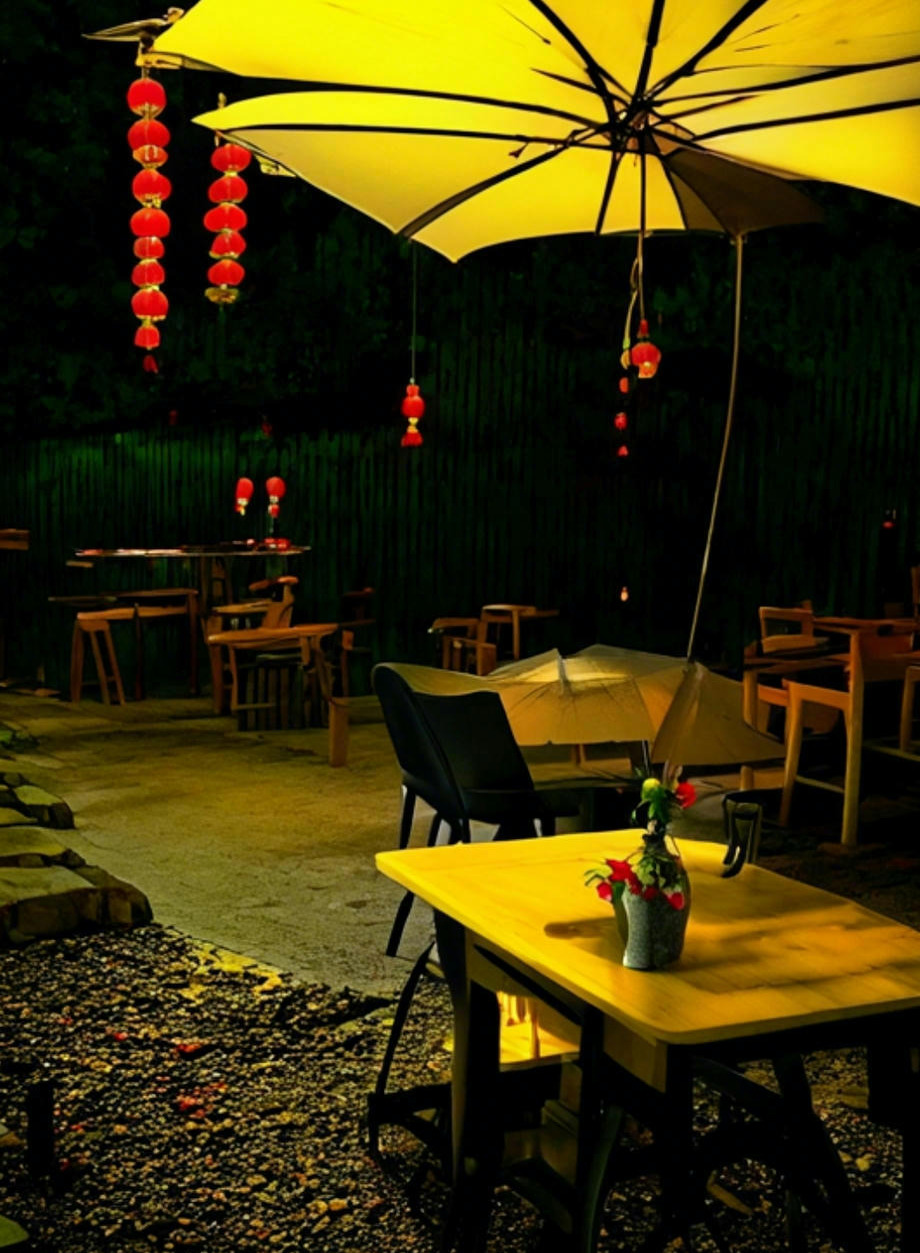}{A cozy outdoor dining area is lit by a large yellow umbrella and several red hanging lanterns. Wooden tables, simple chairs, and warm pools of light create an intimate evening setting with a festive, traditional atmosphere.}
    \hfill
    \titosample{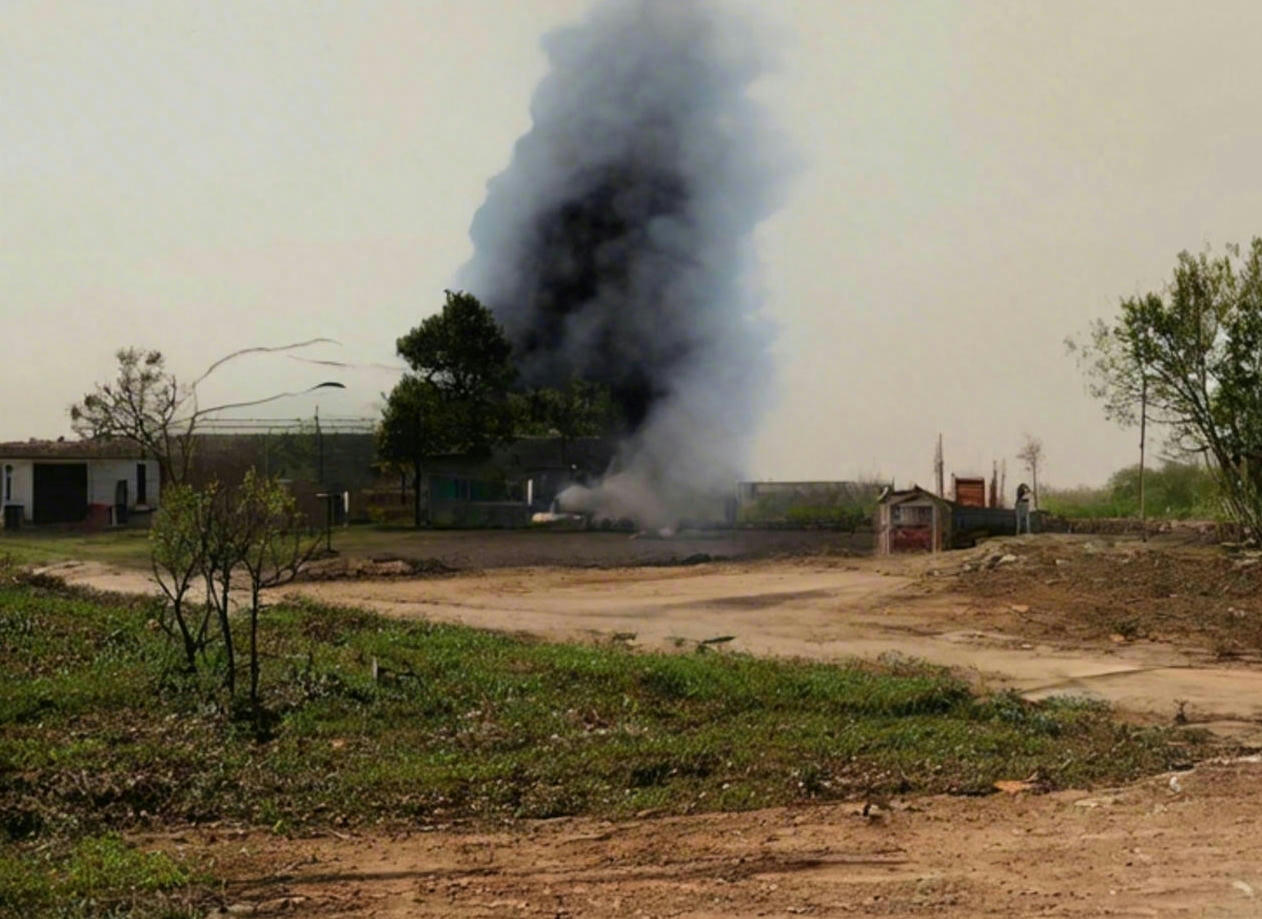}{A dense column of black and gray smoke rises from a rural property near small buildings and scattered trees. The surrounding dirt paths, fields, and pale sky make the sudden fire or burning event feel stark and unsettling.}
    \hfill
    \titosample{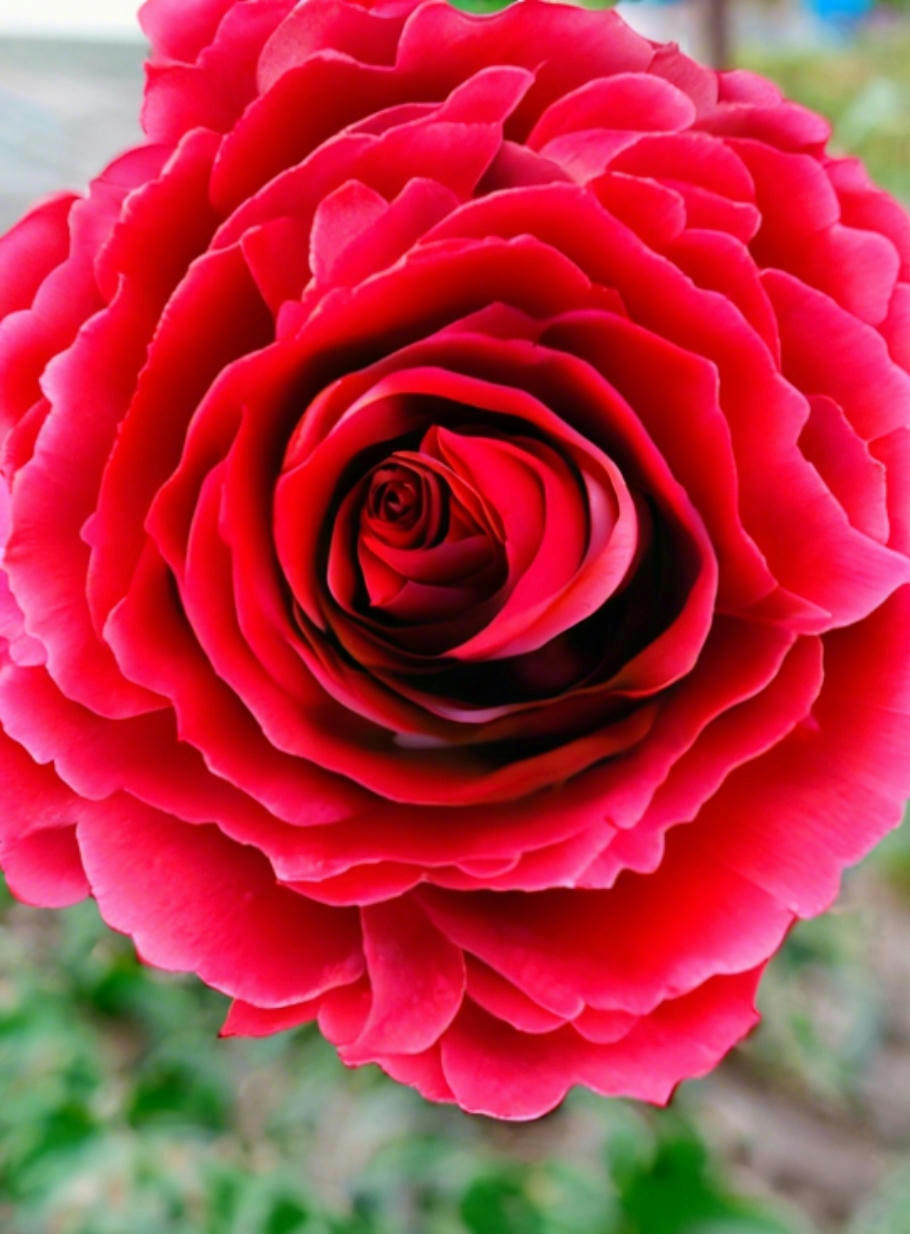}{A vivid red rose fills the image in extreme close-up, revealing many layered petals curling toward a tight spiral at the center. The blurred green background keeps attention on the flower's saturated color and soft texture.}

    \caption{\textbf{Text-to-image examples (Part I).} Qualitative samples generated by the unified prototype from descriptive natural-language prompts, covering open landscapes, interiors, dining scenes, and floral close-ups.}
    \label{fig:app_unified_t2i_1}
\end{figure*}

\begin{figure*}[t]
    \centering
    \titosample{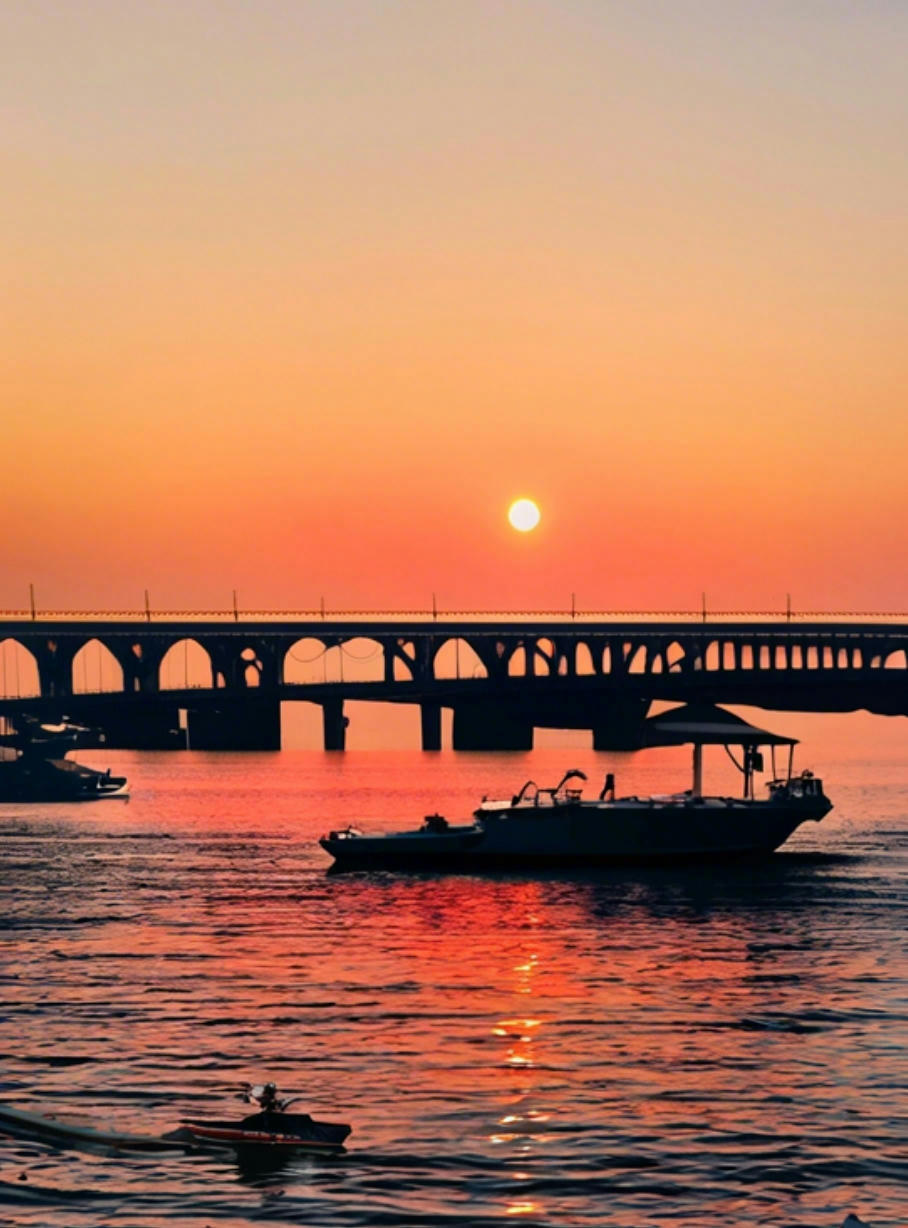}{Several boats glide across rippling water during sunset, appearing as dark silhouettes beneath a long arched bridge. The orange sun and its reflection stretch across the river, giving the scene a warm, tranquil waterfront mood.}
    \hfill
    \titosample{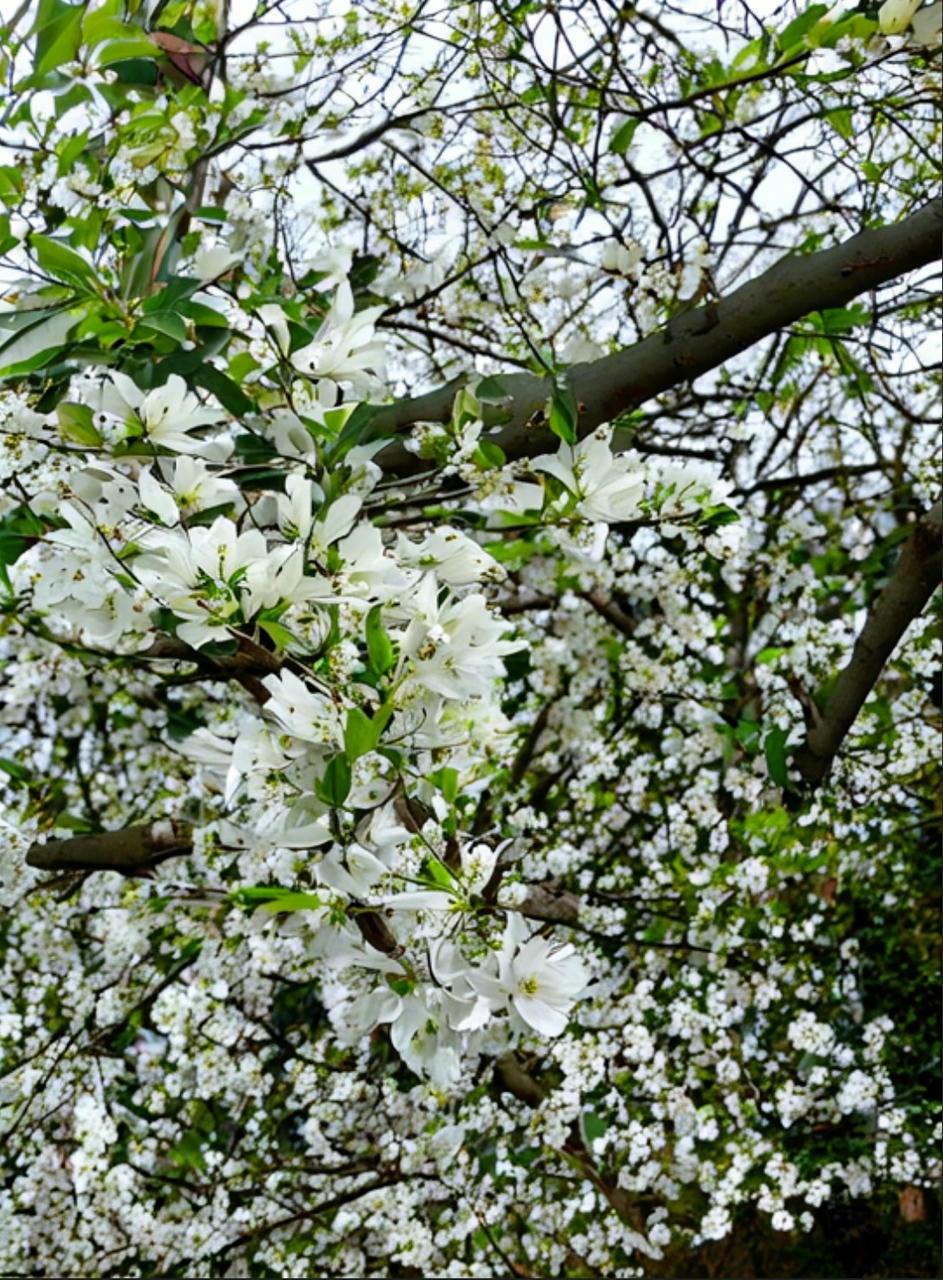}{White spring blossoms cluster thickly along interwoven tree branches, mixed with small green leaves and unopened buds. The shallow depth of field makes the foreground flowers stand out while the background turns into a delicate canopy of pale bloom.}
    \hfill
    \titosample{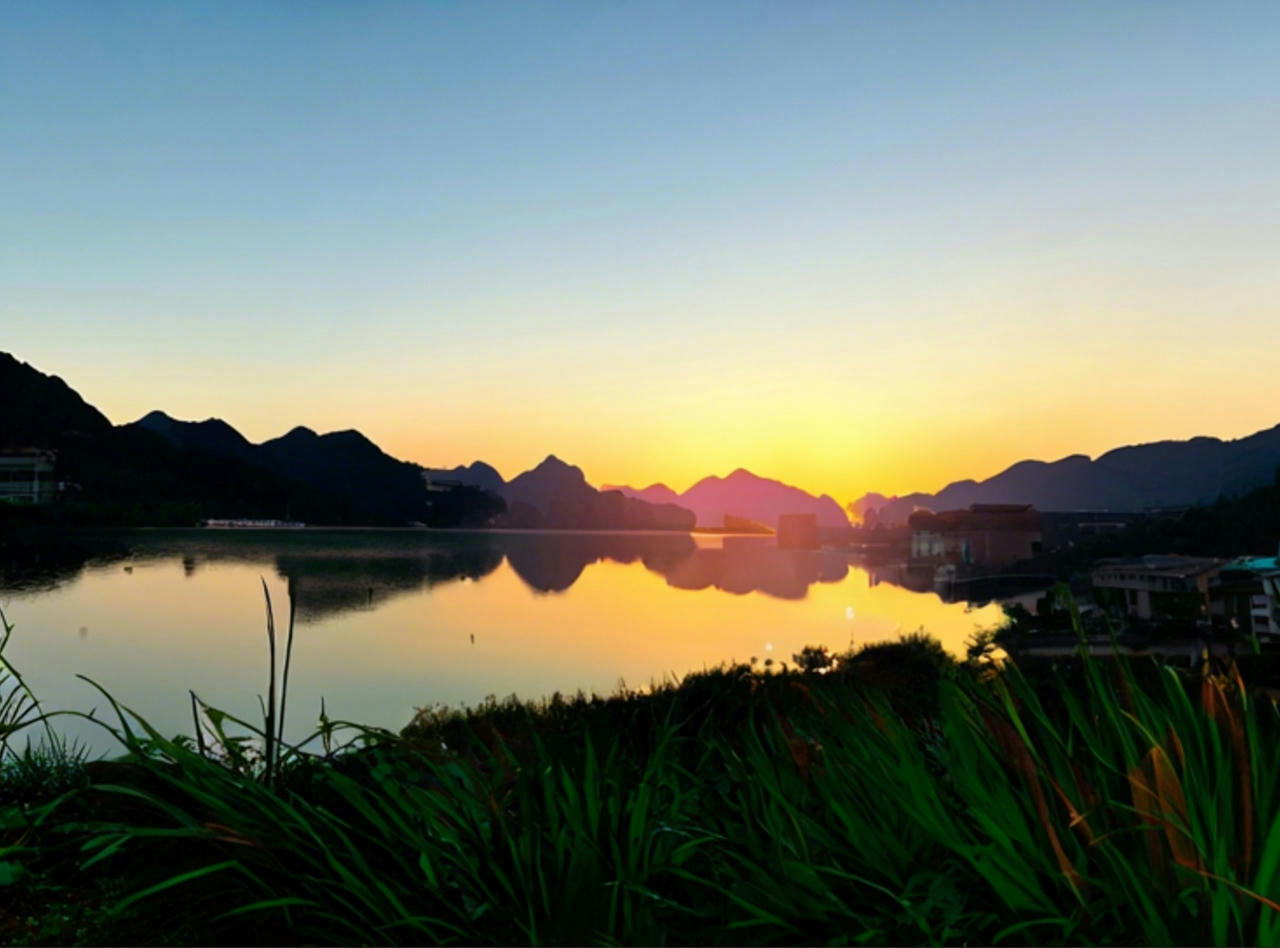}{Golden sunset light spreads behind a line of dark mountain silhouettes and reflects across the still lake below. Grasses in the foreground add depth, while the mirrored peaks and warm sky create a peaceful landscape composition.}

    \vspace{0.9em}

    \titosample{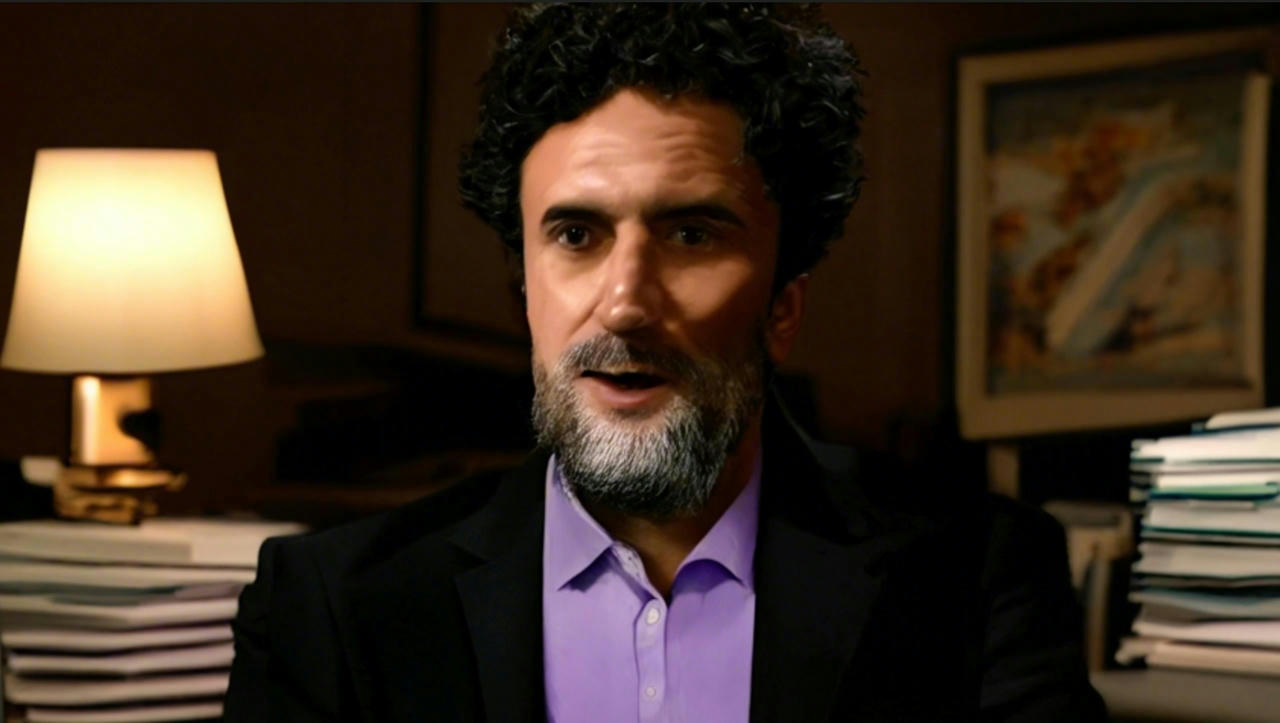}{A bearded man with curly hair sits indoors wearing a dark blazer and lavender shirt. Warm lamp light, stacks of documents, and framed artwork behind him suggest an office or study, giving the portrait a serious conversational tone.}
    \hfill
    \titosample{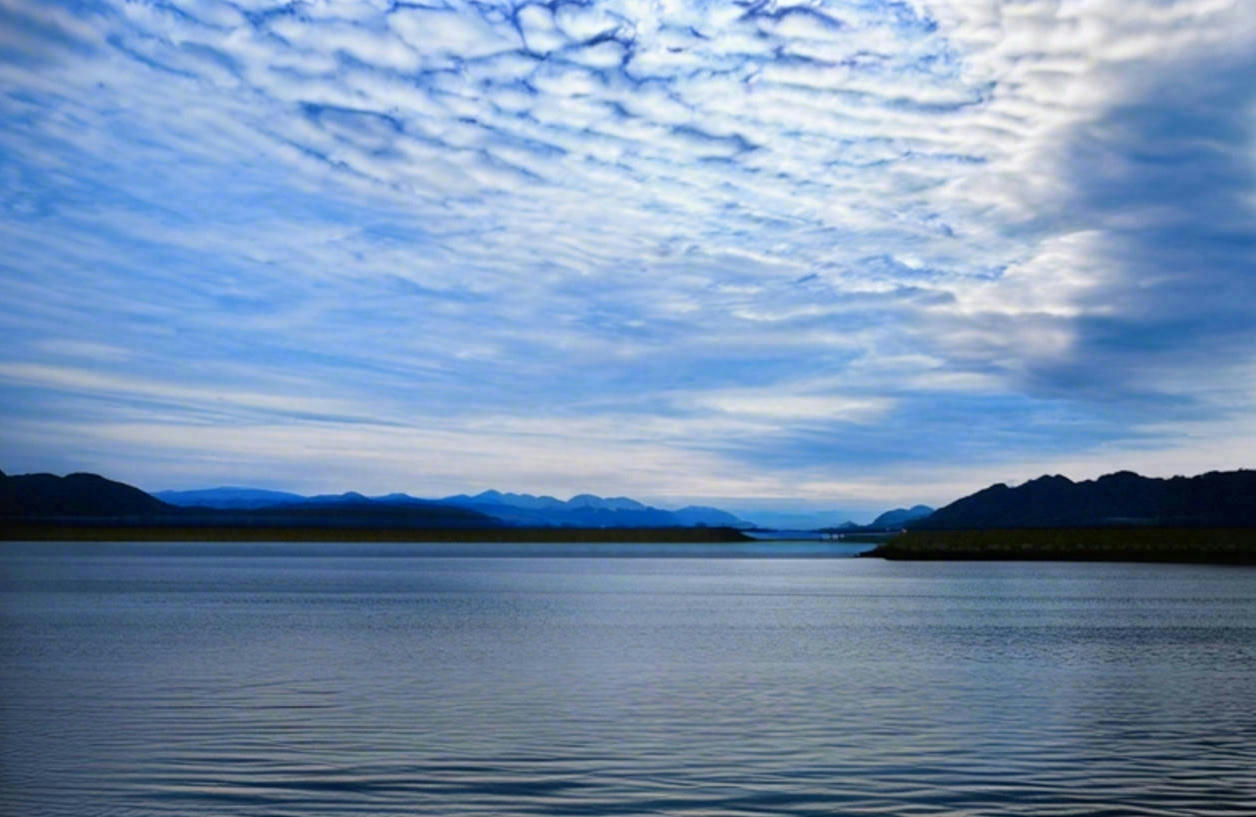}{A broad blue lake stretches toward layered mountains beneath a sweeping sky of textured white clouds. The calm ripples, low dark shorelines, and distant ridges create a cool, expansive landscape with a quiet sense of depth.}
    \hfill
    \titosample{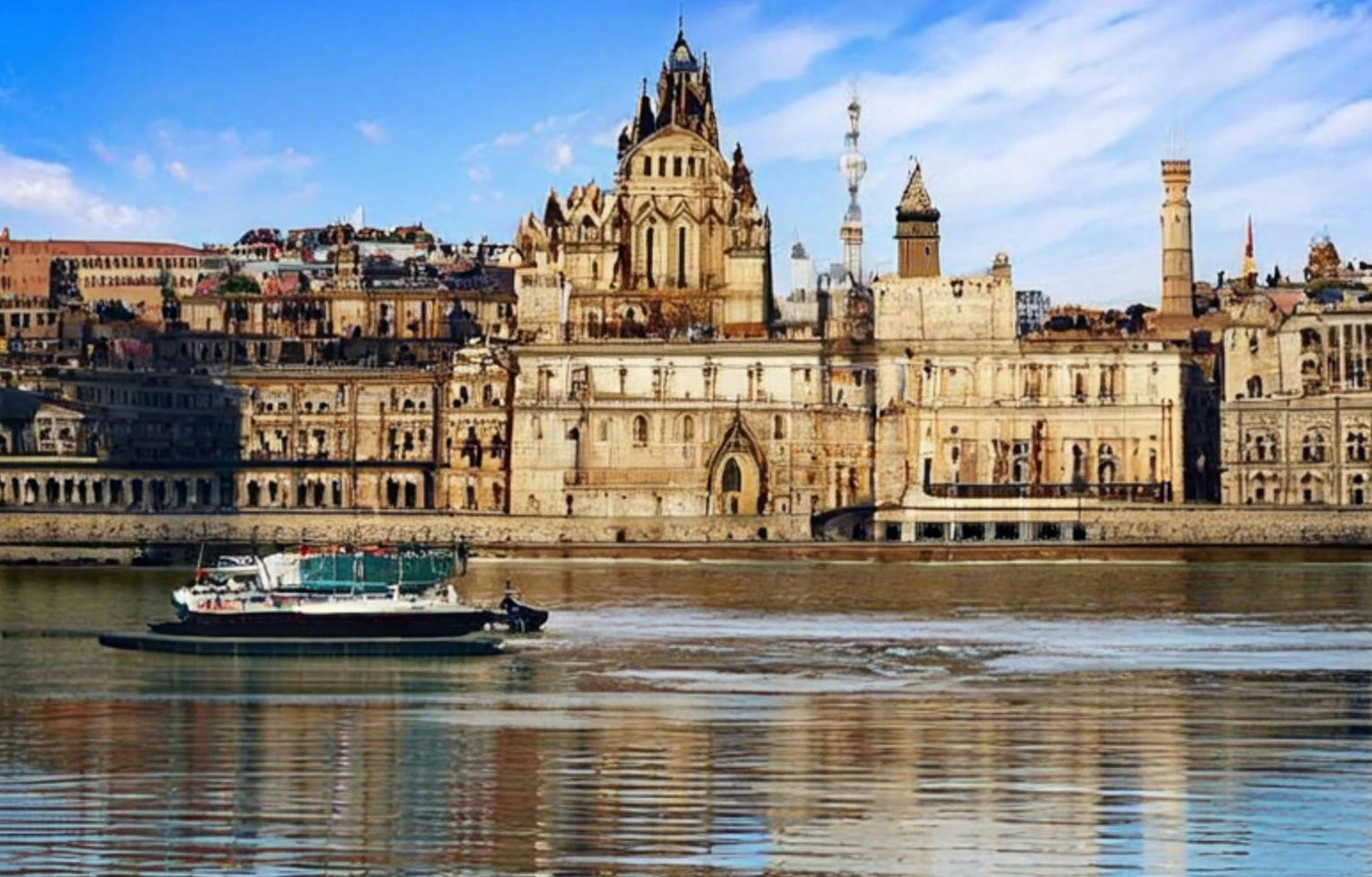}{A sightseeing boat moves along a wide river in front of grand historic architecture, including ornate towers, stone facades, and decorative rooflines. Reflections shimmer on the water, emphasizing the scale and elegance of the riverside cityscape.}

    \caption{\textbf{Text-to-image examples (Part II).} Additional samples spanning waterfront, floral, mountain, and architectural prompts, together with a portrait composition.}
    \label{fig:app_unified_t2i_2}
\end{figure*}

\begin{figure*}[t]
    \centering
    \titosample{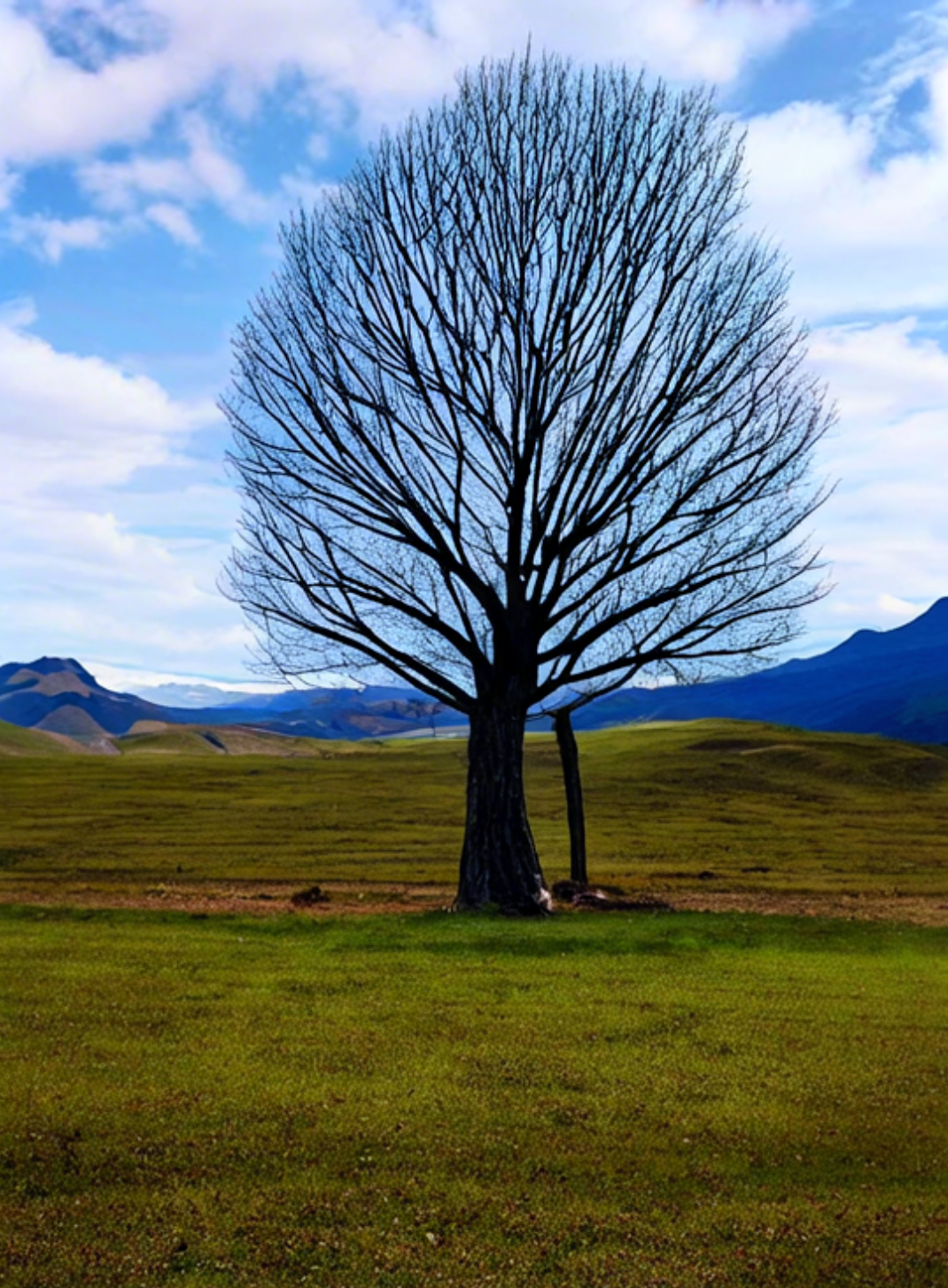}{A single leafless tree stands prominently in rolling green grassland, its intricate branches spreading against a bright sky. Distant mountains and soft clouds frame the isolated tree, giving the image a quiet, almost sculptural quality.}
    \hfill
    \titosample{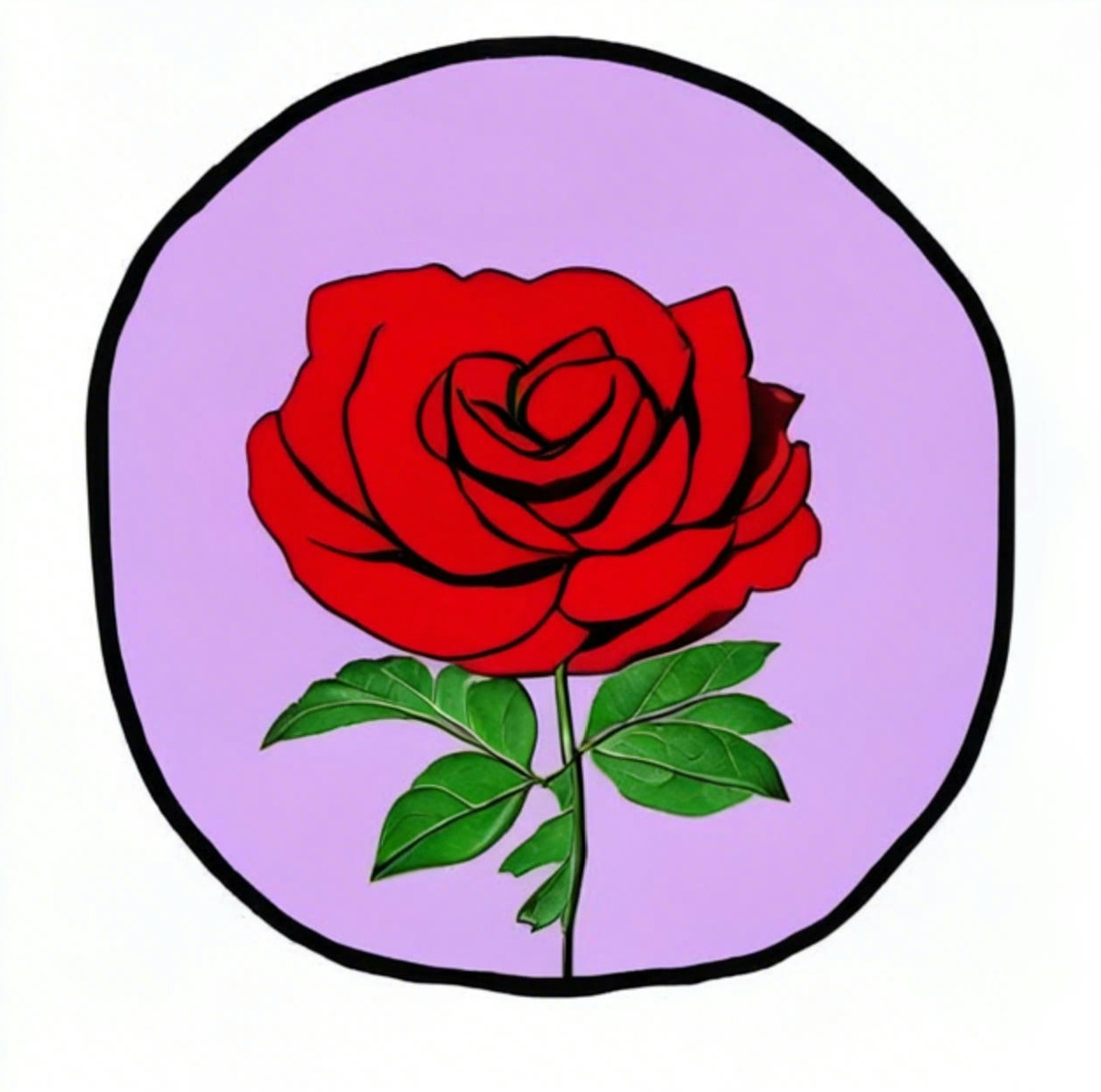}{A simple stylized illustration presents a red rose with green leaves, centered inside a rounded purple background with a thick black outline. The bold colors, clean shapes, and cartoon-like linework make the flower feel decorative and playful.}
    \hfill
    \titosample{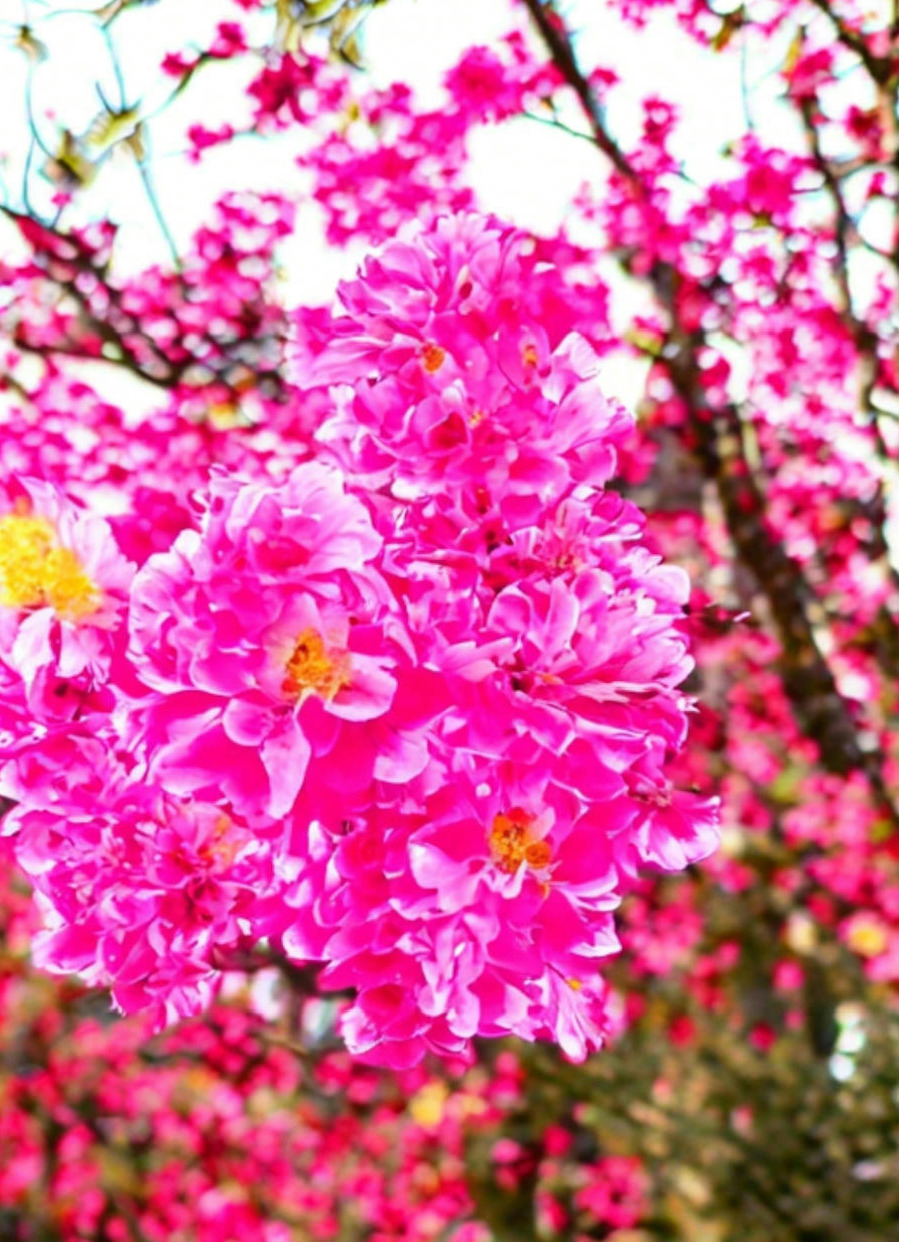}{Bright pink blossoms fill the foreground in dense clusters, their yellow centers catching the light. More branches and flowers blur behind them, creating a saturated spring scene that feels lively, sunny, and full of floral texture.}

    \vspace{0.9em}

    \titosample{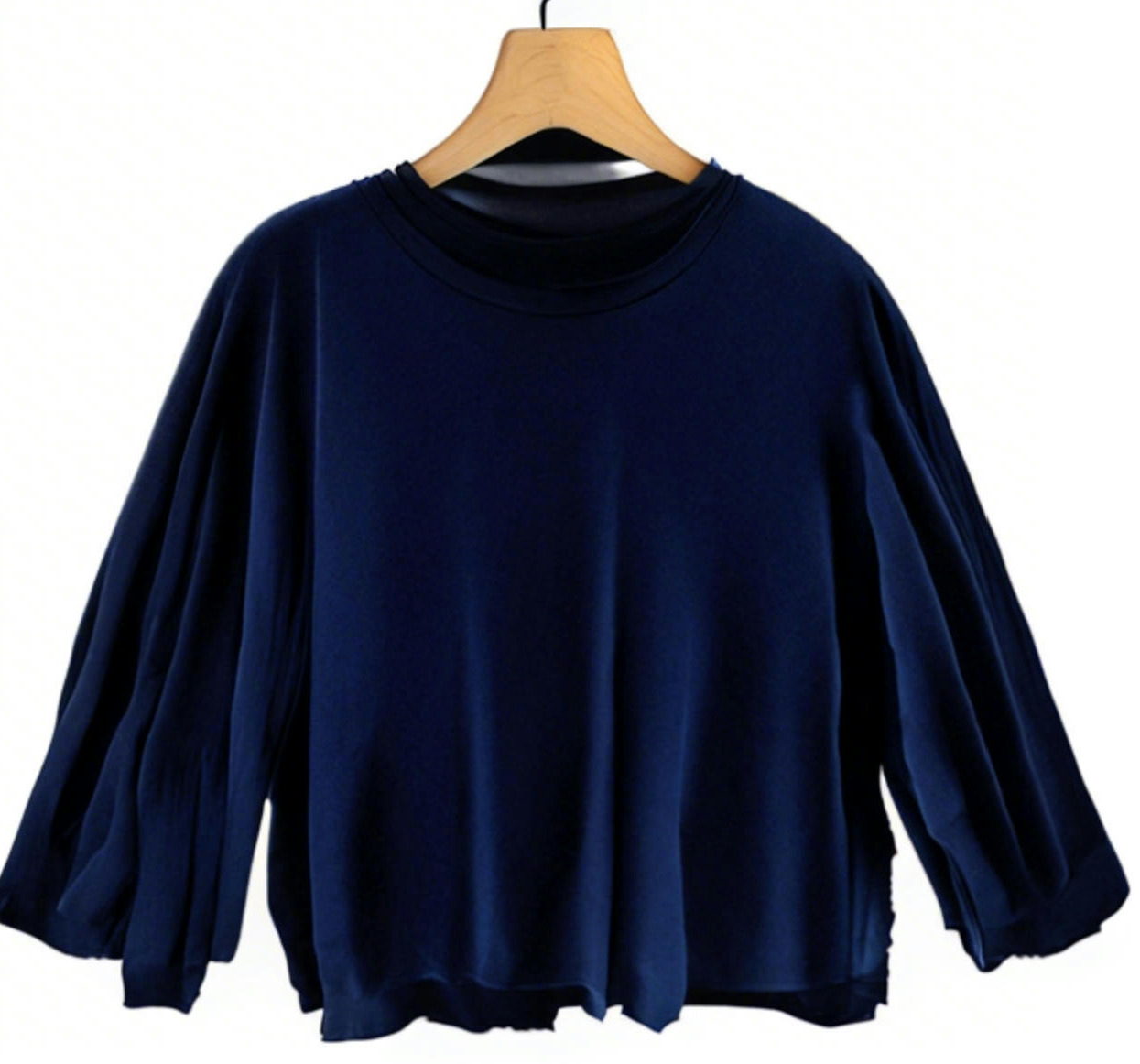}{A loose navy blue sweater hangs from a wooden hanger against a plain white background. The soft fabric, wide sleeves, and relaxed drape are emphasized by the simple product-style composition and uncluttered studio-like setting.}
    \hfill
    \titosample{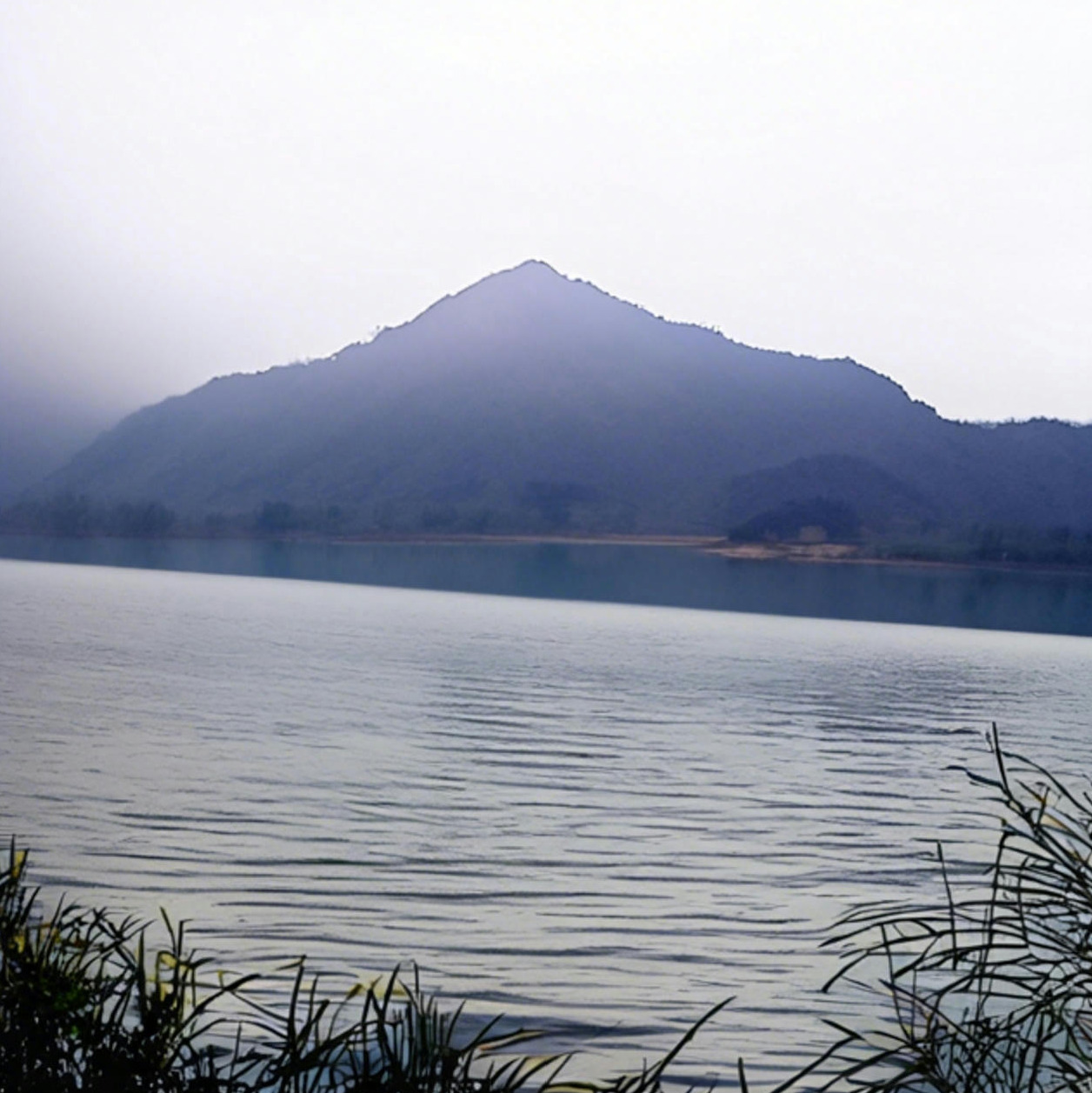}{A misty mountain rises beyond a quiet lake, its slopes softened by gray haze and overcast light. Reeds in the foreground frame the water, adding depth to the subdued scene and strengthening its calm, contemplative mood.}
    \hfill
    \titosample{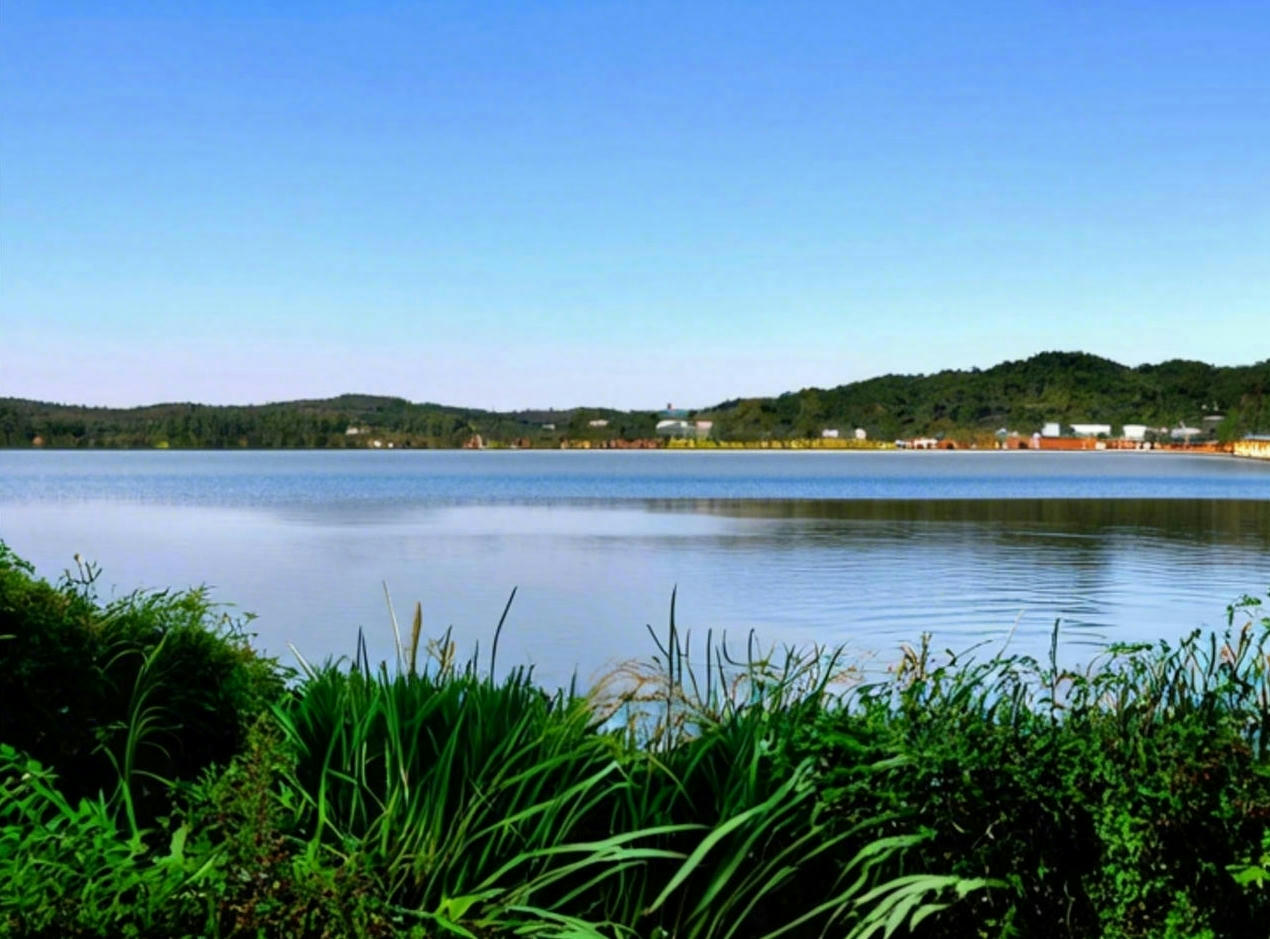}{A calm lake reflects a bright blue sky and the low green hills beyond the opposite shore. Reeds and leafy plants in the foreground create a natural frame, while distant buildings add a small human presence to the peaceful view.}

    \caption{\textbf{Text-to-image examples (Part III).} Additional samples covering pastoral, illustrative, product-style, and lakeside prompts, illustrating the breadth of compositions the model can synthesize.}
    \label{fig:app_unified_t2i_3}
\end{figure*}

\begin{figure*}[t]
    \centering
    \titosample{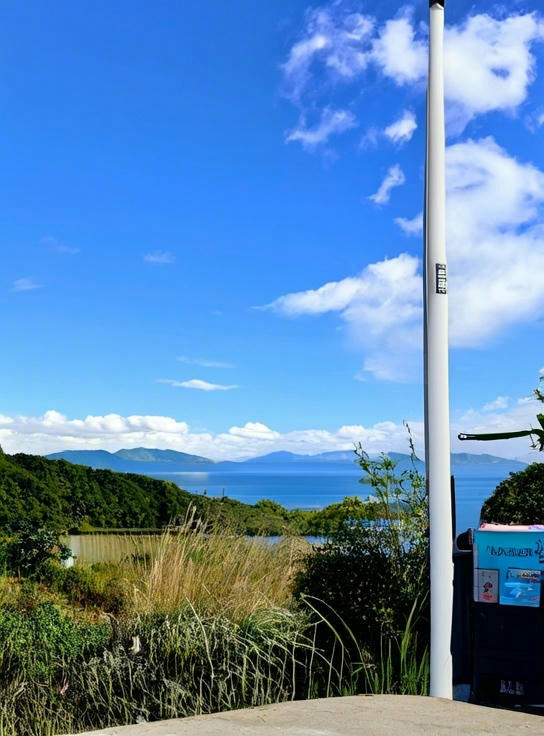}{A roadside viewpoint looks out across blue ocean water toward distant islands and hazy mountains. Green shrubs, tall grasses, a utility pole, and a small signboard sit beside the path, giving the scenic overlook a casual travel feeling.}
    \hfill
    \titosample{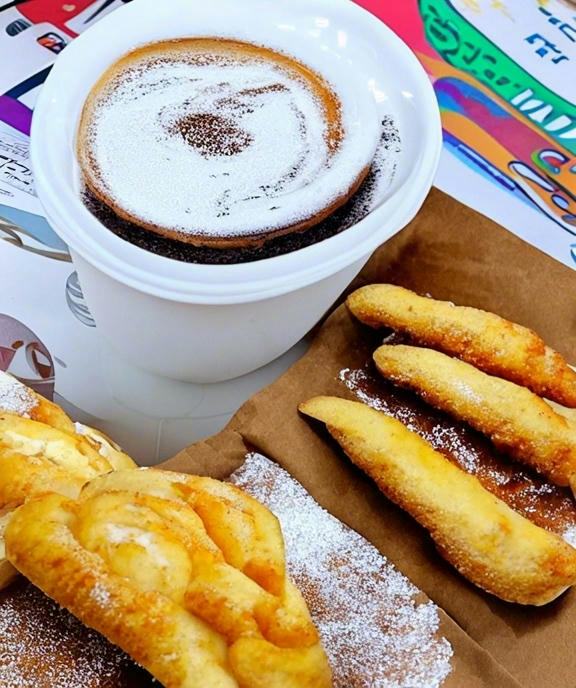}{A cup of foamy coffee sits beside fried pastries sprinkled with powdered sugar on brown paper. The colorful table covering, golden snacks, and creamy drink create a cheerful cafe or street-food scene focused on sweet comfort food.}
    \hfill
    \titosample{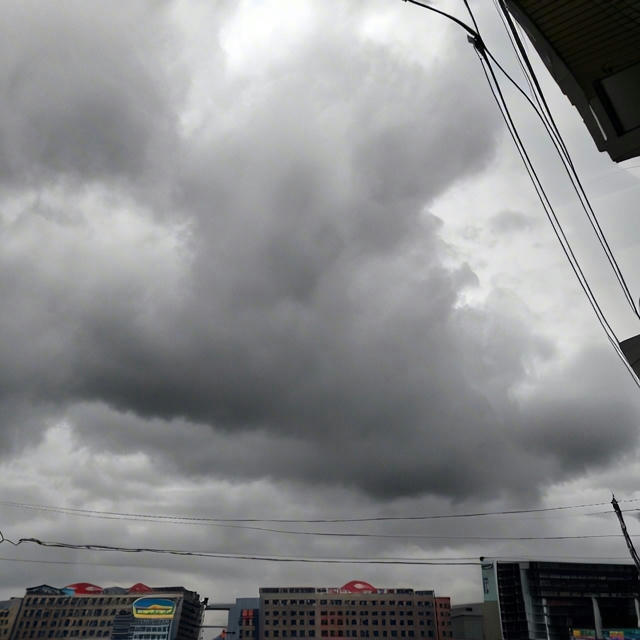}{Heavy gray storm clouds gather low over city buildings, power lines, and rooftops, filling most of the frame with dark, textured sky. The urban skyline beneath looks small and tense, suggesting an approaching downpour or dramatic weather change.}

    \vspace{0.9em}

    \titosample{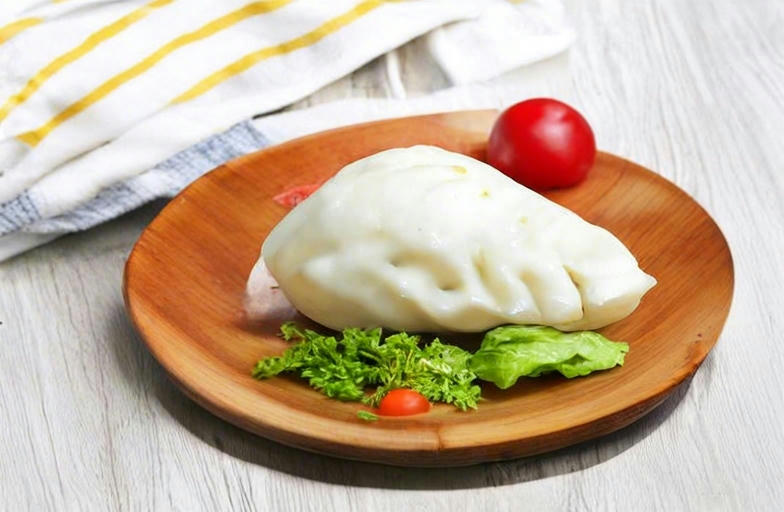}{A large steamed dumpling or bun rests on a round wooden plate, garnished with lettuce and small red tomatoes. The bright tabletop, folded cloth, and clean plating make the simple food presentation look fresh and carefully arranged.}
    \hfill
    \titosample{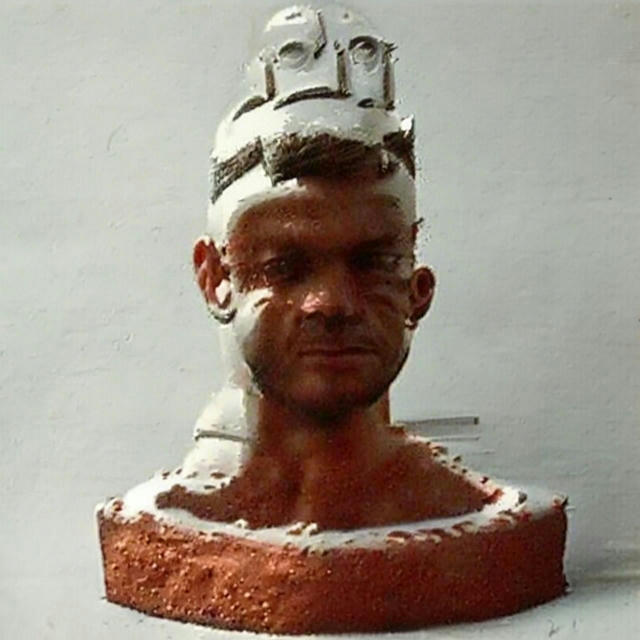}{A sculptural bust shows a realistic human face emerging from a rough, crown-like form and textured reddish base. The plain wall behind it keeps focus on the unusual material contrast, facial detail, and handmade artistic expression.}
    \hfill
    \titosample{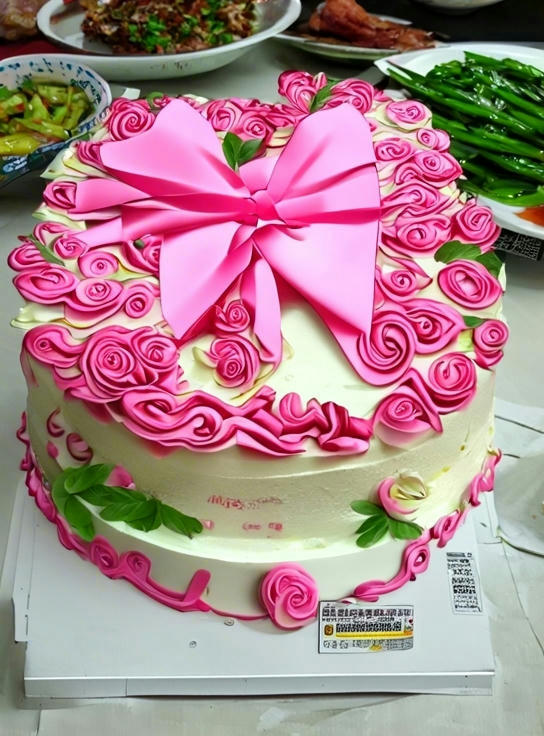}{A round celebration cake is decorated with pink frosting roses, green leaves, and a large bow across the top. Plates of food around it suggest a shared meal, while the elaborate icing makes the cake the table's centerpiece.}

    \caption{\textbf{Text-to-image examples (Part IV).} Additional samples on travel, food, urban weather, and sculptural prompts, showing the model's handling of varied lighting and textures.}
    \label{fig:app_unified_t2i_4}
\end{figure*}

\begin{figure*}[t]
    \centering
    \titosample{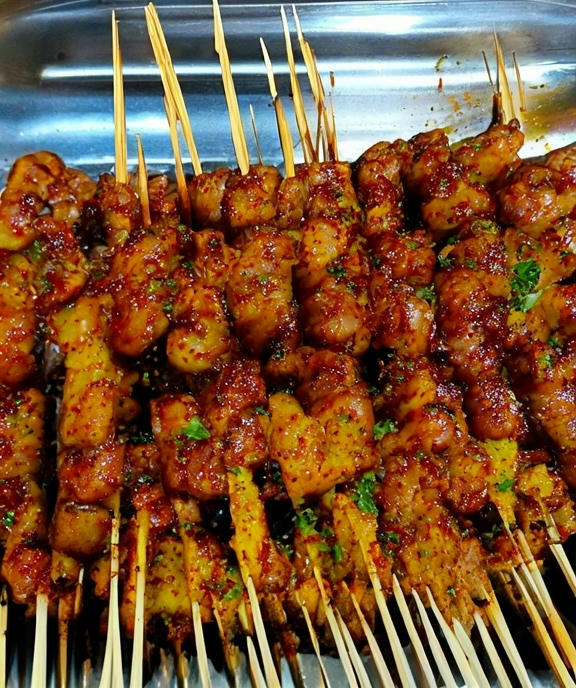}{Rows of grilled skewers are packed tightly on a metal tray, coated in glossy red seasoning, herbs, and chili flakes. The close-up view highlights the rich sauce, charred edges, and spicy street-food appeal of the dish.}
    \hfill
    \titosample{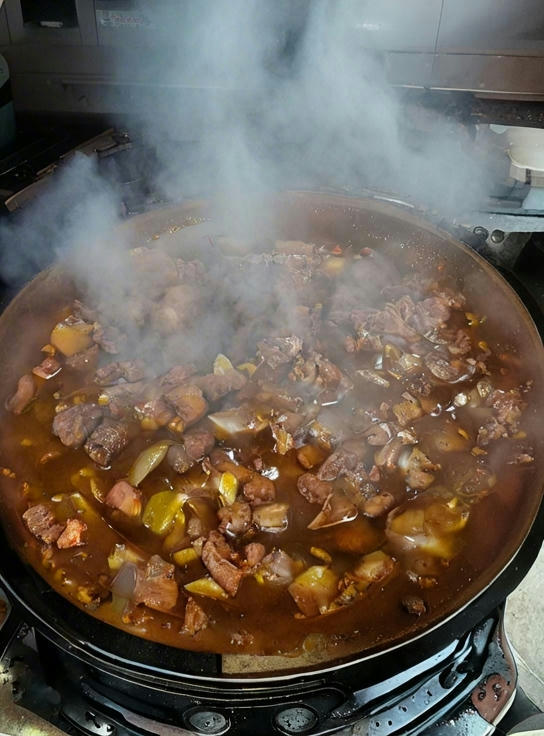}{Thick steam rises from a large pot filled with stew, meat, vegetables, and rich brown broth. The close view captures active cooking, with bubbling liquid and softened ingredients creating a hot, hearty, and aromatic kitchen scene.}
    \hfill
    \titosample{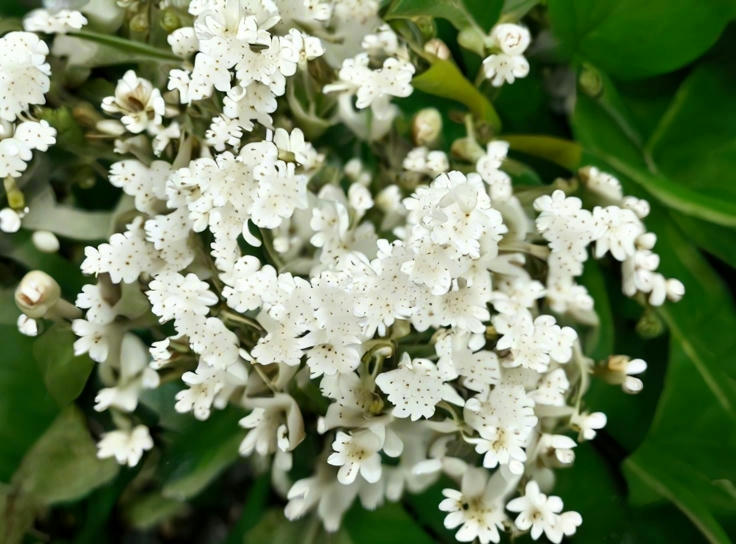}{Small white flowers bloom in dense clusters among glossy green leaves, with many speckled petals and unopened buds visible. The close-up composition emphasizes delicate natural patterns, soft lighting, and the fresh detail of the flowering plant.}

    \vspace{0.9em}

    \titosample{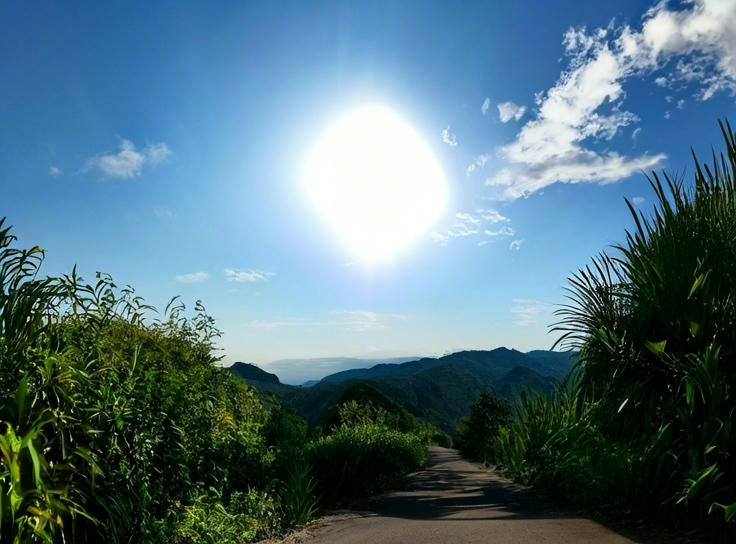}{A brilliant sun shines above a winding mountain road bordered by lush green plants. Distant ridges fade toward the blue horizon, while the intense light and clear sky create a bright, open, high-altitude travel scene.}
    \hfill
    \titosample{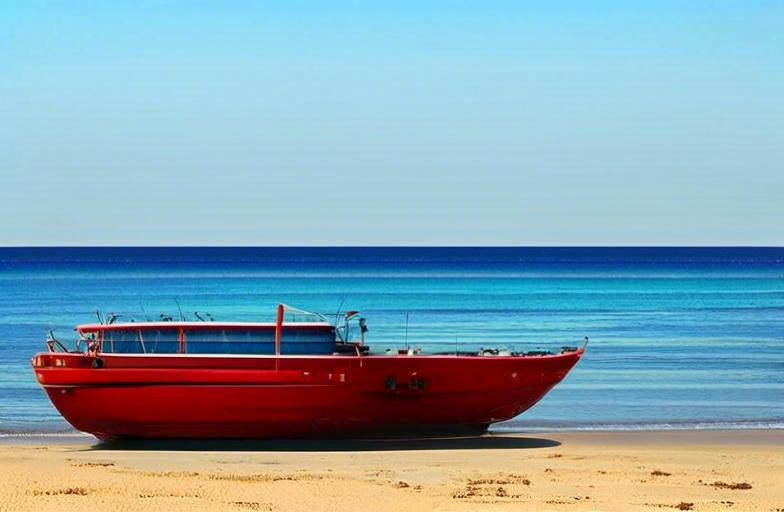}{A vivid red boat rests on pale sand in front of calm turquoise water and a clear horizon. The simple composition emphasizes strong color contrast, seaside quietness, and the unusual stillness of a boat pulled ashore.}
    \hfill
    \titosample{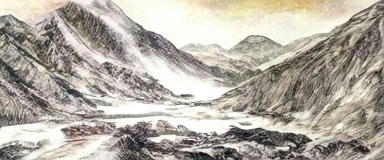}{A monochrome mountain landscape drawing shows steep ridges, deep valleys, winding water, and layered terrain fading into the distance. The sketch-like shading gives the scene a rugged, hand-rendered quality, as if from a travel notebook.}

    \caption{\textbf{Text-to-image examples (Part V).} Additional samples on culinary, floral, and travel prompts, together with an illustrative monochrome landscape.}
    \label{fig:app_unified_t2i_5}
\end{figure*}

\begin{figure*}[t]
    \centering
    \titosample{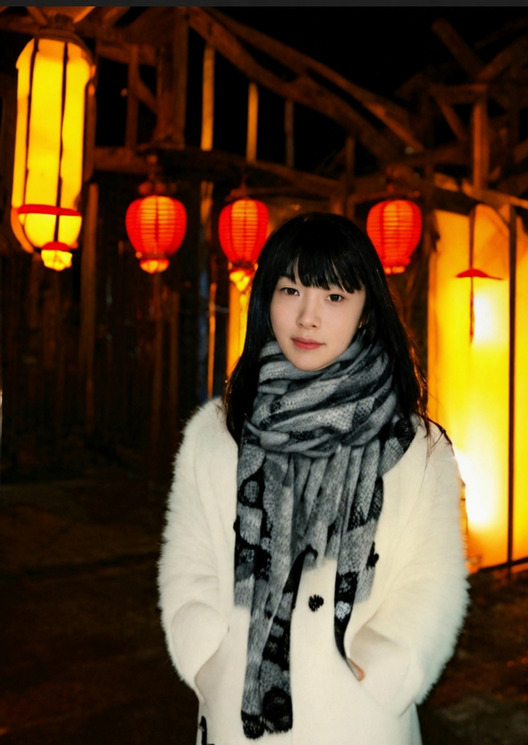}{A young woman in a white coat and patterned scarf stands at night in front of glowing red and yellow lanterns. The warm lights and wooden structures behind her create a festive, atmospheric portrait with soft contrast.}
    \hfill
    \titosample{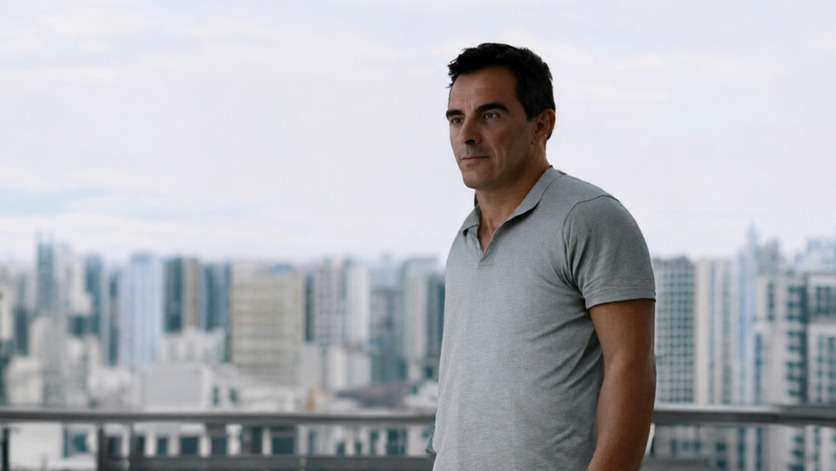}{A man in a gray polo shirt stands on a balcony or rooftop, looking thoughtfully to the side. A blurred city skyline behind him gives the portrait an urban setting and a calm, reflective mood.}
    \hfill
    \titosample{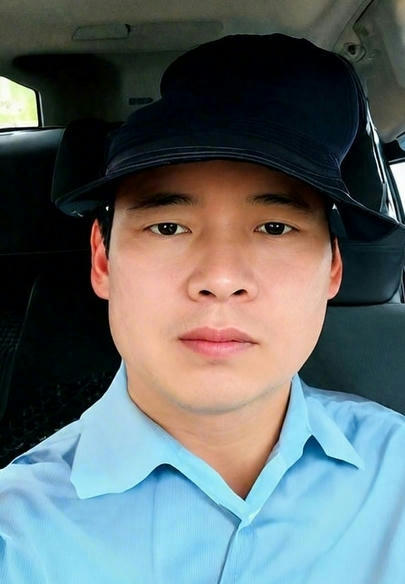}{A close selfie shows a man wearing a dark bucket hat and light blue shirt while seated inside a vehicle. Black seats, side windows, and an orange curtain provide context for the casual travel setting.}

    \vspace{0.9em}

    \titosample{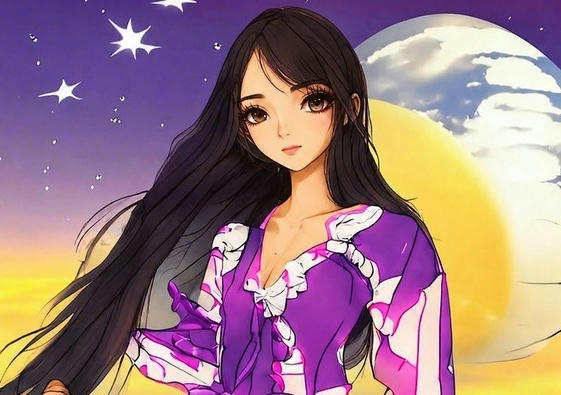}{An anime-style girl with long dark hair wears a purple patterned outfit with lace details, posed against a dreamy sky. Large moonlike circles, stars, and soft gradients create a romantic fantasy illustration style.}
    \hfill
    \titosample{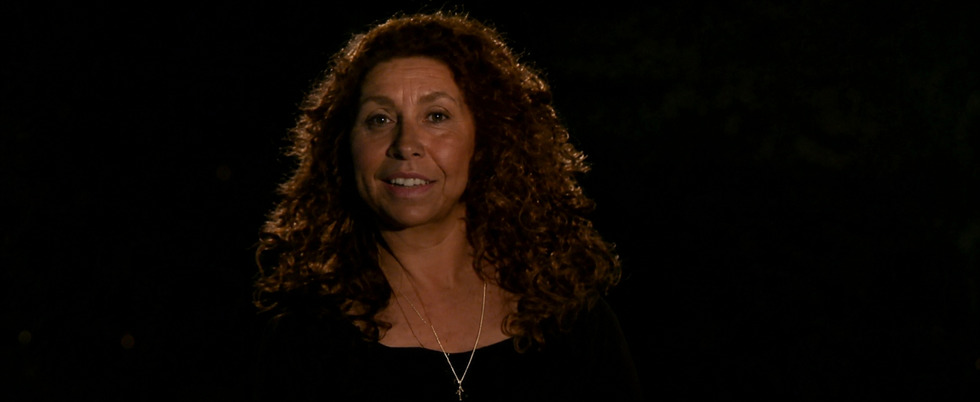}{A woman with curly hair smiles gently in a very dark outdoor setting, lit softly from the front. The nearly black background isolates her face and shoulders, creating an intimate portrait with understated nighttime atmosphere.}
    \hfill
    \titosample{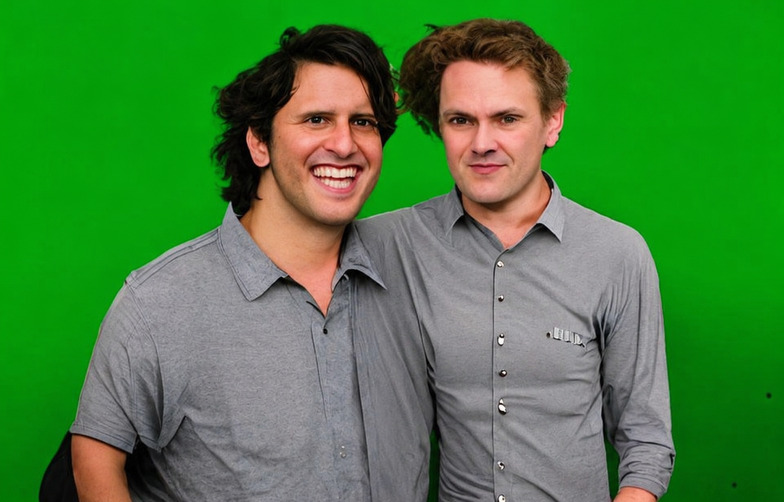}{Two men in gray shirts stand shoulder to shoulder in front of a green screen. One smiles broadly while the other gives a restrained expression, creating a casual studio portrait with contrasting personalities.}

    \caption{\textbf{Text-to-image examples (Part VI).} Portrait-focused samples covering nighttime scenes, urban backdrops, stylized illustration, and studio-style group compositions.}
    \label{fig:app_unified_t2i_6}
\end{figure*}

\begin{figure*}[t]
    \centering
    \titosample{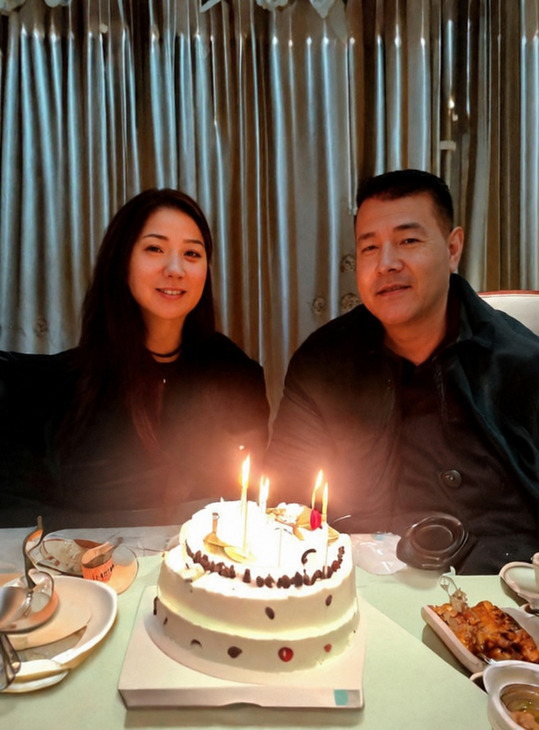}{Two adults sit at a restaurant table behind a birthday cake topped with lit candles. Warm indoor lighting, surrounding dishes, cups, and desserts suggest a celebratory meal shared in a relaxed dining setting.}
    \hfill
    \titosample{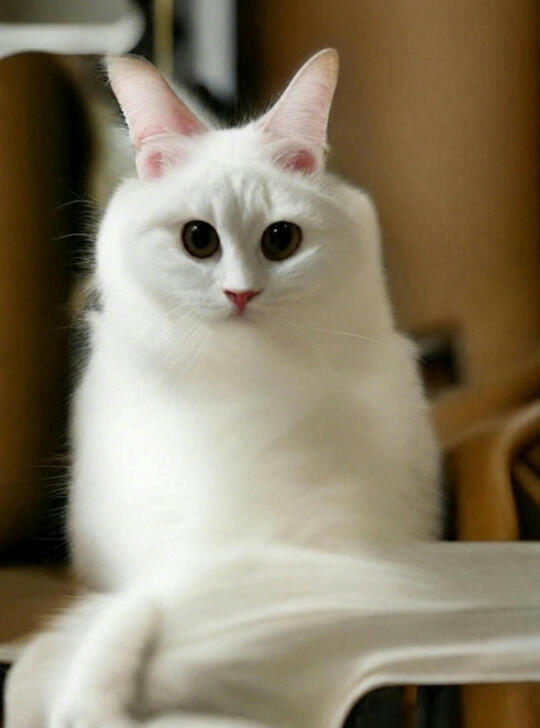}{A fluffy white cat sits indoors with wide dark eyes, pink ears, and crossed front paws. The softly blurred background and centered framing make the cat's round face and clean white fur especially prominent.}
    \hfill
    \titosample{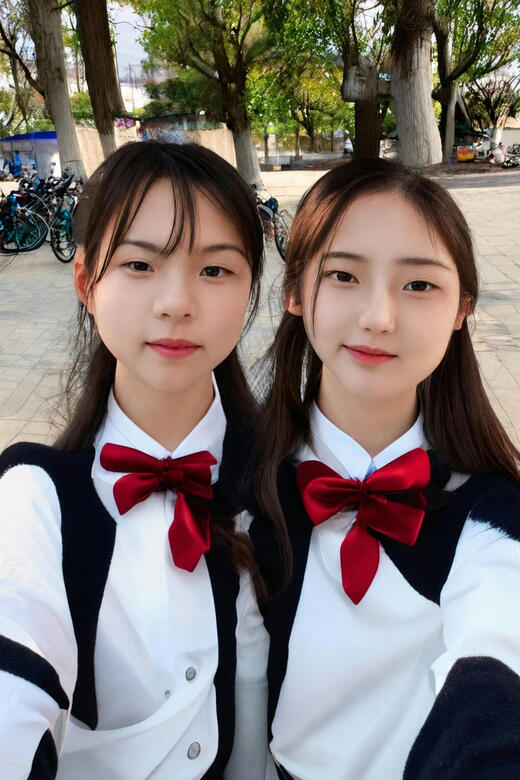}{Two young women in coordinated school-style outfits with red bows take a close selfie outdoors. Trees, bicycles, and paved walkways in the background suggest a campus or park setting on a bright day.}

    \caption{\textbf{Text-to-image examples (Part VII).} Final samples on celebratory dining, animal, and outdoor selfie prompts, completing the qualitative tour of the unified model's text-to-image generation capability.}
    \label{fig:app_unified_t2i_7}
\end{figure*}

\subsubsection{Image \texorpdfstring{$\times$}{x} Text to Text}
In this subsection, we present examples of image-and-text-to-text generation. The outputs are generally fluent and semantically aligned with the inputs. Notably, on top of a model that has received only limited text-generation and image-generation training, we further train with merely approximately 5M image--text pairs, after which the model already exhibits a preliminary captioning ability. This suggests that \method can efficiently acquire cross-modal understanding and generation capabilities under limited multimodal supervision. We also note that all images shown in this subsection are generated by external image generation models from ground-truth captions, and therefore do not involve copyright concerns associated with real images.
\newcommand{\itttsample}[3]{%
\begin{minipage}[t]{0.315\textwidth}
    \centering
    \includegraphics[width=\linewidth]{exp_fig/discussion/unified/it2t_samples/#1.png}
    \vspace{0.35em}

    {\scriptsize\raggedright\textbf{Prompt:} #2\par}
    \vspace{0.2em}
    {\scriptsize\raggedright\textbf{Caption:} #3\par}
\end{minipage}
}

\begin{figure*}[t]
    \centering
    \itttsample{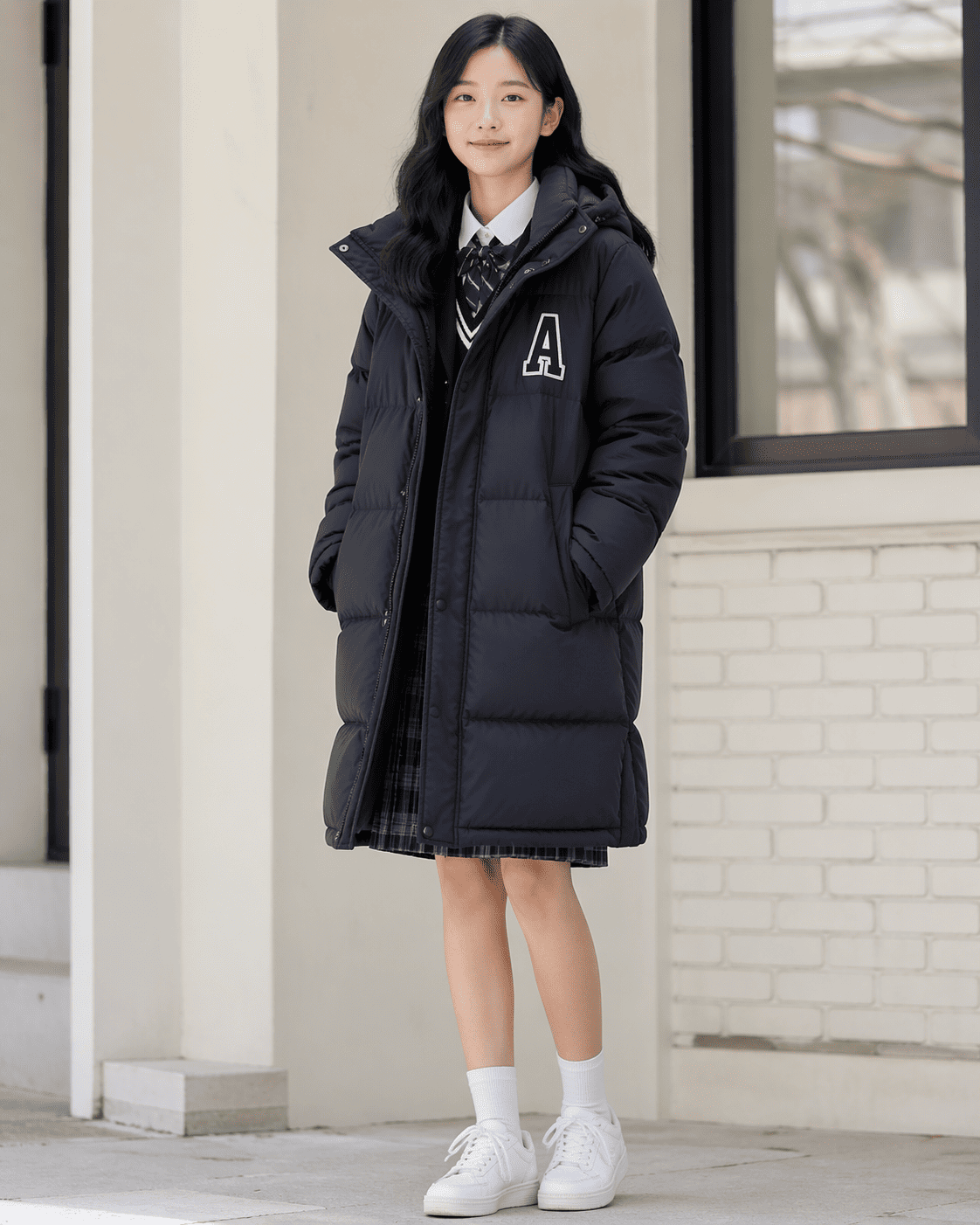}{Describe the content of this image in detail.}{The image shows a woman wearing an oversized long down jacket with a plain collar. The jacket features a letter ``A'' design on it.}
    \hfill
    \itttsample{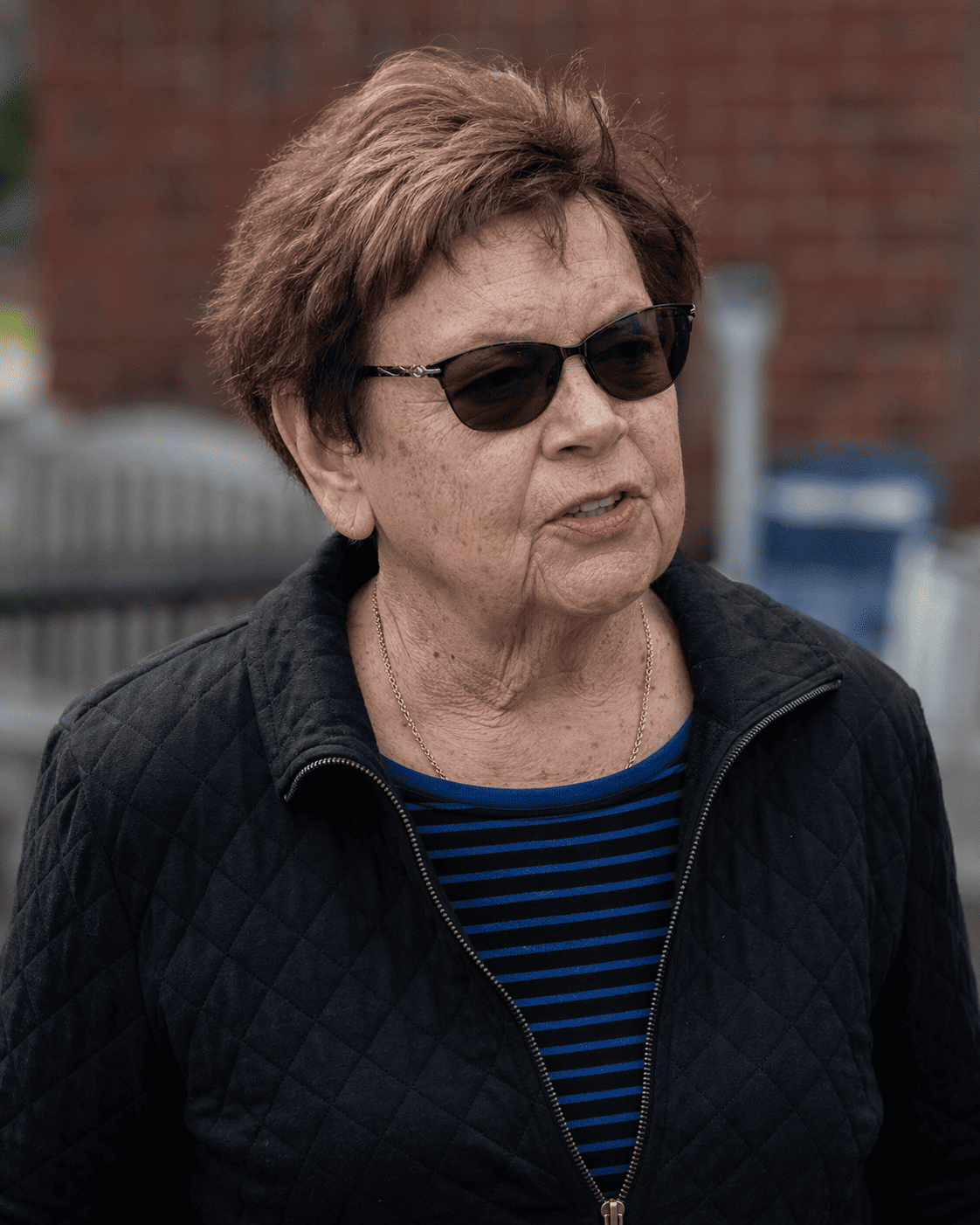}{Describe the content of this image in detail.}{The image shows a person wearing a black jacket and sunglasses, standing on a street.}
    \hfill
    \itttsample{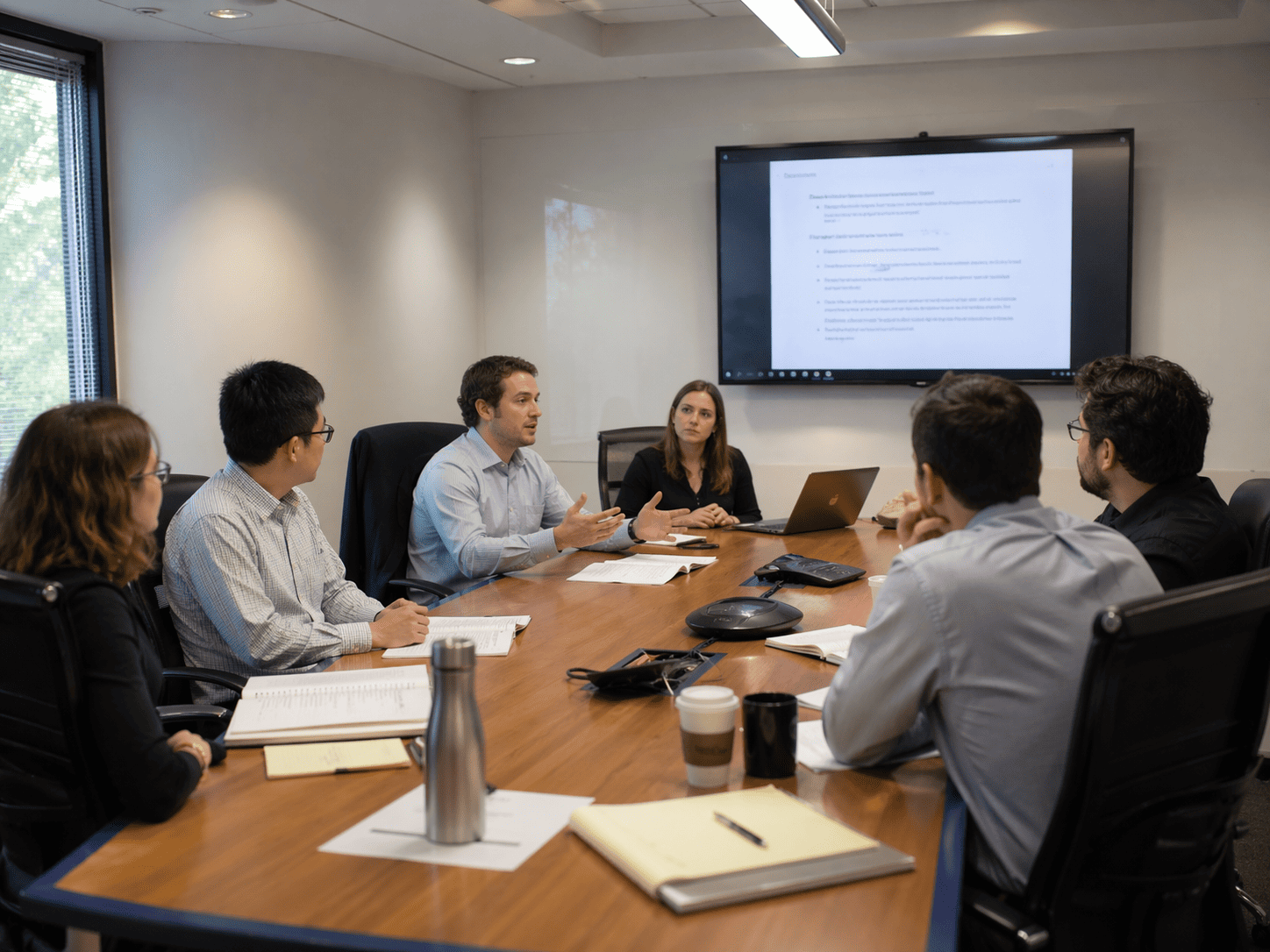}{Describe the content of this image in detail.}{The image shows a meeting room with several people, around six to eight, sitting around a long conference table and holding a meeting. A screen is mounted on the wall, and bowls, cups, and pens are placed on the table.}

    \vspace{0.9em}

    \itttsample{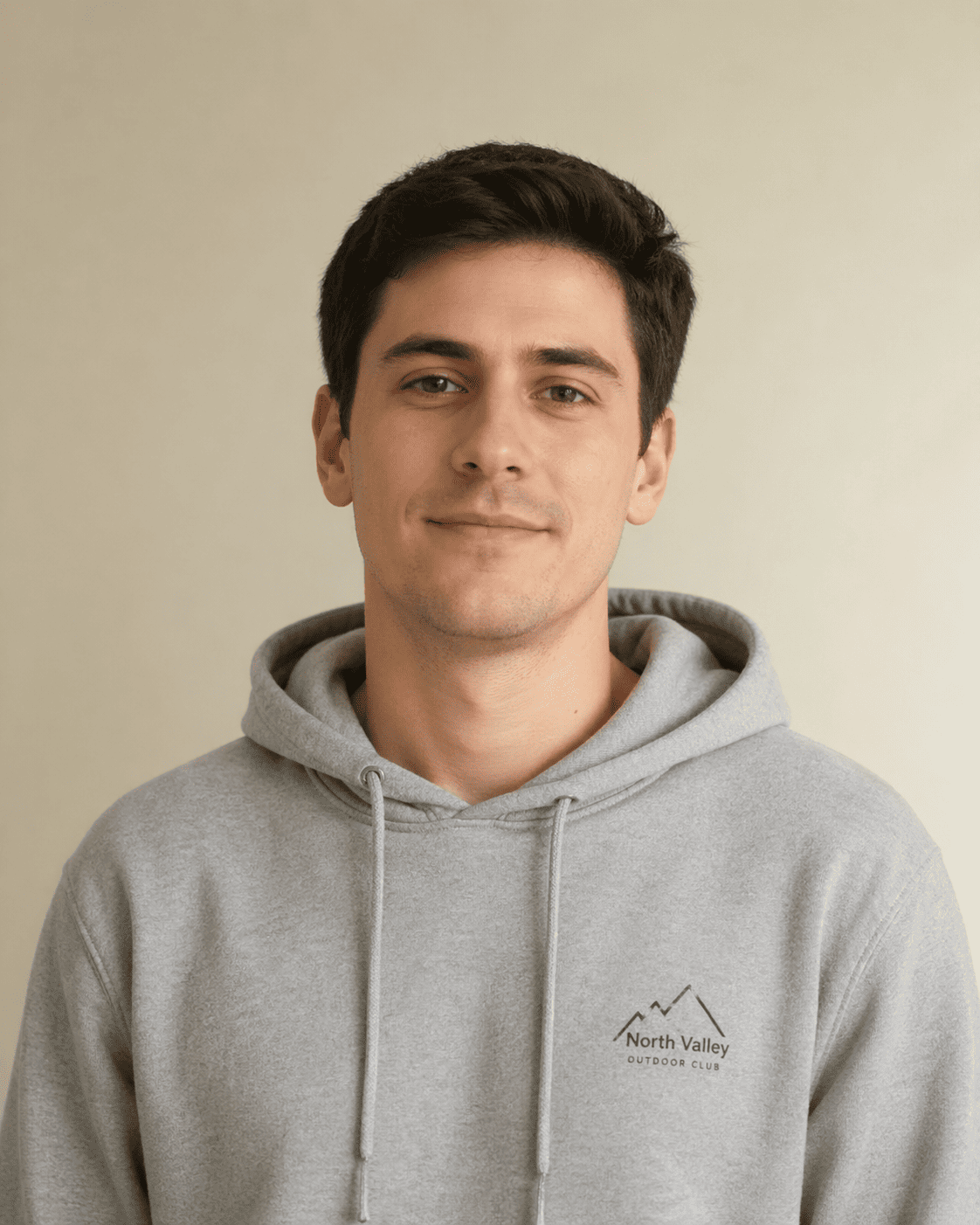}{Describe the content of this image in detail.}{The image shows a person wearing a grey hoodie with a purple design on the front. The background includes a white wall with a blue design.}
    \hfill
    \itttsample{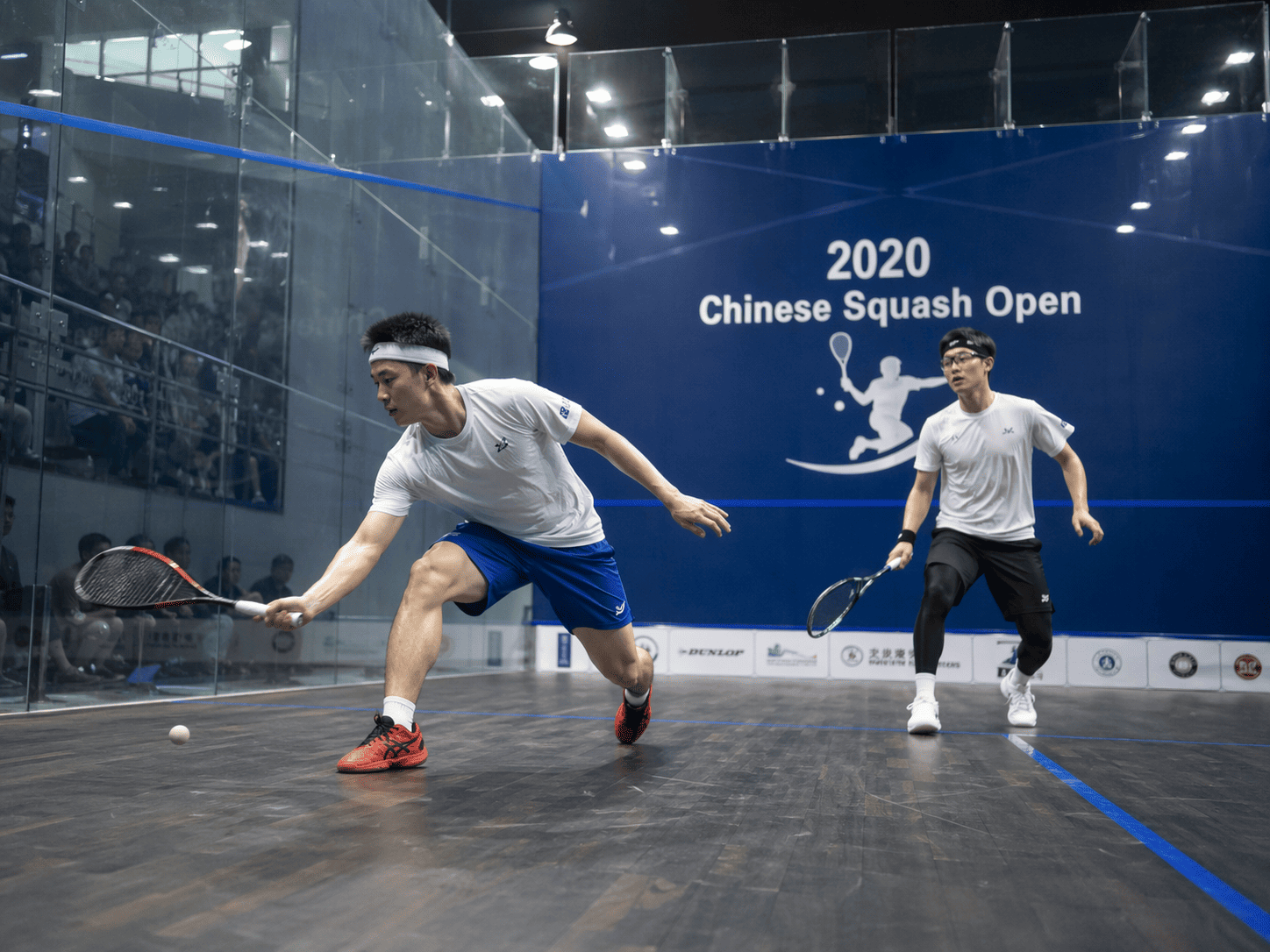}{Describe the content of this image in detail.}{The image shows several people playing a ball game on a court. The background includes a blue wall with some text, and the ground is also blue.}
    \hfill
    \itttsample{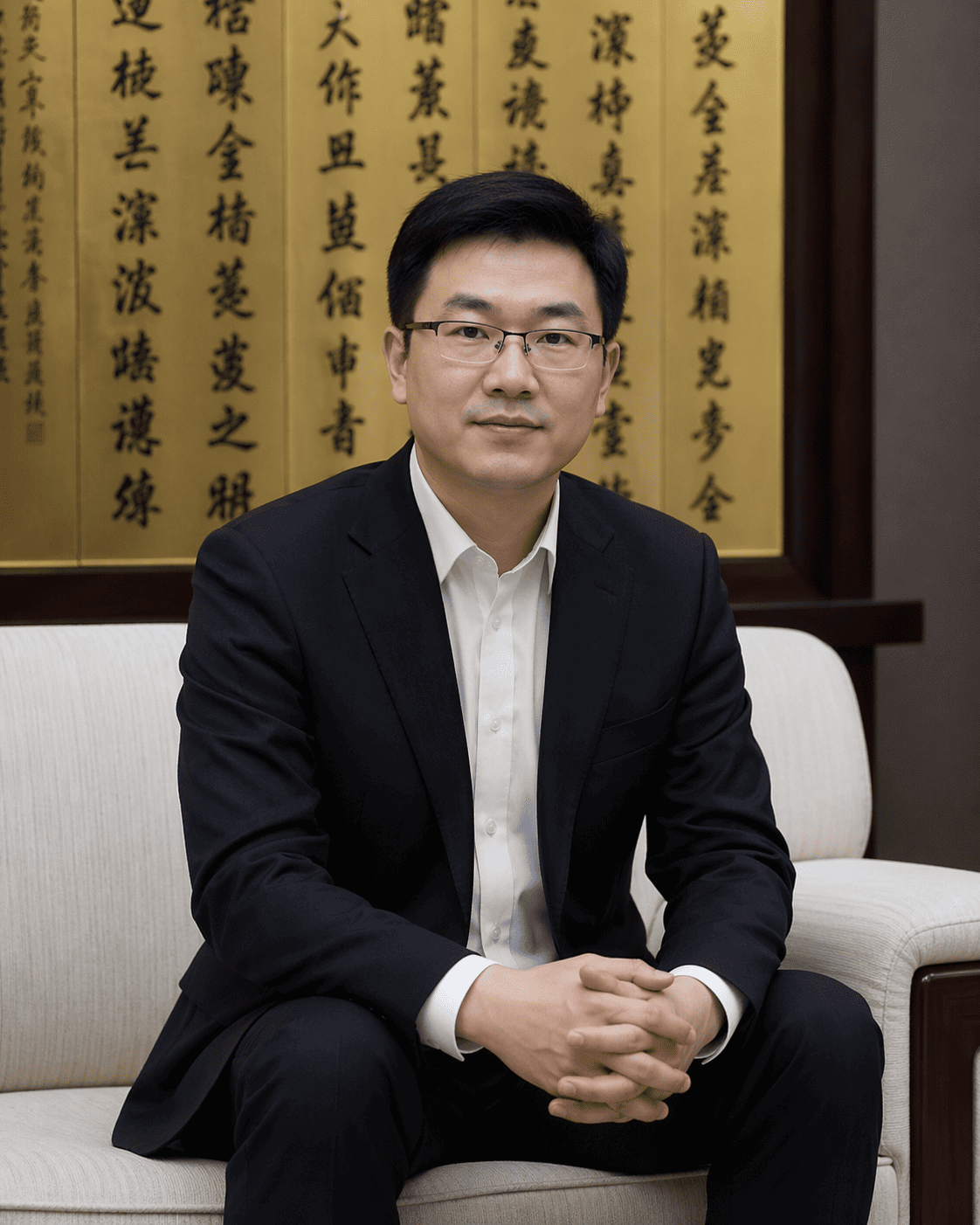}{Describe the content of this image in detail.}{The image shows a person wearing a black suit with a white shirt underneath, seated.}

    \caption{\textbf{Image$\times$text-to-text examples (Part I).} Additional qualitative results under the prompt template \textit{Describe the content of this image in detail.}}
    \label{fig:app_unified_ixt2t_1}
\end{figure*}

\begin{figure*}[t]
    \centering
    \itttsample{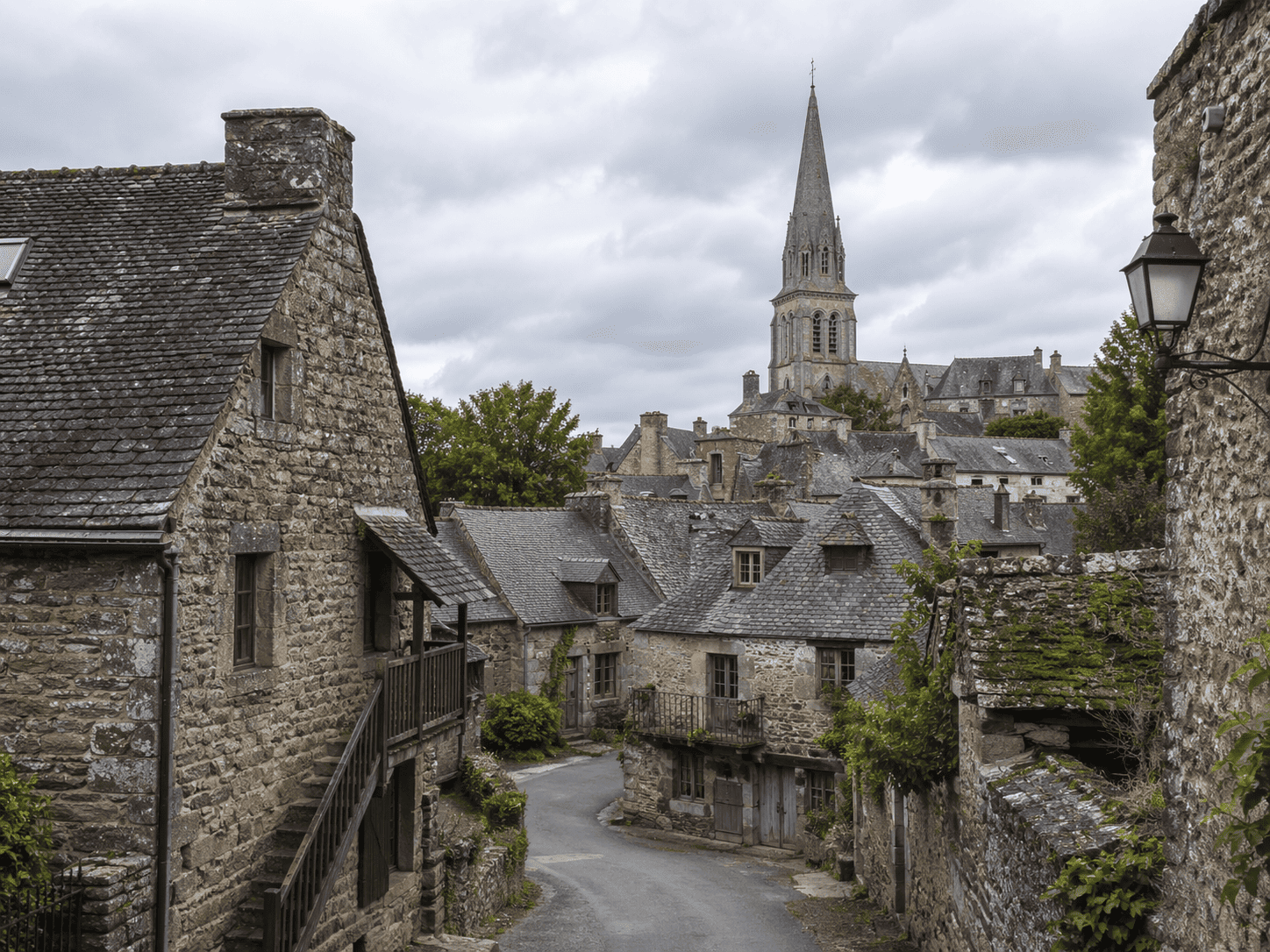}{Describe the content of this image in detail.}{The image shows a cluster of ancient buildings with old-style architecture, featuring several old structures packed closely together.}
    \hfill
    \itttsample{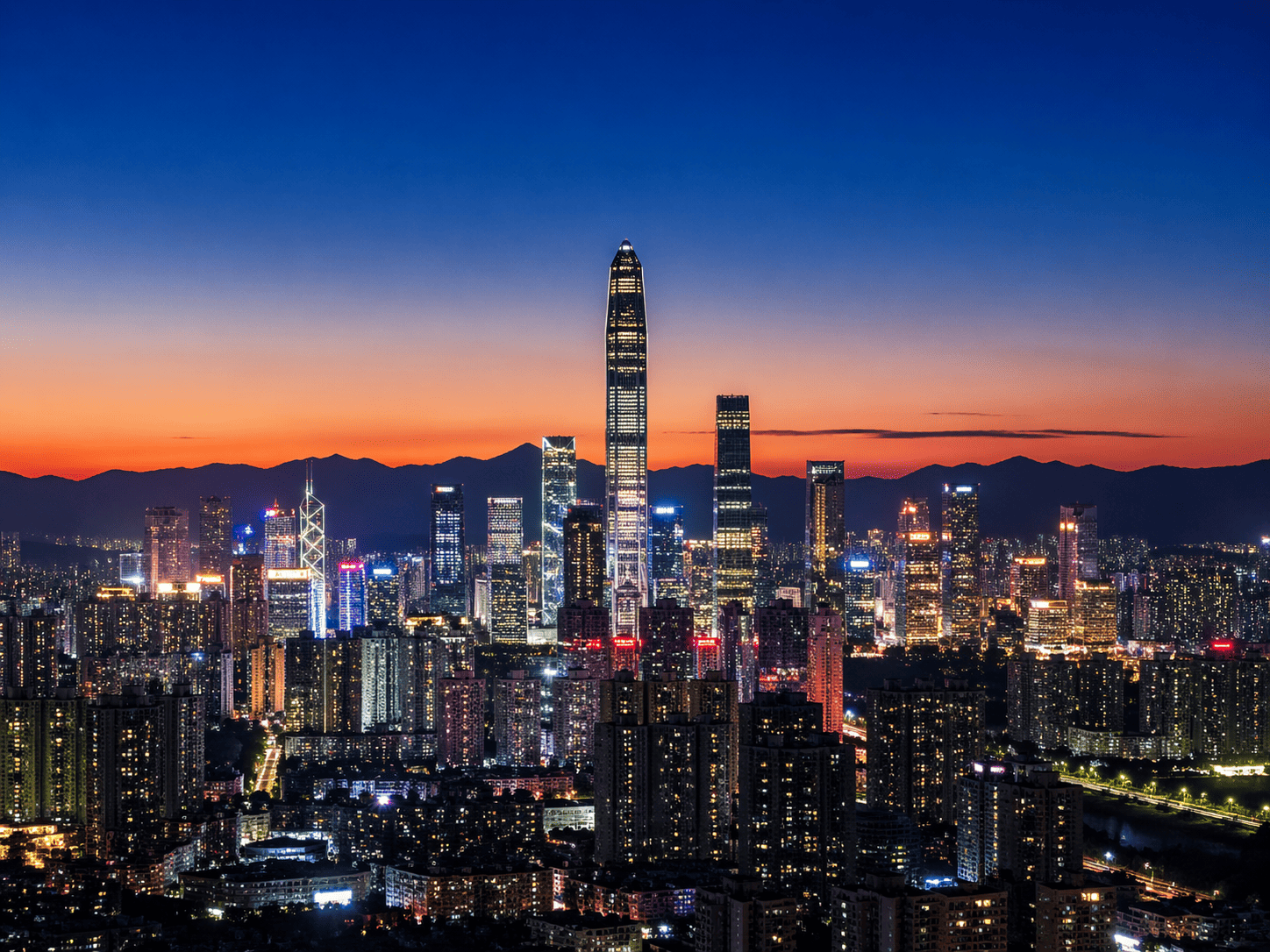}{Describe the content of this image in detail.}{The image is a cityscape shot at dusk, showing an urban night scene with many tall high-rise buildings.}
    \hfill
    \itttsample{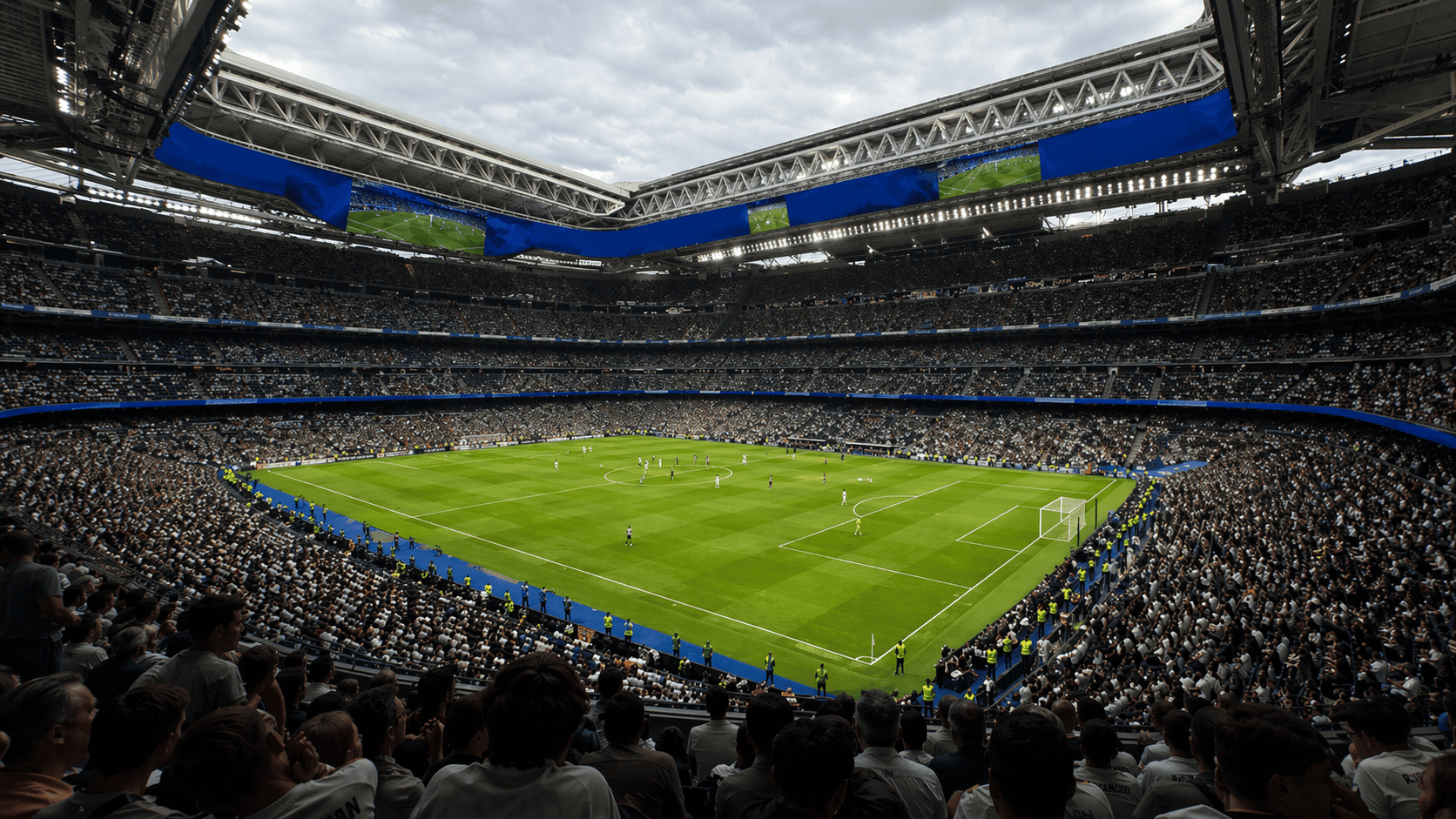}{Describe the content of this image in detail.}{The image shows a large football (soccer) stadium with the field inside, resembling an outdoor stadium. Another row of stands is visible on the other side of the field.}

    \vspace{0.9em}

    \itttsample{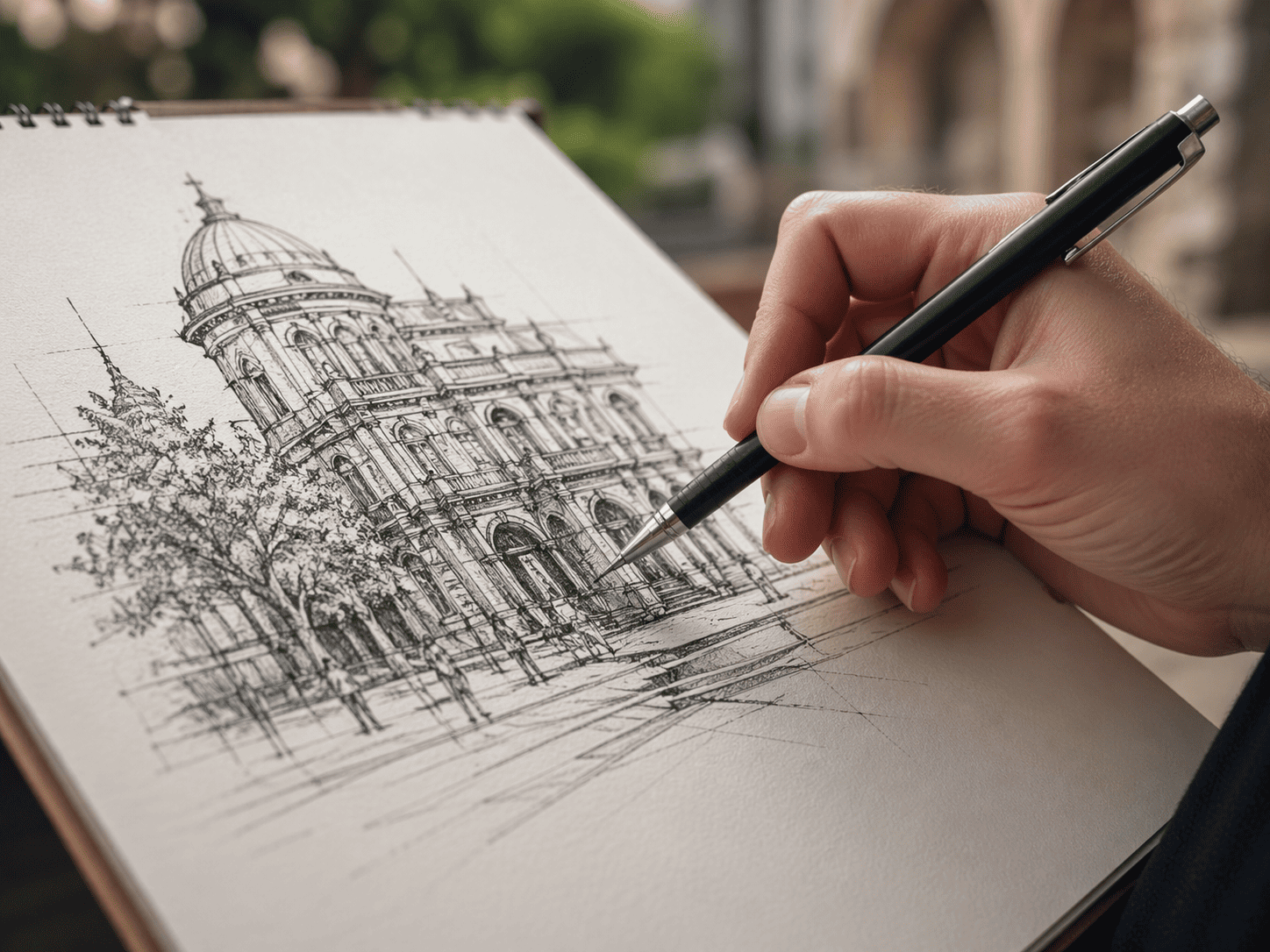}{Describe the content of this image in detail.}{The image shows a person using a pen to write on a piece of paper.}
    \hfill
    \itttsample{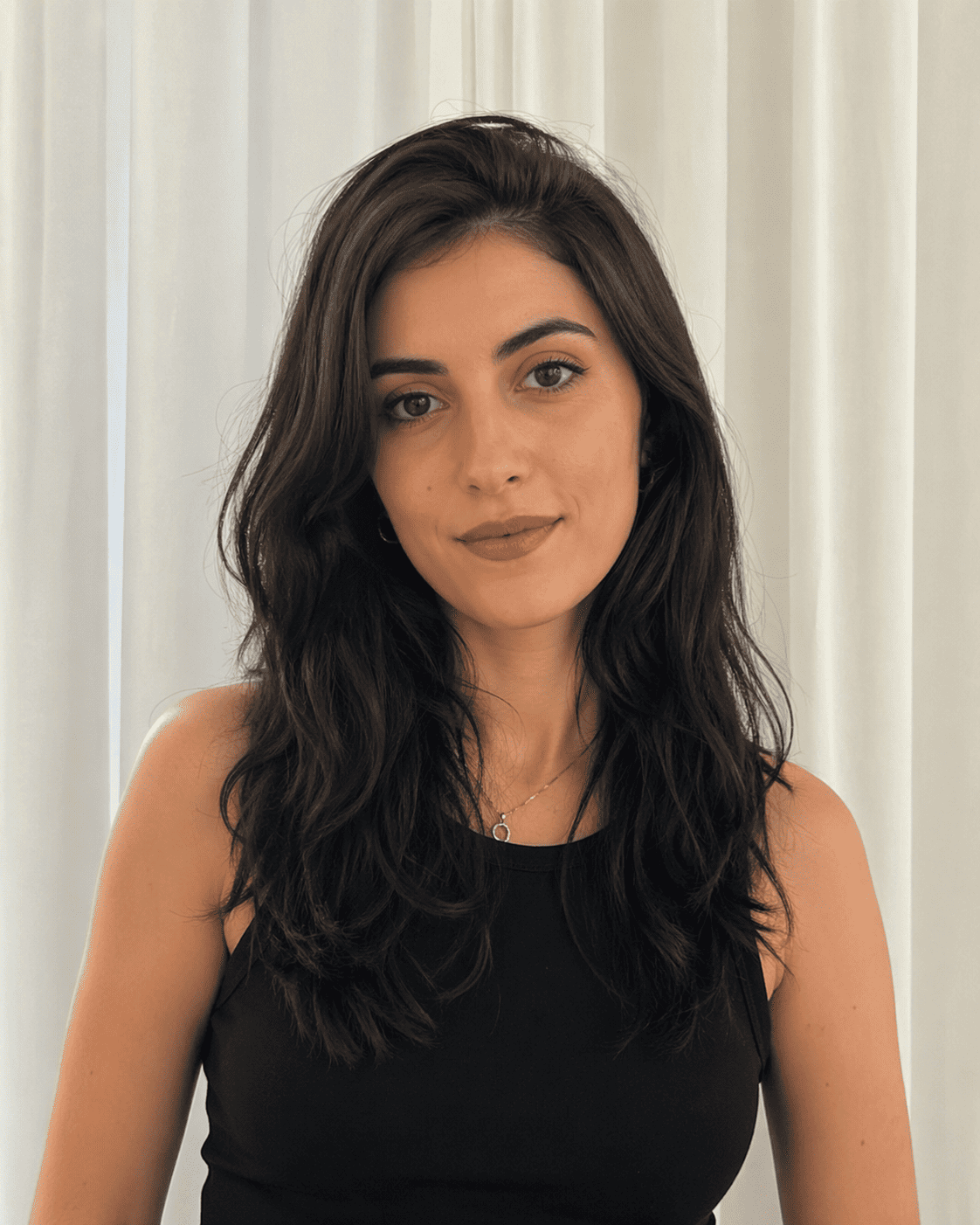}{What does this image depict? Give a comprehensive description.}{The image shows a woman with long, dark hair wearing a black top. The background is a neutral color, with a window and a pink wall visible behind her. She appears to be standing indoors.}
    \hfill
    \itttsample{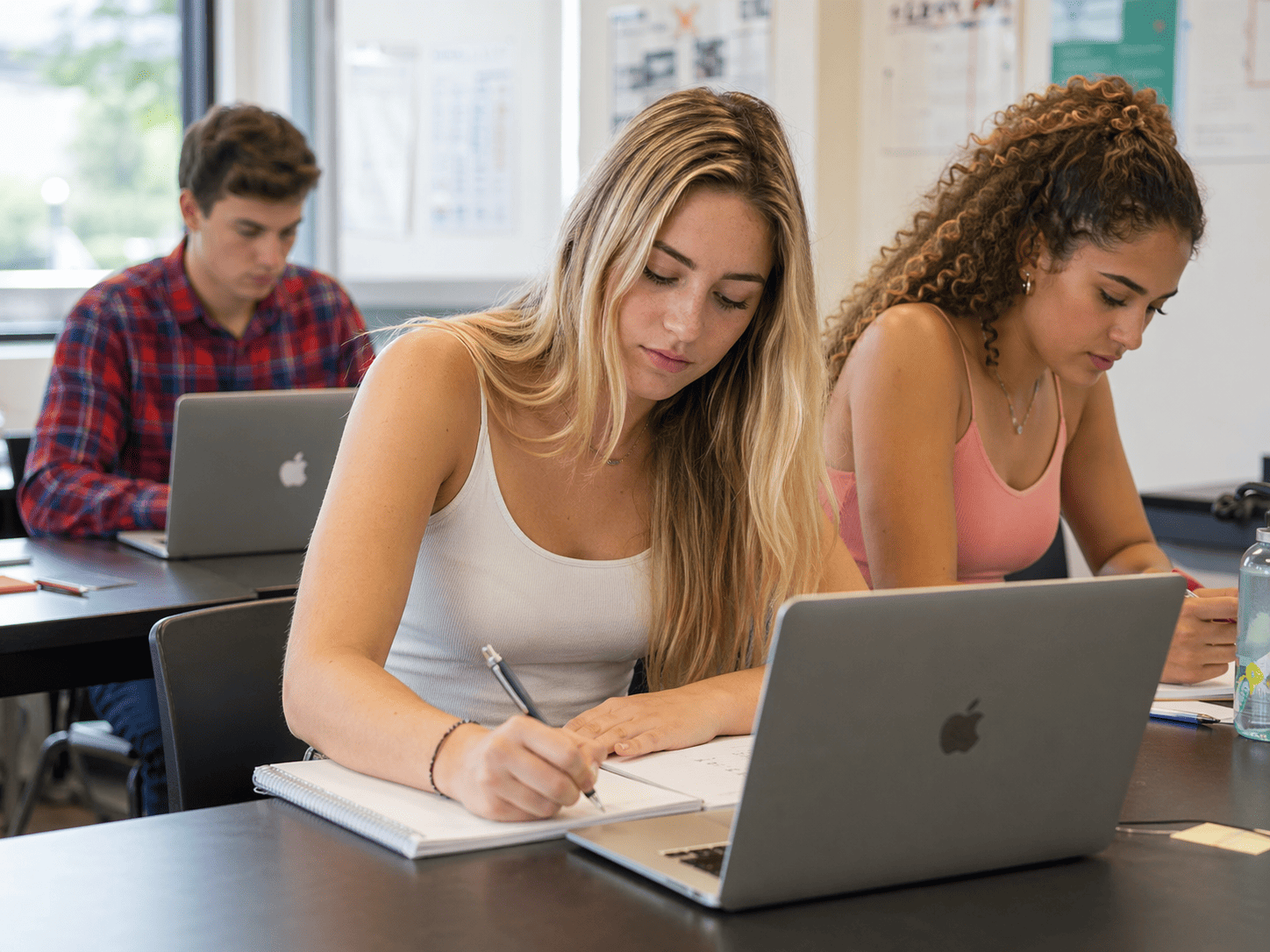}{What does this image depict? Give a comprehensive description.}{The image shows a portrait of a young woman with blonde hair, wearing a white top and seated at a student desk with an open laptop. She appears to be speaking while using the laptop.}

    \caption{\textbf{Image$\times$text-to-text examples (Part II).} Additional qualitative results with detailed descriptive prompts across indoor scenes, cityscapes, and document-oriented imagery.}
    \label{fig:app_unified_ixt2t_2}
\end{figure*}

\begin{figure*}[t]
    \centering
    \itttsample{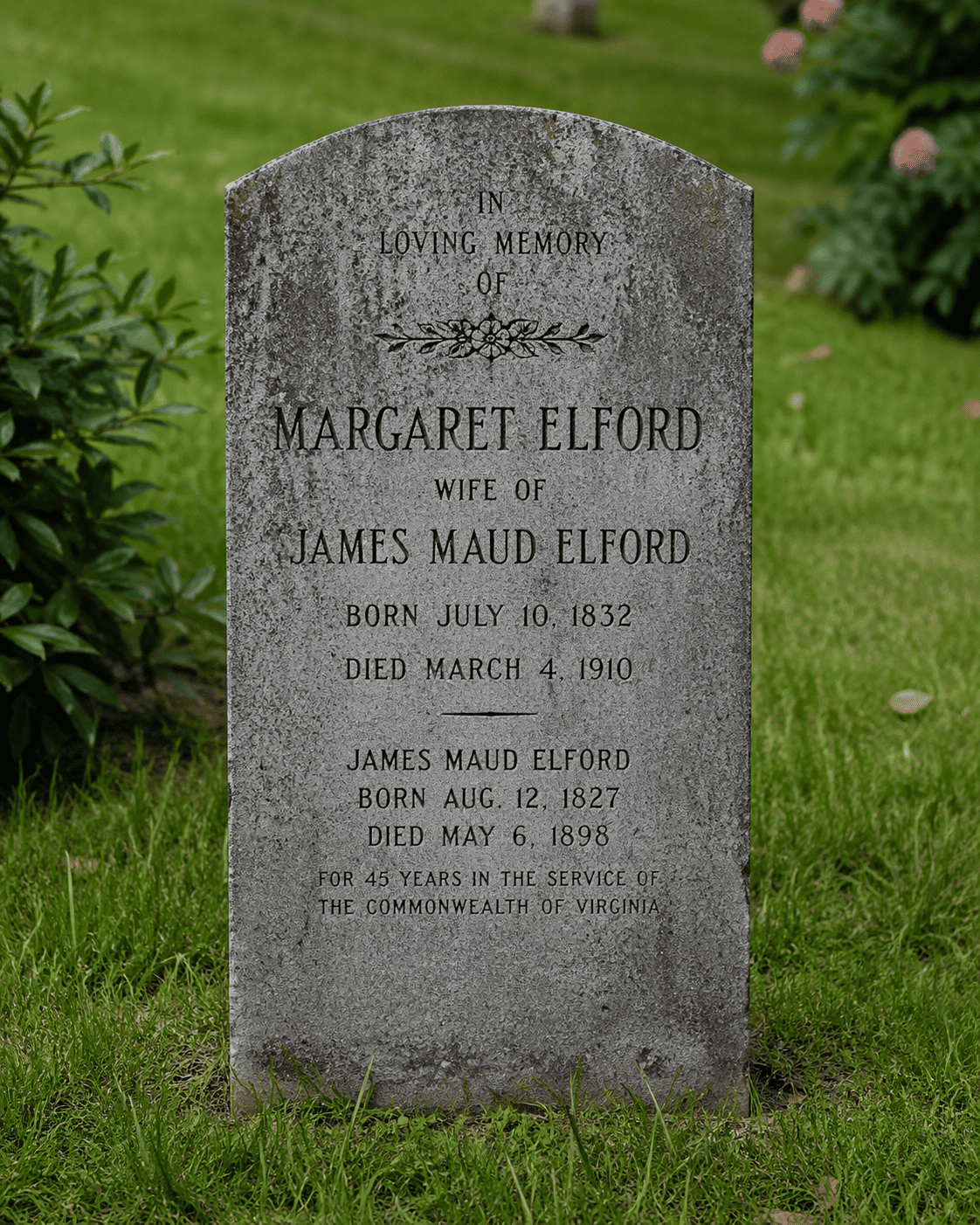}{What does this image depict? Give a comprehensive description.}{The image shows a gravestone with English text carved on it.}
    \hfill
    \itttsample{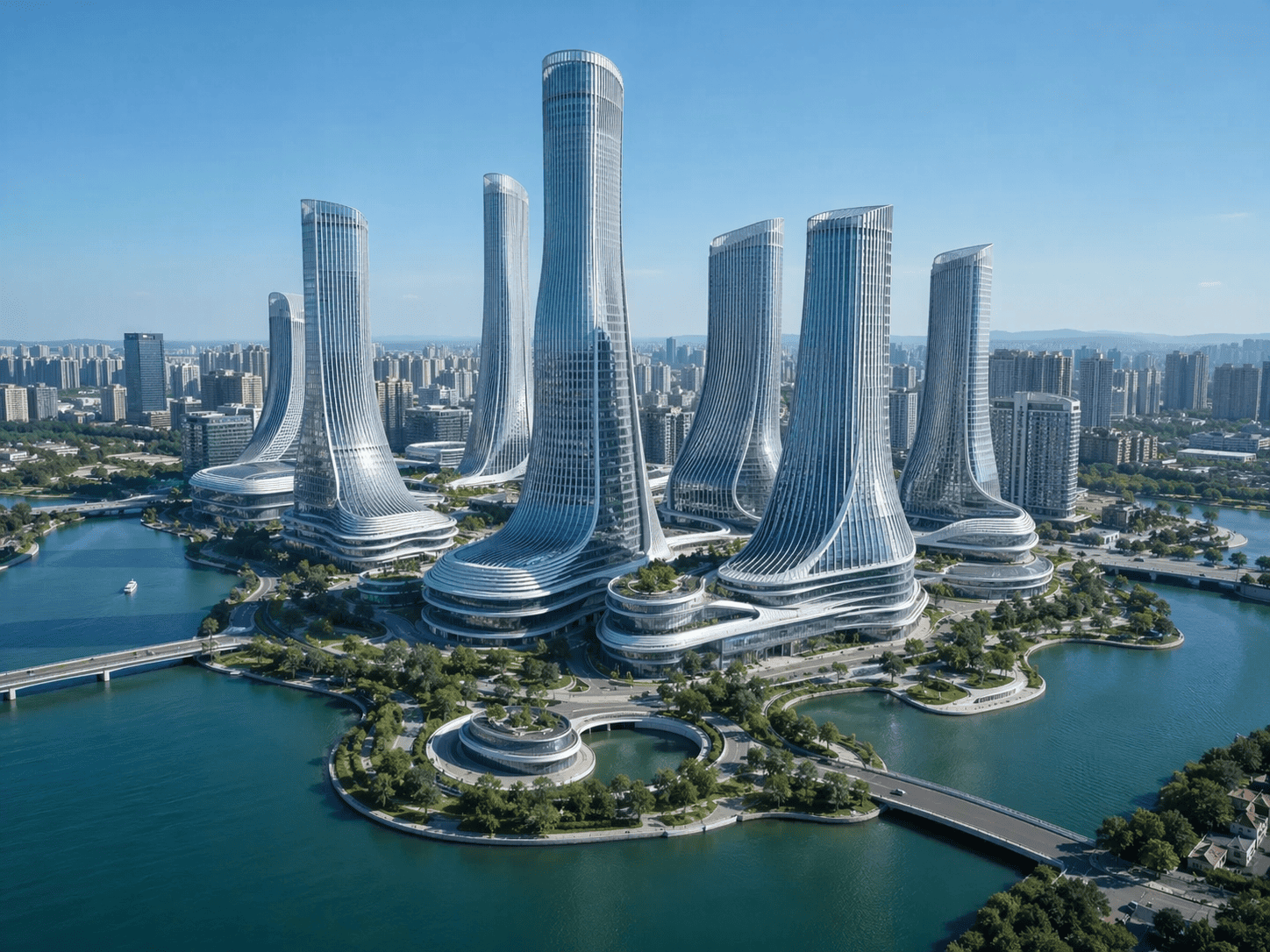}{What is shown in this image? Please explain thoroughly.}{The image shows a cluster of modern city buildings, including tall skyscrapers with distinctive and futuristic architectural designs.}
    \hfill
    \itttsample{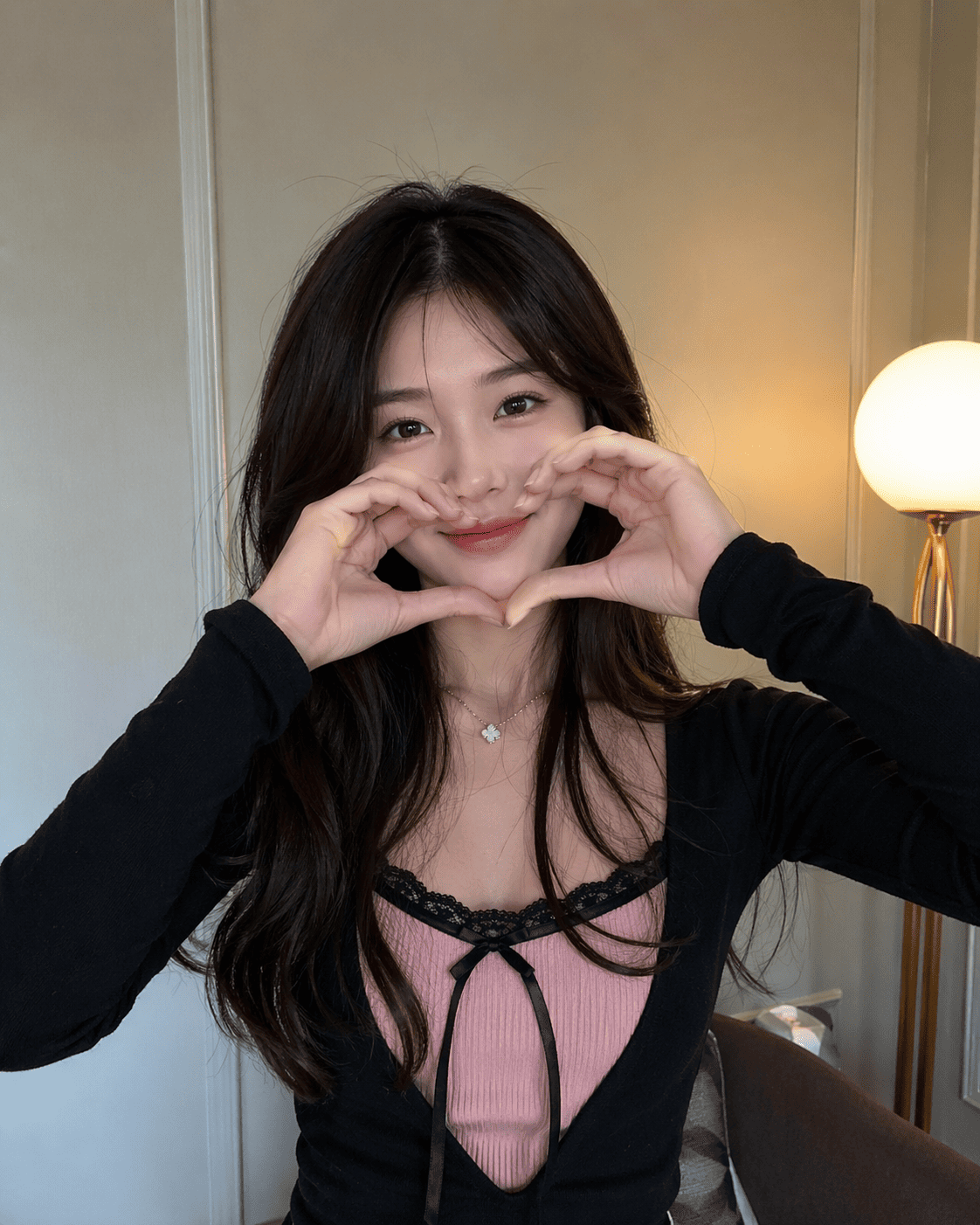}{What is shown in this image? Please explain thoroughly.}{The image is a medium shot of a young East Asian woman with long, straight dark hair, wearing a pink top and a necklace. She is posing with her hands raised near her face, making a gesture with her fingers.}

    \vspace{0.9em}

    \itttsample{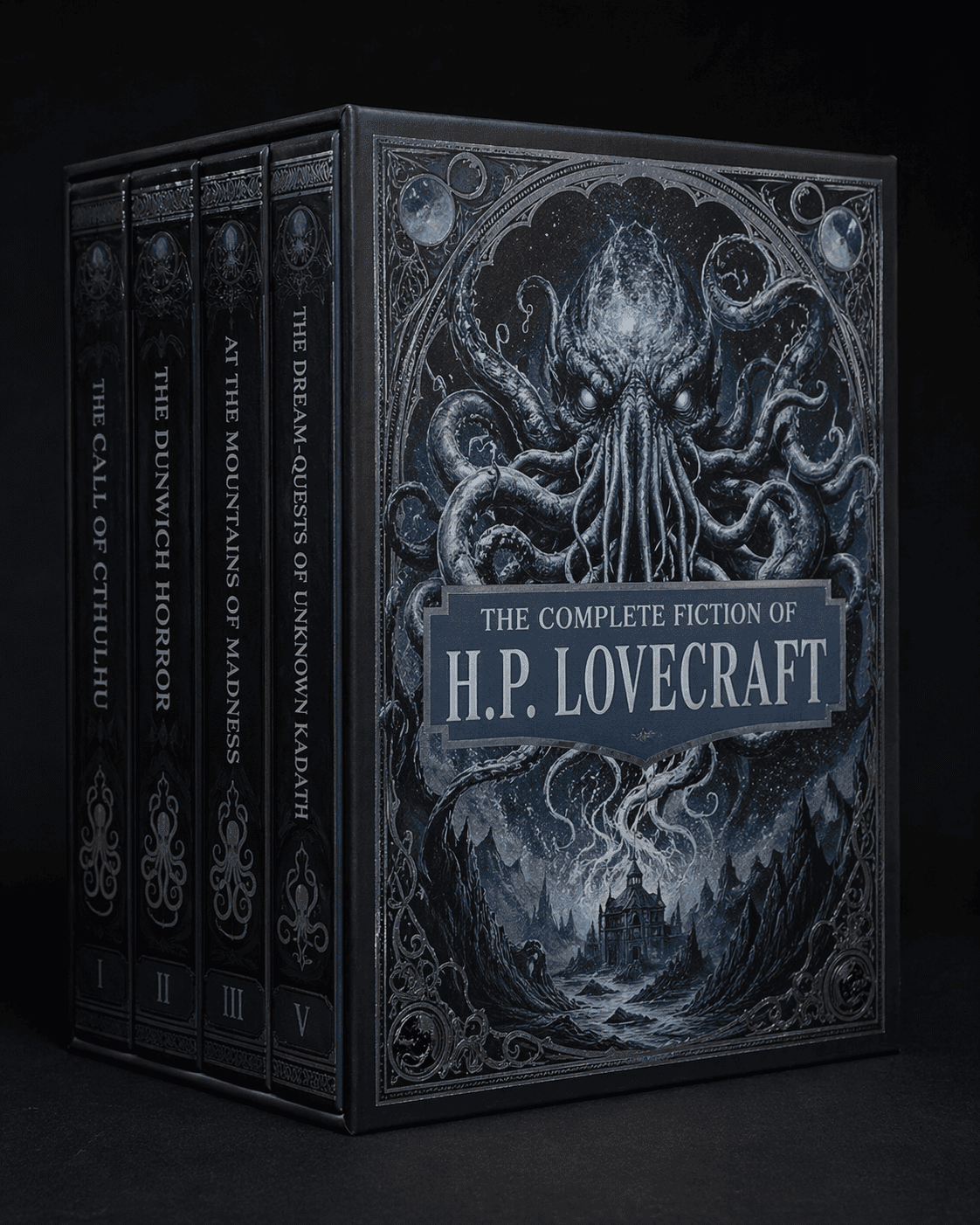}{What is shown in this image? Please explain thoroughly.}{The image shows a blue card or cover with a printed image and some text on it.}
    \hfill
    \itttsample{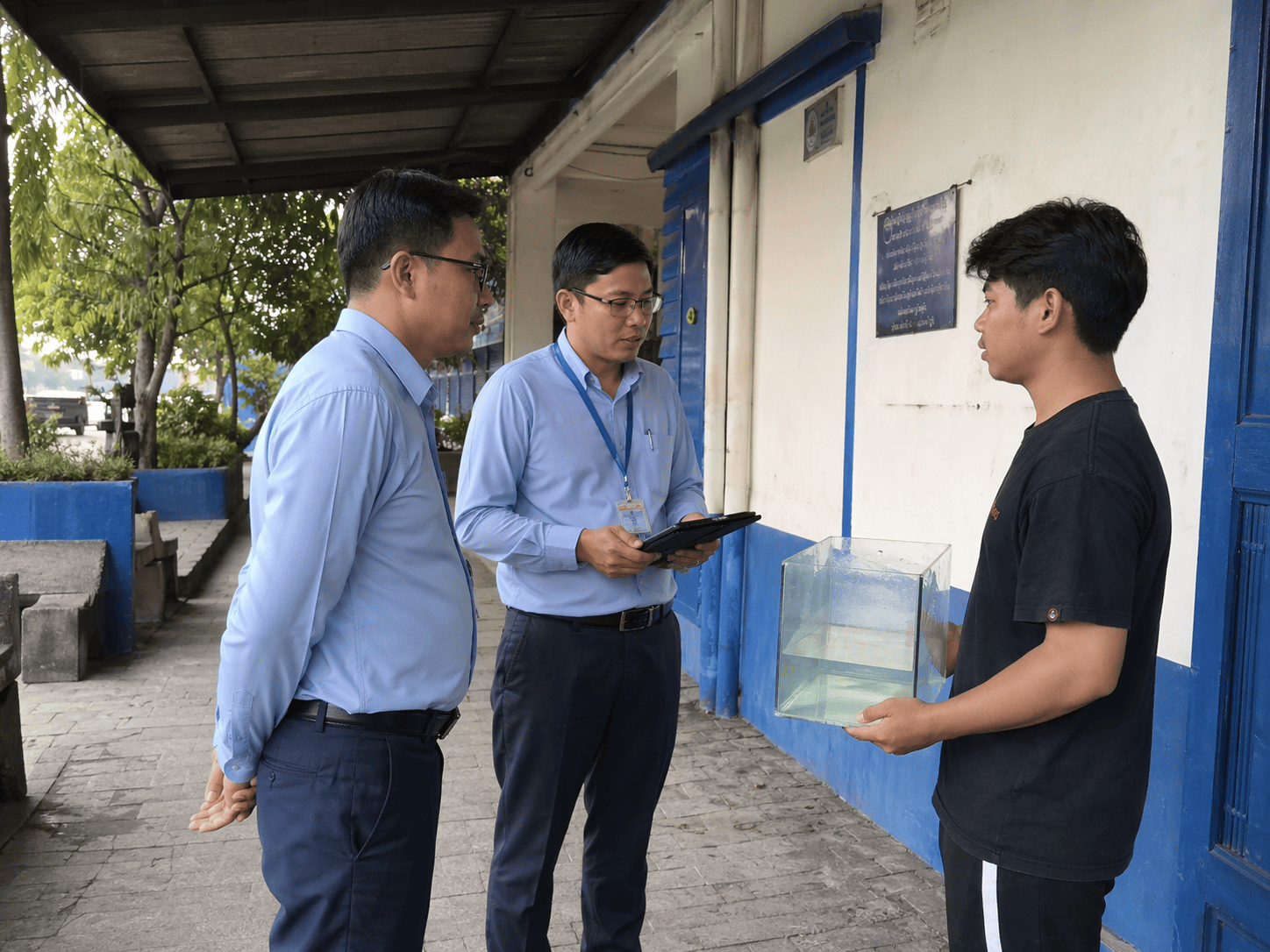}{What is shown in this image? Please explain thoroughly.}{The image shows three people standing in an open outdoor area. They are wearing colored clothing, and one of them appears to be carrying something on their back.}
    \hfill
    \itttsample{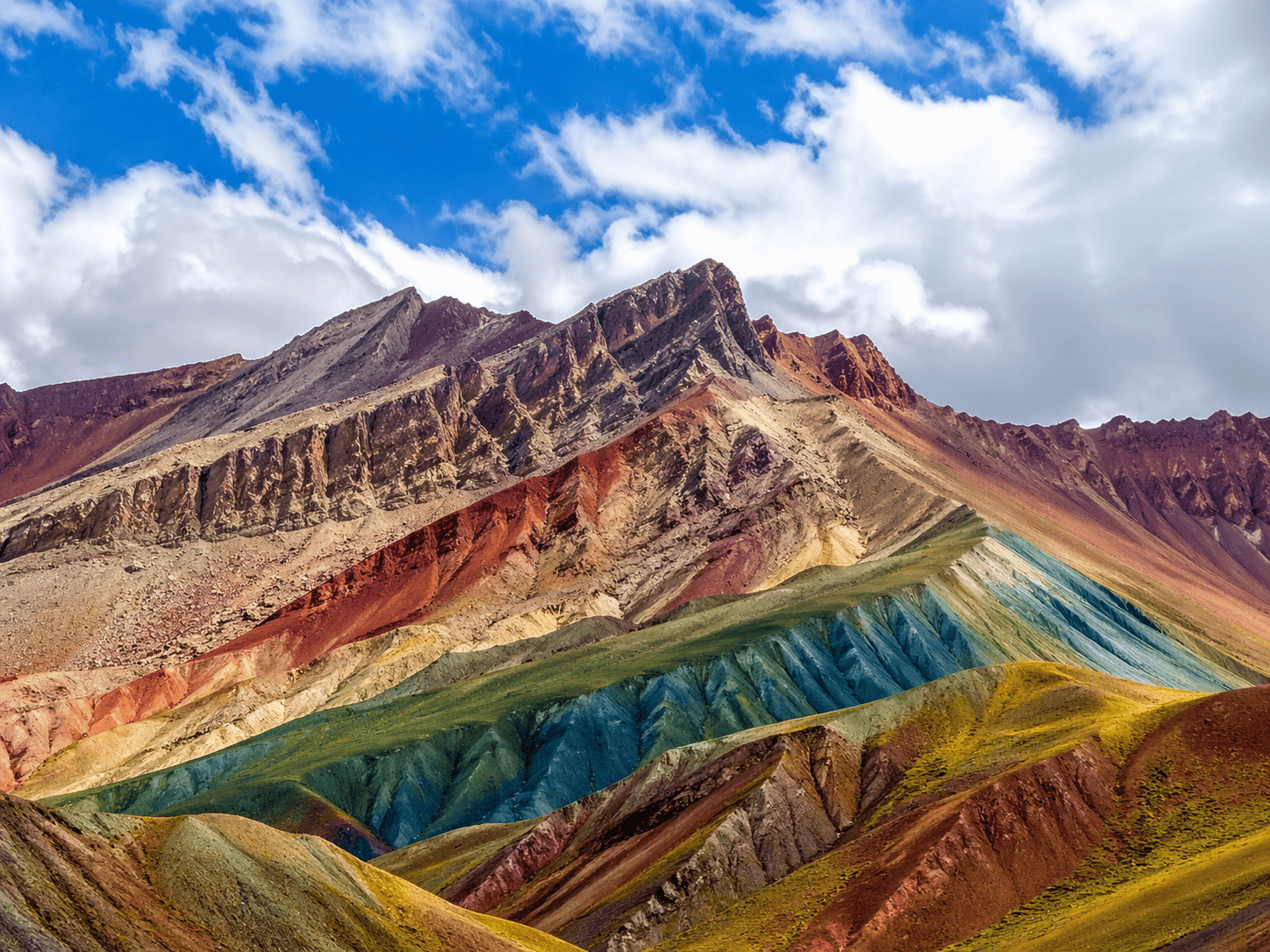}{What is shown in this image? Please explain thoroughly.}{The image shows a mountain range with a massive peak. The surface of the mountain reveals distinctive textures and patterns.}

    \caption{\textbf{Image$\times$text-to-text examples (Part III).} Additional qualitative results under more open-ended explanatory prompts, spanning objects, portraits, outdoor scenes, and landscapes.}
    \label{fig:app_unified_ixt2t_3}
\end{figure*}

\begin{figure*}[t]
    \centering
    \itttsample{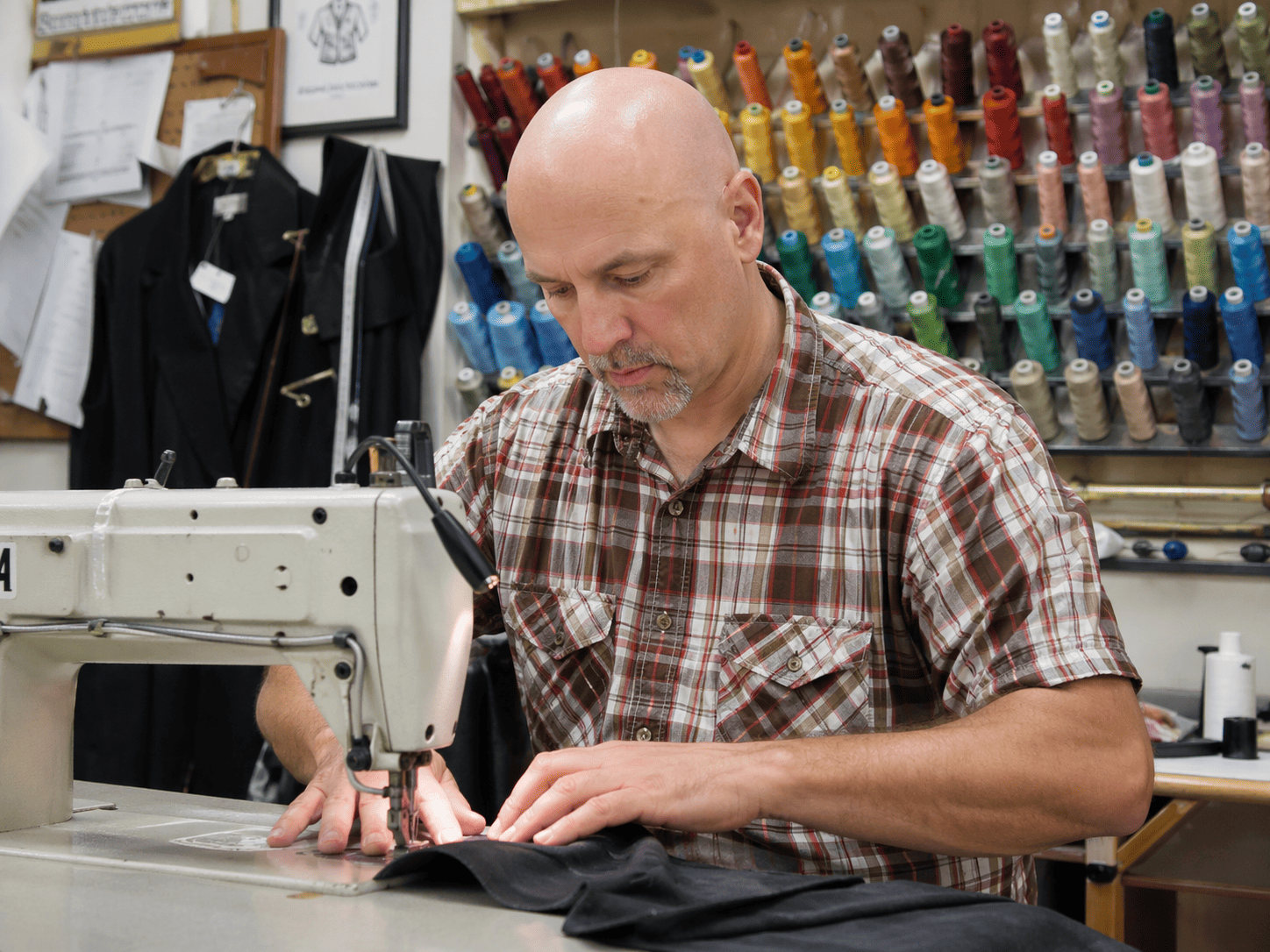}{What is shown in this image? Please explain thoroughly.}{The image shows a man wearing a plaid shirt, working in a studio or workshop. He is focused on an object on a white surface in front of him.}
    \hfill
    \itttsample{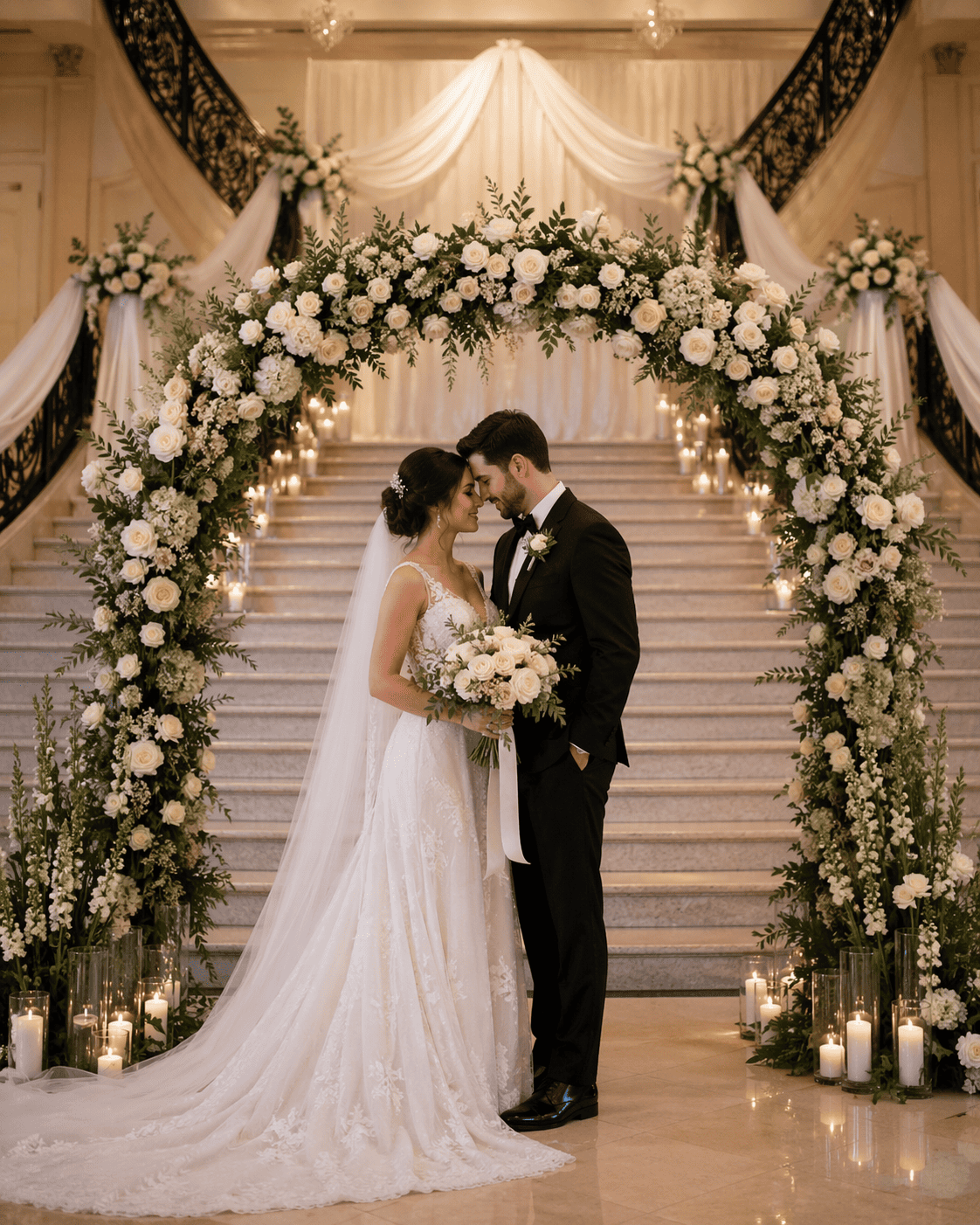}{What is shown in this image? Please explain thoroughly.}{The image shows a newlywed couple at a wedding, dressed in wedding attire.}
    \hfill
    \itttsample{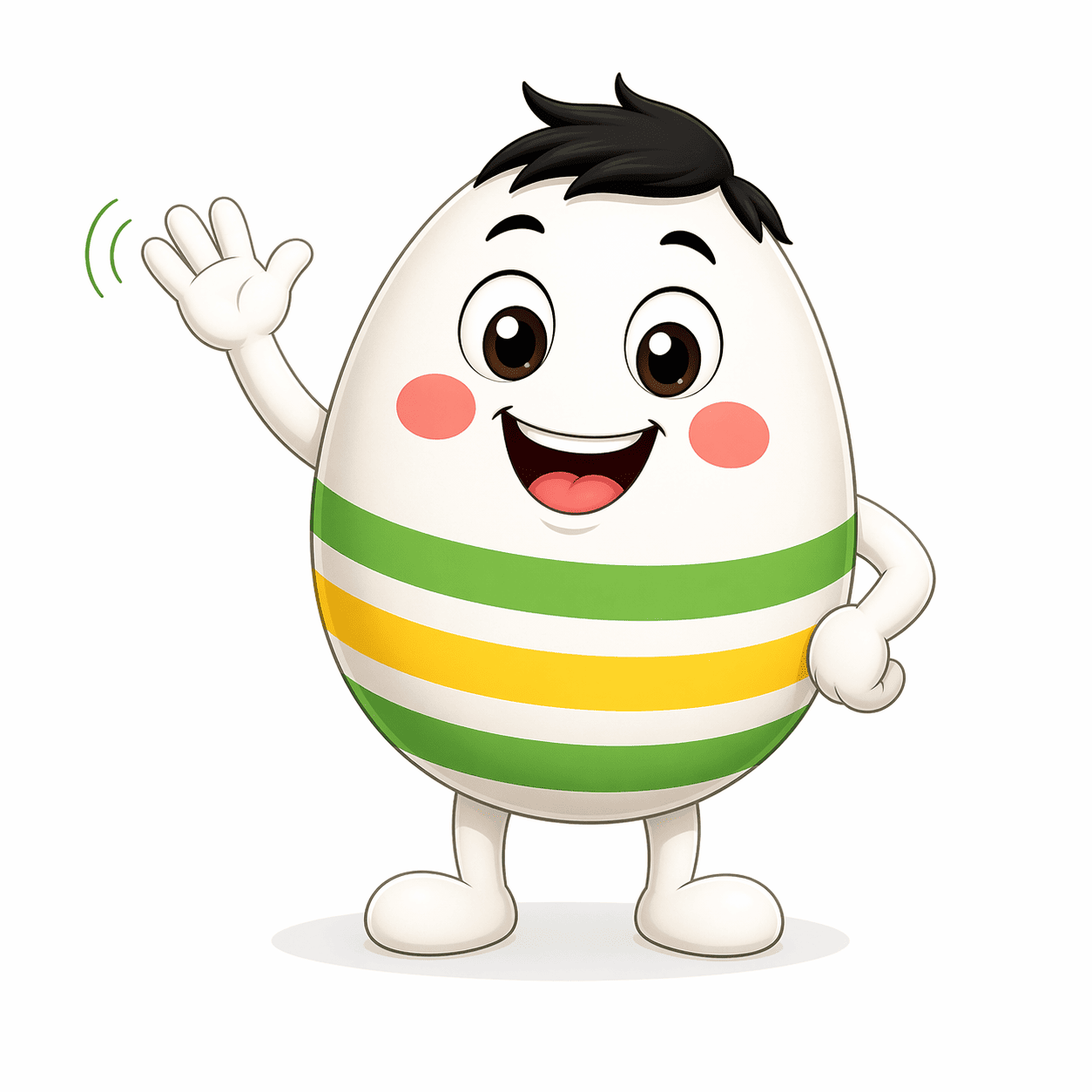}{What is shown in this image? Please explain thoroughly.}{The image is a cartoon-style drawing of a character dressed in a colored outfit, with a body showing green and yellow stripes. The character is touching its face with one hand.}

    \vspace{0.9em}

    \itttsample{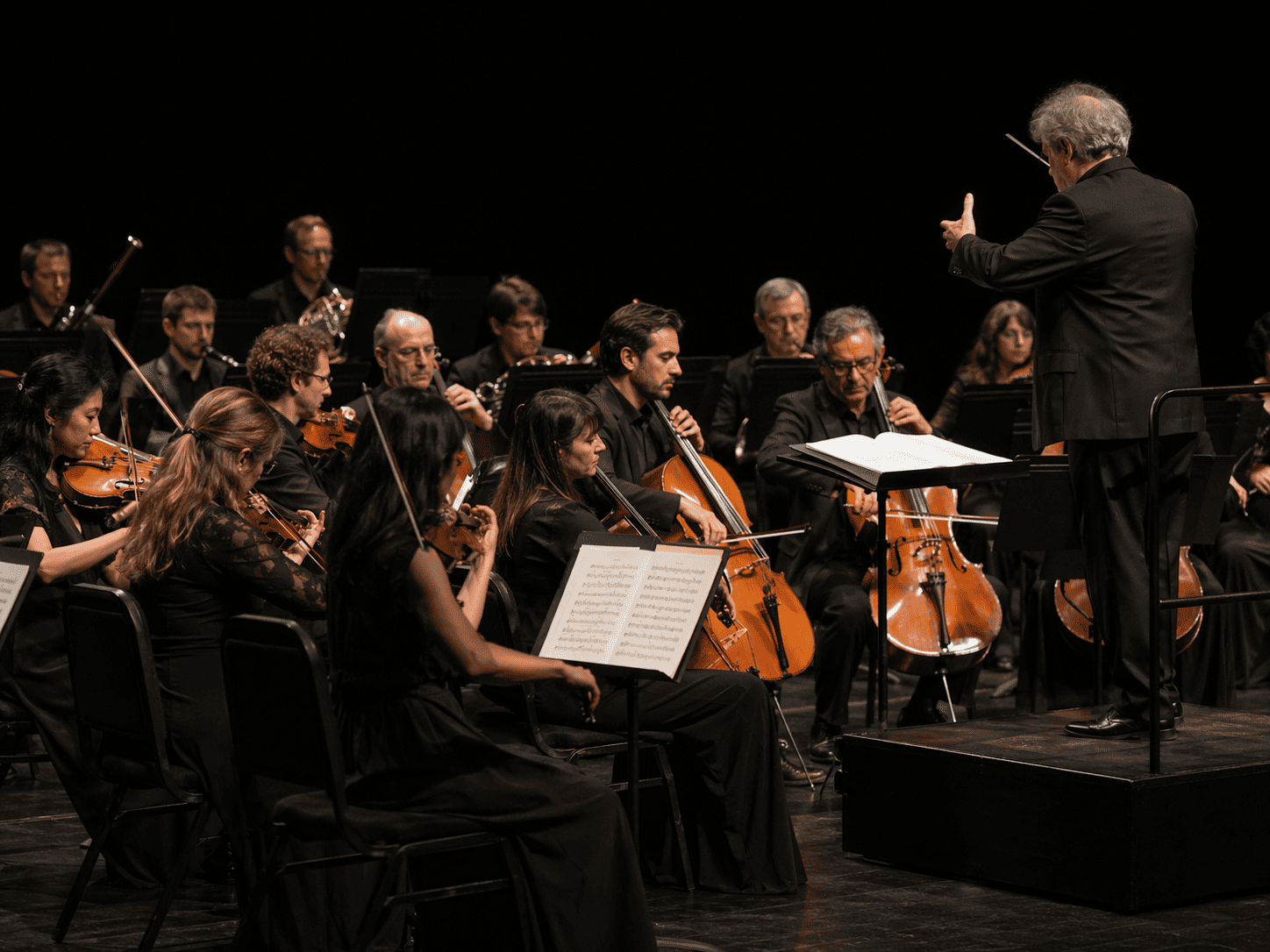}{Describe what you see in this picture.}{The image shows musicians performing with instruments at a ceremony. The background includes a floor and a wall.}
    \hfill
    \itttsample{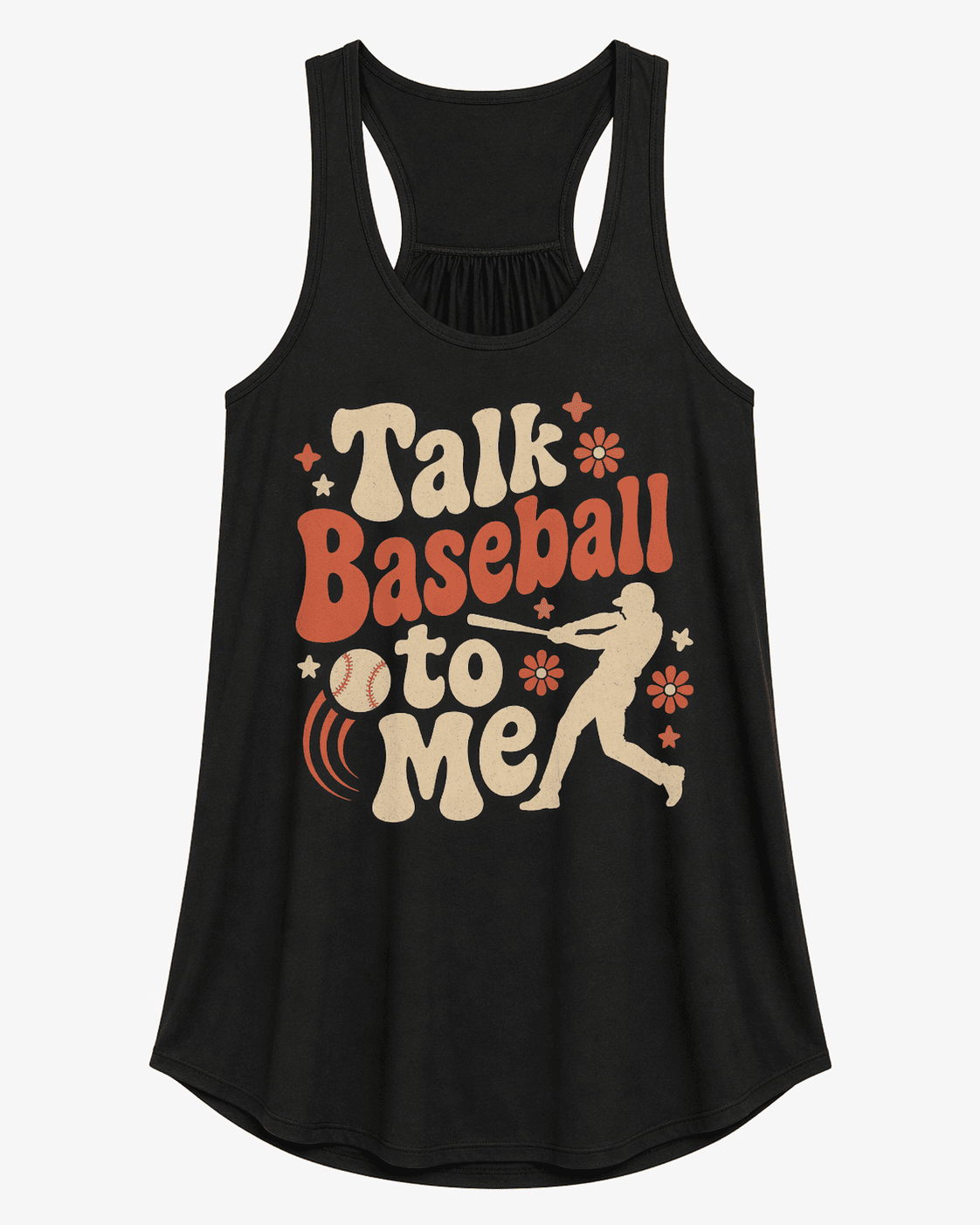}{Describe what you see in this picture.}{The image shows a black sleeveless T-shirt with a printed design on the front.}
    \hfill
    \itttsample{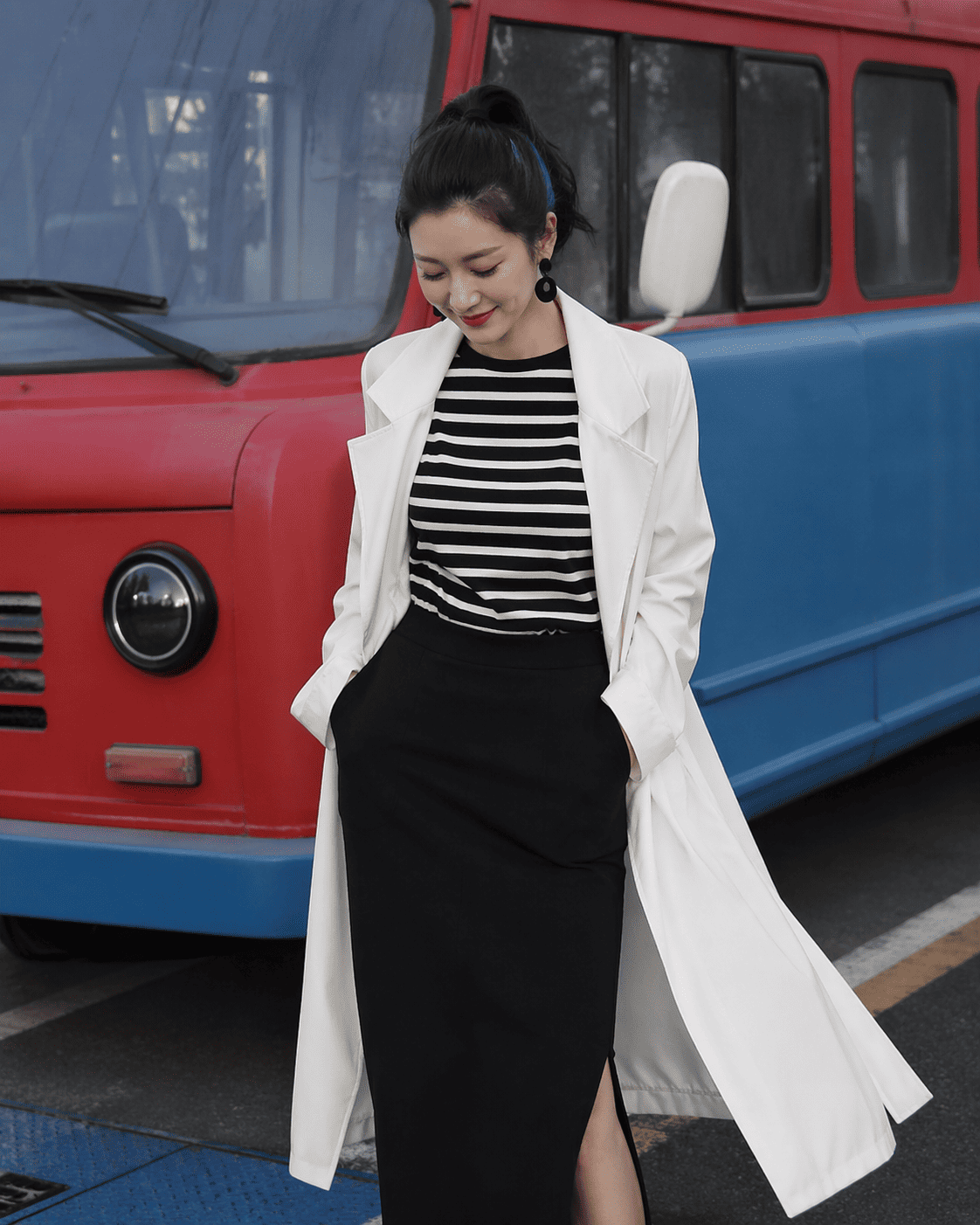}{Describe what you see in this picture.}{The image shows a person wearing a white blouse over a black and white striped skirt.}

    \caption{\textbf{Image$\times$text-to-text examples (Part IV).} Additional qualitative results over people-centric, object-centric, and illustration-style inputs.}
    \label{fig:app_unified_ixt2t_4}
\end{figure*}

\begin{figure*}[t]
    \centering
    \itttsample{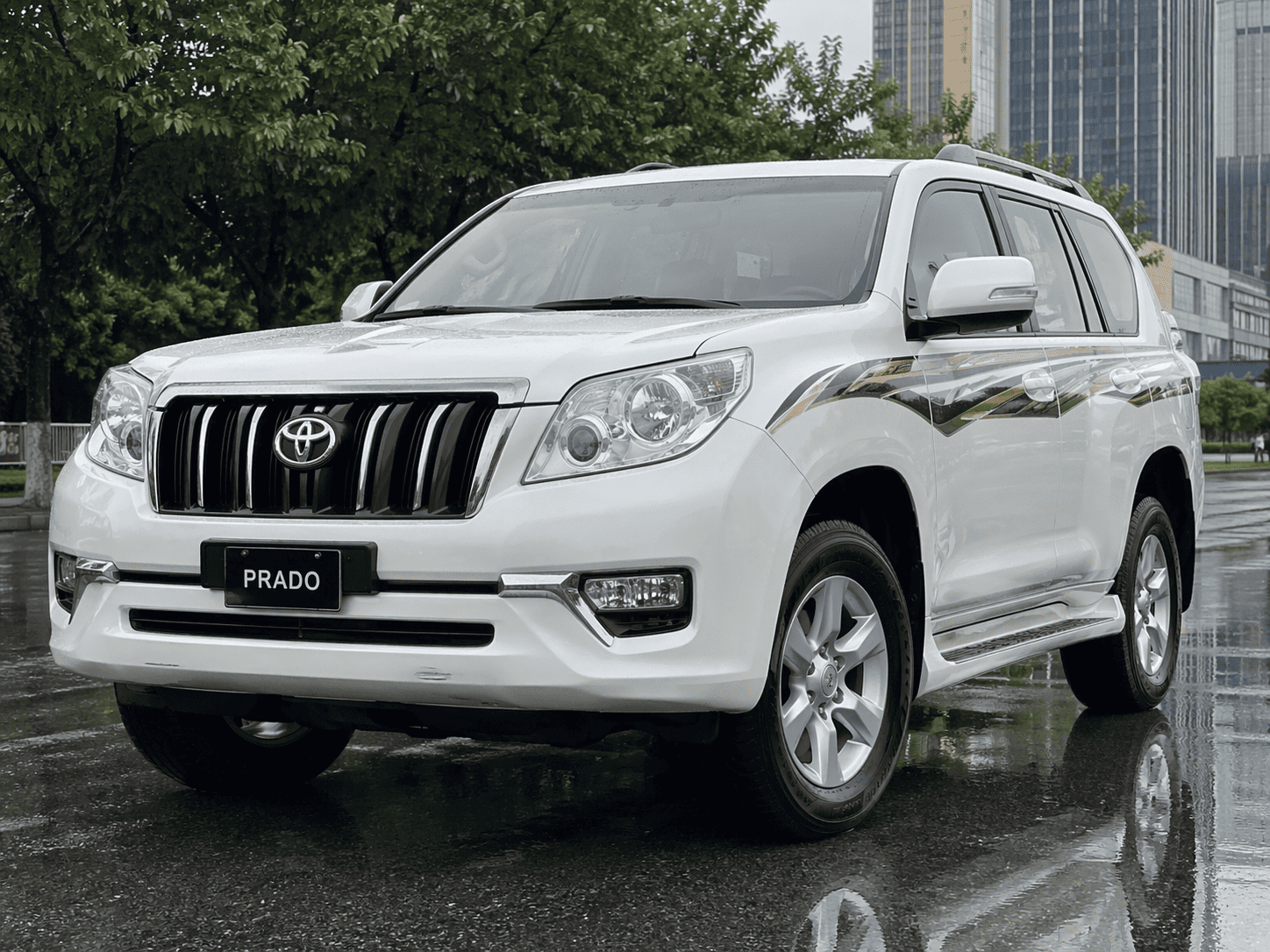}{Describe what you see in this picture.}{The image shows a white Toyota car with a sleek design, featuring prominent front and side headlights.}
    \hfill
    \itttsample{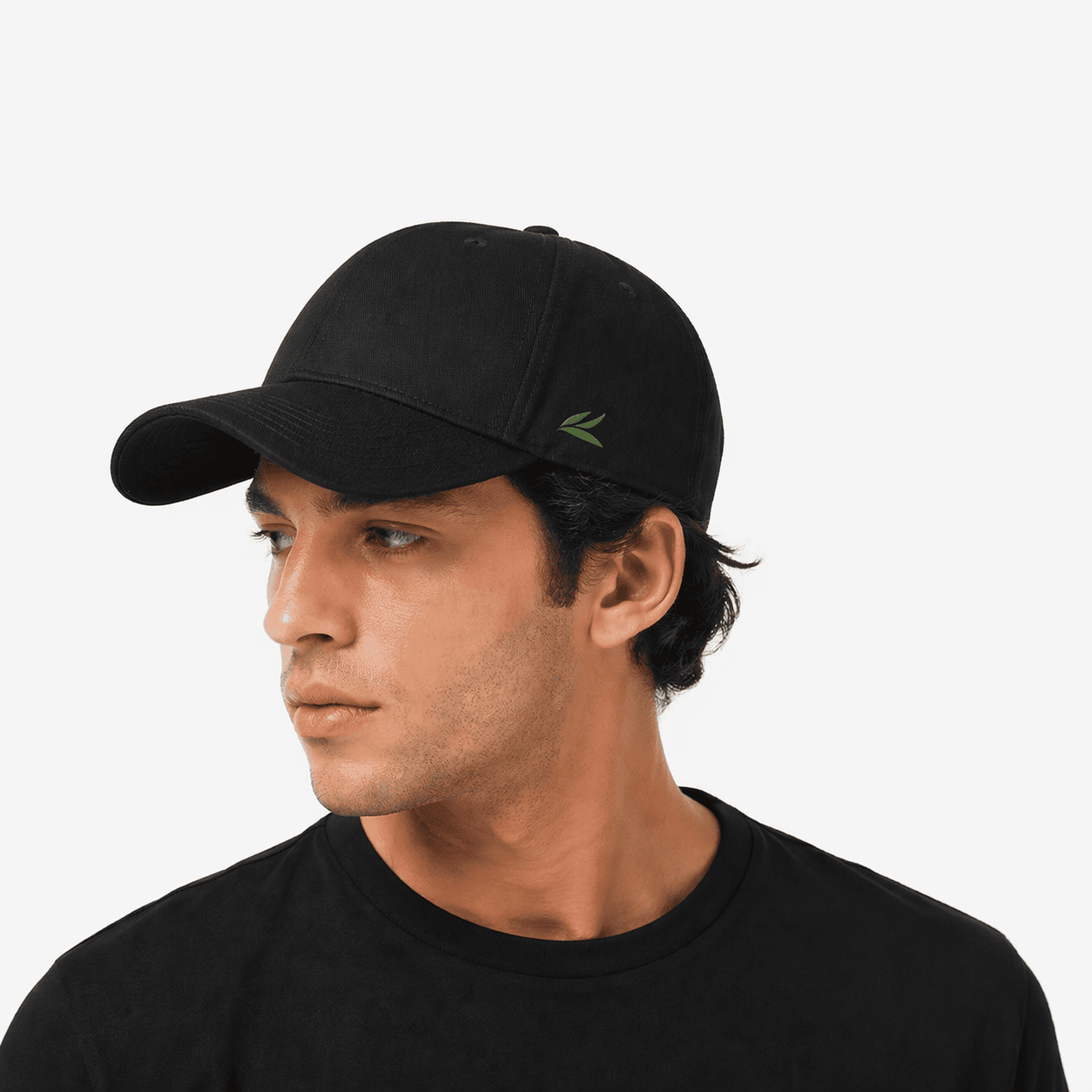}{Describe what you see in this picture.}{The image shows a man wearing a black baseball cap and a grey shirt. The focus is on the cap, and the background is a plain, neutral grey.}
    \hfill
    \itttsample{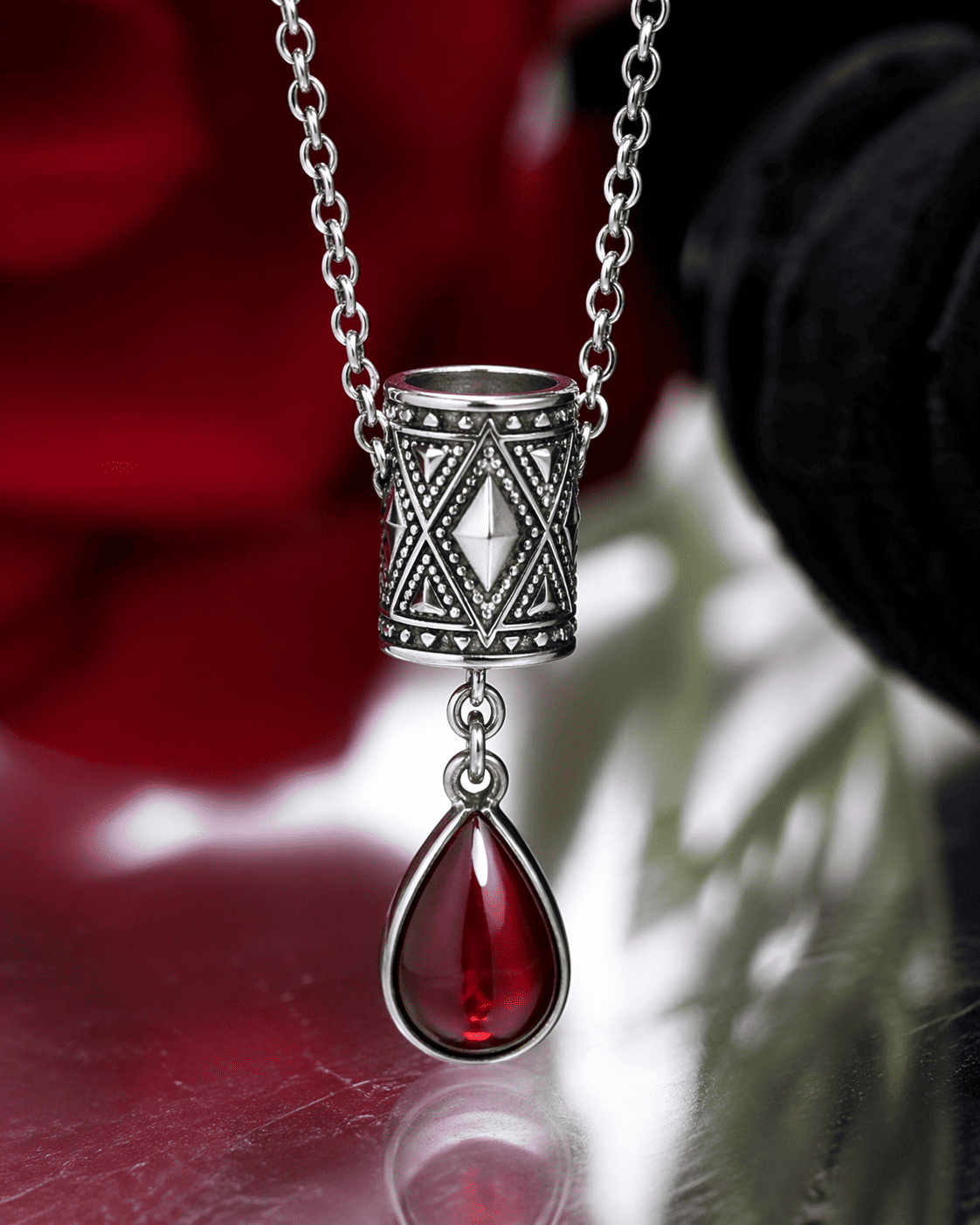}{Describe what you see in this picture.}{The image shows a necklace with a red, ball-shaped pendant.}

    \vspace{0.9em}

    \itttsample{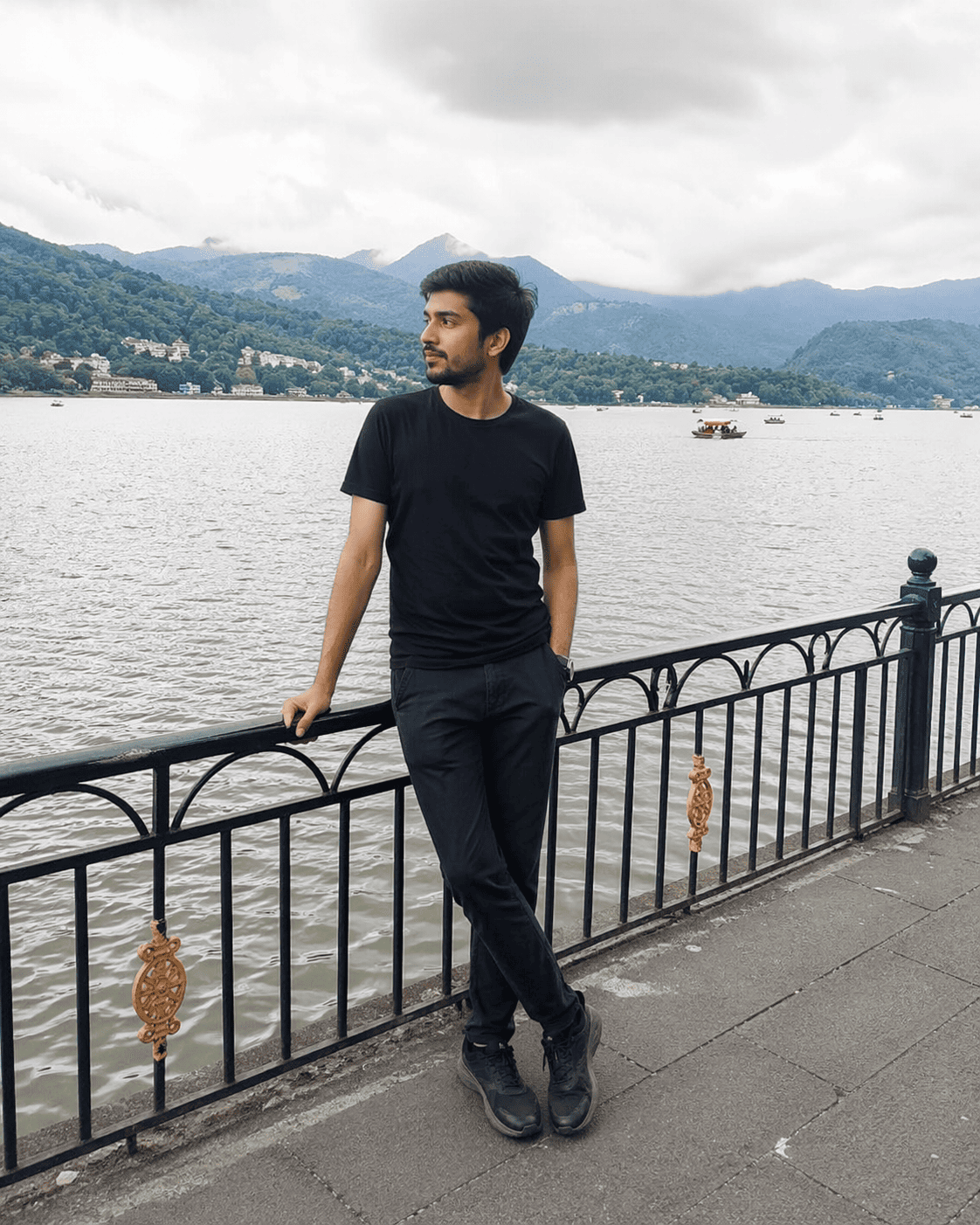}{Describe what you see in this picture.}{The image shows a person standing on a paved area with a scenic green backdrop.}
    \hfill
    \itttsample{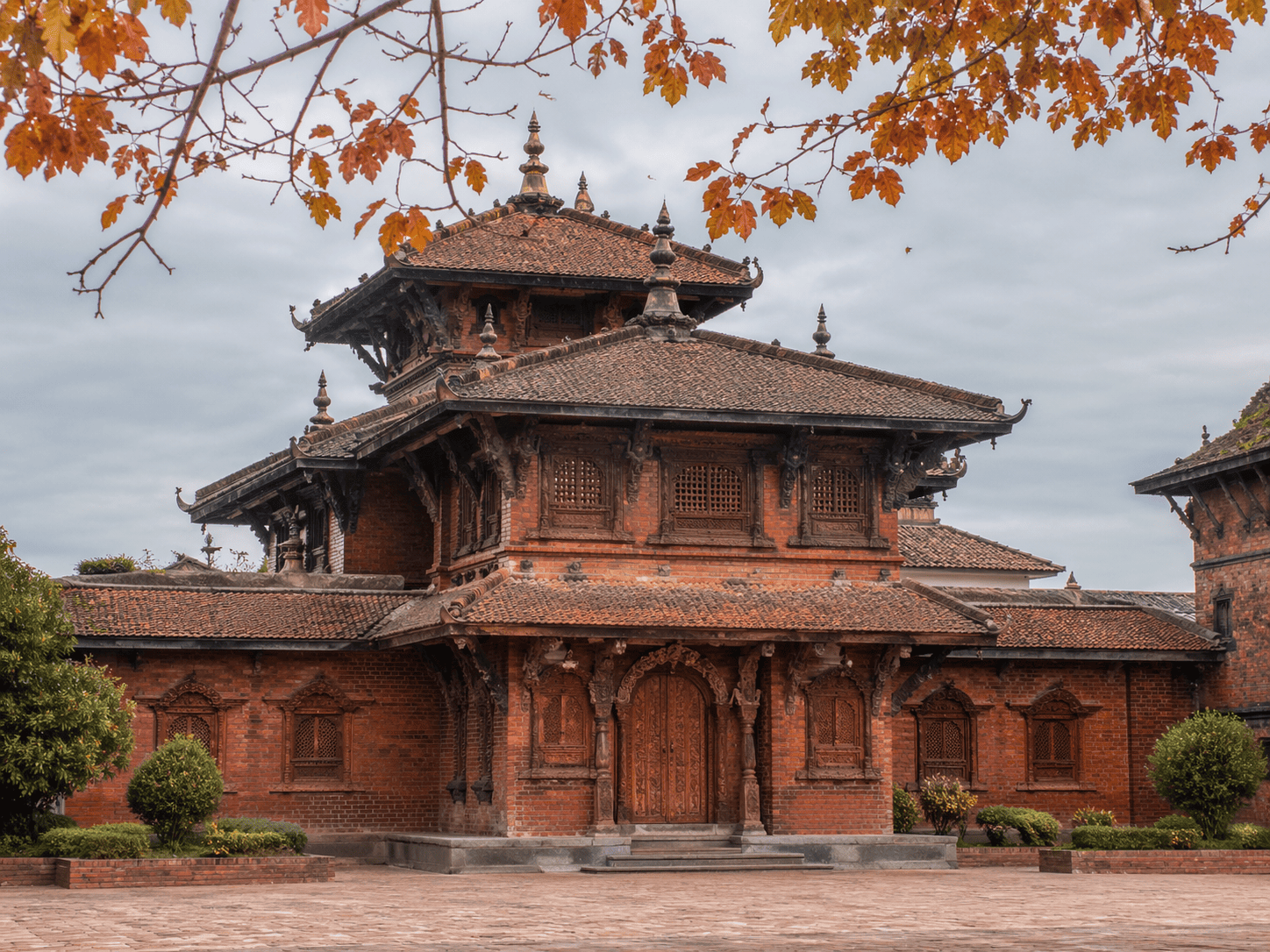}{Describe what you see in this picture.}{The image shows an old building with distinctive traditional Chinese-style architecture. It features traditional entrance details, and an iron structure with a roof is visible above the building.}
    \hfill
    \itttsample{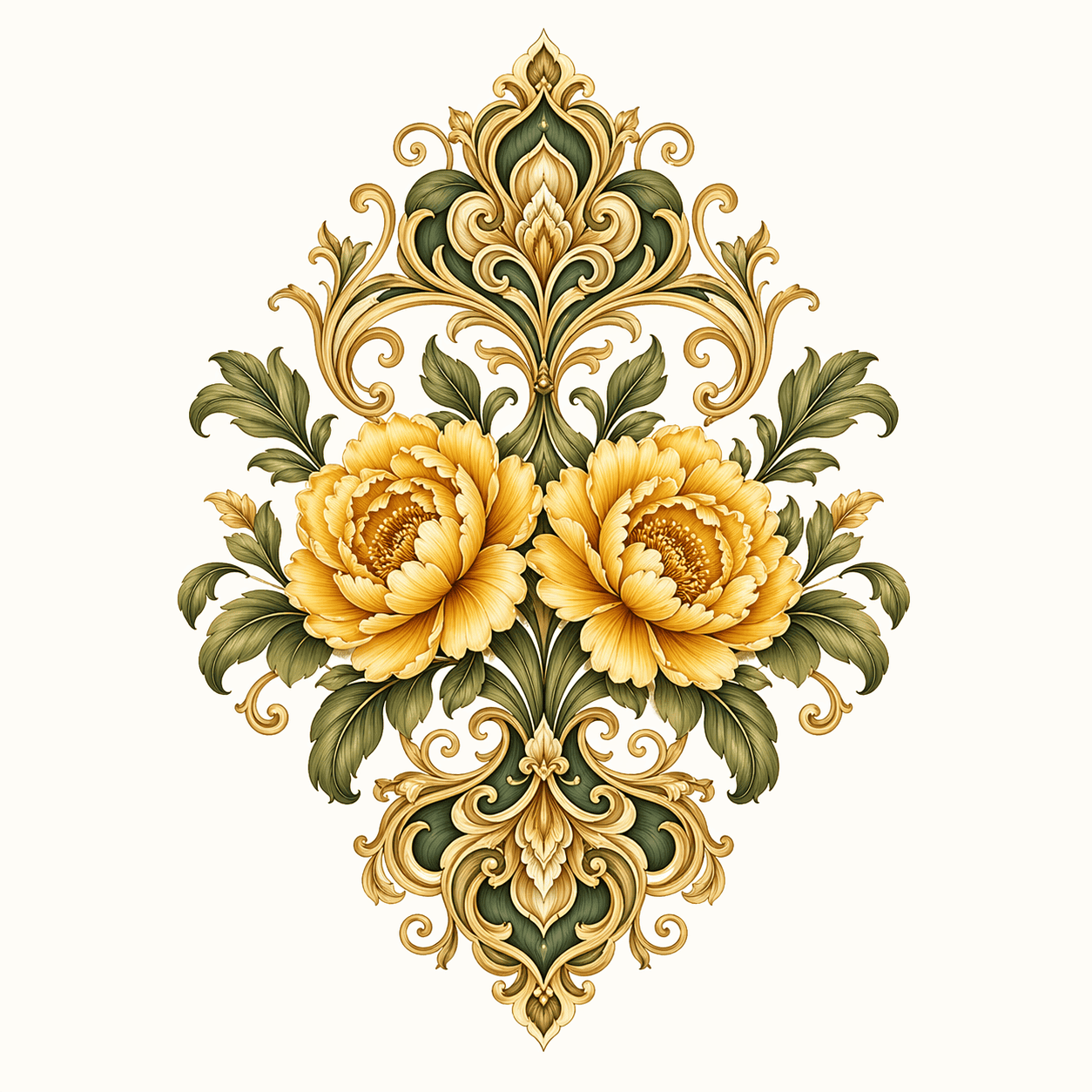}{Describe what you see in this picture.}{The image shows a decorative ornament in the shape of a flower.}

    \caption{\textbf{Image$\times$text-to-text examples (Part V).} Additional qualitative results on apparel, accessories, vehicles, architecture, and scenic human portraits.}
    \label{fig:app_unified_ixt2t_5}
\end{figure*}

\begin{figure*}[t]
    \centering
    \itttsample{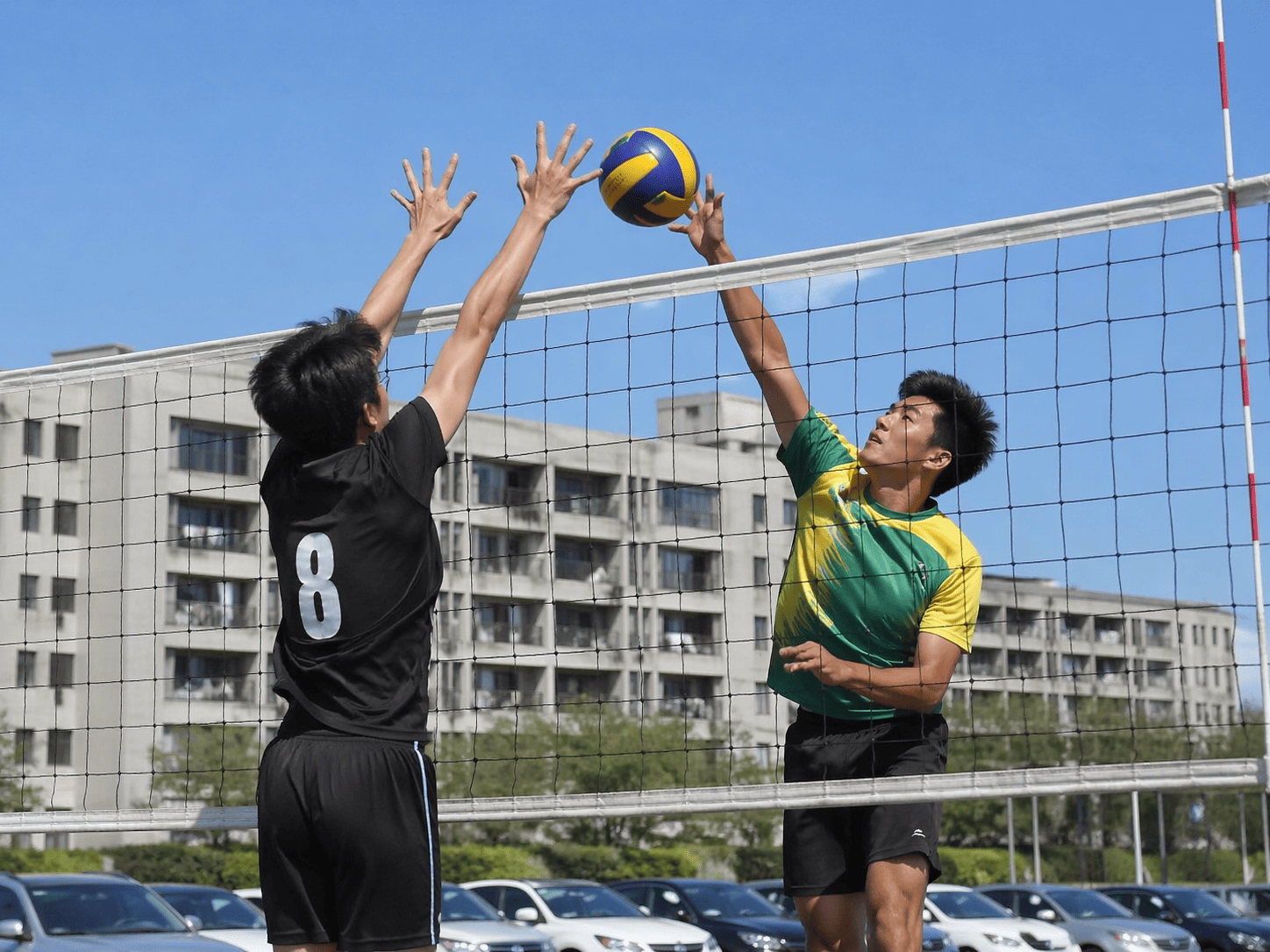}{Describe what you see in this picture.}{The image shows several people playing volleyball, with one of them jumping up. The scene takes place under a clear afternoon sky.}
    \hfill
    \itttsample{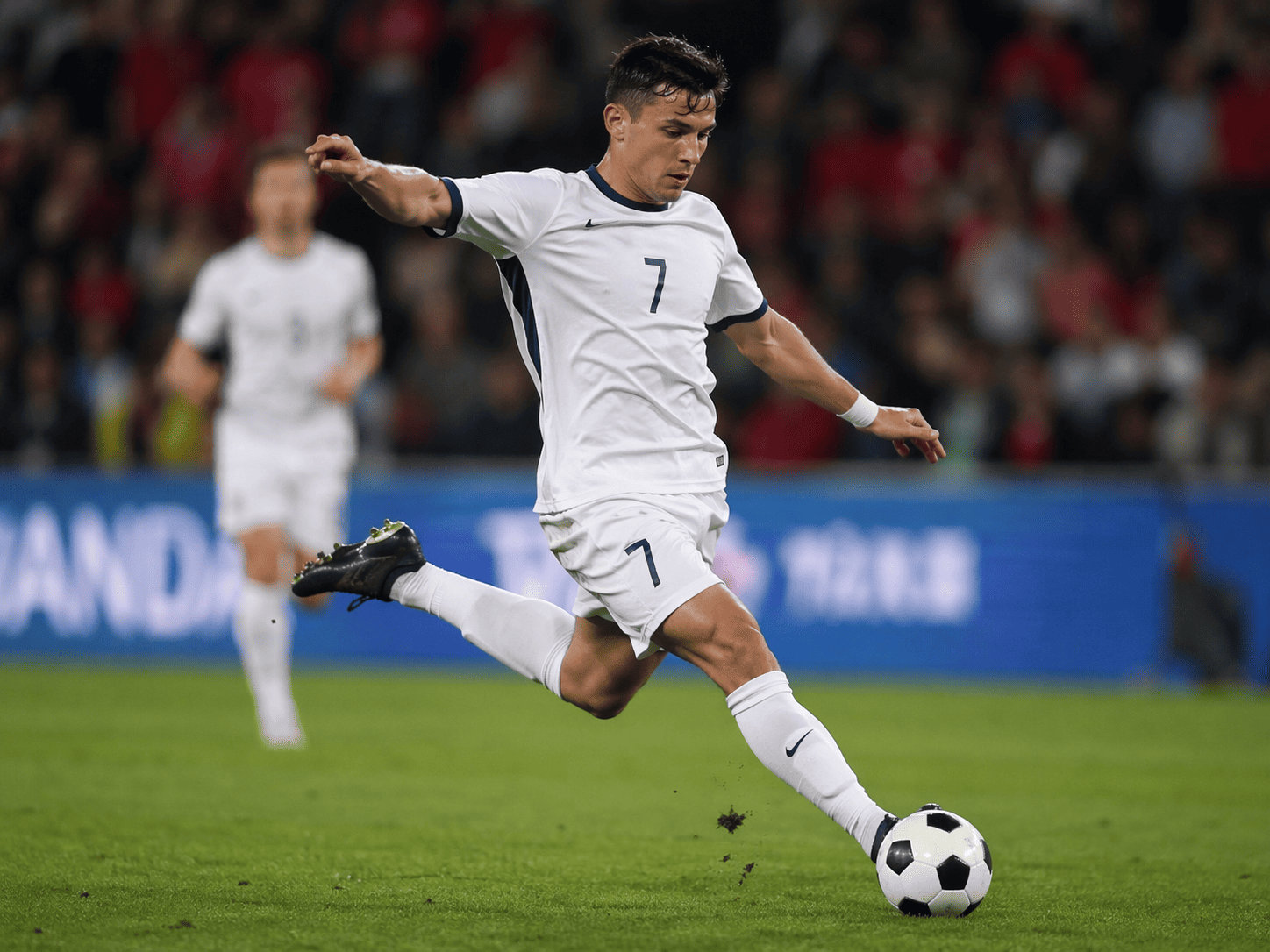}{Generate a detailed caption for this image.}{The image shows a soccer player in action during a match, wearing a numbered uniform. The player appears to be running and kicking the ball on the field.}
    \hfill
    \itttsample{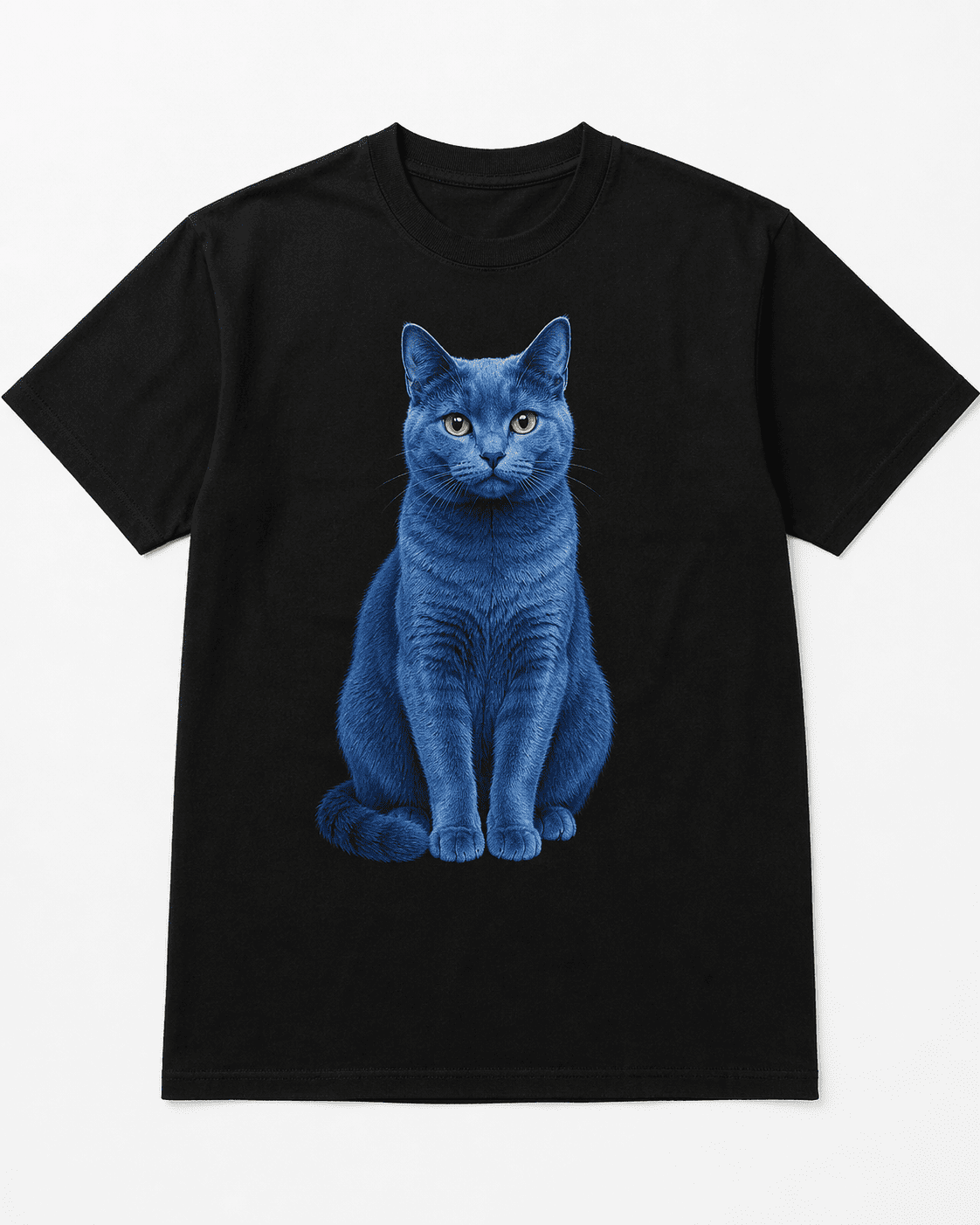}{Generate a detailed caption for this image.}{The image shows a black T-shirt with a design featuring a blue cat printed on the front.}

    \vspace{0.9em}

    \itttsample{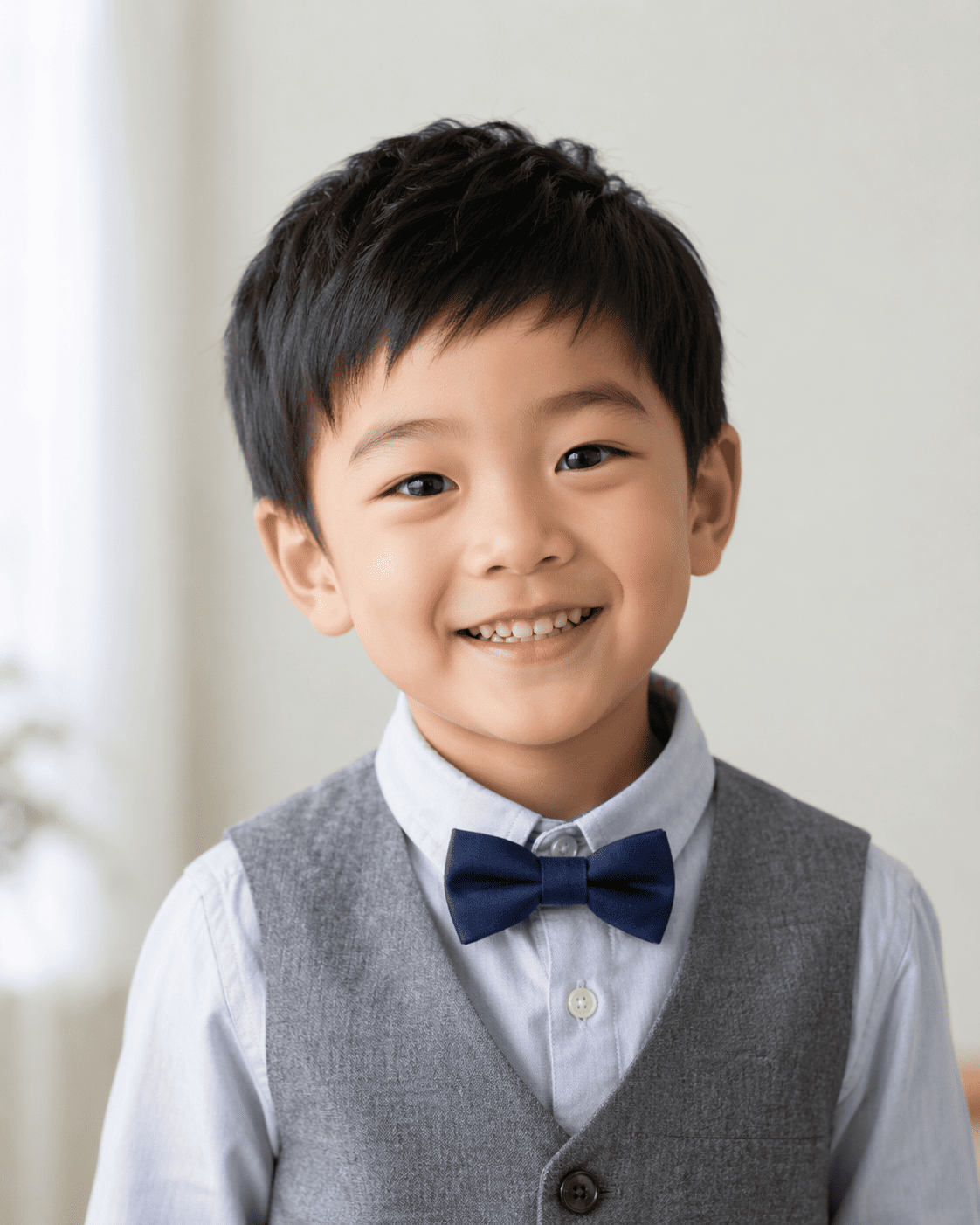}{Generate a detailed caption for this image.}{The image shows a young boy wearing a grey vest over a white shirt, with a blue accent.}
    \hfill
    \itttsample{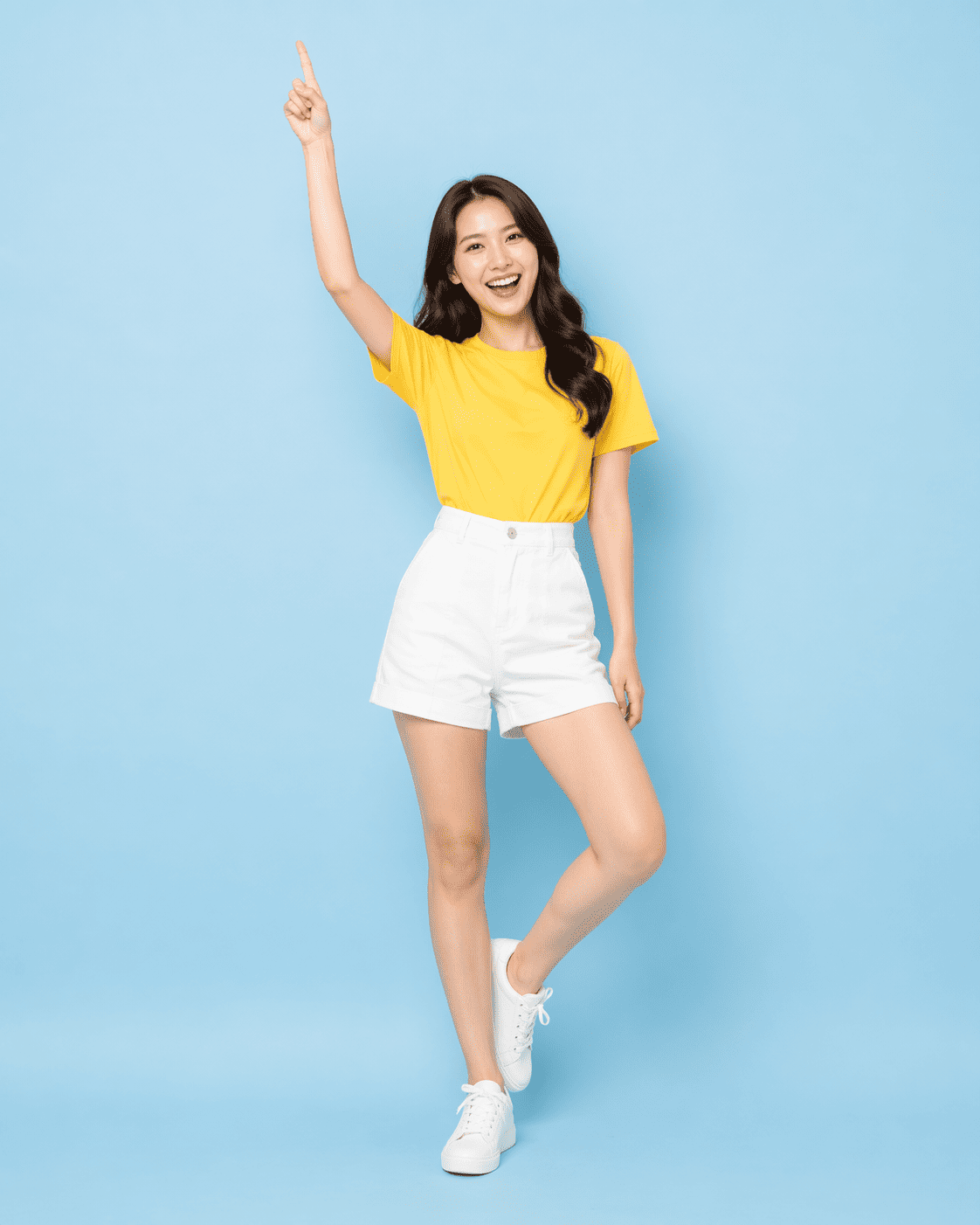}{Generate a detailed caption for this image.}{The image shows a young girl with straight hair, wearing a yellow crop top and white shorts. She is making a fist with one hand and jumping, set against a blue background.}
    \hfill
    \itttsample{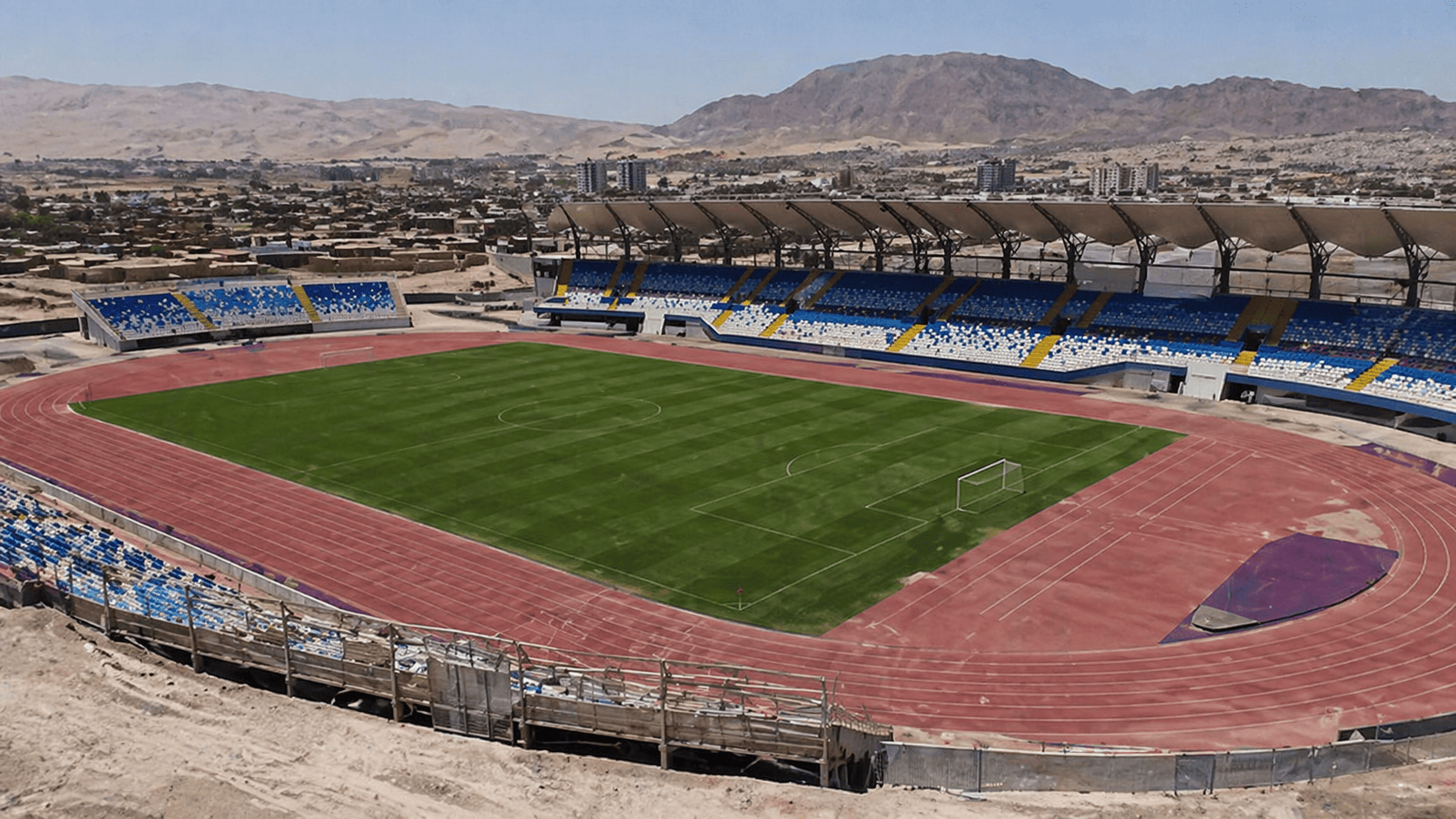}{Generate a detailed caption for this image.}{The image shows the stands of a large soccer stadium. The field and running track are visible, with green grass surrounding the stands and a few spectators present.}

    \caption{\textbf{Image$\times$text-to-text examples (Part VI).} Additional qualitative results for sports, apparel, portrait, and stadium scenes, including caption-generation prompts.}
    \label{fig:app_unified_ixt2t_6}
\end{figure*}